\documentclass{article}

    \usepackage[final, nonatbib]{neurips_2023}

\usepackage[utf8]{inputenc} 
\usepackage[T1]{fontenc}    
\usepackage{hyperref}       
\usepackage{url}            
\usepackage{booktabs}       
\usepackage{amsfonts}       
\usepackage{nicefrac}       
\usepackage{microtype}      
\usepackage{xcolor}         

\usepackage{microtype}
\usepackage{graphicx}
\usepackage{subfigure}
\usepackage{booktabs} 

\usepackage{hyperref}

\usepackage{algorithm}
\usepackage{algorithmic}

\usepackage{amsmath}
\usepackage{amssymb}
\usepackage{mathtools}
\usepackage{amsthm}
\usepackage{bm}
\usepackage{multirow}
\usepackage{pifont}
\newcommand{\xmark}{\text{\ding{55}}}
\newcommand{\cmark}{\text{\ding{51}}}
\theoremstyle{plain}
\newtheorem{theorem}{Theorem}[section]
\newtheorem{proposition}[theorem]{Proposition}
\newtheorem{lemma}[theorem]{Lemma}

\theoremstyle{definition}

\newtheorem{assumption}[theorem]{Assumption}
\theoremstyle{remark}

\title{Communication-Efficient Federated Bilevel Optimization with Global and Local Lower Level Problems}

\author{%
  Junyi Li\\
  Computer Science \\
  University of Maryland \\
  College Park, MD 20742 \\
  \texttt{junyili.ai@gmail.com} \\
  \And
  Feihu Huang \\
  ECE\\
  University of Pittsburgh\\
  Pittsburgh, PA 15261 \\
  \texttt{huangfeihu2018@gmail.com} \\  
  \And
  Heng Huang \thanks{This work was partially supported by NSF IIS 1838627, 1837956, 1956002, 2211492, CNS 2213701, CCF 2217003, DBI 2225775.}\\
  Computer Science \\
  University of Maryland \\
  College Park, MD 20742 \\
  \texttt{henghuanghh@gmail.com}
}

\begin{document}

\maketitle

\begin{abstract}
Bilevel Optimization has witnessed notable progress recently with new emerging efficient algorithms. However, its application in the Federated Learning setting remains relatively underexplored, and the impact of Federated Learning's inherent challenges on the convergence of bilevel algorithms remain obscure.
In this work, we investigate Federated Bilevel Optimization problems and propose a communication-efficient algorithm, named FedBiOAcc. The algorithm leverages an efficient estimation of the hyper-gradient in the distributed setting and utilizes the momentum-based variance-reduction acceleration. Remarkably, FedBiOAcc achieves a communication complexity $O(\epsilon^{-1})$, a sample complexity $O(\epsilon^{-1.5})$ and the linear speed up with respect to the number of clients. We also analyze a special case of the Federated Bilevel Optimization problems, where lower level problems are locally managed by clients. We prove that FedBiOAcc-Local, a modified version of FedBiOAcc, converges at the same rate for this type of problems. Finally, we validate the proposed algorithms through two real-world tasks: Federated Data-cleaning and Federated Hyper-representation Learning. Empirical results show superior performance of our algorithms.
\end{abstract}

\section{Introduction}
Bilevel optimization~\cite{willoughby1979solutions, solodov2007explicit} has increasingly drawn attention due to its wide-ranging applications in numerous machine learning tasks, including hyper-parameter optimization~\cite{okuno2018hyperparameter}, meta-learning~\cite{zintgraf2019fast} and neural architecture search~\cite{liu2018darts}. A bilevel optimization problem involves an upper problem and a lower problem, wherein the upper problem is a function of the minimizer of the lower problem. Recently, great progress has been made to solve this type of problems, particularly through the development of efficient single-loop algorithms that rely on diverse gradient approximation techniques~\cite{ji2020provably}. However, the majority of existing bilevel optimization research concentrates on standard, non-distributed settings, and how to solve the bilevel optimization problems under distributed settings have received much less attention. Federated learning (FL)~\cite{mcmahan2017communication} is a recently promising distributed learning paradigm. In FL, a set of clients jointly solve a machine learning task under the coordination of a central server. To protect user privacy and mitigate communication overhead, clients perform multiple steps of local update before communicating with the server. A variety of algorithms~\cite{wang2019adaptive, yu2019parallel,haddadpour2019convergence, karimireddy2019scaffold, bayoumi2020tighter} have been proposed to accelerate this training process. However, most of these algorithms primarily address standard single-level optimization problems. In this work, we study the bilevel optimization problems in the Federated Learning setting and investigate the following research question: \textit{Is it possible to develop communication-efficient federated algorithms tailored for bilevel optimization problems that also ensure a rapid convergence rate?}

\begin{table}
  \centering
  \caption{ \textbf{Comparisons of the Federated/Non-federated bilevel optimization algorithms for finding an $\epsilon$-stationary point of~\eqref{eq:fed-bi-multi}}. $Gc(f,\epsilon)$ and $Gc(g,\epsilon)$ denote the number of gradient evaluations \emph{w.r.t.} $f^{(m)}(x,y)$ and $g^{(m)}(x,y)$; $JV(g,\epsilon)$ denotes the number of Jacobian-vector products; $HV(g,\epsilon)$ is the number of Hessian-vector products; $\kappa=L/\mu$ is the condition number, $p(\kappa)$ is used when no dependence is provided. Sample complexities are measured by client.} \label{tab:1}
  \resizebox{\textwidth}{!}{
 \begin{tabular}{c|c|c|c|c|c|c|c}
  \toprule
    \textbf{Setting} & \textbf{Algorithm} & \textbf{Communication} & $\bm{Gc(f,\epsilon)}$ & $\bm{Gc(g,\epsilon)}$ & $\bm{JV(g,\epsilon)}$ & $\bm{HV(g,\epsilon)}$ & \textbf{Heterogeneity} \\ \midrule
   \multirow{2}{*}{\textbf{Non-Fed}} & StocBiO~\cite{ji2021bilevel} &  &$O(\kappa^5\epsilon^{-2})$ & $O(\kappa^{9}\epsilon^{-2})$ &$O(\kappa^{5}\epsilon^{-2})$  & $O(\kappa^{6}\epsilon^{-2})$ \\
    & MRBO~\cite{yang2021provably} & & $O(p(\kappa)\epsilon^{-1.5})$ & $O(p(\kappa)\epsilon^{-1.5})$ & $O(p(\kappa)\epsilon^{-1.5})$ & $O(p(\kappa)\epsilon^{-1.5})$\\ \midrule
  \multirow{6}{*}{\textbf{Federated}} & CommFedBiO~\cite{li2022communication} & $O(p(\kappa)\epsilon^{-2})$ &$O(p(\kappa)\epsilon^{-2})$ & $O(p(\kappa)\epsilon^{-2})$ &$O(p(\kappa)\epsilon^{-2})$ & $O(p(\kappa)\epsilon^{-2})$ & $\cmark$\\ 
   & FedNest~\cite{tarzanagh2022fednest} & $O(\kappa^9\epsilon^{-2})$ & $O(\kappa^{5}\epsilon^{-2})$ &$O(\kappa^9\epsilon^{-2})$ & $O(\kappa^{5}\epsilon^{-2})$& $O(\kappa^{9}\epsilon^{-2})$ & $\cmark$\\ 
   & AggITD~\cite{xiao2023communication} & $O(p(\kappa)\epsilon^{-2})$ & $O(p(\kappa)\epsilon^{-2})$ &$O(p(\kappa)\epsilon^{-2})$ & $O(p(\kappa)\epsilon^{-2})$& $O(p(\kappa)\epsilon^{-2})$ & $\cmark$\\ 
   & FedMBO~\cite{huang2023achieving} & $O(M^{-1}p(\kappa)\epsilon^{-2})$ & $O(M^{-1}p(\kappa)\epsilon^{-2})$ &$O(M^{-1}p(\kappa)\epsilon^{-2})$ & $O(M^{-1}p(\kappa)\epsilon^{-2})$& $O(M^{-1}p(\kappa)\epsilon^{-2})$ & $\cmark$\\
   & SimFBO~\cite{yang2023simfbo} & $O(p(\kappa)\epsilon^{-1})$ & $O(M^{-1}p(\kappa)\epsilon^{-2})$ &$O(M^{-1}p(\kappa)\epsilon^{-2})$ & $O(M^{-1}p(\kappa)\epsilon^{-2})$& $O(M^{-1}p(\kappa)\epsilon^{-2})$ & $\cmark$ \\
   & Local-BSGVR~\cite{gao2022convergence} & $O(p(\kappa)\epsilon^{-1})$ & $O(M^{-1}p(\kappa)\epsilon^{-1.5})$ &$O(M^{-1}p(\kappa)\epsilon^{-1.5})$ & $O(M^{-1}p(\kappa)\epsilon^{-1.5})$& $O(M^{-1}p(\kappa)\epsilon^{-1.5})$ & $\xmark$ \\
    & \textbf{FedBiOAcc (Ours)} & $\bm{O(\kappa^{19/3}\epsilon^{-1})}$ &$\bm{O(M^{-1}\kappa^8\epsilon^{-1.5})}$ & $\bm{O(M^{-1}\kappa^{8}\epsilon^{-1.5})}$ & $\bm{O(M^{-1}\kappa^8\epsilon^{-1.5})}$  & $\bm{O(M^{-1}\kappa^8\epsilon^{-1.5})}$ & $\cmark$\\\bottomrule
 \end{tabular}
 }
\vspace{-0.1in}
\end{table}

More specifically, a general Federated Bilevel Optimization problem has the following form:
\begin{align}\label{eq:fed-bi-multi}
    \underset{x \in \mathbb{R}^p }{\min}\ h(x) &\coloneqq \frac{1}{M}\sum_{m=1}^{M} f^{(m)}(x , y_{x}), \;\mbox{s.t.}\ y_{x} = \underset{y\in \mathbb{R}^{d}}{\arg\min}\ \frac{1}{M}\sum_{m=1}^{M} g^{(m)}(x,y)
\end{align}
A federated bilevel optimization problem consists of an upper and a lower level problem, the upper problem $f(x,y) \coloneqq \frac{1}{M} \sum_{m=1}^M f^{(m)}(x,y)$ relies on the solution $y_x$ of the lower problem, and $g(x,y) \coloneqq \frac{1}{M} \sum_{m=1}^M g^{(m)}(x,y)$. Meanwhile, both the upper and the lower level problems are federated: In Eq.\eqref{eq:fed-bi-multi}, we have $M$ clients, and each client has a local upper problem $f^{(m)}(x,y)$ and a lower level problem $g^{(m)}(x,y)$. Compared to single-level federated optimization problems, the estimation of the hyper-gradient in federated bilevel optimization problems is much more challenging. In Eq.\eqref{eq:fed-bi-multi}, the hyper-gradient is not linear \emph{w.r.t} the local hyper-gradients of clients, whereas the gradient of a single-level Federated Optimization problem is the average of local gradients. Consequently, directly applying the vanilla local-sgd method~\cite{mcmahan2017communication} to federated bilevel problems results in a large bias. In the literature~\cite{tarzanagh2022fednest, li2022communication, huang2023achieving, xiao2023communication}, researchers evaluate the hyper-gradient through multiple rounds of client-server communication, however, this approach leads to high communication overhead. In contrast, we view the hyper-gradient estimation as solving a quadratic federated problem and solving it with the local-sgd method. More specifically, we formulate the solution of the federated bilevel optimization as three intertwined federated problems: the upper problem, the lower problem and the quadratic problem for the hyper-gradient estimation. Then we address the three problems using alternating gradient descent steps, furthermore, to manage the noise of the stochastic gradient and obtain the fast convergence rate, we employ a momentum-based variance reduction technique.

Beyond the standard federated bilevel optimization problem as defined in Eq.~\ref{eq:fed-bi-multi}, another variant of Federated Bilevel Optimization problem, which entails locally managed lower-level problems, is also frequently utilized in practical applications. For this type of problem, we can get an unbiased estimate of the global hyper-gradient using local hyper-gradient, thus we can solve it with a local-SGD like algorithm, named FedBiOAcc-Local. However, it is challenging to analyze the convergence of the algorithm. In particular, we need to bound the intertwined client drift error, which is intrinsic to FL and the bilevel-related errors \emph{e.g.} the lower level solution bias. In fact, we prove that the FedBiOAcc-Local algorithm attains the same fast rate as FedBiO algorithm.

Finally, we highlight the main \textbf{contributions} of our paper as follows:

\begin{enumerate}
\vspace{-0.1in}
\setlength{\itemsep}{-1pt}
    \item We propose FedBiOAcc to solve Federated Bilevel Optimization problems, the algorithm evaluates the hypergradient of federated bilevel optimization problems efficiently and achieves optimal convergence rate through momentum-based variance reduction. FedBiOAcc has sample complexity of $O(\epsilon^{-1.5})$, communication complexity of $O(\epsilon^{-1})$ and achieves linear speed-up \emph{w.r.t} the number of clients.
    \item We study Federated Bilevel Optimization problem with local lower level problem for the first time, where we show the convergence of a modified version of FedBiOAcc, named FedBiOAcc-Local for this type of problems.
    \item We validate the efficacy of the proposed  FedBiOAcc algorithm through two real-world tasks: Federated Data Cleaning and Federated Hyper-representation Learning.
\vspace{-0.1in}
\end{enumerate}

\textbf{Notations} $\nabla$ denotes full gradient, $\nabla_{x}$ denotes partial derivative for variable x, higher order derivatives follow similar rules. $[K]$ represents the sequence of integers from 1 to $K$, $\bar{x}$ represents average of the sequence of variables $\{x^{(m)}\}_{m=1}^M$. $\bar{t}_s$ represents the global communication timestamp $s$.

\section{Related Works}
Bilevel optimization dates back to at least the 1960s when~\cite{willoughby1979solutions} proposed a regularization method, and then followed by many research works~\cite{ferris1991finite,solodov2007explicit,yamada2011minimizing,sabach2017first}, while in machine learning community, similar ideas in the name of implicit differentiation were also used in Hyper-parameter Optimization~\cite{larsen1996design,chen1999optimal,bengio2000gradient, do2007efficient}. Early algorithms for Bilevel Optimization solved the accurate solution of the lower problem for each upper variable. Recently, researchers developed algorithms that solve the lower problem with a fixed number of steps, and use the `back-propagation through time' technique to compute the hyper-gradient~\cite{domke2012generic, maclaurin2015gradient, franceschi2017forward, pedregosa2016hyperparameter, shaban2018truncated}. 
Very Recently, it witnessed a surge of interest in using implicit differentiation to derive single loop algorithms~\cite{ghadimi2018approximation, hong2020two, ji2020provably, khanduri2021near, chen2021single, yang2021provably, huang2021biadam, li2022fully, dagreou2022framework, huang2022enhanced, huang2023momentum}. In particular, \cite{li2022fully, dagreou2022framework} proposes a way to iteratively evaluate the hyper-gradients to save computation. In this work, we view the hyper-gradient estimation of Federated Bilevel Optimization as solving a quadratic federated optimization problem and use a similar iterative evaluation rule as~\cite{li2022fully, dagreou2022framework} in local update.

The bilevel optimization problem is also considered in the more general settings. For example, bilevel optimization with multiple lower tasks is considered in~\cite{guo2021randomized}, furthermore,~\cite{chen2022decentralized, yang2022decentralized, lu2022stochastic, gao2023convergence} studies the bilevel optimization problem in the decentralized setting, \cite{jiao2022asynchronous} studies the bilevel optimization problem in the asynchronous setting. In contrast, we study bilevel optimization problems under Federated Learning~\cite{mcmahan2017communication} setting. Federated learning is a promising privacy-preserving learning paradigm for distributed data. Compared to traditional data-center distributed learning, Federated Learning poses new challenges including data heterogeneity, privacy concerns, high communication cost, and unfairness. To deal with these challenges, various methods~\cite{karimireddy2019scaffold, li2019convergence, sahu2018convergence, zhao2018federated, mohri2019agnostic, li2021ditto} are proposed. However, bilevel optimization problems are less investigated in the federated learning setting. \cite{xingbig} considered the distributed bilevel formulation, but it needs to communicate the Hessian matrix for every iteration, which is computationally infeasible. More recently, FedNest~\cite{tarzanagh2022fednest} has been proposed to tackle the general federated nest problems, including federated bilevel problems. However, this method evaluates the full hyper-gradient at every iteration; this leads to high communication overhead; furthermore, FedNest also uses SVRG to accelerate the training. Similar works that evaluate the hyper-gradient with multiple rounds of client-server communication are~\cite{li2022communication, huang2023achieving, xiao2023communication, yang2023simfbo}. Finally, there is a concurrent work~\cite{gao2022convergence} that investigates the possibility of local gradients on Federated Bilevel Optimization, however, it only considers the homogeneous case, this setting is quite constrained and much simpler than the more general heterogeneous case we considered. Furthermore, \cite{gao2022convergence} only considers the case where both the upper and the lower problem are federated, and omit the equally important case where the lower level problem is not federated.

\section{Federated Bilevel Optimization}

\subsection{Some Mild Assumptions}
Note that the formulation of Eq.\eqref{eq:fed-bi-multi} is very general, and we consider the stochastic heterogeneous case in this work. More specifically, we assume:
\[f^{(m)}(x ,y) \coloneqq \mathbb{E}_{\xi \sim \mathcal{D}_f^{(m)}}[f^{(m)}(x,y, \xi)] 
\text{, } g^{(m)}(x,y) \coloneqq \mathbb{E}_{\xi \sim \mathcal{D}_g^{(m)}} [g^{(m)}(x,y; \xi)]\] 
where $\mathcal{D}_f^{(m)}$ and $\mathcal{D}_g^{(m)}$ are some probability distributions. Furthermore, we assume the local objectives could be potentially different: $f^{(m)}(x,y) \neq f^{(k)}(x,y)$ or $g^{(m)} (x,y) \neq g^{(k)}(x,y)$ for $m \neq k, m,k \in [M]$. Furthermore, we assume the following assumptions in our subsequent discussion:
\begin{assumption}\label{assumption:function}
	Function $f^{(m)}(x,y)$ is possibly non-convex and $g^{(m)}(x,y)$ is $\mu$-strongly convex \emph{w.r.t} $y$ for any given $x$.
\end{assumption}
\begin{assumption}\label{assumption:f_smoothness}
	Function $f^{(m)}(x,y)$ is $L$-smooth and has $C_f$-bounded gradient;
\end{assumption}
\begin{assumption}\label{assumption:g_smoothness}
	Function $g^{(m)}(x,y)$ is  $L$-smooth, and $\nabla_{xy} g^{(m)}(x,y)$ and $\nabla_{y^2} g^{(m)}(x,y)$ are Lipschitz continuous with constants $L_{xy}$ and $L_{y^2}$ respectively;
\end{assumption}
\begin{assumption}\label{assumption:noise_assumption}
	We have unbiased stochastic first-order and second-order gradient oracle with bounded variance.
\end{assumption}
\begin{assumption} \label{assumption:hetero}
For any $m, j \in [M]$ and $z = (x,y)$, we have: $ \| \nabla f^{(m)} (z) -  \nabla f^{(j)} (z) \| \leq \zeta_f$, $ \| \nabla g^{(m)} (z) -  \nabla g^{(j)}(z) \| \leq \zeta_g$, $ \| \nabla_{xy} g^{(m)} (z) -  \nabla_{xy} g^{(j)} (z) \| \leq \zeta_{g,xy}$, $ \| \nabla_{y^2} g^{(m)} (z) -  \nabla_{y^2} g^{(j)} (z) \| \leq \zeta_{g,yy}$, where $\zeta_f$, $\zeta_g$, $\zeta_{g,xy}$, $\zeta_{g,yy}$, are constants.
\end{assumption}
As stated in The assumption~\ref{assumption:function}, we study the \textbf{non-convex-strongly-convex} bilevel optimization problems, this class of problems is widely studied in the non-distributed bilevel literature~\cite{ji2021lower, ghadimi2018approximation}. Furthermore, Assumption~\ref{assumption:f_smoothness} and Assumption~\ref{assumption:g_smoothness} are also standard assumptions made in the non-distributed bilevel literature. Assumption~\ref{assumption:noise_assumption} is widely used in the study of stochastic optimization problems. For Assumption~\ref{assumption:hetero}, gradient difference is widely used in single level Federated Learning literature as a measure of client heterogeneity~\cite{khanduri2021near, woodworth2021minimax}. Please refer to the full version of Assumptions in Appendix.

\begin{algorithm}[t]
\caption{Accelerated Federated Bilevel Optimization (\textbf{FedBiOAcc})}
\label{alg:FedBiOAcc}
\begin{algorithmic}[1]
\STATE {\bfseries Input:} Constants $c_{\omega}$, $c_{\nu}$, $c_u$, $\gamma$, $\eta$, $\tau$, $r$; learning rate schedule $\{\alpha_t\}$, $t \in [T]$, initial state ($x_1$, $y_1$, $u_1$);
\STATE{\bfseries Initialization:} Set $y^{(m)}_{1} = y_{1}$, $x^{(m)}_{1} = x_{1}$, $u^{(m)}_1 = u_1$, $\omega_{1}^{(m)} = \nabla_y g^{(m)} (x_{1}, y_{1}, \mathcal{B}_{y})$, $\nu_{1}^{(m)} = \nabla_x f^{(m)}(x_1, y_1; \mathcal{B}_{f,1}) - \nabla_{xy}g^{(m)}(x_1, y_1; \mathcal{B}_{g,1})u_1$ and $q_1 = \nabla_{y^2} g^{(m)}(x^{(m)}_{1}, y^{(m)}_{1}; \mathcal{B}_{g,2})u_{1} - \nabla_y f^{(m)}(x^{(m)}_{1}, y^{(m)}_{1}; \mathcal{B}_{f,2})$ for $m \in [M]$
\FOR{$t=1$ \textbf{to} $T$}
\STATE $\hat{y}^{(m)}_{t+1} = y^{(m)}_{t} - \gamma\alpha_{t}  \omega_{t}^{(m)}$, $\hat{x}^{(m)}_{t+1} = x^{(m)}_{t} - \eta\alpha_{t} \nu_{t}^{(m)}$, $\hat{u}_{t+1}^{(m)} = \mathcal{P}_{r}(u^{(m)}_t - \tau\alpha_tq^{(m)}_t)$
\STATE Get $\hat{\omega}^{(m)}_{t+1}$, $\hat{\nu}^{(m)}_{t+1}$ and $\hat{q}^{(m)}_{t+1}$ following Eq.~\eqref{eq:fedbioacc-alg1}
\IF{$t$ mod I $ = 0$}
\STATE $y^{(m)}_{t+1} = \frac{1}{M}\sum_{j=1}^{M} \hat{y}^{(j)}_{t+1}$;  $x^{(m)}_{t+1} = \frac{1}{M}\sum_{j=1}^{M} \hat{x}^{(j)}_{t+1}$, $u^{(m)}_{t+1} = \frac{1}{M}\sum_{j=1}^{M} \hat{u}^{(j)}_{t+1}$
\STATE $\omega^{(m)}_{t+1} =  \frac{1}{M}\sum_{j=1}^{M} \hat{\omega}^{(j)}_{t+1}$, $\nu^{(m)}_{t+1} =  \frac{1}{M}\sum_{j=1}^{M} \hat{\nu}^{(j)}_{t+1} $, $q^{(m)}_{t+1} = \frac{1}{M}\sum_{j=1}^{M} \hat{q}^{(j)}_{t+1}$,
\ELSE
\STATE $y^{(m)}_{t+1} = \hat{y}^{(m)}_{t+1}$, $x^{(m)}_{t+1} = \hat{x}^{(m)}_{t+1}$, $u^{(m)}_{t+1} = \hat{u}^{(m)}_{t+1}$
\STATE  $\omega^{(m)}_{t+1} = \hat{\omega}^{(m)}_{t+1}$, $\nu^{(m)}_{t+1} = \hat{\nu}^{(m)}_{t+1}$, $q^{(m)}_{t+1} = \hat{q}^{(m)}_{t+1}$
\ENDIF
\ENDFOR
\end{algorithmic}
\end{algorithm}

\subsection{The FedBiOAcc Algorithm}
A major difficulty in solving a Federated Bilevel Optimization problem Eq.~\eqref{eq:fed-bi-multi} is \textbf{evaluating the hyper-gradient $\bm{\nabla h(x)}$}. For the function class (non-convex-strongly-convex) we consider, the explicit form of hypergradient $h(x)$ exists as $\nabla h(x) = \Phi(x,y_x)$, where $\Phi(x, y)$ is denoted as:
\begin{align}\label{eq:outer_grad_other}
    \Phi(x,y) = &\nabla_x f(x, y) -  \nabla_{xy} g(x, y)\times [\nabla_{y^2} g(x, y)]^{-1} \nabla_y f(x, y),
\end{align}
Based on Assumption~\ref{assumption:function}$\sim$\ref{assumption:g_smoothness}, we can verify $\Phi(x,y_x)$ is the hyper-gradient~\cite{ghadimi2018approximation}. But since the clients only have access to their local data, for $\forall m \in [M]$, the client evaluates:
\begin{align}\label{eq:outer_grad_other_local}
    \Phi^{(m)}(x,y) = &\nabla_x f^{(m)}(x, y)  -  \nabla_{xy} g^{(m)}(x, y)\times [\nabla_{y^2} g^{(m)}(x, y)]^{-1} \nabla_y f^{(m)}(x, y),
\end{align}
It is straightforward to verify that $\Phi^{(m)}(x,y)$ is not an unbiased estimate of the full hyper-gradient, \emph{i.e.} $\Phi(x,y_x) \neq \frac{1}{M}\sum_{m=1}^M\Phi^{(m)}(x,y_x)$. To address this difficulty, we can view the \textbf{Hyper-gradient computation as the process of solving a federated optimization problem}. 

In fact, Evaluating Eq.~\eqref{eq:outer_grad_other} is equivalent to the following two steps: first, we solve the quadratic federated optimization problem $l(u)$:
\begin{align}\label{eq:hg-quad}
\underset{u \in \mathbb{R}^d }{\min}\; l(u) =  \frac{1}{M}\sum_{m=1}^{M}  u^T(\nabla_{y^2} g^{(m)}(x, y))u - \langle \nabla_y f^{(m)}(x, y), u\rangle
\end{align}
Suppose that we denote the solution of the above problem as $u^*$, then we have the following linear operation to get the hypergradient:
\begin{align}\label{eq:hg-est}
\nabla h(x) = \frac{1}{M}\sum_{m=1}^M \big(\nabla_x f^{(m)}(x, y_x) -  \nabla_{xy} g^{(m)}(x, y_x)u^{*}\big)
\end{align}
Compared to the formulation Eq.~\eqref{eq:outer_grad_other}, Eq.~\eqref{eq:hg-quad} and Eq.~\eqref{eq:hg-est} are more suitable for the distributed setting. In fact, both Eq.~\eqref{eq:hg-quad} and Eq.~\eqref{eq:hg-est} have a linear structure.
Eq.~\eqref{eq:hg-quad} is a (single-level) quadratic federated optimization problem, and we could solve Eq.~\eqref{eq:hg-quad} through local-sgd~\cite{mcmahan2017communication}, suppose that each client maintains a variable $u^{(m)}_t$, and performs the following update:
\begin{align*}
   u_{t+1}^{(m)} &= \mathcal{P}_r(u^{(m)}_t - \tau_t\nabla l^{(m)}(u^{(m)}_t;\mathcal{B})) \nonumber\\
   \nabla l^{(m)}(u^{(m)}_t;\mathcal{B}) &= \nabla_{y^2} g^{(m)}(x^{(m)}_t, y^{(m)}_t; \mathcal{B}_{g,2}))u^{(m)}_t  - \nabla_y f^{(m)}(x^{(m)}_t, y^{(m)}_t; \mathcal{B}_{f,2})
\end{align*}
where $\nabla l^{(m)}(u^{(m)}_t;\mathcal{B})$ is client $m$'s the stochastic gradient of Eq.~\eqref{eq:hg-quad}, and $(x^{(m)}_t, y^{(m)}_t)$ denotes the upper and lower variable state at the timestamp $t$, the $\mathcal{P}_r(\cdot)$ denotes the projection to a bounded ball of radius-$r$. Note that Clients perform multiple local updates of $u^{(m)}_t$ before averaging. As for Eq.~\eqref{eq:hg-est}, each client evaluates $\nabla h^{(m)}(x)$ locally:
$\nabla h^{(m)}(x) =  \nabla_x f^{(m)}(x, y_x) -  \nabla_{xy} g^{(m)}(x, y_x)u^{*}$
and the server averages $\nabla h^{(m)}(x)$ to get $\nabla h(x)$. In summary, the linear structure of  Eq.~\eqref{eq:hg-quad} and Eq.~\eqref{eq:hg-est} makes it suitable for local updates, therefore, reduce the communication cost.

More specifically, we perform alternative update of upper level variable $x^{(m)}_t$, the lower level variable $y^{(m)}_t$ and hyper-gradient computation variable $u^{(m)}_t$. For example, for each client $m \in [M]$, we perform the following local updates:
\begin{align}\label{eq:fedbio_alg}
   y^{(m)}_{t+1} &= y^{(m)}_{t} - \gamma_t  \nabla_y g^{(m)} (x^{(m)}_{t}, y^{(m)}_{t}, \mathcal{B}_y) 
   ,\; u_{t+1}^{(m)} = \mathcal{P}_r(u^{(m)}_t - \tau_t\nabla l^{(m)}(u^{(m)}_t;\mathcal{B})) \nonumber\\
   x^{(m)}_{t+1} &= x^{(m)}_{t} - \eta_t \big(\nabla_x f^{(m)}(x^{(m)}_t, y^{(m)}_t; \mathcal{B}_{f,1}) - \nabla_{xy}g^{(m)}(x^{(m)}_t, y^{(m)}_t; \mathcal{B}_{g,1})u_{t}^{(m)}\big)
\end{align}
Every $I$ steps, the server averages clients' local states, this resembles the local-sgd method for single level federated optimization problems. Note that in the update of the upper variable $x^{(m)}_{t}$, we use $u_{t}^{(m)}$ as an estimation of $u^*$ in Eq.~\eqref{eq:hg-est}. An algorithm follows Eq.~\eqref{eq:fedbio_alg} is shown in Algorithm 2 of Appendix and we refer to it as FedBiO.

\textbf{Comparison with FedNest.} The update rule of Eq.~\ref{eq:fedbio_alg} is very different from that of FedNest~\cite{tarzanagh2022fednest} and its follow-ups~\cite{huang2023achieving, xiao2023communication}. In FedNest, a sub-routine named FedIHGP is used to evaluate Eq.~\eqref{eq:outer_grad_other} at every global epoch. This involves multiple rounds of client-server communication and leads to higher communication overhead. In contrast, Eq.~\eqref{eq:fedbio_alg} formulates the hyper-gradient estimation as an quadratic federated optimization problem, and then solves three intertwined federated problems through alternative updates of $x$, $y$ and $u$. 

Note that Eq.~\ref{eq:fedbio_alg} updates the related variables through vanilla gradient descent steps. In the non-federated setting, gradient-based methods such as stocBiO~\cite{ji2020provably} requires large-batch size ($O(\epsilon^{-1})$) to reach an $\epsilon$-stationary point, and we also analyze Algorithm 2 in Appendix to show the same dependence. To control the noise and remove the dependence over large batch size, we apply the momentum-based variance-reduction technique STORM~\cite{cutkosky2019momentum}.  In fact, Eq.~\eqref{eq:fedbio_alg} solves three intertwined optimization problems: the bilevel problem $h(x)$, the lower level problem $g(x, y)$ and the hyper-gradient computation problem Eq~\eqref{eq:hg-quad}. So we control the noise in the process of solving each of the three problems. More specifically, we have $\omega^{(m)}_{t}$, $\nu^{(m)}_{t}$ and $q^{(m)}_{t}$
to be the momentum estimator for $x^{(m)}_{t}$, $y^{(m)}_{t}$ and $u^{(m)}_{t}$ respectively, and we update them following the rule of STORM~\cite{cutkosky2019momentum}:
\begin{align}\label{eq:fedbioacc-alg1}
\hat{\omega}_{t+1}^{(m)} &= \nabla_y g^{(m)} (x^{(m)}_{t+1}, y^{(m)}_{t+1}, \mathcal{B}_{y}) + (1 - c_{\omega}\alpha_{t}^2) (\omega_{t}^{(m)} - \nabla_y g^{(m)} (x^{(m)}_{t}, y^{(m)}_{t}, \mathcal{B}_{y})) \nonumber\\
\hat{\nu}_{t+1}^{(m)} &= \big( \nabla_x f^{(m)}(x^{(m)}_{t+1}, y^{(m)}_{t+1}; \mathcal{B}_{f,1}) - \nabla_{xy}g^{(m)}(x^{(m)}_{t+1}, y^{(m)}_{t+1}; \mathcal{B}_{g,1})u_{t+1}^{(m)}\big) \nonumber\\
&\quad + (1 - c_{\nu}\alpha_{t}^2) \big(\nu_{t}^{(m)} - \big( \nabla_x f^{(m)}(x^{(m)}_{t}, y^{(m)}_{t}; \mathcal{B}_{f,1}) - \nabla_{xy}g^{(m)}(x^{(m)}_{t}, y^{(m)}_{t}; \mathcal{B}_{g,1})u_{t}^{(m)}\big)\big) \nonumber\\
\hat{q}_{t+1}^{(m)} &= \big(\nabla_{y^2} g^{(m)}(x^{(m)}_{t+1}, y^{(m)}_{t+1}; \mathcal{B}_{g,2})u^{(m)}_{t+1} - \nabla_y f^{(m)}(x^{(m)}_{t+1}, y^{(m)}_{t+1}; \mathcal{B}_{f,2})\big)\nonumber\\
&\quad + (1 - c_u\alpha_t^2)\big(q^{(m)}_t -  \big(\nabla_{y^2} g^{(m)}(x^{(m)}_{t}, y^{(m)}_{t}; \mathcal{B}_{g,2})u^{(m)}_{t} - \nabla_y f^{(m)}(x^{(m)}_{t}, y^{(m)}_{t}; \mathcal{B}_{f,2})\big)\big)
\end{align}
where $c_\omega$, $c_\nu$ and $c_u$ are constants, $\alpha_t$ is the learning rate. Then we update the $x^{(m)}_{t}$, $y^{(m)}_{t}$ and $u^{(m)}_{t}$ as follows:
\begin{align}\label{eq:fedbioacc_alg2}
    \hat{y}^{(m)}_{t+1} = y^{(m)}_{t} - \gamma\alpha_{t}  \omega_{t}^{(m)}, \hat{x}^{(m)}_{t+1} = x^{(m)}_{t} - \eta\alpha_{t} \nu_{t}^{(m)}, \hat{u}_{t+1}^{(m)} = \mathcal{P}_r(u^{(m)}_t - \tau\alpha_tq^{(m)}_t)
\end{align}
where $\gamma$, $\eta$, $\tau$ are constants and $\alpha_t$ is the learning rate. The FedBiOAcc algorithm following Eq.~\eqref{eq:fedbioacc_alg2} is summarized in Algorithm~\ref{alg:FedBiOAcc}. As shown in line 6 and 12 of Algorithm~\ref{alg:FedBiOAcc}, Every $I$ iterations, we average both variables and the momentum.

\subsection{Convergence Analysis}
In this section, we study the convergence property for the FedBiOAcc algorithm. For any $t \in [T]$, we define the following virtual sequence:
\[\bar{x}_t = \frac{1}{M}\sum_{m=1}^M x^{(m)}_{t}, \bar{y}_t = \frac{1}{M}\sum_{m=1}^M y^{(m)}_{t}, \bar{u}_t = \frac{1}{M}\sum_{m=1}^M u^{(m)}_{t} \] 
we denote the average of the momentum similarly as $\bar{\omega}_t$, $\bar{\nu}_t$ and $\bar{q}_t$.  Then we consider the following Lyapunov function $\mathcal{G}_t$:
\begin{align}\label{eq:lyapunov}
    \mathcal{G}_t &= h(\bar{x}_{t}) + \frac{18\eta\tilde{L}^2}{\mu\gamma}(\|\bar{y}_t - y_{\bar{x}_{t}} \|^2 + \|\bar{u}_{t} - u_{\bar{x}_{t}}\|^2) + \frac{9bM\eta}{64\alpha_{t}} \|\bar{\omega}_t - \frac{1}{M} \sum_{m=1}^M\nabla_y g^{(m)}(x^{(m)}_{t}, y^{(m)}_{t} ) \|^2 \nonumber\\
    &\qquad \qquad + \frac{9bM\eta}{64\alpha_{t}}\|\bar{q}_{t} - \frac{1}{M} \sum_{m=1}^M(\nabla_{y^2} g^{(m)}(x^{(m)}_{t}, y^{(m)}_{t})u^{(m)}_{t} - \nabla_y f^{(m)}(x^{(m)}_{t}, y^{(m)}_{t})))\|^2 \nonumber \\
    & \qquad\qquad + \frac{9bM\eta}{64\alpha_{t}}\| \bar{\nu}_{t} - \frac{1}{M} \sum_{m=1}^M(\nabla_x f^{(m)}(x^{(m)}_{t}, y^{(m)}_{t}) - \nabla_{xy}g^{(m)}(x^{(m)}_{t}, y^{(m)}_{t})u_{t}^{(m)}) \|^2
\end{align}
where $ y_{\bar{x}_{t}}$ denotes the solution of the lower level problem $g(\bar{x}_t, \cdot)$, $u_{\bar{x}_t} = [\nabla_{y^2} g(\bar{x}, y_{\bar{x}})]^{-1}\nabla_y f(\bar{x}, y_{\bar{x}})$ denotes the solution of Eq~\eqref{eq:hg-quad} at state $\bar{x}_t$. Besides, $\gamma, \eta, \tau$ are learning rates and $L, \tilde{L}$ are constants. Note that the first three terms of $\mathcal{G}_t$: $h(\bar{x}_{t})$, $\|\bar{y}_t - y_{\bar{x}_{t}} \|^2$, $\|\bar{u}_{t} - u_{\bar{x}_{t}}\|^2$ measures the errors of three federated problems: the upper level problem, the lower level problem and the hyper-gradient estimation. Then the last three terms measure the estimation error of the momentum variables: $\bar{\omega}_t$, $\bar{\nu}_t$ and $\bar{q}_t$. The convergence proof primarily concentrates on bounding these errors, please see Lemma C.2 - C.6 in the Appendix for more details. Meanwhile, as in the single level federated optimization problems, local updates lead to client-drift error. More specifically, we need to bound $\|x_t^{(m)}-  \bar{x}_t \|^2$, $\|y_t^{(m)}-  \bar{y}_t \|^2$ and $\|u_t^{(m)}-  \bar{u}_t \|^2$, please see Lemma C.7 - C.11 for more details. Finally, we have the following convergence theorem:
\begin{theorem}\label{theorem:FedBiOAcc_multi_main}
Suppose in Algorithm~\ref{alg:FedBiOAcc}, we choose learning rate $\alpha_t = \frac{\delta}{(u + t)^{1/3}}, t \in [T]$, for some constant $\delta$ and $u$, and let $c_{\nu}$, $c_{\omega}$, $c_{u}$ choose some value, $\eta$, $\gamma$ and $\tau$, $r$ be some small values decided by the Lipschitz constants of $h(x)$, we choose the minibatch size to be $b_x = b_y = b$ and the first batch to be $b_1 = O(Ib)$, then we have:
\begin{align*}
    \frac{1}{T}\sum_{t = 1}^{T-1} \mathbb{E} \big[ \|\nabla h(\bar{x}_t) \|^2 \big] = O\big(\frac{\kappa^{19/3}I}{T} + \frac{\kappa^{16/3}}{(bMT)^{2/3}}\big)
\end{align*}
To reach an  $\epsilon$-stationary point, we need $T = O(\kappa^{8}(bM)^{-1}\epsilon^{-1.5})$, $I = O(\kappa^{5/3}(bM)^{-1}\epsilon^{-0.5})$.
\end{theorem}
As stated in the Theorem, to reach an  $\epsilon$-stationary point, we need $T = O(\kappa^{8}(bM)^{-1}\epsilon^{-1.5})$, then the sample complexity for each client is $Gc(f, \epsilon) = O(M^{-1}\kappa^{8}\epsilon^{-1.5})$, $Gc(g, \epsilon) = O(M^{-1}\kappa^{8}\epsilon^{-1.5})$, $Jv(g, \epsilon) = O(M^{-1}\kappa^{8}\epsilon^{-1.5})$, $Hv(g, \epsilon) = O(M^{-1}\kappa^{8}\epsilon^{-1.5})$. So FedBiOAcc achieves the linear speed up \emph{w.r.t.} to the number of clients $M$. Next, suppose we choose $I = O(\kappa^{5/3}(bM)^{-1}\epsilon^{-0.5})$, then the number of communication round $E = O(\kappa^{19/3}\epsilon^{-1})$. This matches the optimal communication complexity of the single level optimization problems as in the STEM~\cite{khanduri2021stem}. Furthermore, compared to FedNest and its variants, FedBiOAcc has improved both the communication complexity and the iteration complexity. As for LocalBSCVR~\cite{gao2022convergence}, FedBiOAcc obtains same rate, but incorporates the heterogeneous case. Note that it is much more challenging to analyze the heterogeneous case. In fact, if we assume homogeneous clients, we have local hyper-gradient (Eq.~\eqref{eq:outer_grad_other_local}) equals the global hyper-gradient (Eq.~\eqref{eq:outer_grad_other}), then we do not need to use the quadratic federated optimization problem view in Section 3.2, while the theoretical analysis is also simplified significantly.

\section{Federated Bilevel Optimization with Local Lower Level Problems}
In this section, we consider an alternative formulation of the Federated Bilevel Optimization problems as follows:
\begin{align}\label{eq:fed-bi}
    \underset{x \in \mathbb{R}^p }{\min}\ h(x) &\coloneqq \frac{1}{M}\sum_{m=1}^{M} f^{(m)}(x , y_{x}^{(m)}), \;\mbox{s.t.}\ y_{x}^{(m)} = \underset{y\in \mathbb{R}^{d}}{\arg\min}\ g^{(m)}(x,y)
\end{align}
Same as Eq.~\eqref{eq:fed-bi-multi}, Eq.~\eqref{eq:fed-bi} has a federated upper level problem, however, Eq.~\eqref{eq:fed-bi}  has a unique lower level problem for each client, which is different from Eq.~\eqref{eq:fed-bi-multi}. In fact, federated bilevel optimization problem Eq~\eqref{eq:fed-bi} can be viewed as a special type of standard federated learning problems. If we denote $h^{(m)} (x) = f^{(m)}(x , y_{x}^{(m)})$, then Eq.~\eqref{eq:fed-bi} can be written as $\underset{x \in \mathbb{R}^p }{\min}\ h(x) \coloneqq \frac{1}{M} \sum_{m=1}^M h^{(m)}(x)$. But due to the bilevel structure of $h^{(m)} (x)$, Eq.~\eqref{eq:fed-bi} is more challenging than the standard Federated Learning problems. 

\textbf{Hyper-gradient Estimation.} Assume Assumption~\ref{assumption:function}$\sim$Assumption~\ref{assumption:g_smoothness} hold, then the hyper-gradient is $\Phi(x,y_x) = \frac{1}{M}\sum_{m=1}^M\Phi^{(m)}(x,y_x)$, where $\Phi^{(m)}(x,y)$ is defined in Eq.~\eqref{eq:outer_grad_other_local}, in other words, the local hyper-gradient $\Phi^{(m)}(x,y)$ is an unbiased estimate of the full hyper-gradient. This fact makes it possible to solve Eq.~\eqref{eq:fed-bi} with local-sgd like methods. More specifically, we solve the local bilevel problem $h^{(m)}(x)$ multiple steps on each client and then the server averages the local states from clients. Please refer to 
Algorithm~3
and the variance-reduction acceleration 
Algorithm~4
in the Appendix. For ease of reference, we name them FedBiO-Local and FedBiOAcc-Local, respectively.

Several challenges exist in analyzing 
FedBiO-Local and FedBiOAcc-Local.
First, Eq.~\eqref{eq:outer_grad_other_local} involves Hessian inverse, so we only evaluate it approximately through the Neumann series~\cite{lorraine2020optimizing} as:
\begin{align}
\label{eq:outer_grad_est_local}
\Phi^{(m)}(x, y;\xi_x) &= \nabla_x f^{(m)}(x, y; \xi_{f}) - \tau\nabla_{xy}g^{(m)}(x, y; \xi_{g}) \nonumber\\
&\qquad\times \sum_{q=-1}^{Q-1} \prod_{j=Q-q}^{Q} (I-\tau \nabla_{y^2} g^{(m)}(x, y; \xi_j))\nabla_y f^{(m)}(x,y;\xi_{f})
\end{align}
where $\xi_x=\{\xi_j(j=1,\ldots,Q), \xi_f,\xi_g\}$, and we assume its elements are mutually independent. $\Phi^{(m)}(x, y;\xi_x)$ is a biased estimate of $\Phi^{(m)}(x,y)$, but with bounded bias and variance (Please see Proposition D.2
for more details.) Furthermore, to reduce the computation cost, each client solves the local lower level problem approximately and we update the upper and lower level variable alternatively. The idea of alternative update is widely used in the non-distributed bilevel optimization~\cite{ji2020provably, yang2021provably}. However, in the federated setting, client variables drift away when performing multiple local steps. As a result, the variable drift error and the bias caused by inexact solution of the lower level problem intertwined with each other. For example, in the local update, clients optimize the lower level variable $y^{(m)}$ towards the minimizer $y^{(m)}_{x^{(m)}}$, but after the communication step, $x^{(m)}$ is smoothed among clients, as a result, the target of $y^{(m)}_t$ changes which causes a huge bias.

In the appendix, we show the FedBiOAcc-Local algorithm achieves the same optimal convergence rate as FedBiOAcc, which has iteration complexity $O(\epsilon^{-1.5})$ and communication complexity $O(\epsilon^{-1})$. However, since the lower level problem in Eq.~\eqref{eq:fed-bi} is unique for each client, FedBiOAcc-Local does not have the property of linear speed-up \emph{w.r.t} the number of clients as FedBiOAcc does.

\section{Numerical Experiments}
In this section, we assess the performance of the proposed FedBiOAcc algorithm through two federated bilevel tasks: Federated Data Cleaning and Federated Hyper-representation Learning. The Federated Data Cleaning task involves global lower level problems, while the Hyper-representation Learning task involves local lower level problems. The implementation is carried out using PyTorch, and the Federated Learning environment is simulated using the PyTorch.Distributed package. Our experiments were conducted on servers equipped with an AMD EPYC 7763 64-core CPU and 8 NVIDIA V100 GPUs.

\subsection{Federated Data Cleaning}
In this section, we consider the Federated Data Cleaning task. In this task, we are given a noisy training dataset whose labels are corrupted by noise and a clean validation set. Then we aim to find weights for training samples such that a model that is learned over the weighted training set performs well on the validation set. This is a federated bilevel problem when the noisy training set is distributed over multiple clients. The formulation of the task is included in 
Appendix B.1.
This task is a specialization of Eq.~\eqref{eq:fed-bi-multi}.

\begin{figure}[t]
\begin{center}
    \includegraphics[width=0.24\columnwidth]{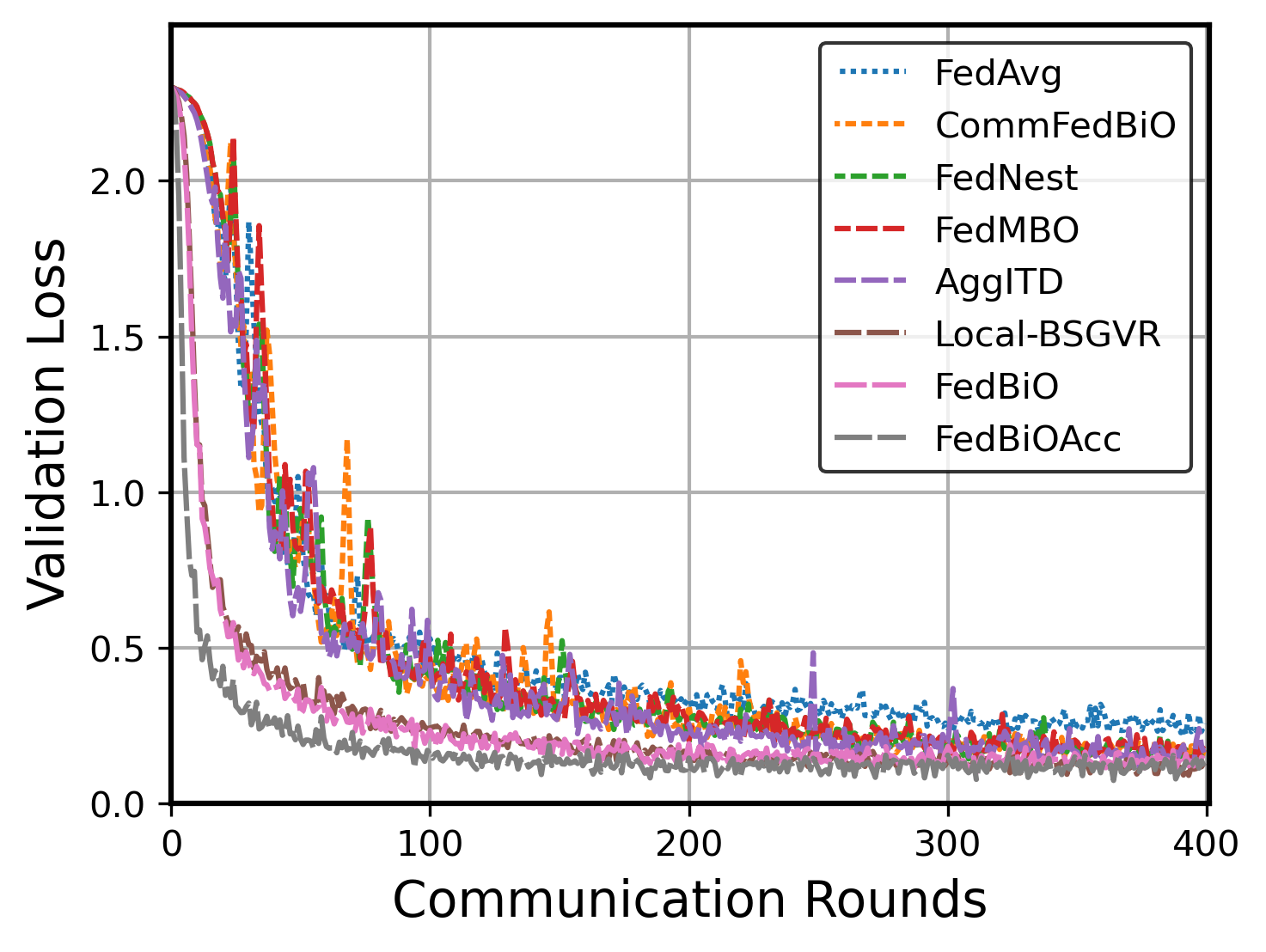}
    \includegraphics[width=0.24\columnwidth]{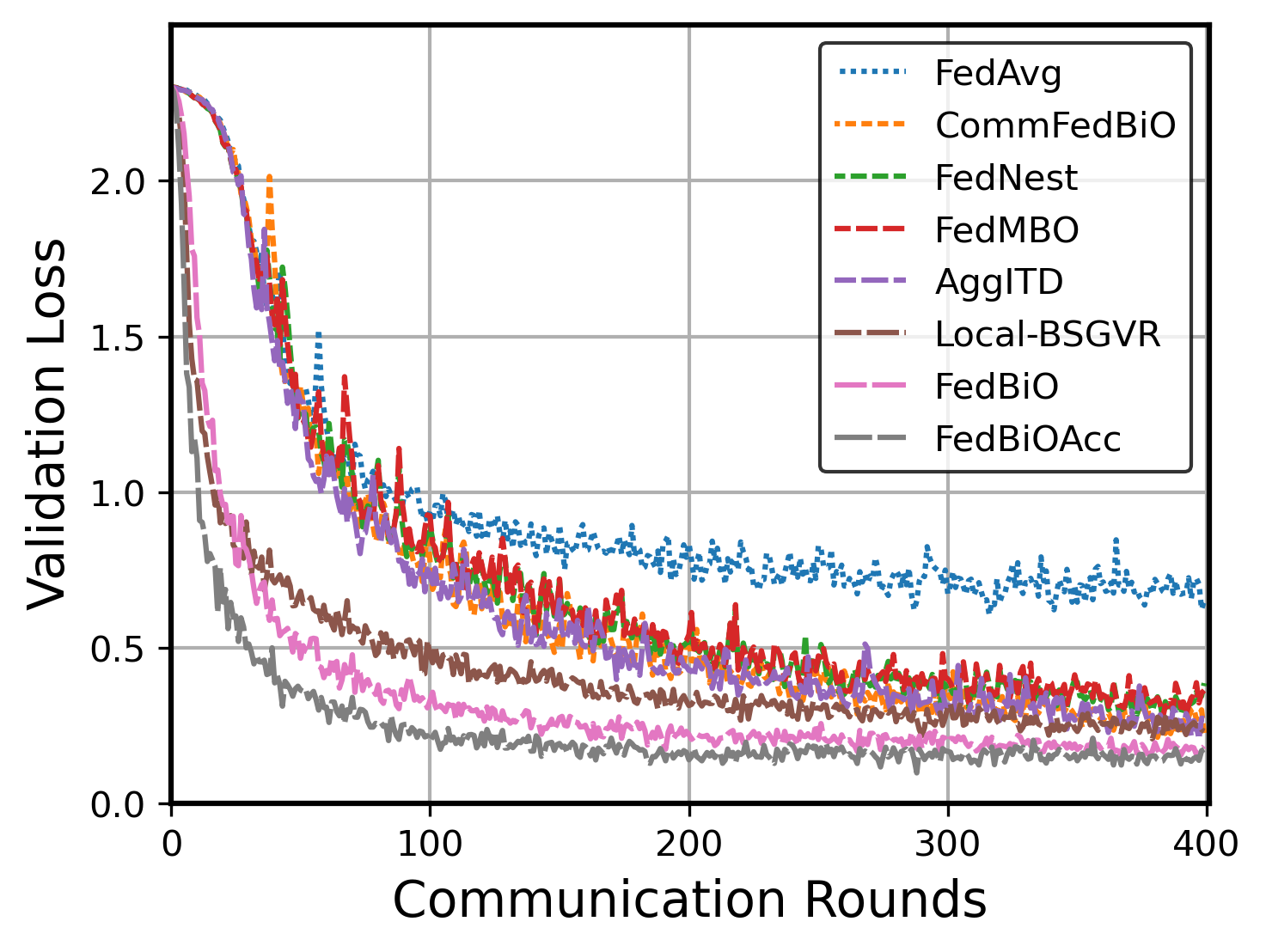}
    \includegraphics[width=0.24\columnwidth]{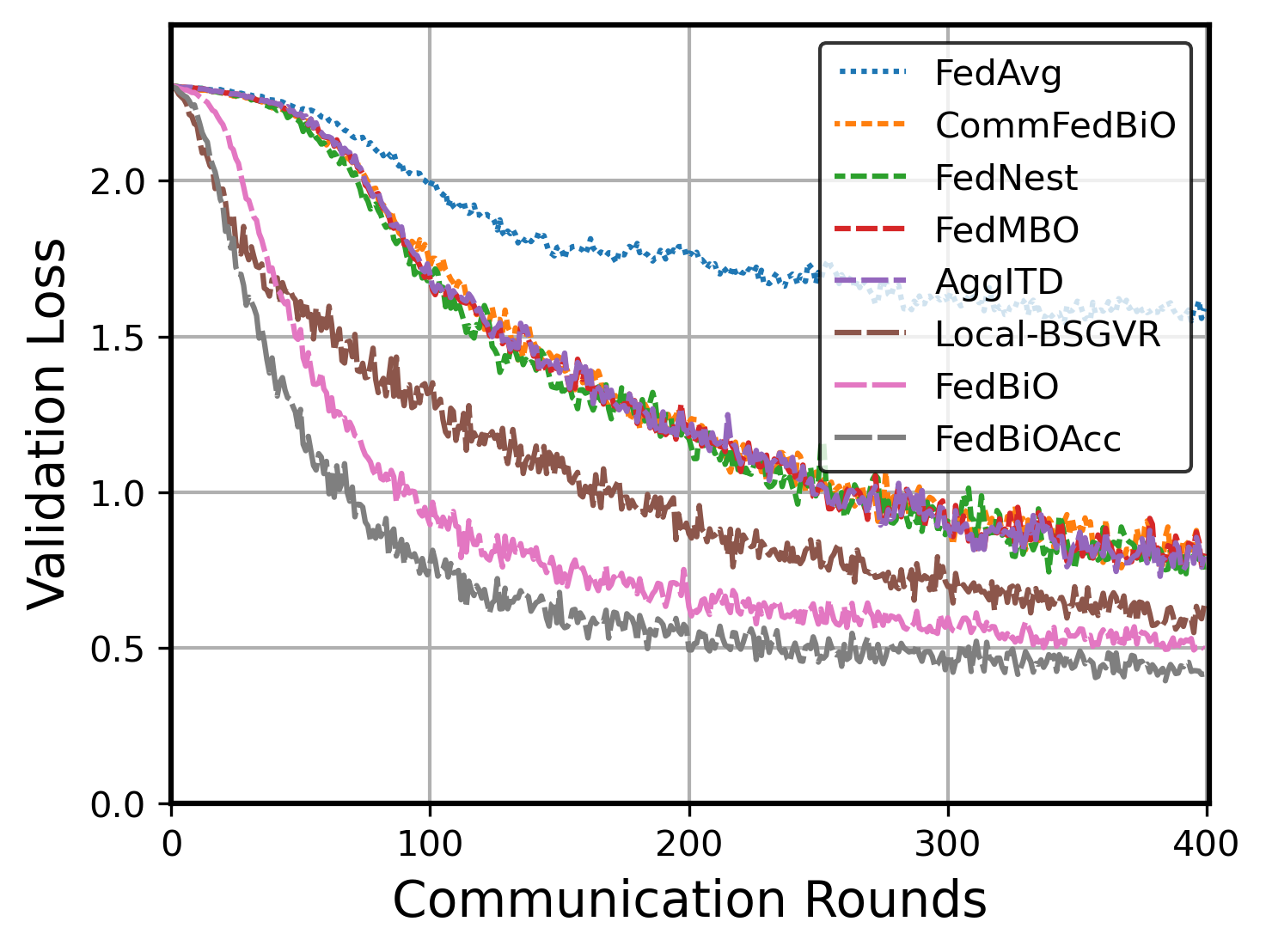}
    \includegraphics[width=0.24\columnwidth]{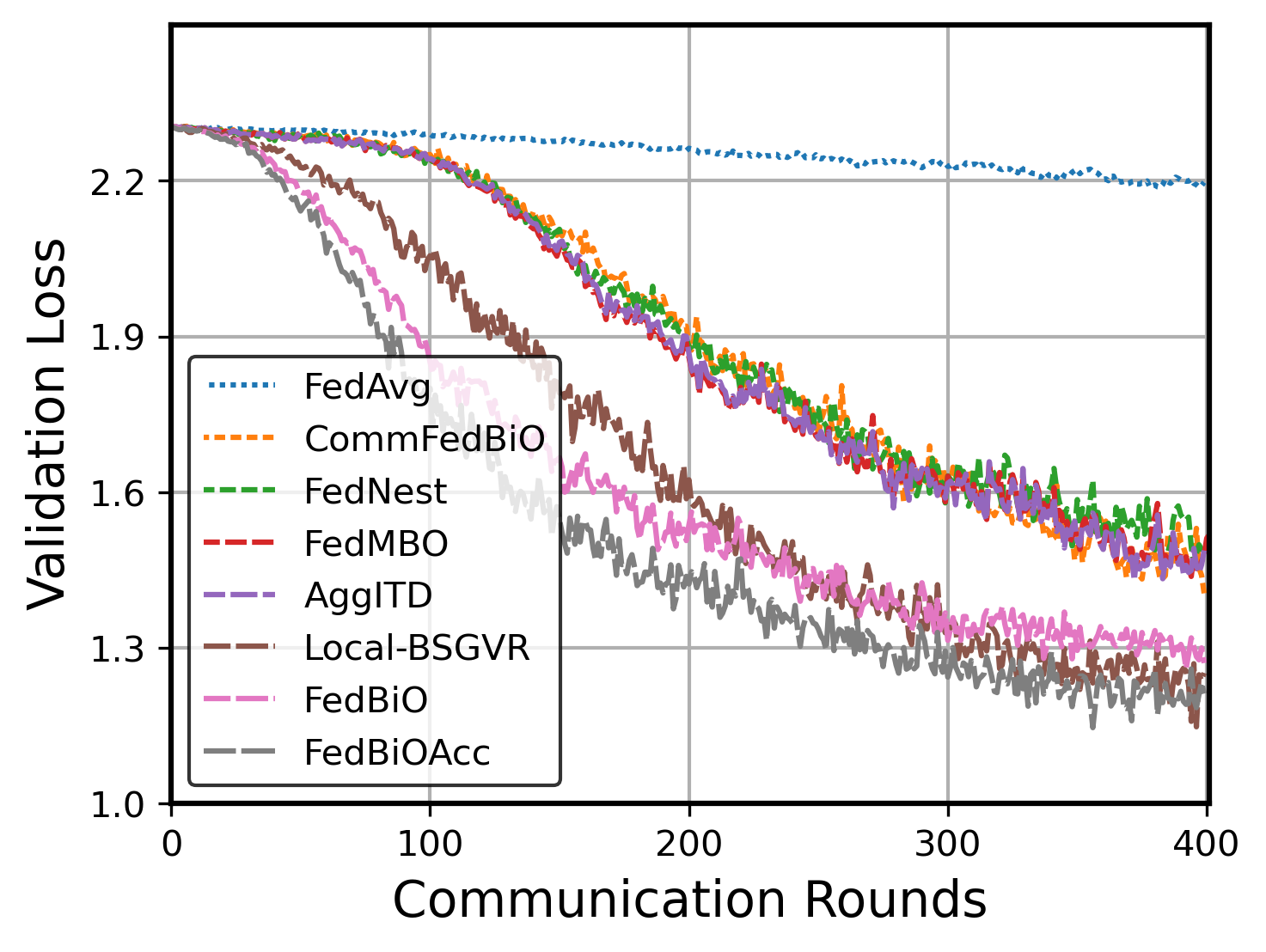}
\end{center}
\caption{Validation Error vs Communication Rounds. From Left to Right: $\rho=0.1, 0.4, 0.8, 0.95$. The local step $I$ is set as 5 for FedBiO, FedBiOAcc and FedAvg.}
\label{fig:clean}
\end{figure}

\begin{figure}[t]
\begin{center}
    \includegraphics[width=0.24\columnwidth]{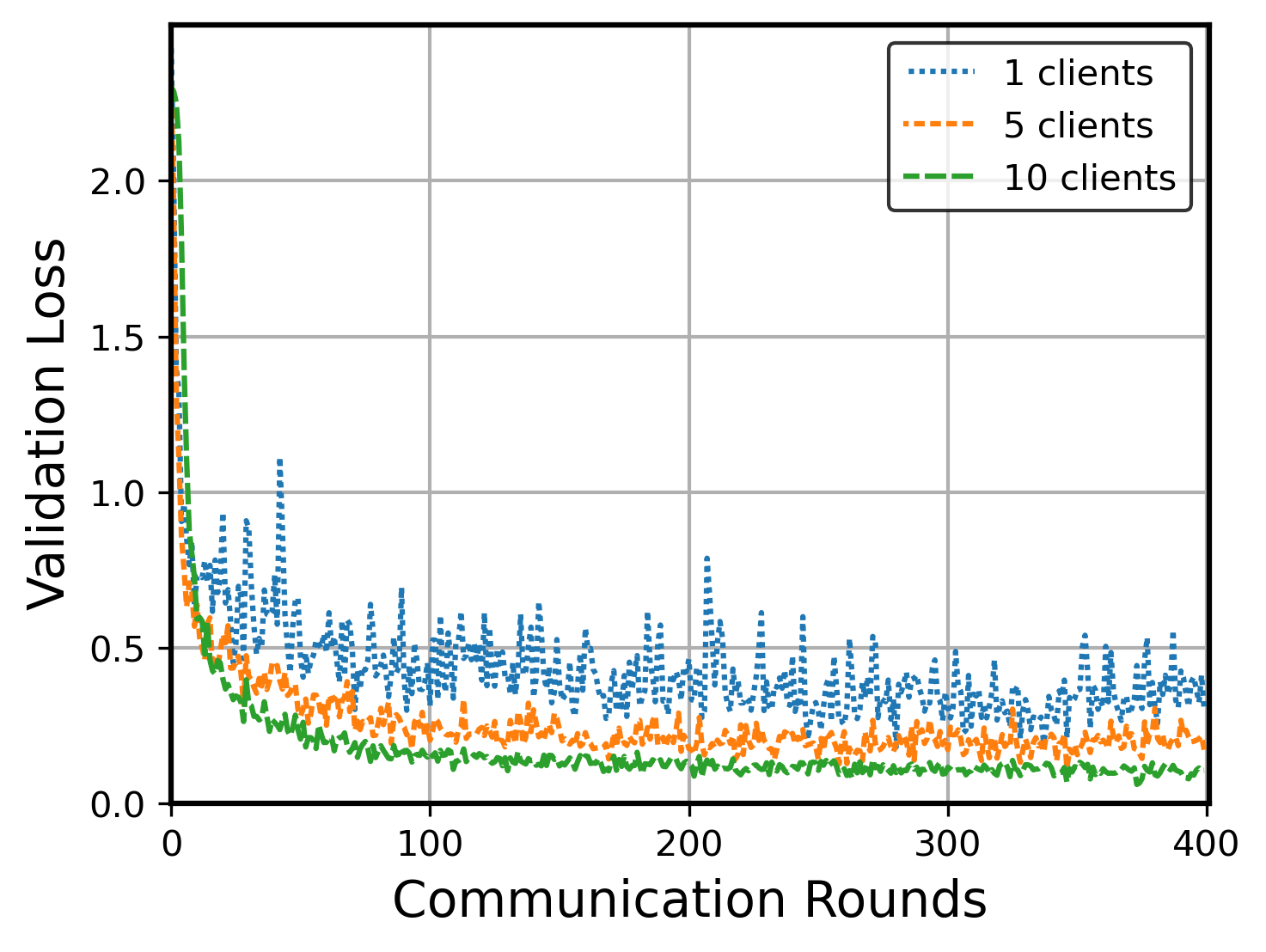}
    \includegraphics[width=0.24\columnwidth]{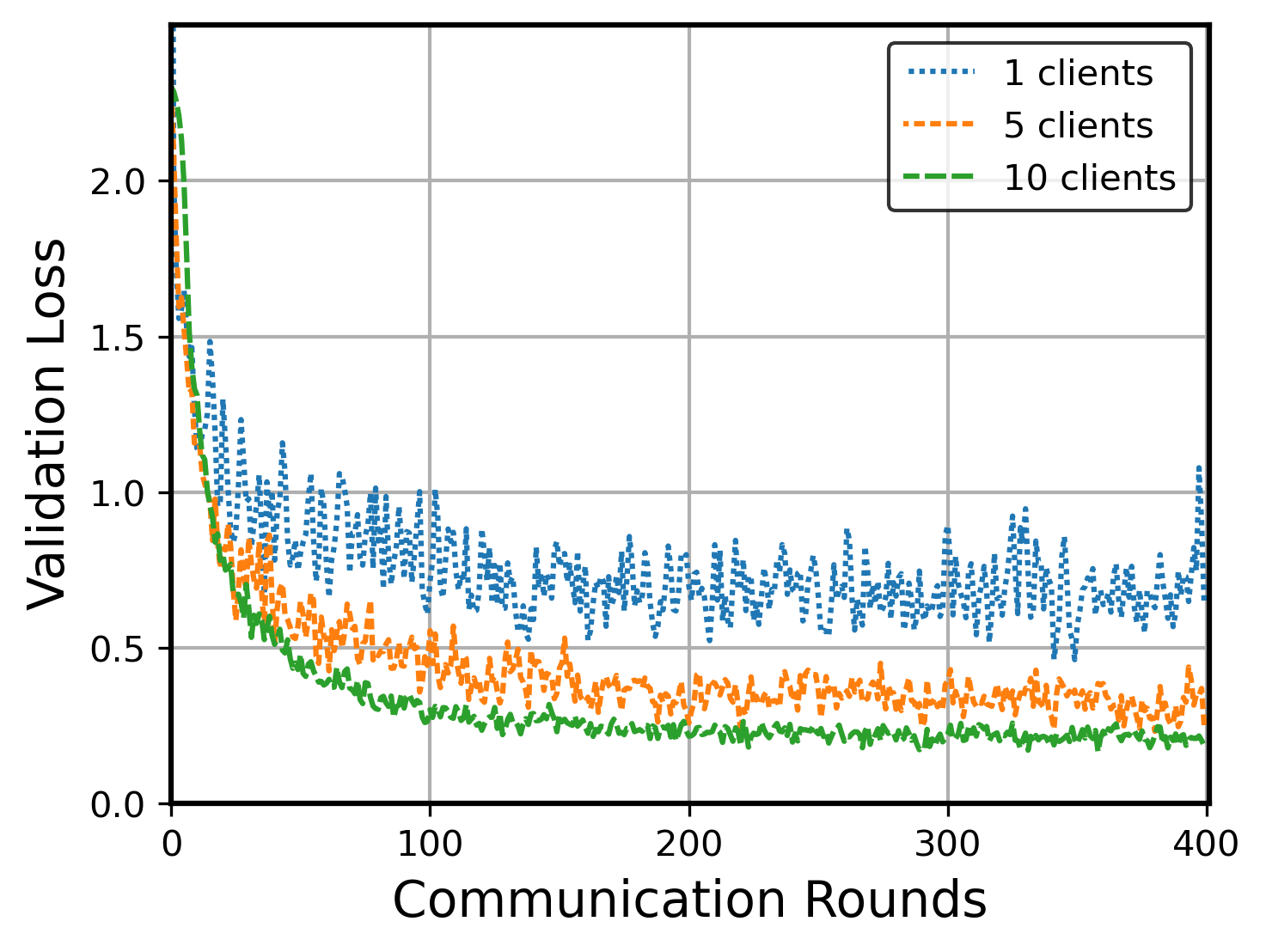}
    \includegraphics[width=0.24\columnwidth]{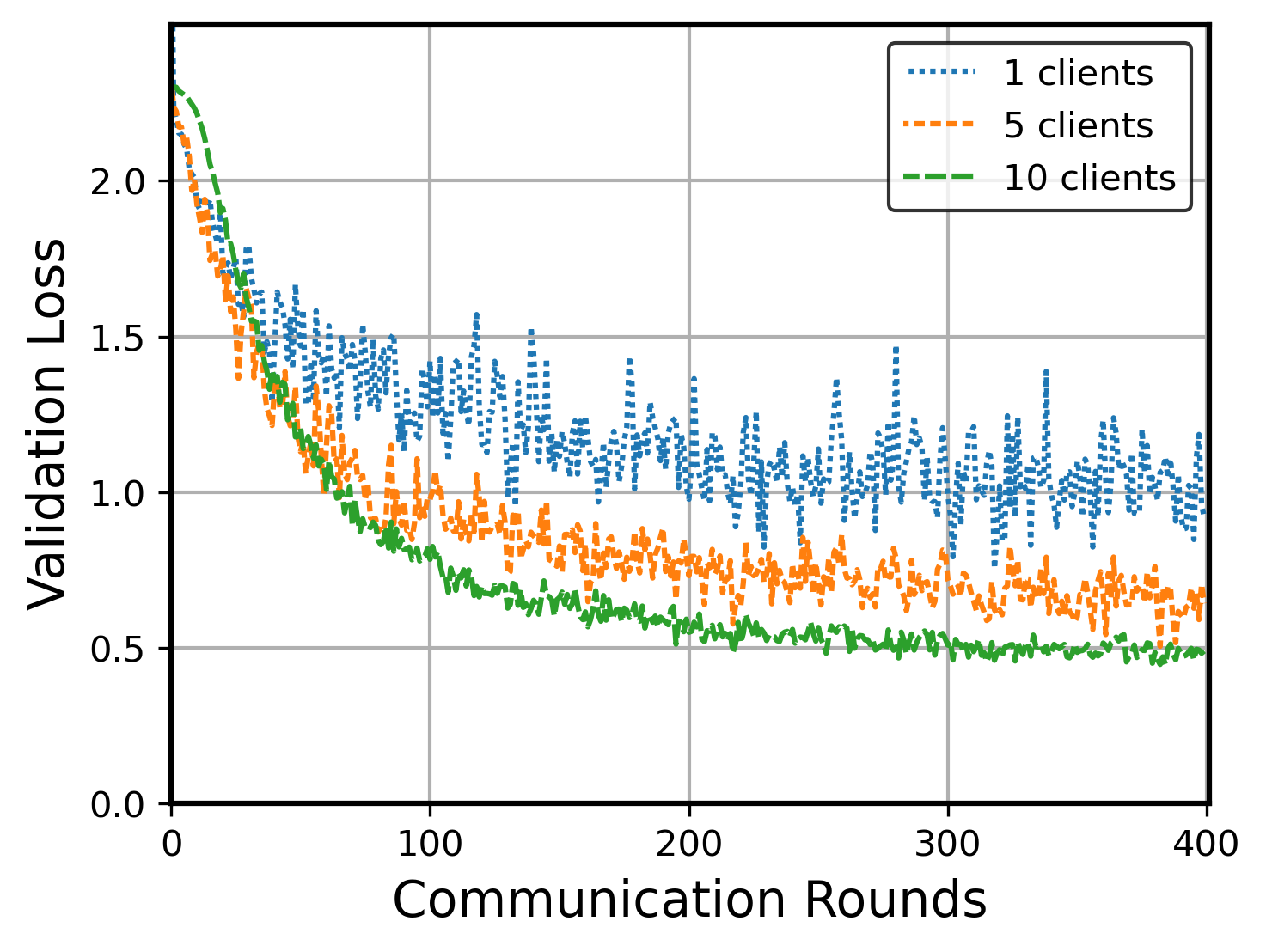}
    \includegraphics[width=0.24\columnwidth]{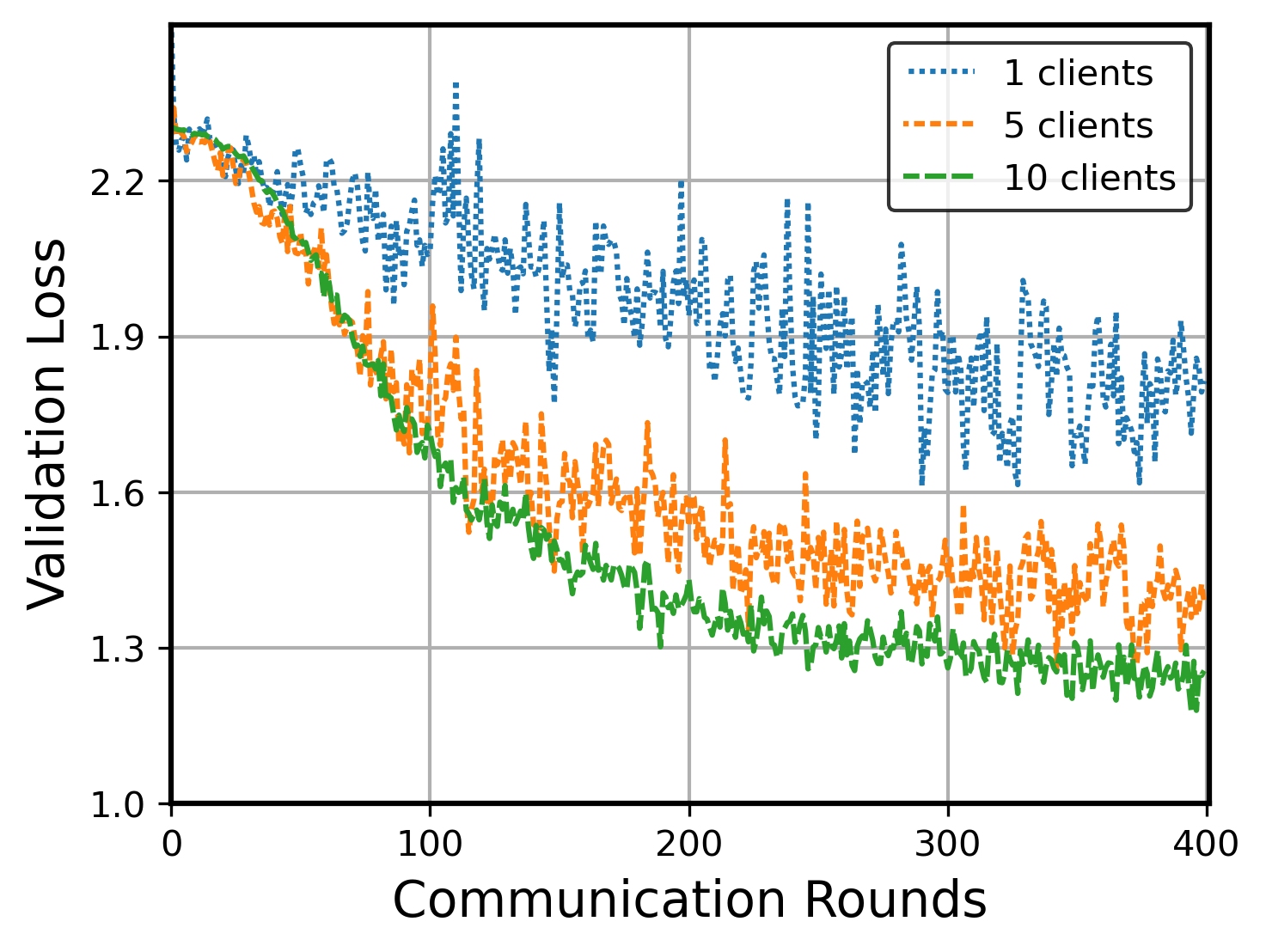}
\end{center}
\caption{Validation Error vs Communication Rounds with different number of clients per epoch. From Left to Right: $\rho=0.1, 0.4, 0.8, 0.95$. The local step $I$ is set as 5.}
\label{fig:clean-sample}
\end{figure}

\begin{figure}[t]
\begin{center}
    \includegraphics[width=0.24\columnwidth]{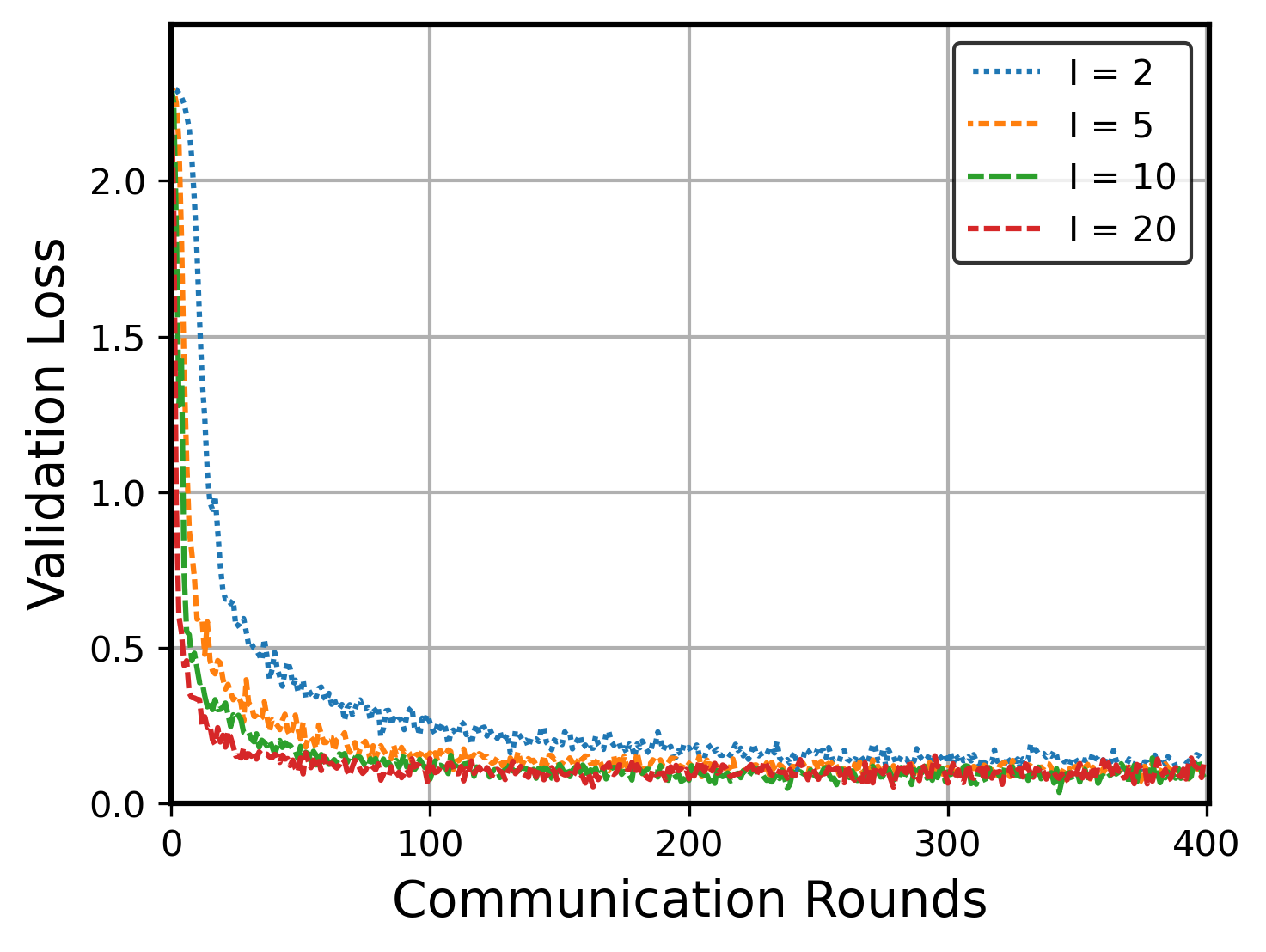}
    \includegraphics[width=0.24\columnwidth]{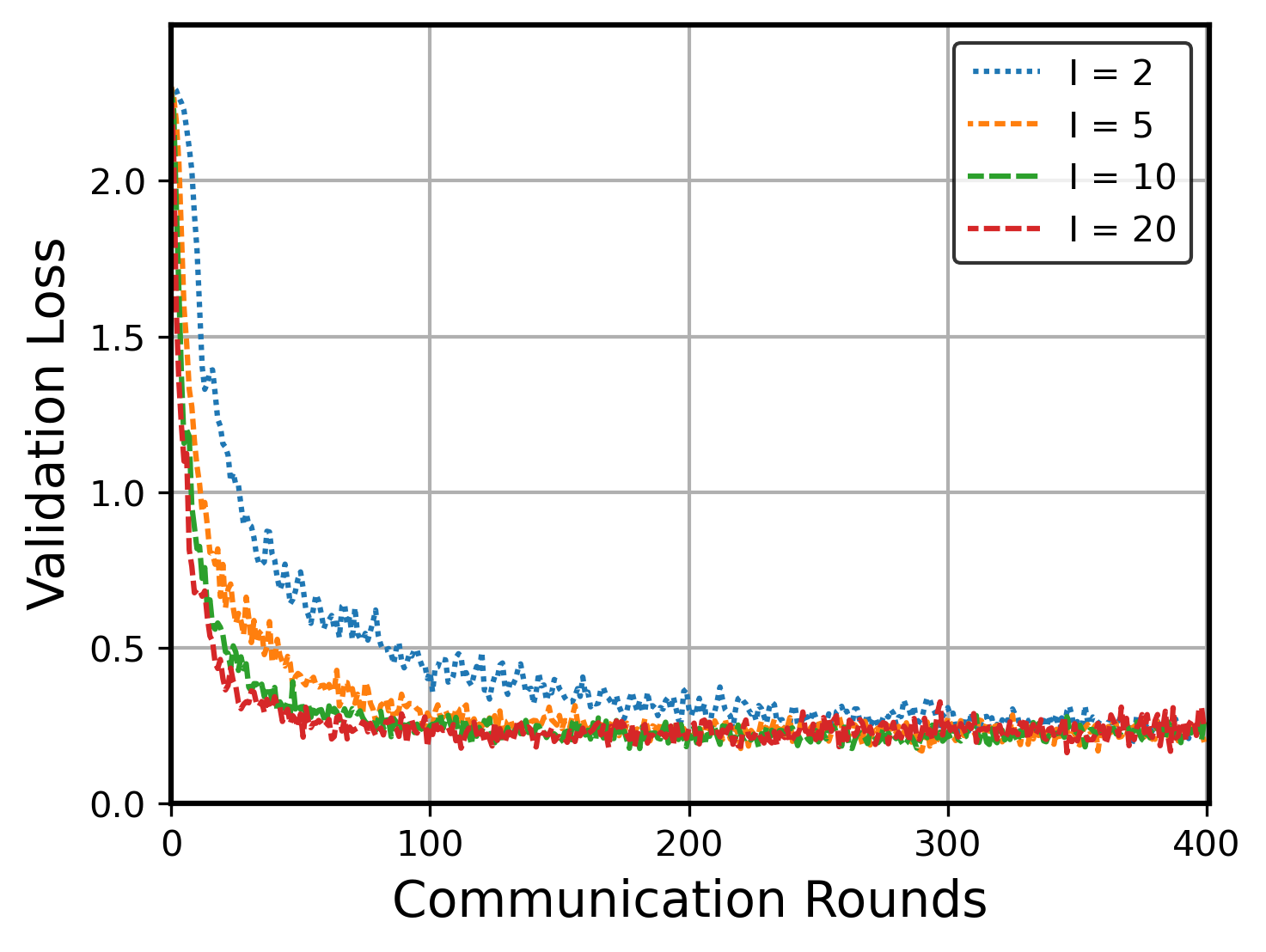}
    \includegraphics[width=0.24\columnwidth]{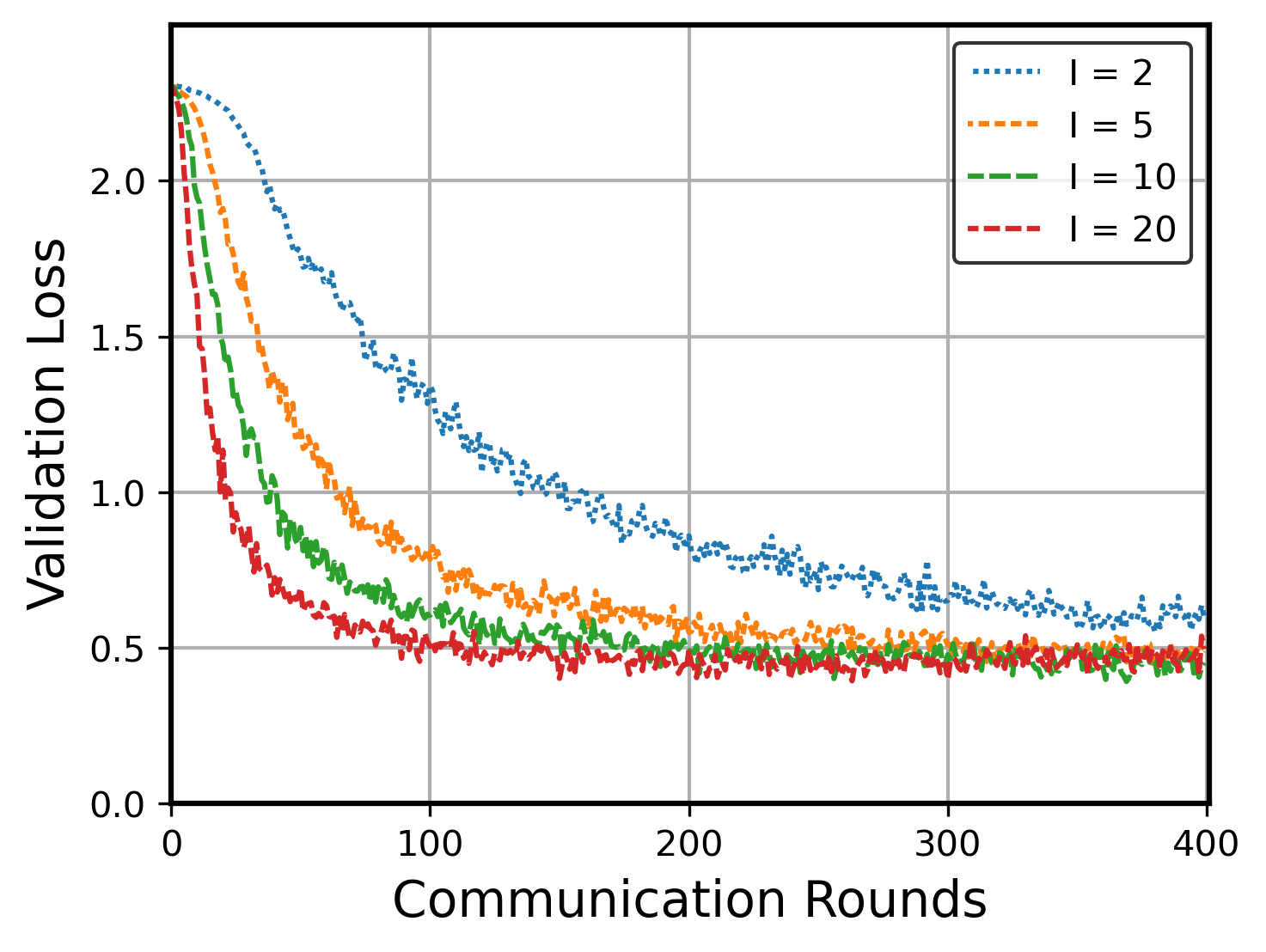}
    \includegraphics[width=0.24\columnwidth]{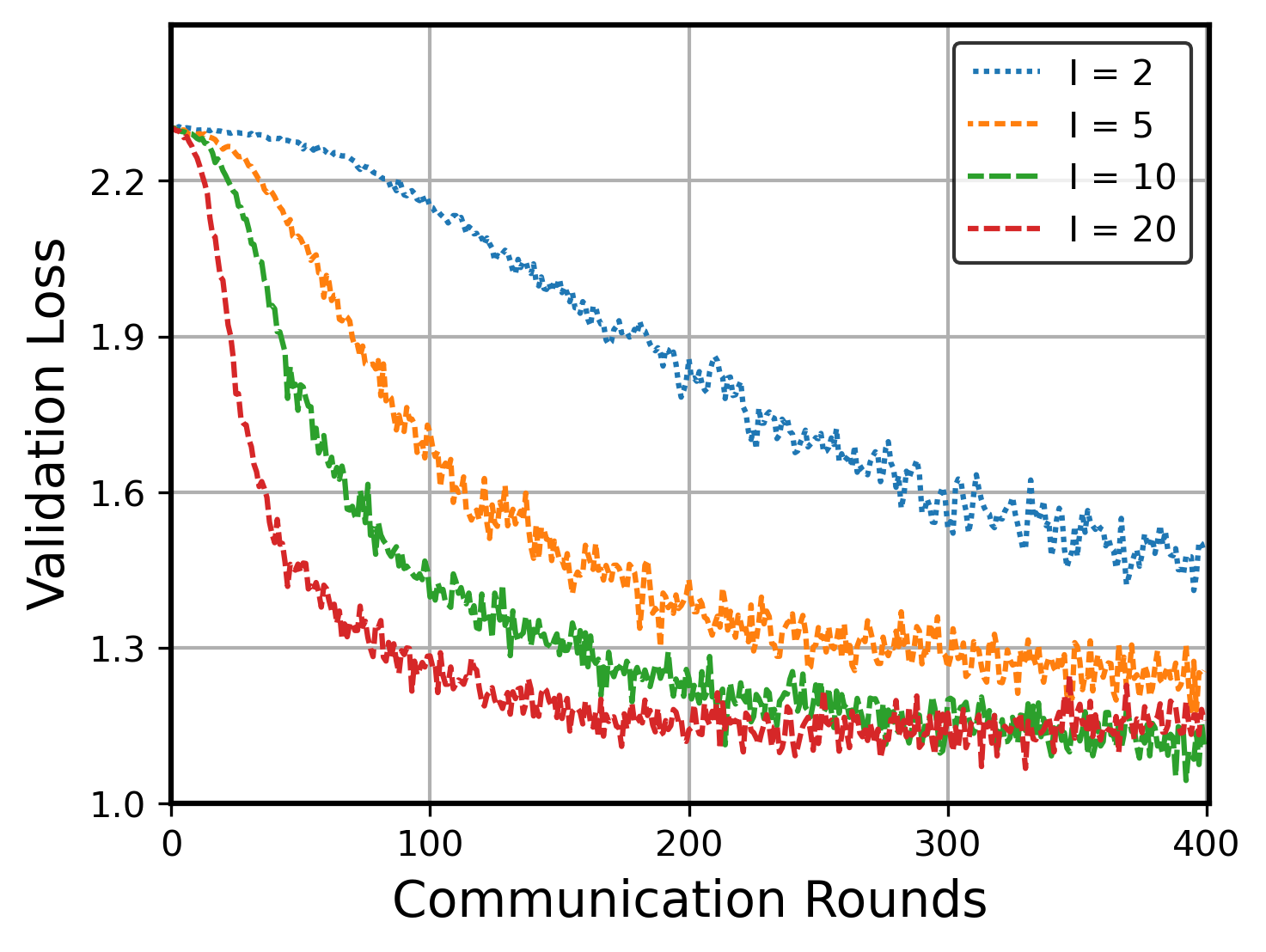}
\end{center}
\caption{Validation Error vs Communication Rounds with different number of local steps $I$. From Left to Right: $\rho=0.1, 0.4, 0.8, 0.95$.}
\label{fig:clean-interval}
\end{figure}

\textbf{Dataset and Baselines.} We create 10 clients and construct datasets based on MNIST~\cite{lecun1998gradient}. For the training set, each client randomly samples 4500 images (no overlap among clients) from 10 classes and then randomly uniformly perturb the labels of $\rho$ ($0\leq\rho\leq 1$) percent samples. For the validation set, each client randomly selects 50 clean images from a different class. In other words, the $m_{th}$ client only has validation samples from the $m_{th}$ class. This single-class validation setting introduces a high level of heterogeneity, such that individual clients are unable to conduct local cleaning due to they only have clean samples from one class. In our experiments, we test our FedBiOAcc algorithm, including the FedBiO algorithm (Algorithm 2 in Appendix) which does not use variance reduction; additionally, we also consider some baseline methods: a baseline that directly performs FedAvg~\cite{mcmahan2017communication} on the noisy dataset, this helps to verify the usefulness of data cleaning; Local-BSGVR~\cite{gao2022convergence}, FedNest~\cite{tarzanagh2022fednest}, CommFedBiO~\cite{li2022communication}, AggITD~\cite{xiao2023communication} and FedMBO~\cite{huang2023achieving}. Note that Local-BSGVR is designed for the homogeneous setting, and the last four baselines all need multiple rounds of client-server communication to evaluate the hyper-gradient at each global epoch.  We perform grid search to find the best hyper-parameters for each method and report the best results. Specific choices are included in Appendix B.1.

In figure~\ref{fig:clean}, we compare the performance of different methods at various noise levels $\rho$. Note that the larger the $\rho$ value, the more noisy the training data are. The noise level can be illustrated by the performance of the FedAvg algorithm, which learns over the noisy data directly. As shown in the figure, FedAvg learns almost nothing when $\rho = 0.95$. Next, our algorithms are robust under various heterogeneity levels. When the noise level in the training set increases as the value of $\rho$ increases, learning relies more on the signal from the heterogeneous validation set, and our algorithms consistently outperform other baselines. Finally, in figure~\ref{fig:clean-sample}, we vary the number of clients sampled per epoch, and the experimental results show that our FedBiOAcc converges faster with more clients in the training per epoch; in figure~\ref{fig:clean-interval}, we vary the number of local steps under different noisy levels. Interestingly, the algorithm benefit more from the local training under larger noise.

\subsection{Federated Hyper-Representation Learning}
In the Hyper-representation learning task, we learn a hyper-representation of the data such that a linear classifier can be learned quickly with a small number of data samples. A mathematical formulation of the task is included in Appendix B.2.
Note that this task is an instantiation of Eq.~\eqref{eq:fed-bi}, due to the fact that each client has its own tasks, and thus only the upper level problem is federated. We consider the Omniglot~\cite{lake2011one} and MiniImageNet~\cite{ravi2017optimization} data sets.  As in the non-distributed setting, we perform $N$-way-$K$-shot classification. 

\begin{figure*}[t]
\begin{center}
    \includegraphics[width=0.24\columnwidth]{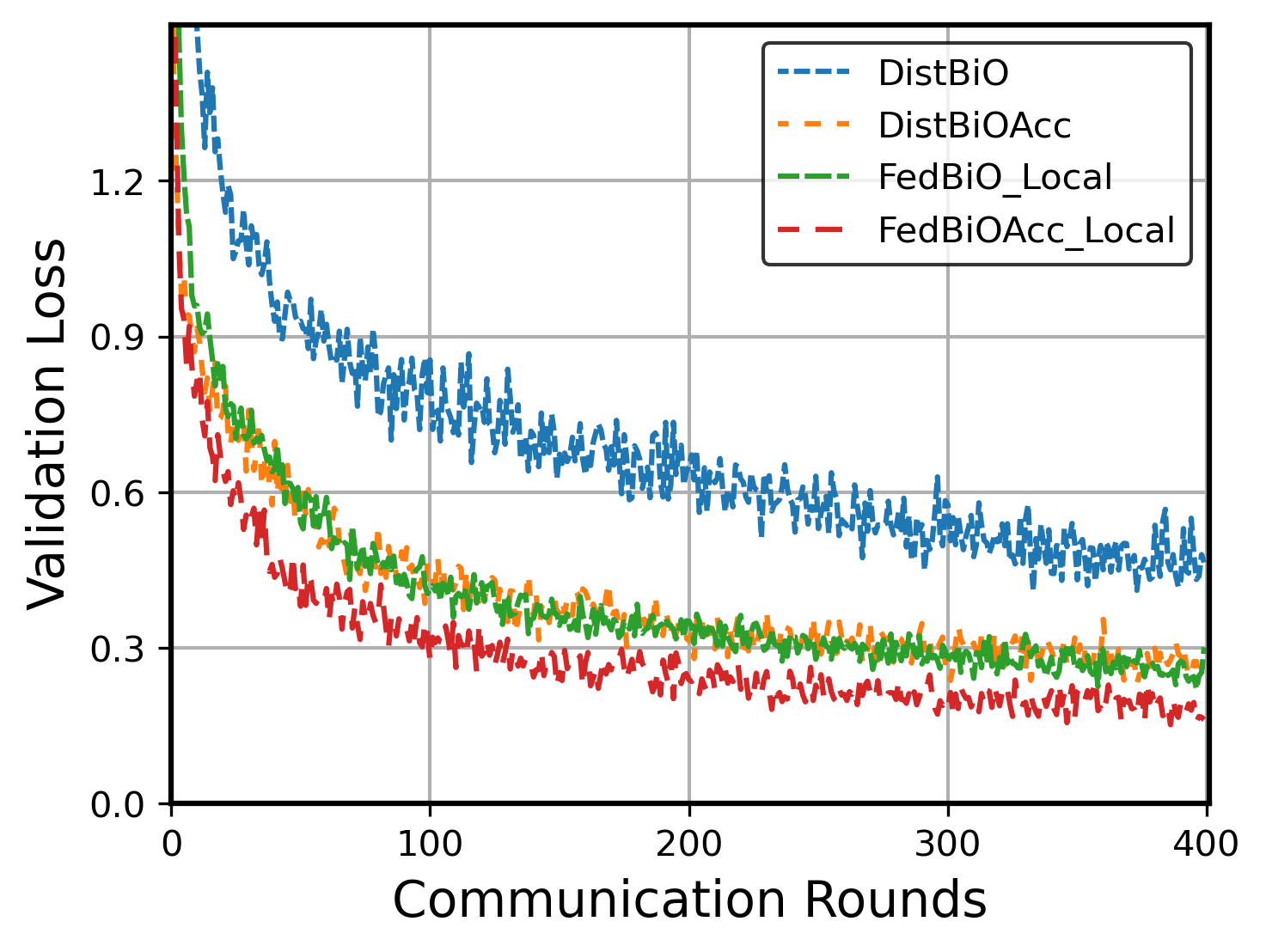}
    \includegraphics[width=0.24\columnwidth]{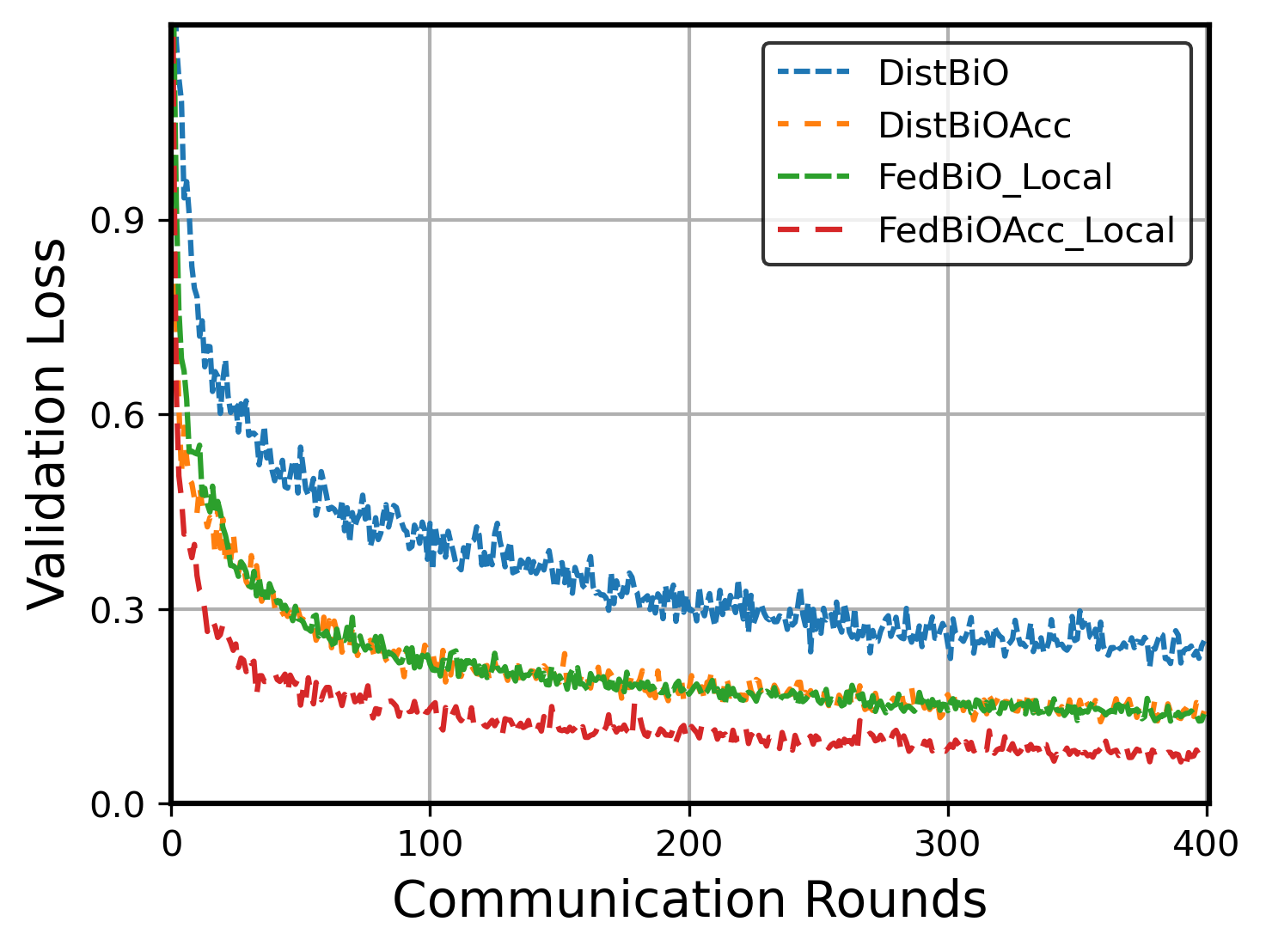}
    \includegraphics[width=0.24\columnwidth]{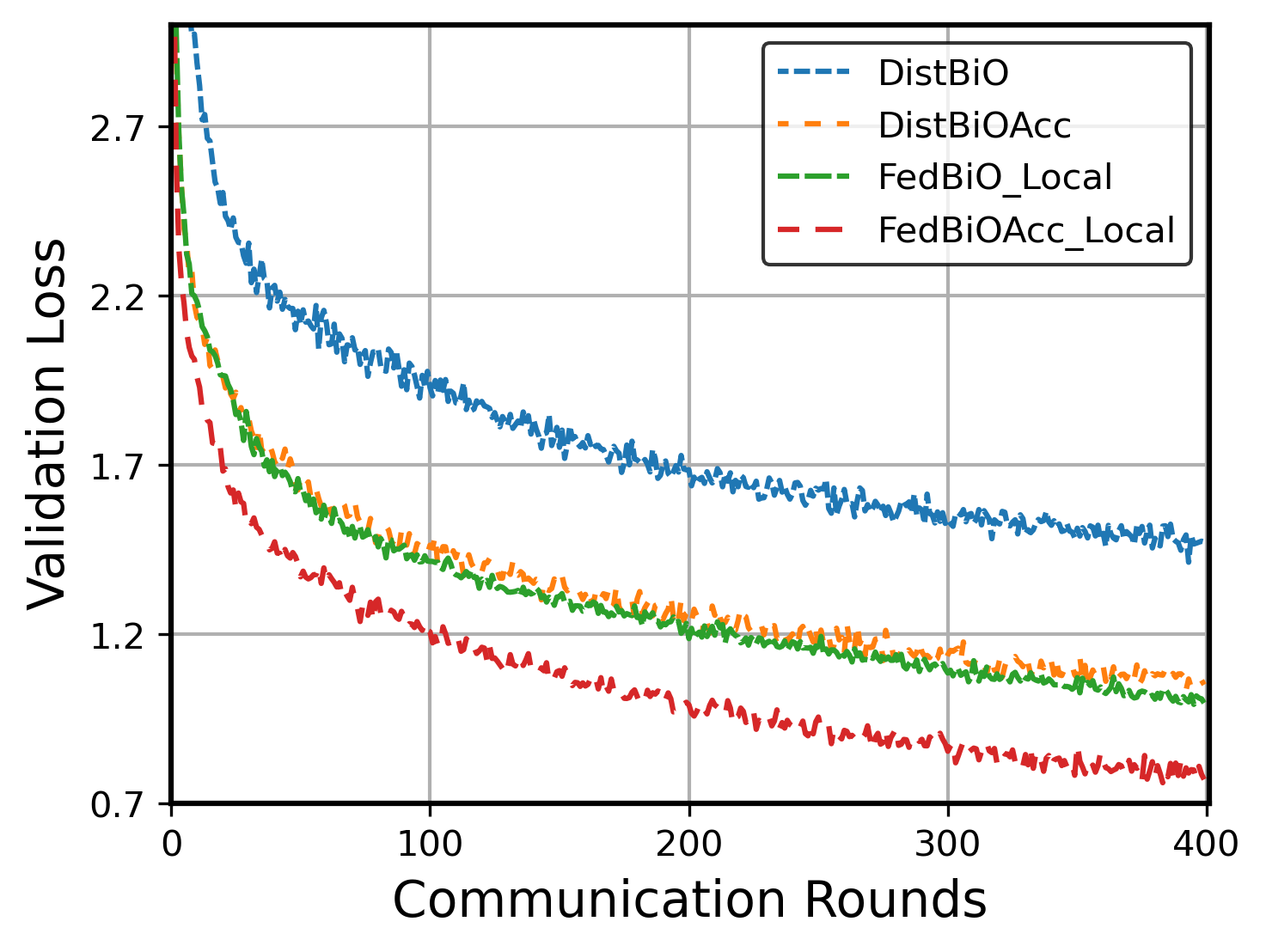}
    \includegraphics[width=0.24\columnwidth]{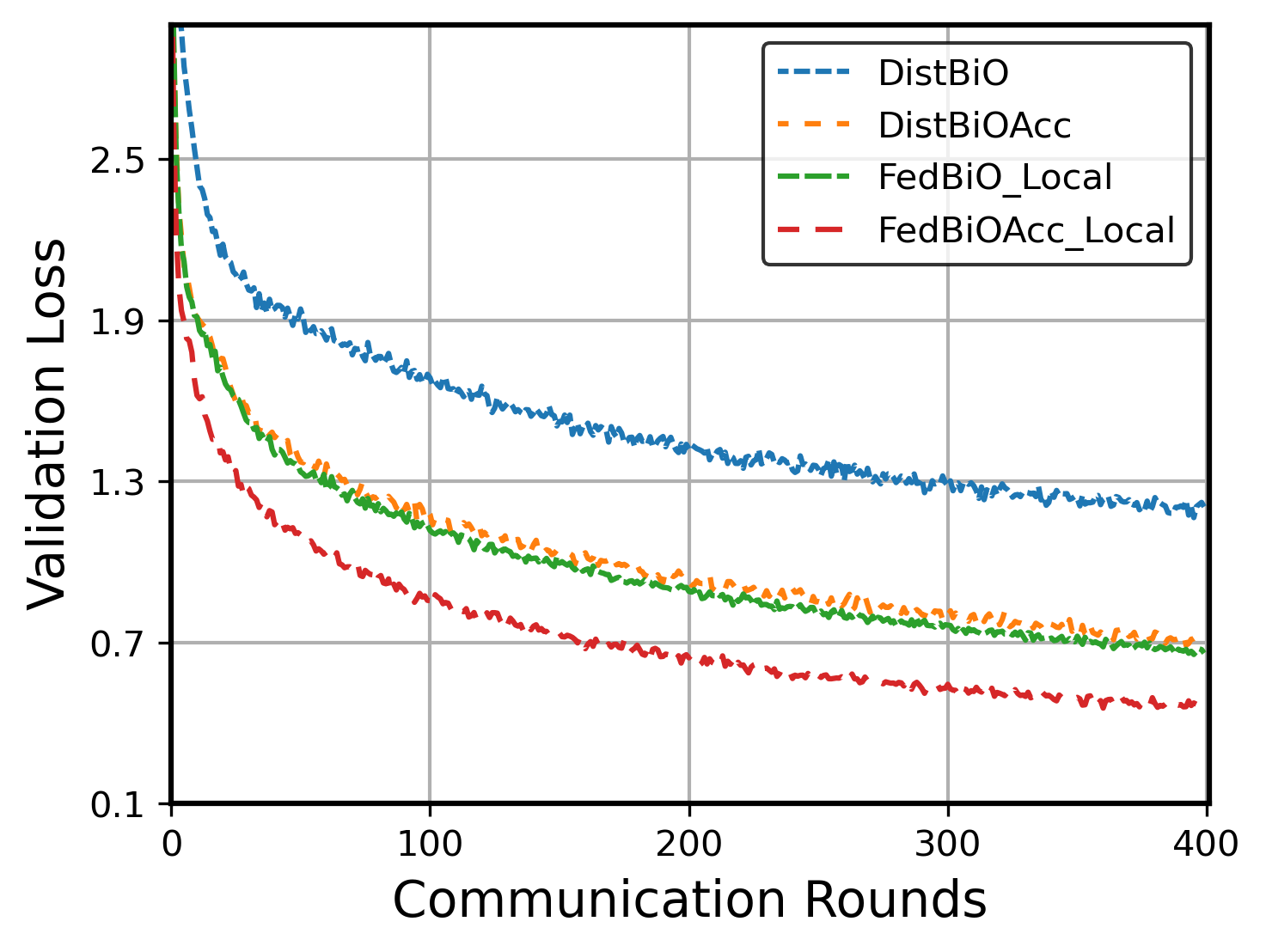}
    \includegraphics[width=0.24\columnwidth]{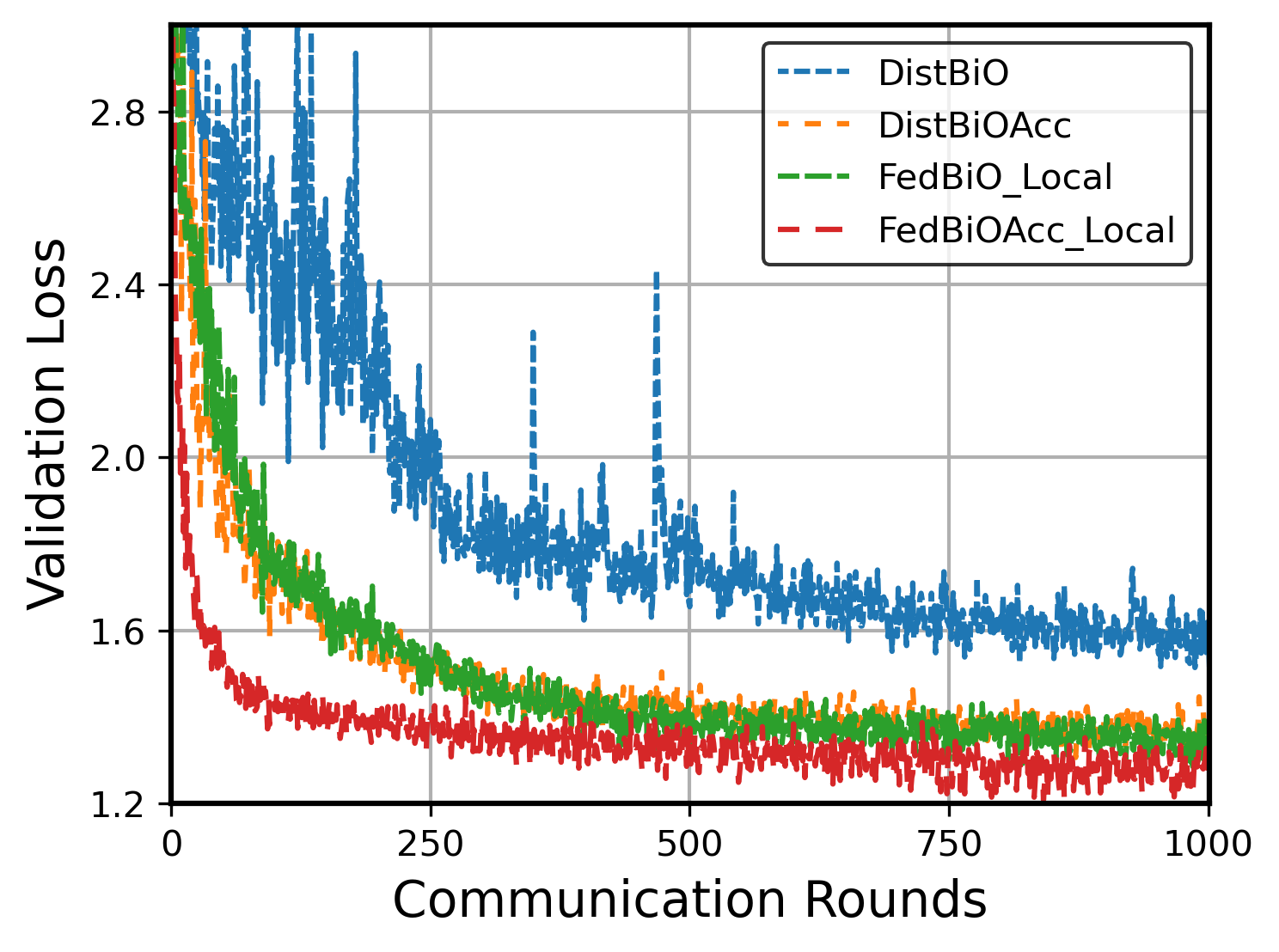}
    \includegraphics[width=0.24\columnwidth]{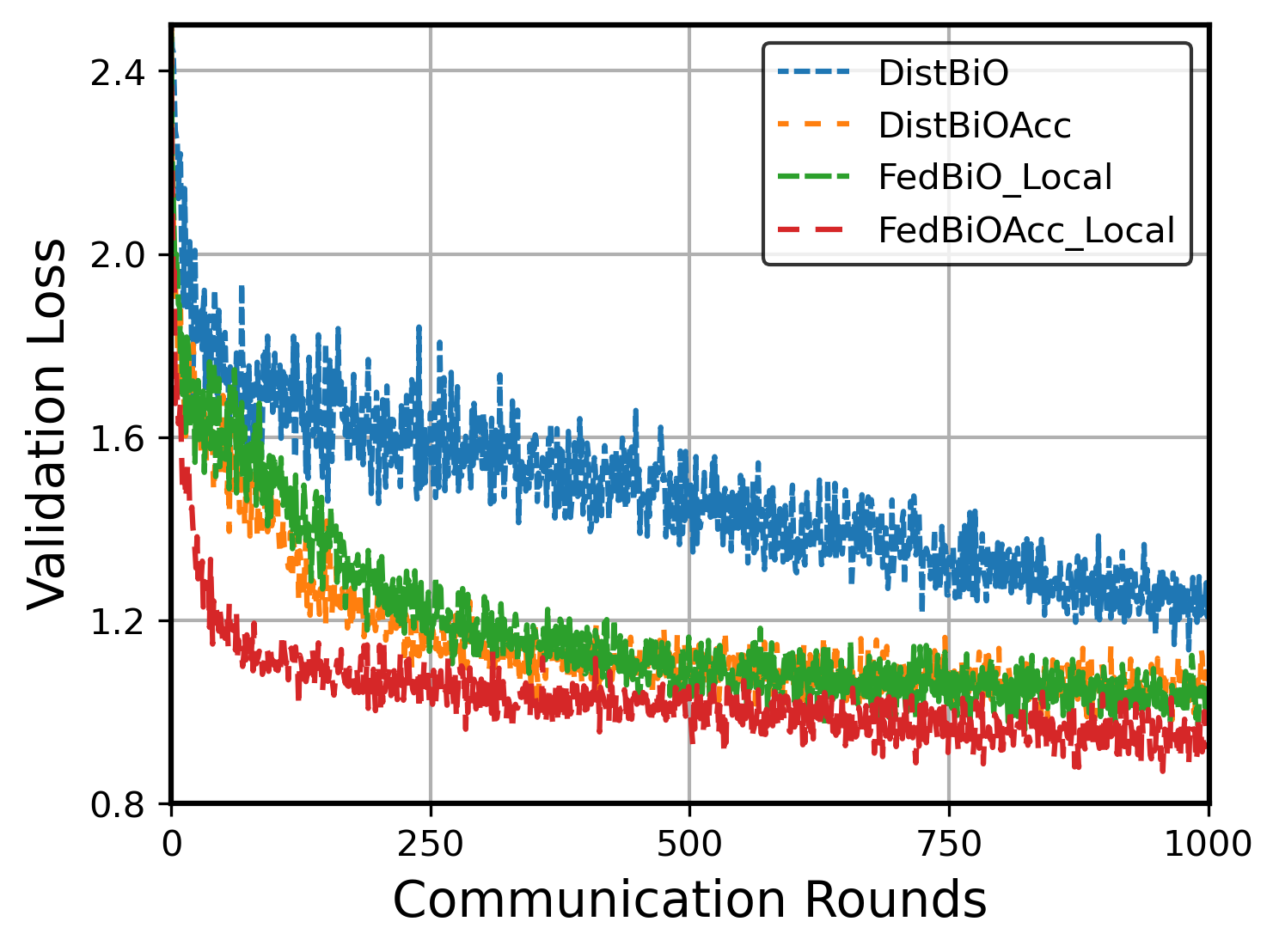}
    \includegraphics[width=0.24\columnwidth]{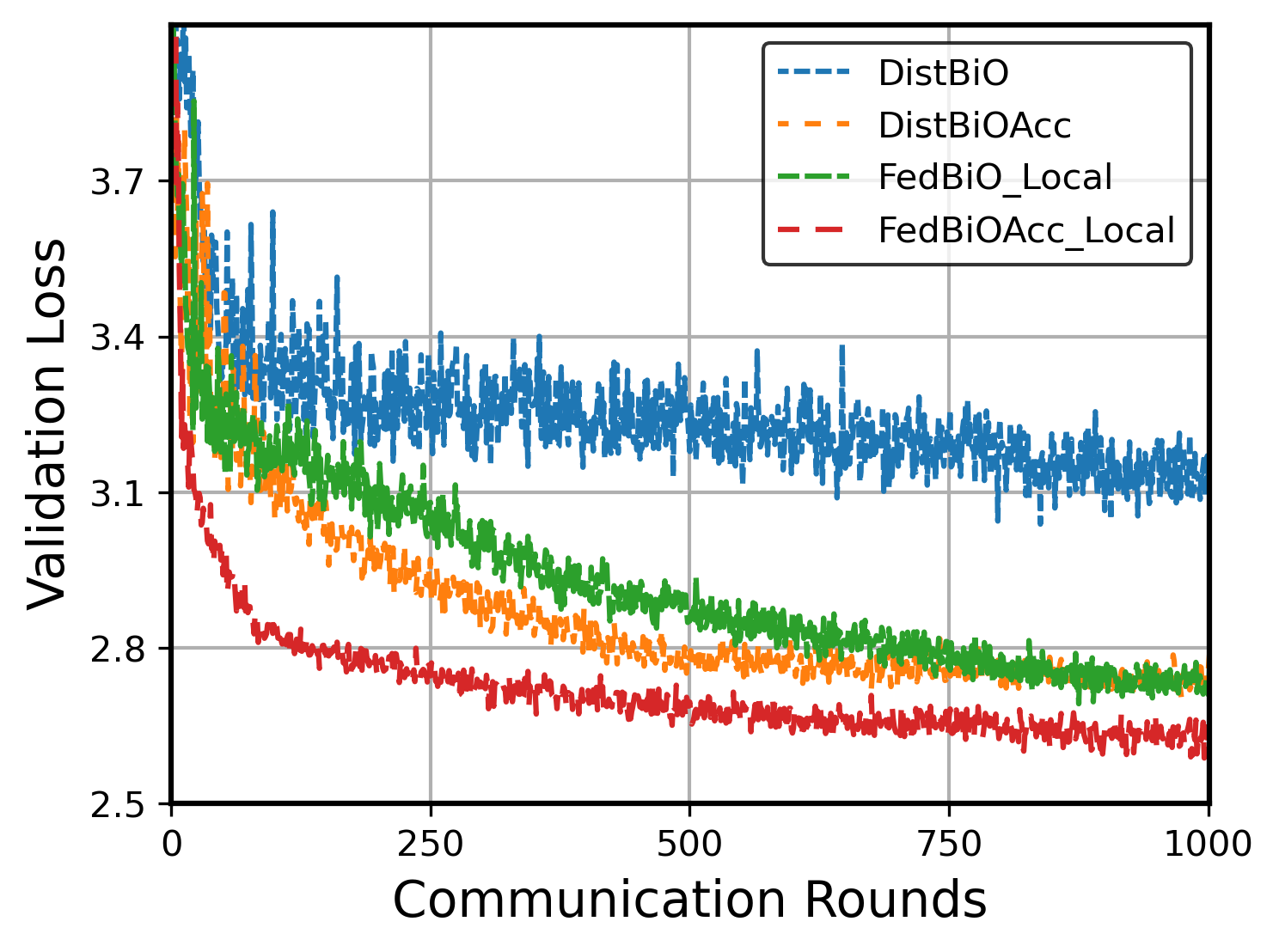}
    \includegraphics[width=0.24\columnwidth]{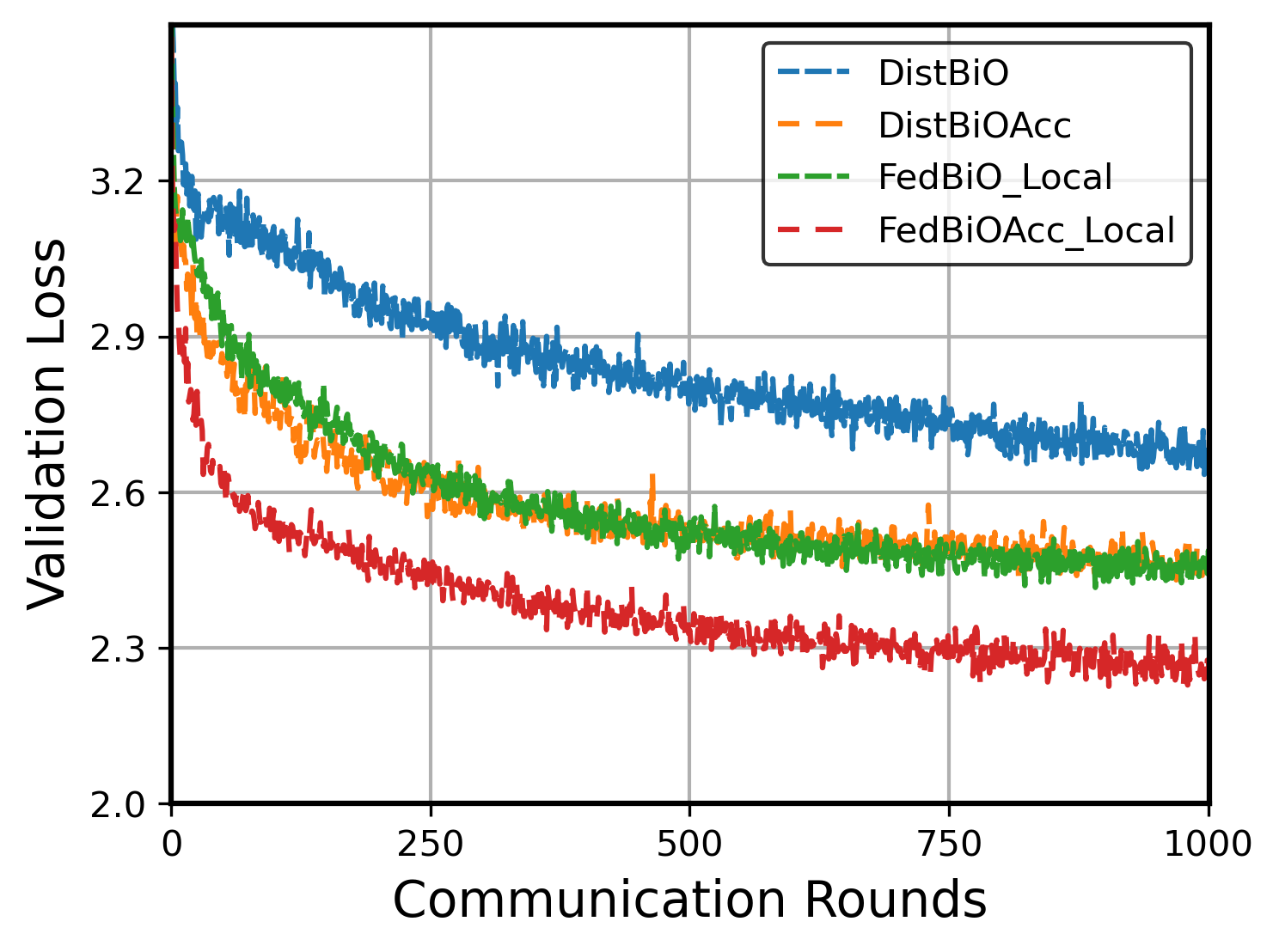}
\end{center}
\caption{Validation Error vs Communication Rounds. The top row shows the result for the Omniglot Dataset and the bottom row shows MiniImageNet. From Left to Right: 5-way-1-shot, 5-way-5-shot, 20-way-1-shot, 20-way-5-shot. The local step $I$ is set to 5.}
\label{fig:hyper-rep}
\end{figure*}

In this experiment, we compare FedBiOAcc-Local (Algorithm 4 in the Appendix) with three baselines FedBiO-Local (Algorithm 3 in the Appendix), DistBiO and DistBiOAcc. Note that DistBiO and DistBiOAcc are the distributed version of FedBiO-Local and FedBiOAcc-Local, respectively. In the experiments, we implement DistBiO and DistBiOAcc by setting the local steps as 1 for FedBiO-Local and FedBiOAcc-Local.  We perform grid search for the hyper-parameter selection for both methods and choose the best ones, the specific choices of hyper-parameters are deferred to 
Appendix~B.2.
The results are summarized in Figure~\ref{fig:hyper-rep} (full results are included in 
Figure~5 and Figure~6
of Appendix. As shown by the results, FedBiOAcc converges faster than the baselines on both datasets and on all four types of classification tasks, which demonstrates the effectiveness of variance reduction and multiple steps of local training.

\section{Conclusion}
In this paper, we study the Federated Bilevel Optimization problems and introduce FedBiOAcc. In particular, FedBiOAcc evaluates the hyper-gradient by solving a federated quadratic problem, and mitigates the noise through momentum-based variance reduction technique.  We provide a rigorous convergence analysis for our proposed method and show that FedBiOAcc has the optimal iteration complexity $O(\epsilon^{-1.5})$ and communication complexity $O(\epsilon^{-1})$, and it also achieves linear speed-up \emph{w.r.t} the number of clients. Besides, we study a type of novel Federated Bilevel Optimization problems with local lower level problems. We modify FedBiO for this type of problems and propose FedBiOAcc-Local. FedBiOAcc-Local achieves the same optimal convergence rate as FedBiOAcc. Finally, we validate our algorithms with real-world tasks.


\bibliography{neurips_2023}
\bibliographystyle{abbrv}

\newpage
\appendix
\onecolumn


\section{Assumptions}
In this section, we restate all assumptions needed in our proof below:
\begin{assumption} [Assumption 1]
The function $f^{(m)}(x ,y)$ is possibly non-convex and $g^{(m)}(x,y)$ is $\mu$-strongly convex \emph{w.r.t} $y$ for any given $x$, \emph{i.e.} for any $y_1$, $y_2 \in \mathbb{R}^d$, we have: \[g^{(m)}(x, y_1) \geq g^{(m)}(x,y_2) + \langle \nabla_y g^{(m)} (x, y_2), y_2 - y_1\rangle + \frac{\mu}{2}||y_2 - y_1||^2.\] 
\end{assumption}

\begin{assumption} [Assumption 2] Function $f^{(m)}(x,y)$ is $L$-Lipschitz, \emph{i.e.} for for any $x_1$, $x_2 \in \mathcal{X}$ and for any $y_1$, $y_2 \in \mathbb{R}^d$, and we denote $z_1 = (x_1, y_1)$, $z_2 = (x_2, y_2)$, then we have: \[f^{(m)}(z_1) \leq f^{(m)}(z_2) + \langle \nabla f^{(m)} (z_2),  z_1 - z_2\rangle + \frac{L}{2}||z_1 - z_2||^2.\] or equivalently: $||\nabla f^{(m)}(z_1) - \nabla f^{(m)}(z_2)|| \leq L||z_1 - z_2||$. We also assume and $f^{(m)}(x,y)$ has $C_f$-bounded gradient, \emph{i.e.} for for any $x \in \mathcal{X}$ and any $y \in \mathbb{R}^d$, and we denote $z = (x, y)$, then we have $||\nabla f(z)|| \leq C_f$.
\end{assumption}

\begin{assumption} [Assumption 3] Function $g^{(m)}(x,y)$ is $L$-Lipschitz. \emph{i.e.} for for any $x_1$, $x_2 \in \mathcal{X}$ and for any $y_1$, $y_2 \in \mathbb{R}^d$, and we denote $z_1 = (x_1, y_1)$, $z_2 = (x_2, y_2)$, then we have: \[g^{(m)}(z_1) \leq g^{(m)}(z_2) + \langle \nabla g^{(m)} (z_2),  z_1 - z_2\rangle + \frac{L}{2}||z_1 - z_2||^2.\] equivalently: $||\nabla g^{(m)}(z_1) - \nabla g^{(m)}(z_2)|| \leq L||z_1 - z_2||$. For higher-order derivatives, we have:
\begin{itemize}
	\item[a)] $\nabla_{xy} g^{(m)}(x,y)$ and $\nabla_{y^2} g^{(m)}(x,y)$ are Lipschitz continuous with constant $L_{xy}$ and $L_{y^2}$ respectively, \emph{i.e.} for for any $x_1$, $x_2 \in \mathcal{X}$ and for any $y_1$, $y_2 \in \mathbb{R}^d$, and we denote $z_1 = (x_1, y_1)$, $z_2 = (x_2, y_2)$, then we have: $||\nabla_{xy} g^{(m)}(z_1) - \nabla_{xy} g^{(m)}(z_2)|| \leq L_{xy}||z_1 - z_2||$ and $||\nabla_{y^2} g^{(m)}(z_1) - \nabla_{y^2} g^{(m)}(z_2)|| \leq L_{y^2}||z_1 - z_2||$.
\end{itemize}
\end{assumption}

\begin{assumption} [Assumption 4]
We have an unbiased stochastic first order and second order derivative oracle with bounded variance, more specifically, denote $z = (x,y)$, we have:
\begin{itemize}
\item[a)] we have $\nabla f^{(m)}(z; \xi)$, such that: $E[\nabla f^{(m)}(z; \xi)] = \nabla f^{(m)}(z)$ and $var(\nabla f^{(m)}(z; \xi)) \leq \sigma^2$.
\item[b)] we have $\nabla g^{(m)}(z; \xi)$, such that: $E[\nabla g^{(m)}(z; \xi)] = \nabla g^{(m)}(z)$ and $var(\nabla g^{(m)}(z; \xi)) \leq \sigma^2$.
\item[c)] we have $\nabla_{y^2} g^{(m)}(z, \xi)$, such that: $E[\nabla_{y^2} g^{(m)}(z; \xi)] = \nabla_{y^2} g^{(m)}(z)$ and $var(\nabla_{y^2} g^{(m)}(z; \xi)) \leq \sigma^2$;
\item[d)] we have $\nabla_{xy} g^{(m)}(z; \xi)$, such that: $E[\nabla_{xy} g^{(m)}(z; \xi)] = \nabla_{xy} g^{(m)}(z)$ and $var(\nabla_{xy} g^{(m)}(x,y; \xi)) \le \sigma^2$; 
\end{itemize}
\end{assumption}

\begin{assumption} [Assumption 5]
For any $m, j \in [M]$ and $z = (x,y)$, we have: $ \| \nabla f^{(m)} (z) -  \nabla f^{(j)} (z) \| \leq \zeta_f$, $ \| \nabla g^{(m)} (z) -  \nabla g^{(j)}(z) \| \leq \zeta_g$, $ \| \nabla_{xy} g^{(m)} (z) -  \nabla_{xy} g^{(j)} (z) \| \leq \zeta_{g,xy}$, $ \| \nabla_{y^2} g^{(m)} (z) -  \nabla_{y^2} g^{(j)} (z) \| \leq \zeta_{g,yy}$, where $\zeta_f$, $\zeta_g$, $\zeta_{g,xy}$, $\zeta_{g,yy}$, are constants.
\end{assumption}

\begin{assumption} [Assumption 6]
For any $m, j \in [M]$ and $z = (x,y)$, we have: $ \| \nabla f^{(m)} (z) -  \nabla f^{(j)} (z) \| \leq \zeta_f$, $ \| \nabla_{xy} g^{(m)} (z) -  \nabla_{xy} g^{(j)} (z) \| \leq \zeta_{g,xy}$, $ \| \nabla_{y^2} g^{(m)} (z) -  \nabla_{y^2} g^{(j)} (z) \| \leq \zeta_{g,yy}$, $ \| y^{(m)}_x -  y^{(j)}_x\| \leq \zeta_{g^{\ast}}$, where $\zeta_f$, $\zeta_{g,xy}$, $\zeta_{g,yy}$, $\zeta_{g^{\ast}}$ are constants.
\end{assumption}

\newpage

\section{More Experimental Details and Results}
In this section, we introduce more details of the experiments.

\subsection{Federated Data Cleaning}\label{sec:exp-clean}
The formulation of the problem is as follows:
\begin{align*}
    \underset{x \in \mathbb{R}^p }{\min}\ h(x) &\coloneqq \frac{1}{M}\sum_{m=1}^{M} f^{(m)}(x , y_{x}^{(m)}) = \frac{1}{M}\sum_{m=1}^M \big(\frac{1}{N^{(val)}_m}\sum_{n=1}^{N^{(val)}_m} \Theta(y_x; \xi_{m,n}^{val})\big) \nonumber\\
    &\mbox{s.t.}\ y_{x} = \underset{y\in \mathbb{R}^{d}}{\arg\min}\; g(x,y) = \frac{1}{M}\sum_{m=1}^M\sum_{n=1}^{N^{(tr)}_m} x_{m,n} \Theta(y; \xi_{m,n}^{tr})
\end{align*}
In the above formulation, we have $M$ clients, each client $m \in [M]$ has a pair of (noisy) training set $\{\xi_{m,n}^{tr}\}_{n=1}^{N^{(tr)}_m}$ and validation set $\{\xi_{m,n}^{val}\}_{n=1}^{N^{(val)}_m}$, and $x_{m,n}, n \in [N^{(tr)}_m]$ are weights for training samples, $y$ is the parameter of a model, and we denote the model by $\Theta$. Note that $y_x$ is the model learned over the weighted training set. We fit a model with 3 fully connected layers for the MNIST dataset. We also use $L_2$ regularization with coefficient $10^{-3}$ to satisfy the strong convexity condition.

In the Experiments, for FedNest and CommFedBiO, we choose learning rate 1 and hyper-learning rate 10000, for FedBiO, we choose learning rate 0.5, hyper learning rate 1000, for FedBiOAcc, we choose $\delta$ as 30, $u$ as 10000, $c_{\eta}$ as 0.2, $C_{\gamma}$ as 0.2, $\tau$ as 0.01, $\eta$ as 200 and $\gamma$ as 1.

\subsection{Federated Hyper-Representation Learning}\label{sec:exp-hyper}
\begin{align*}
    \underset{x \in \mathbb{R}^p }{\min}\ h(x) &\coloneqq \frac{1}{M}\sum_{m=1}^{M} f^{(m)}(x , y_{x}^{(m)}) = \frac{1}{M}\sum_{m=1}^M\big(\frac{1}{N_m}\sum_{n=1}^{N_m} \big(\frac{1}{N_{m,n}^{val}}\sum_{i=1}^{N_{m,n}^{val}} \Theta(x, y_x^{(\mathcal{T}_{m,n})}; \xi_{i}^{val})\big)\big) \nonumber\\
    &\mbox{s.t.}\ y_{x}^{(\mathcal{T}_{m,n})} = \underset{y\in \mathbb{R}^{d}}{\arg\min}\; g^{(\mathcal{T}_{m,n})}(x,y) = \frac{1}{N_{m,n}^{tr}}\sum_{i=1}^{N_{m,n}^{tr}} \Theta(x, y; \xi_{i}^{tr})
\end{align*}
In the above formulation, we have $M$ clients, each client $m \in [M]$ has $N_m$ tasks and each task $\mathcal{T}_{m,n}$ is defined by a pair of training set $\{\xi_{i}^{tr}\}_{i=1}^{N_{m,n}^{tr}}$ and validation set $\{\xi_{i}^{val}\}_{i=1}^{N_{m,n}^{val}}$. $\Theta$ defines the model, $x$ is the parameter of the backbone model and $y$ is the parameter of the linear classifier. In summary, the lower level problem is to learn the optimal linear classifier $y$ given the backbone $x$, and the upper level problem is to learn the optimal backbone parameter $x$.

The Omniglot dataset includes 1623 characters from 50 different alphabets and each character consists of 20 samples. We create the Federated version of the Omniglot dataset. Firstly, we follow the experimental protocols of~\cite{vinyals2016matching} to  divide the alphabets to train/validation/test with 33/5/12, respectively. Then we distribute three alphabets to a client, in other words, we consider 11 clients in experiments. As in the non-distributed setting, we perform $N$-way-$K$-shot classification, more specifically, for each task, we randomly sample $N$ characters from the alphabet over that client and for each character, we sample $K$ samples for training and 15 samples for validation.  We augment the characters by performing rotation operations (multipliers of 90 degrees). We use a 4-layer convolutional neural network where each convolutional layer has 64 filters of 3$\times$3~\cite{finn2017model}. For the MiniImageNet, it has 64 training classes and 16 validation classes. We distribute the training classes into four clients, similar to Omniglot, we also perform the $N$-way-$K$-shot classification. We use a 4-layer convolutional neural network where each convolutional layer has 64 filters of 3$\times$3~\cite{finn2017model} for experiments.

\begin{figure}[ht]
\begin{center}
    \includegraphics[width=0.24\columnwidth]{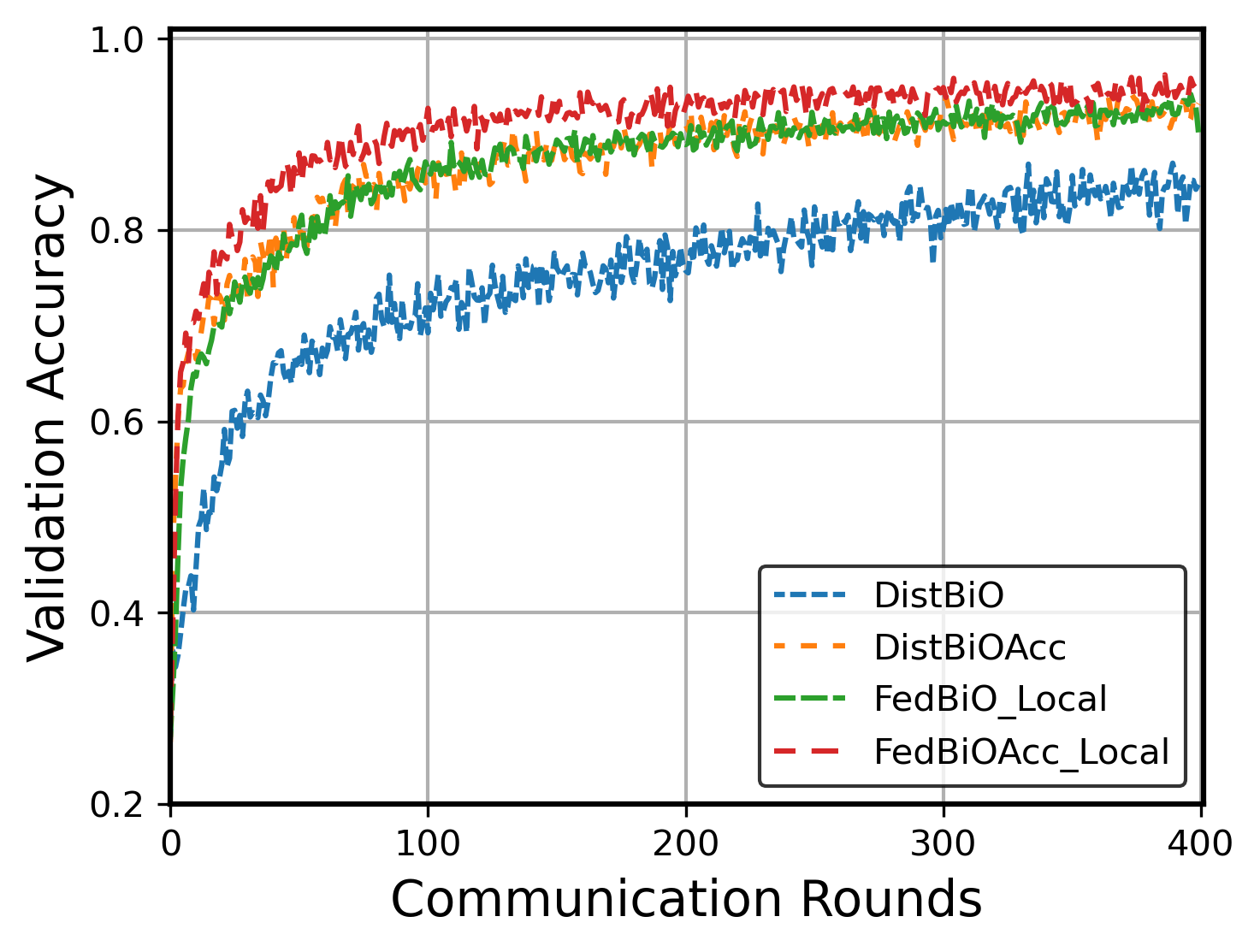}
    \includegraphics[width=0.24\columnwidth]{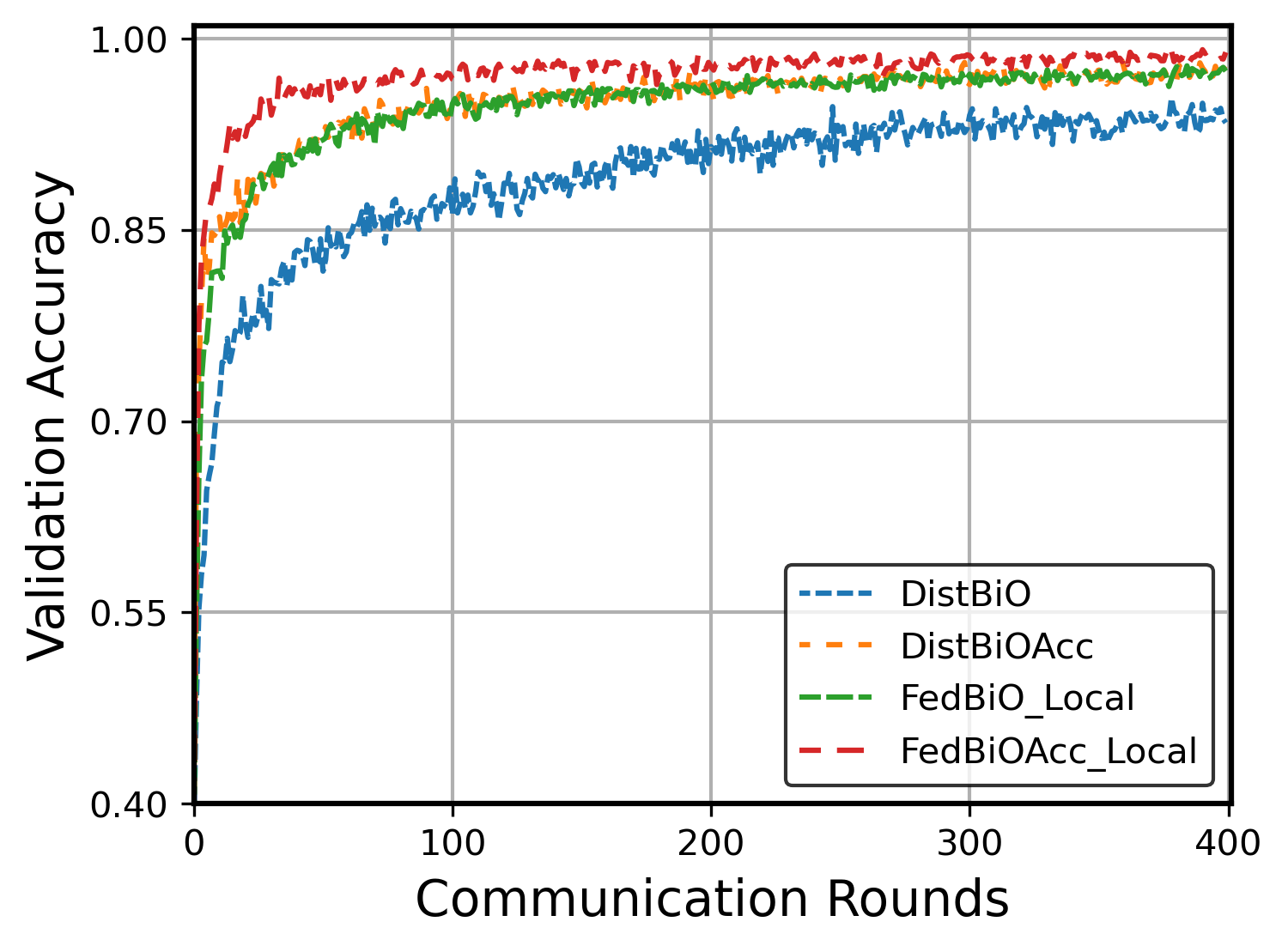}
    \includegraphics[width=0.24\columnwidth]{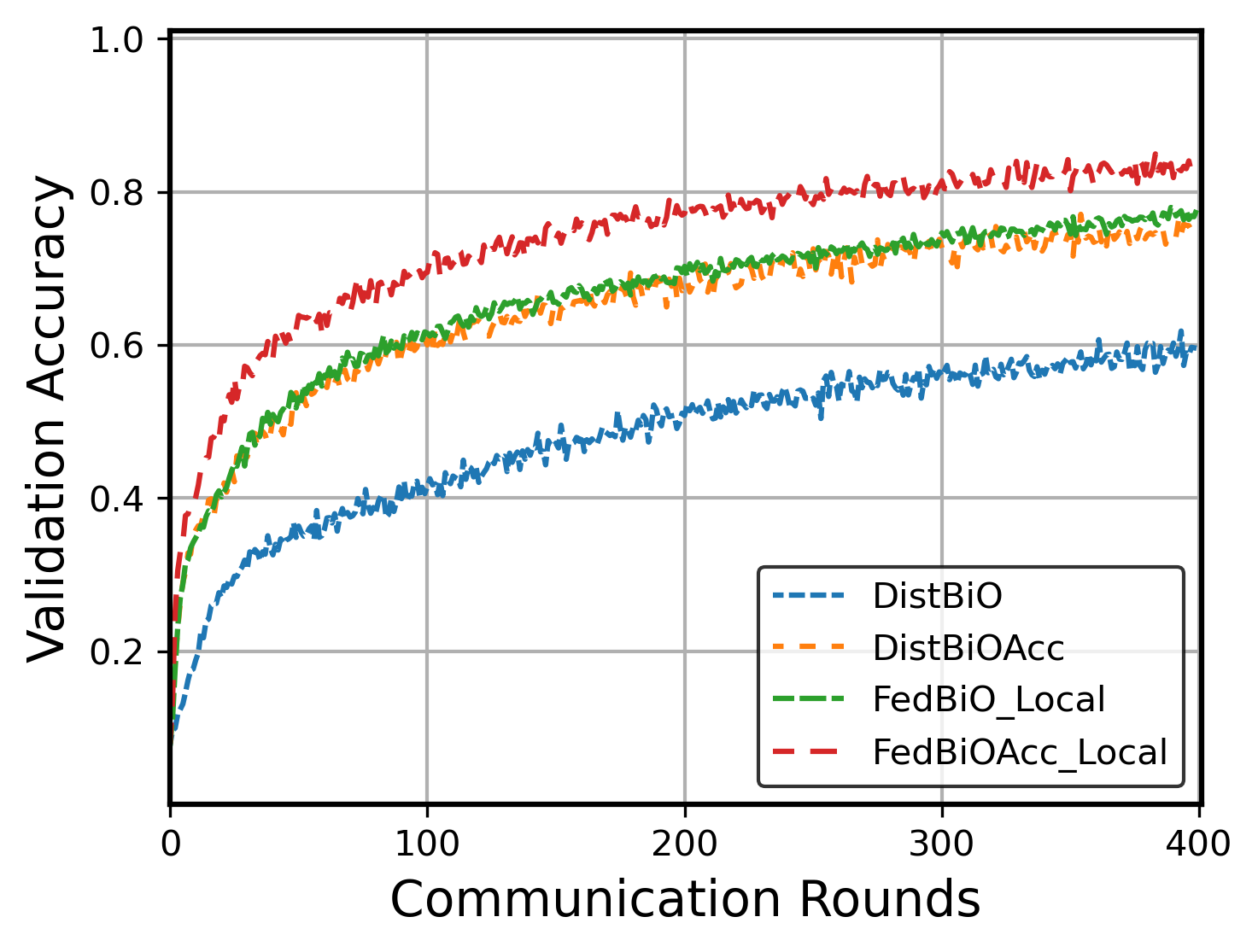}
    \includegraphics[width=0.24\columnwidth]{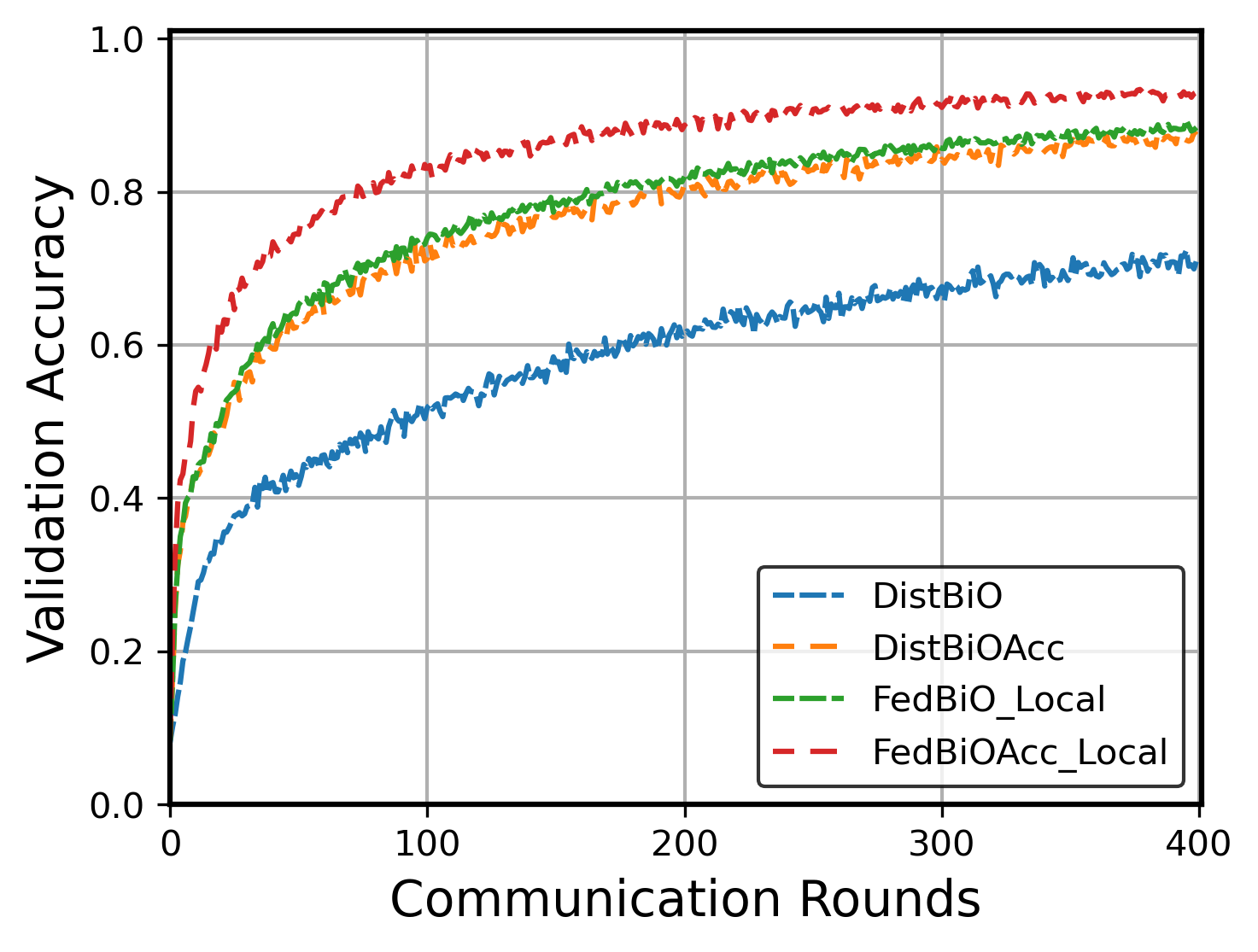}
    \includegraphics[width=0.24\columnwidth]{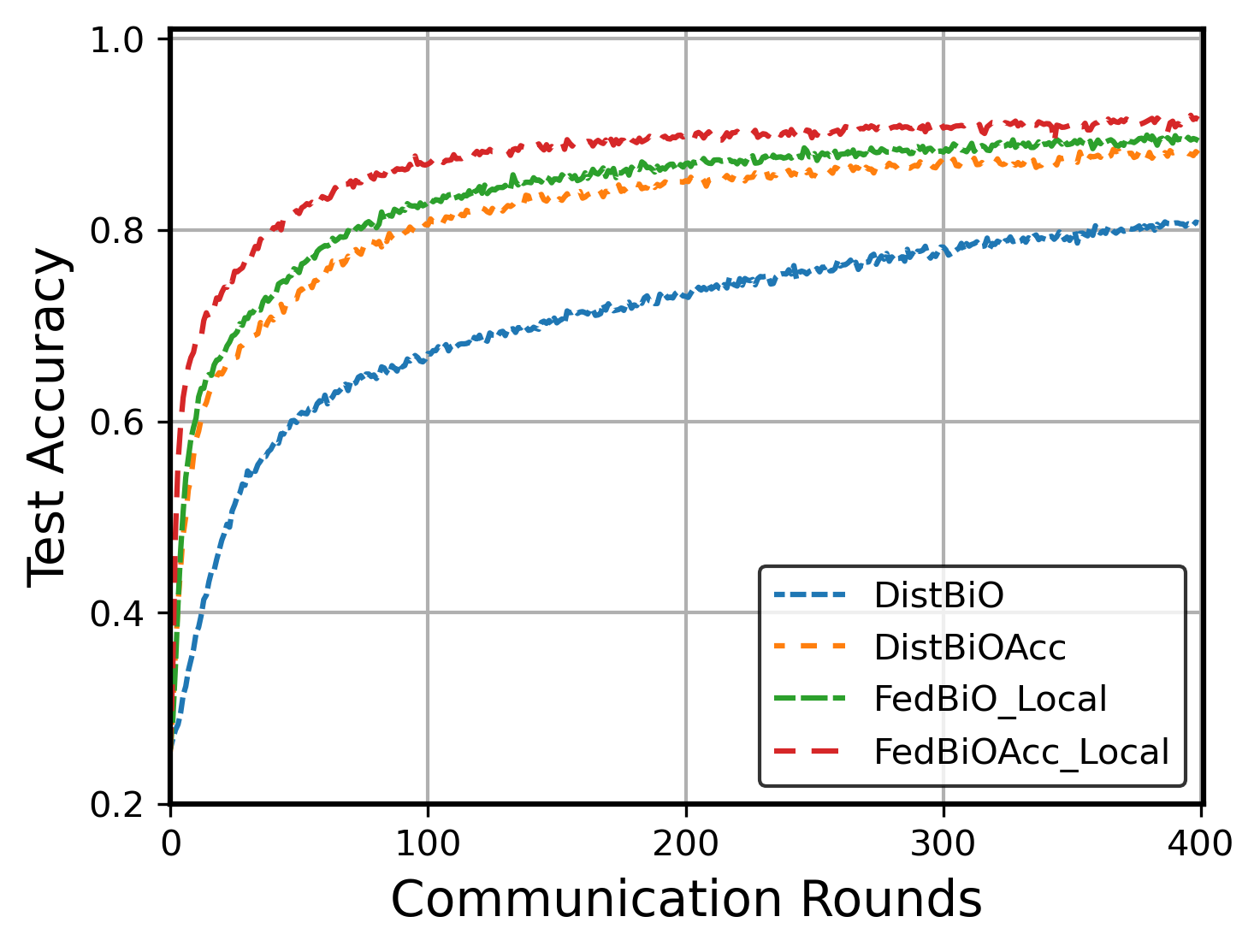}
    \includegraphics[width=0.24\columnwidth]{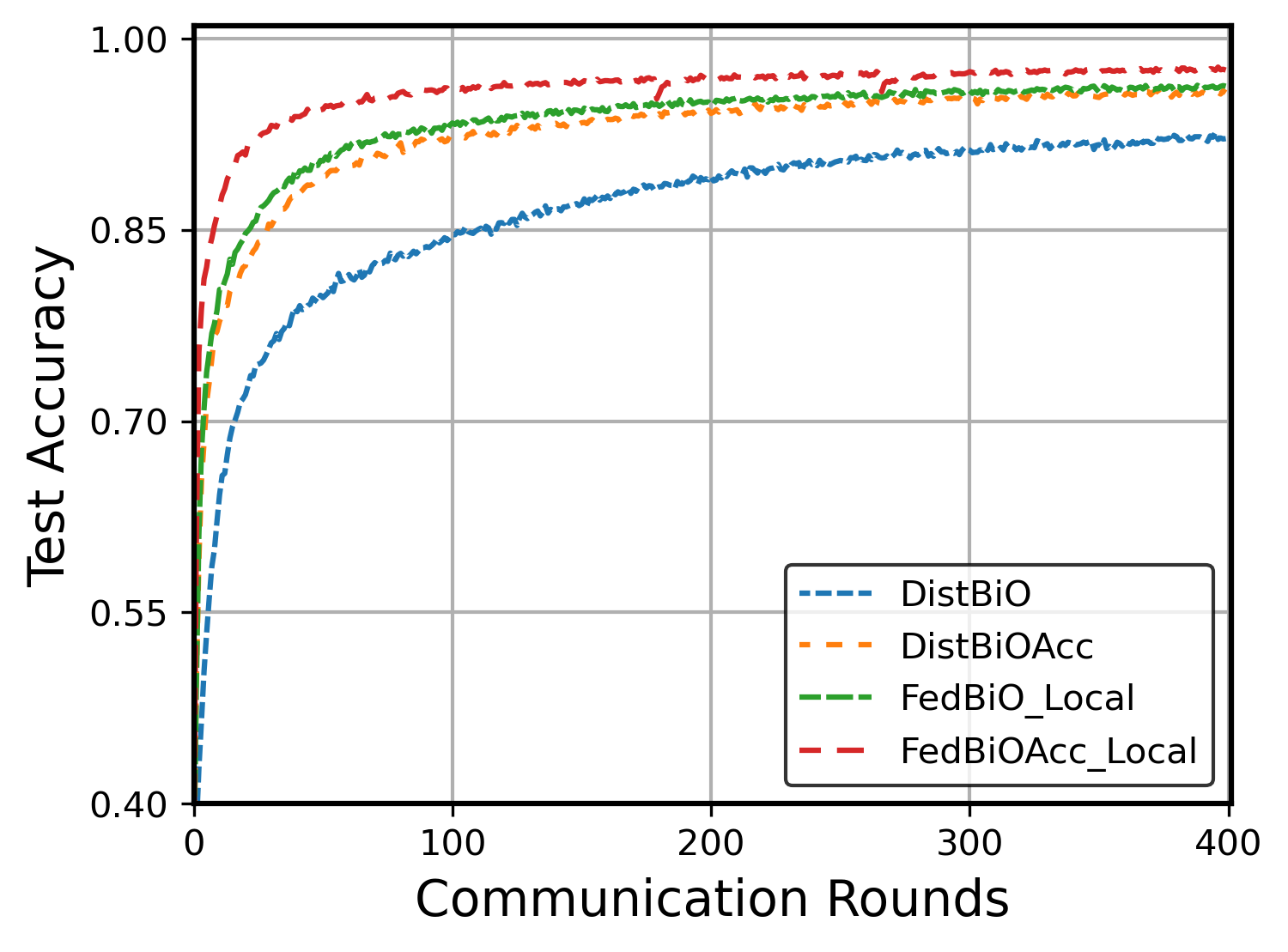}
    \includegraphics[width=0.24\columnwidth]{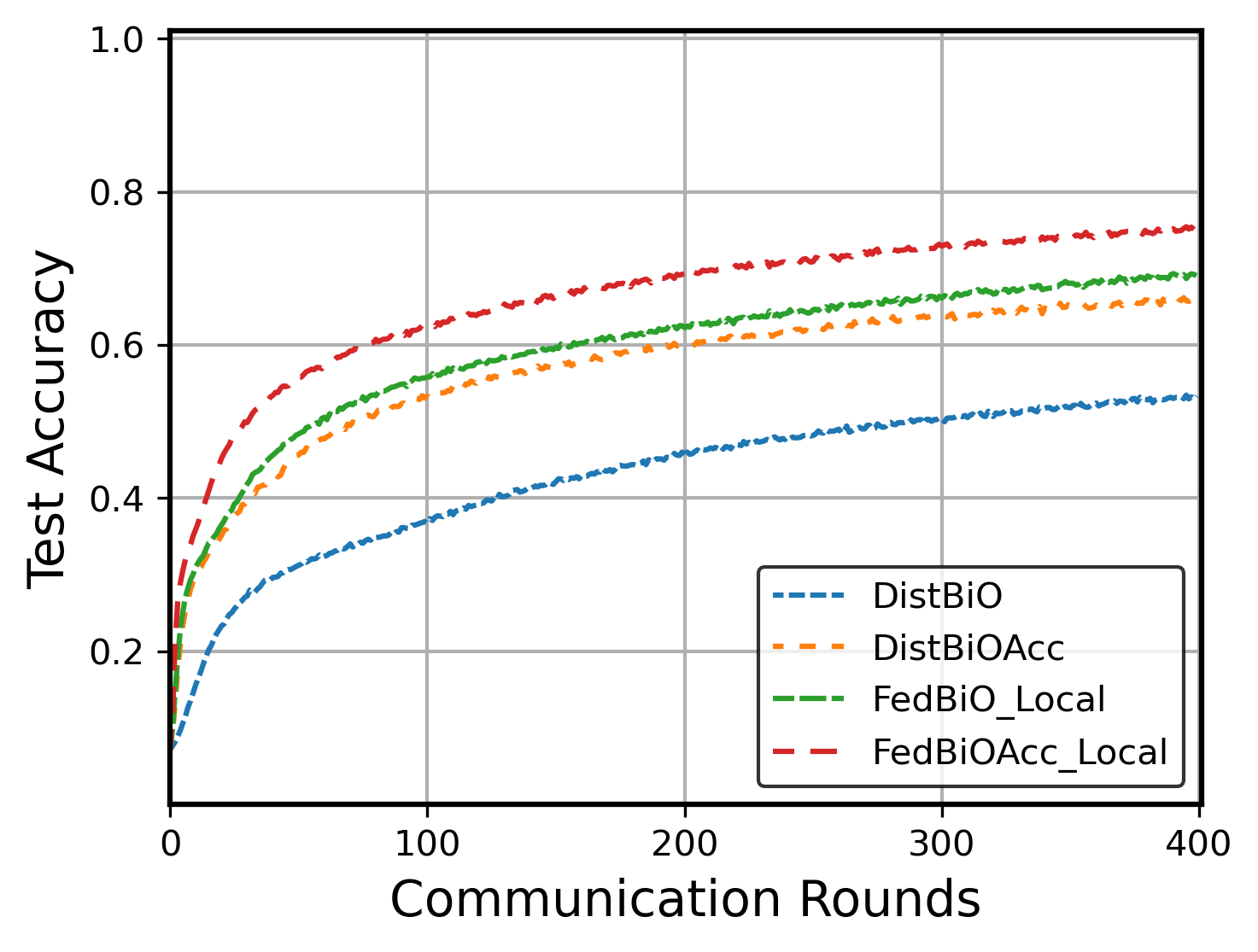}
    \includegraphics[width=0.24\columnwidth]{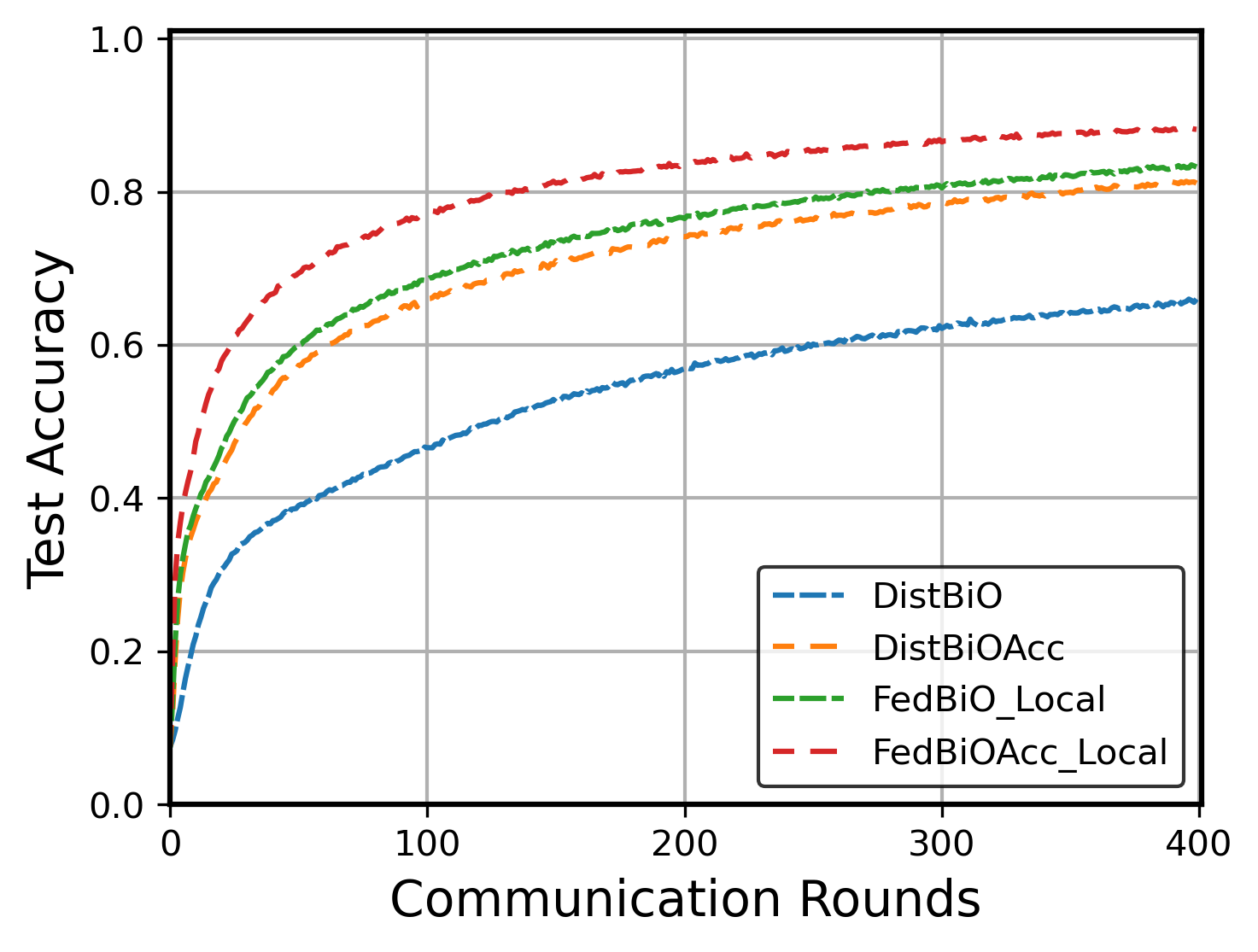}
    \includegraphics[width=0.24\columnwidth]{omni-5-1-val-err.png}
    \includegraphics[width=0.24\columnwidth]{omni-5-5-val-err.png}
    \includegraphics[width=0.24\columnwidth]{omni-20-1-val-err.png}
    \includegraphics[width=0.24\columnwidth]{omni-20-5-val-err.png}
    \includegraphics[width=0.24\columnwidth]{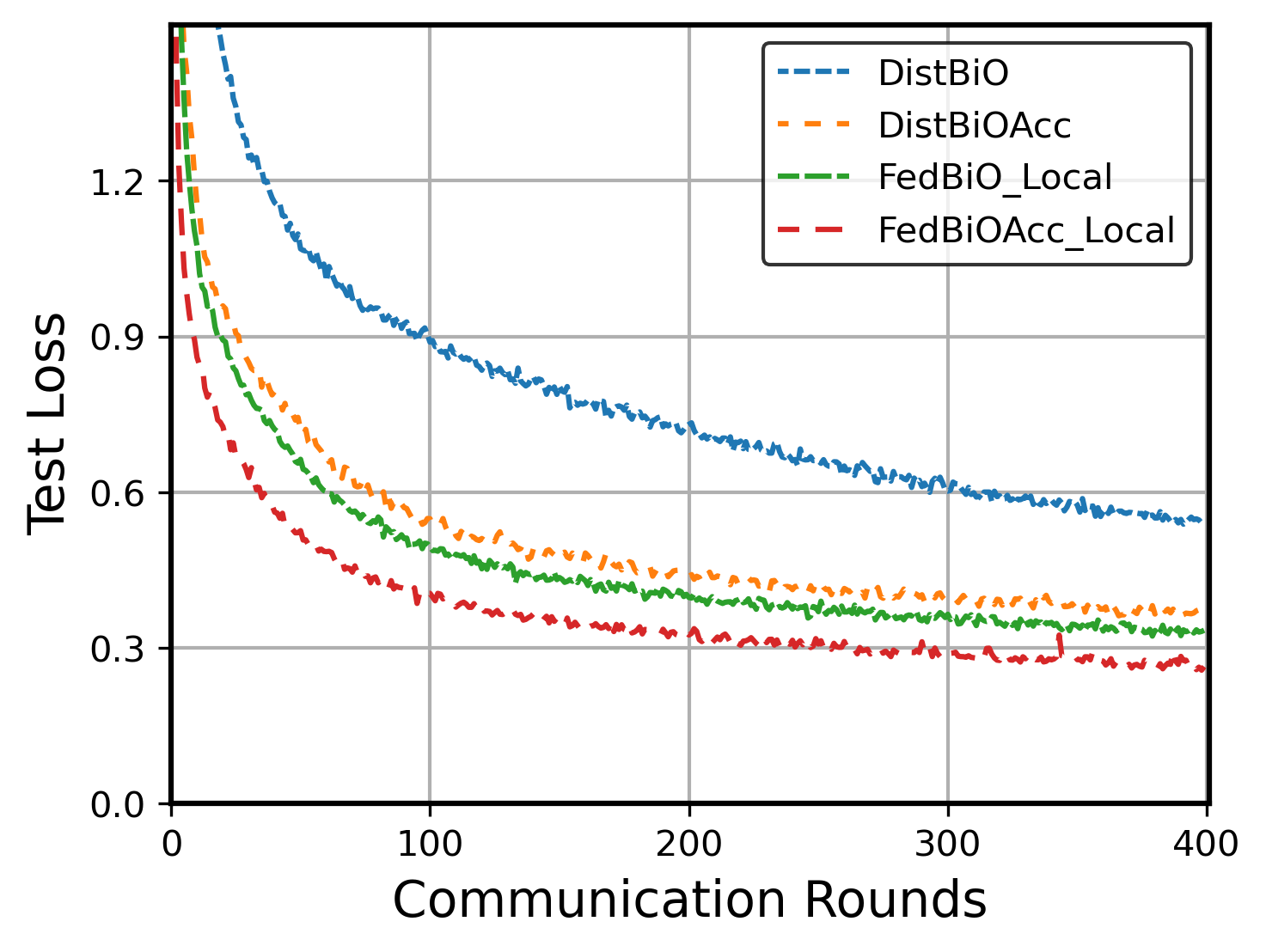}
    \includegraphics[width=0.24\columnwidth]{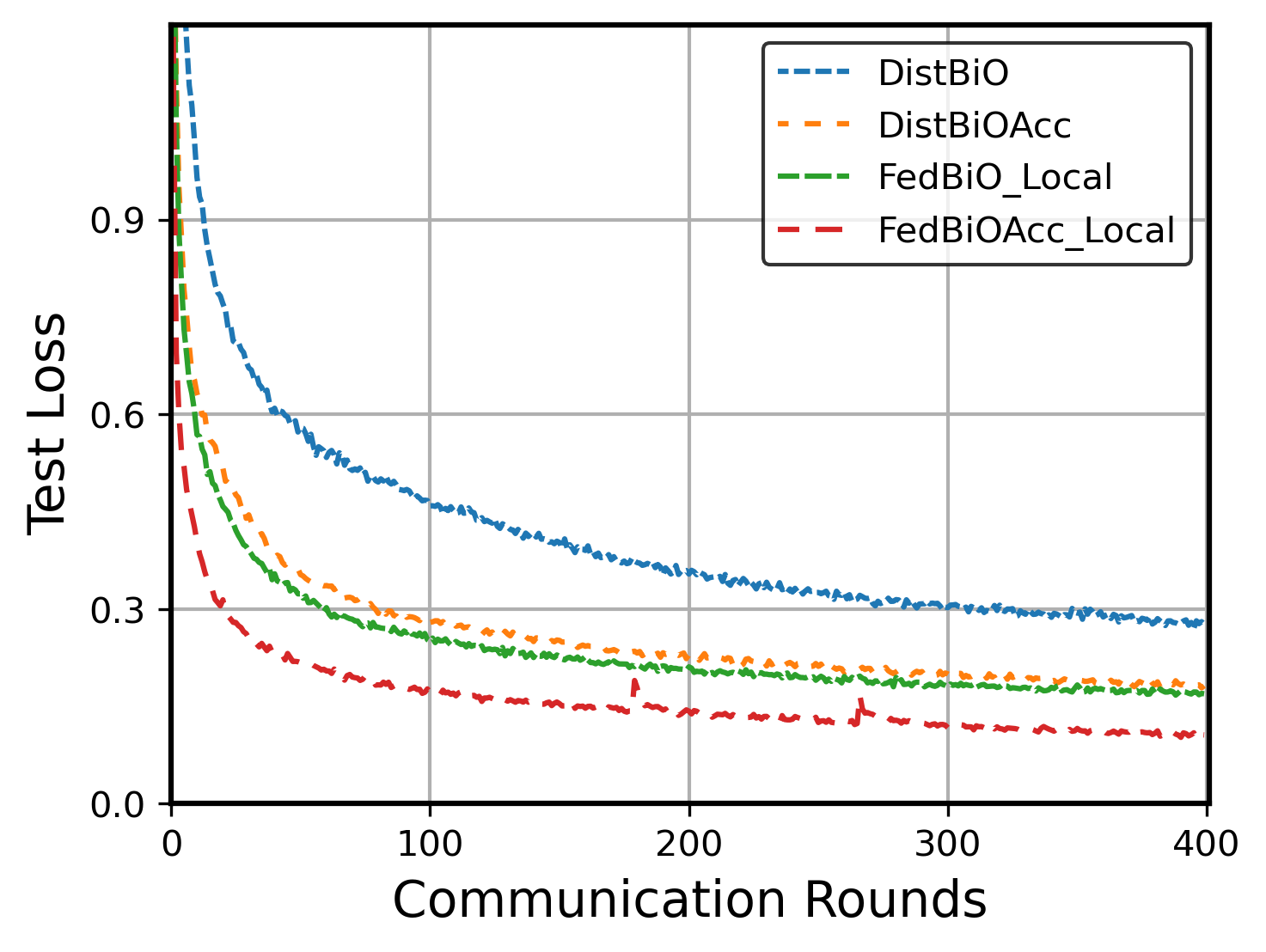}
    \includegraphics[width=0.24\columnwidth]{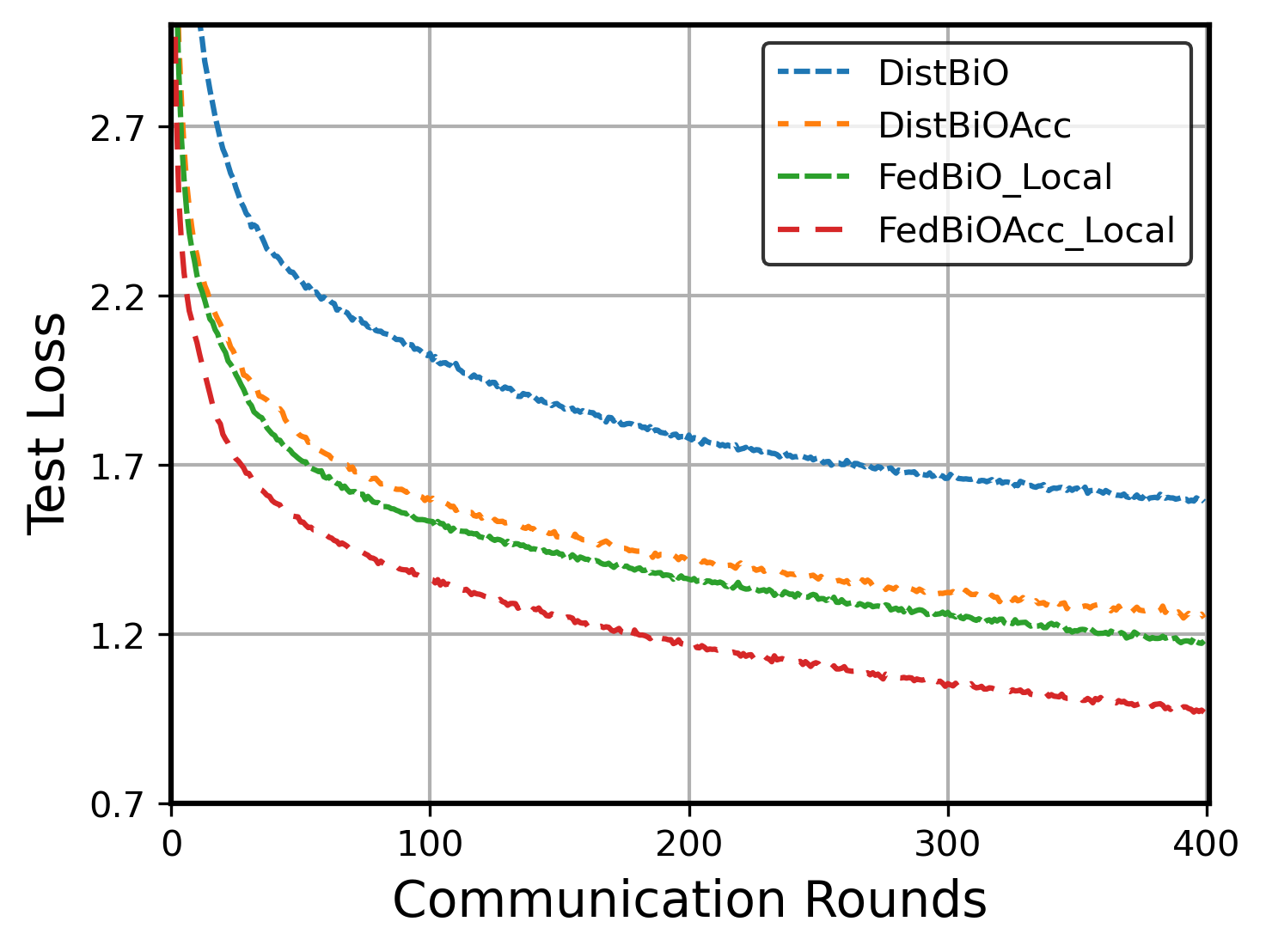}
    \includegraphics[width=0.24\columnwidth]{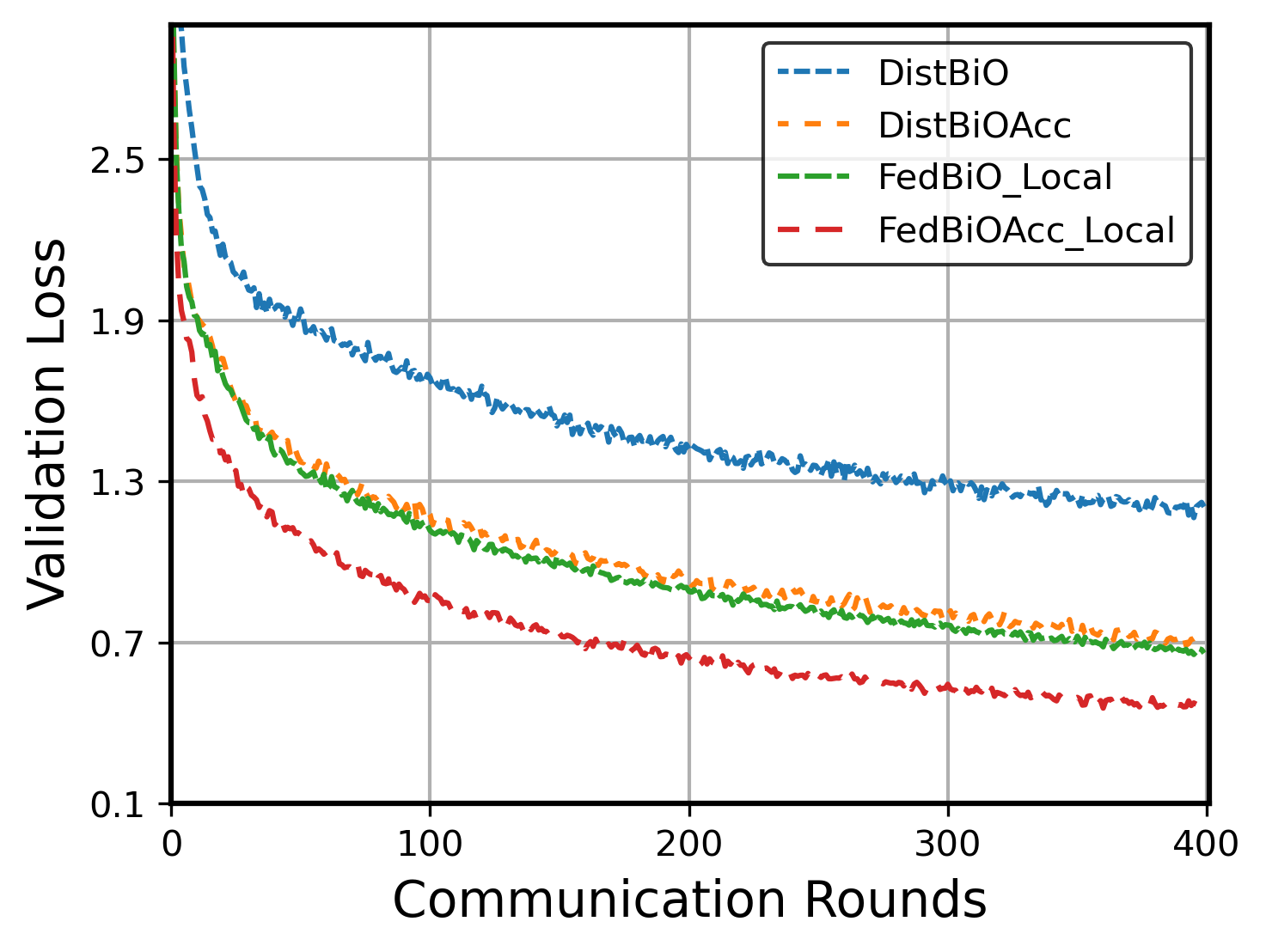}
\end{center}
\caption{Results for the Omniglot Dataset. From Left to Right: 5-way-1-shot, 5-way-5-shot, 20-way-1-shot, 20-way-5-shot.}
\label{fig:hyper-rep-omniglot}
\end{figure}

\begin{figure}[ht]
\begin{center}
    \includegraphics[width=0.24\columnwidth]{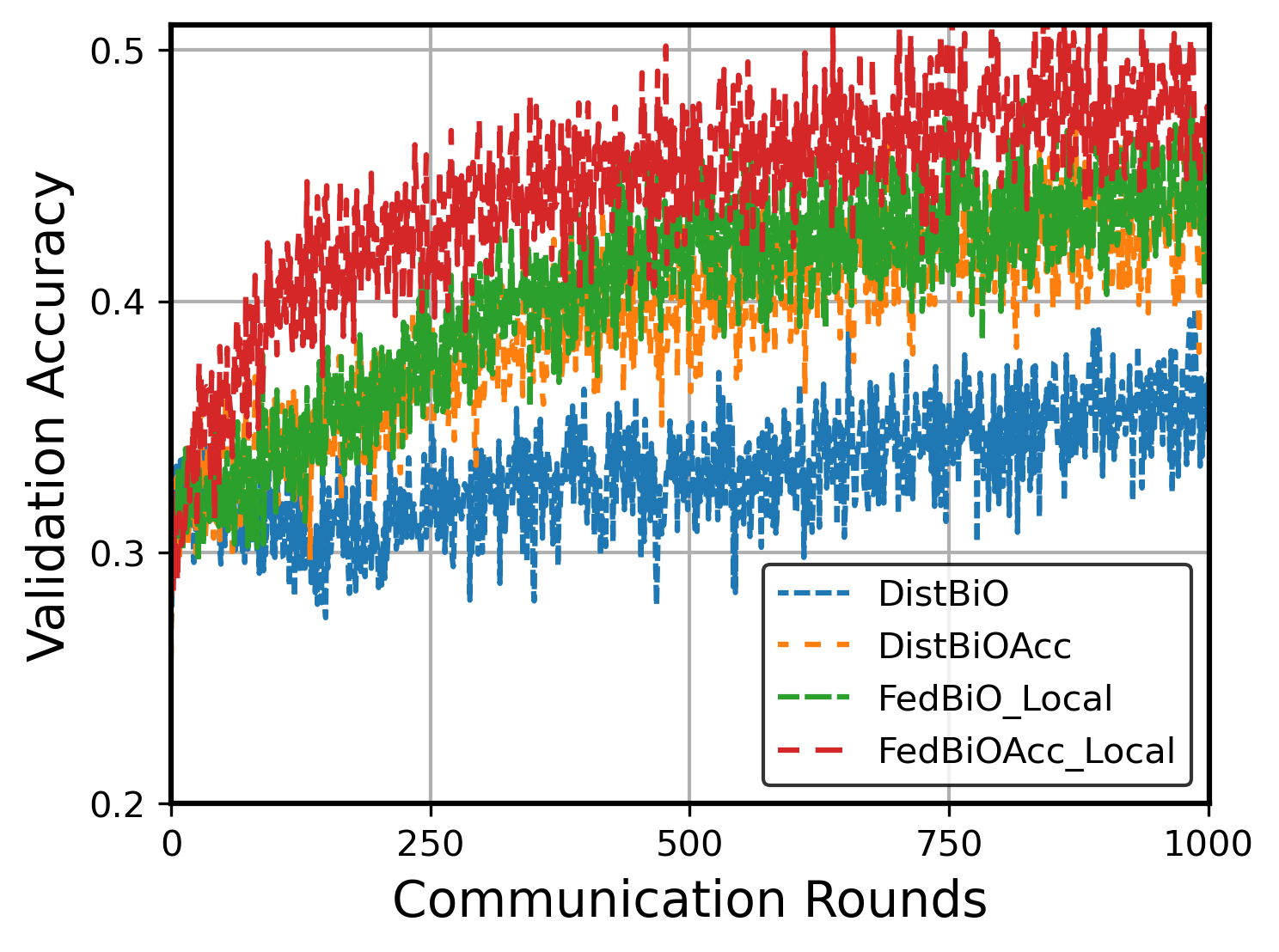}
    \includegraphics[width=0.24\columnwidth]{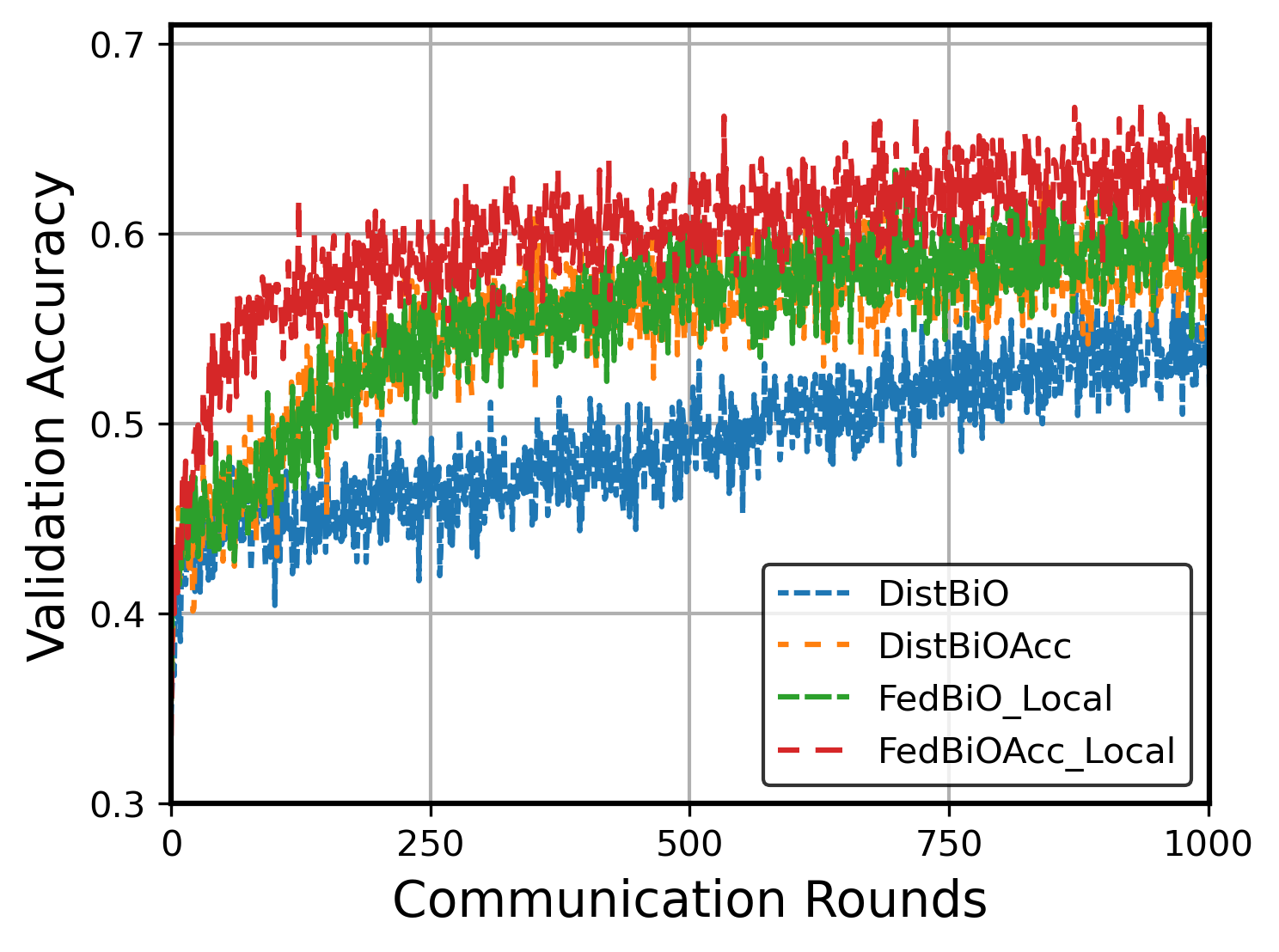}
    \includegraphics[width=0.24\columnwidth]{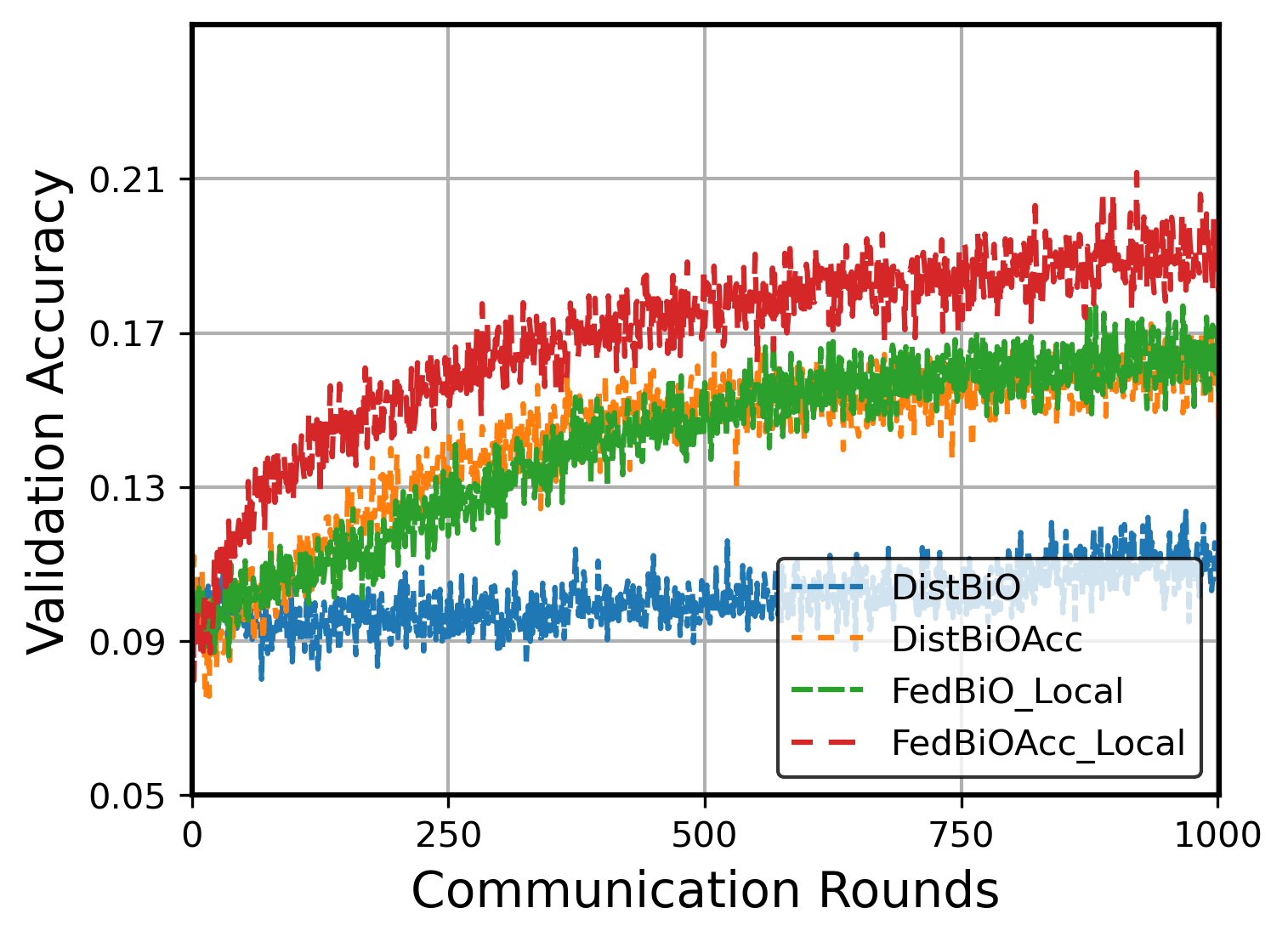}
    \includegraphics[width=0.24\columnwidth]{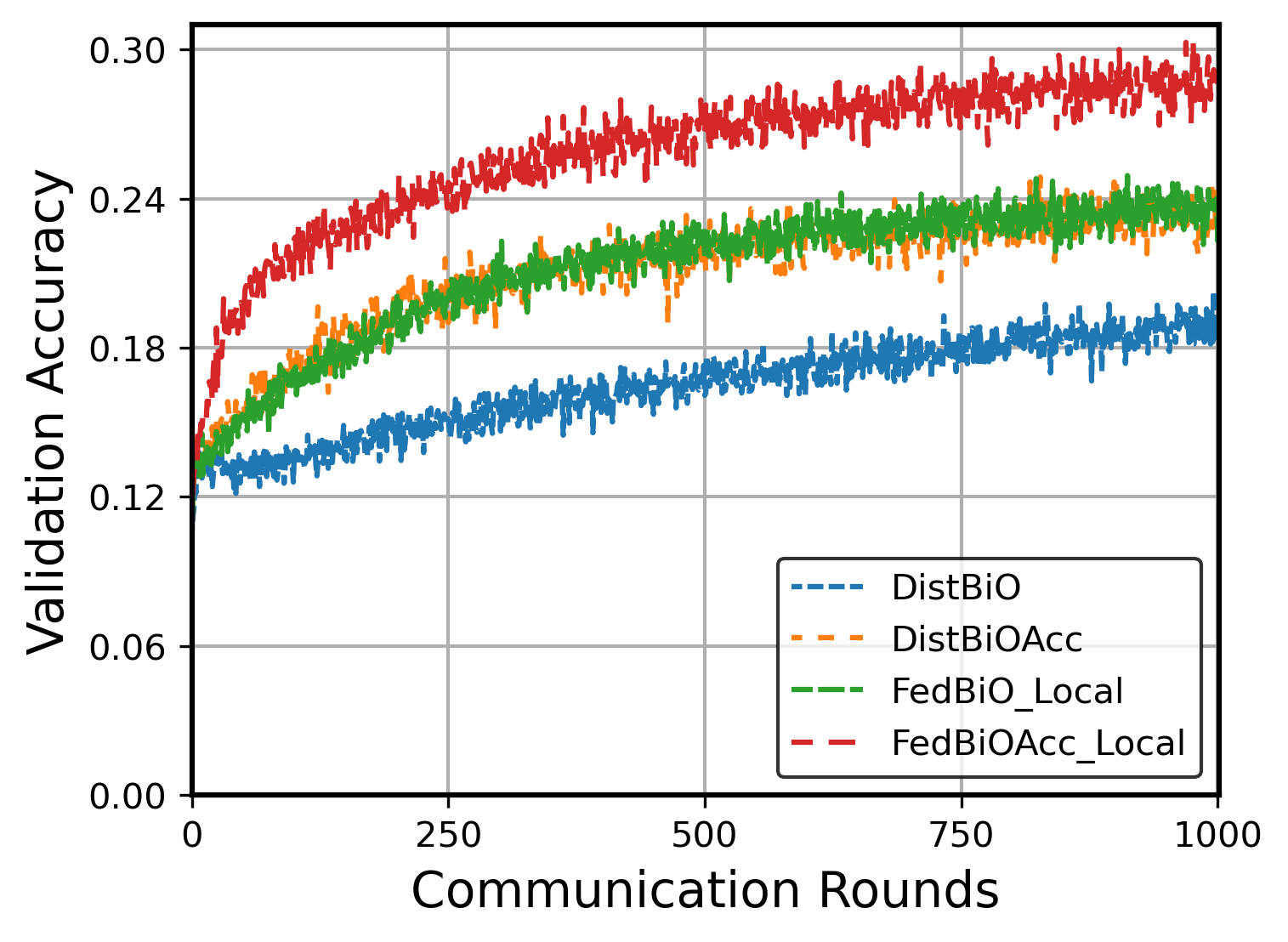}
    \includegraphics[width=0.24\columnwidth]{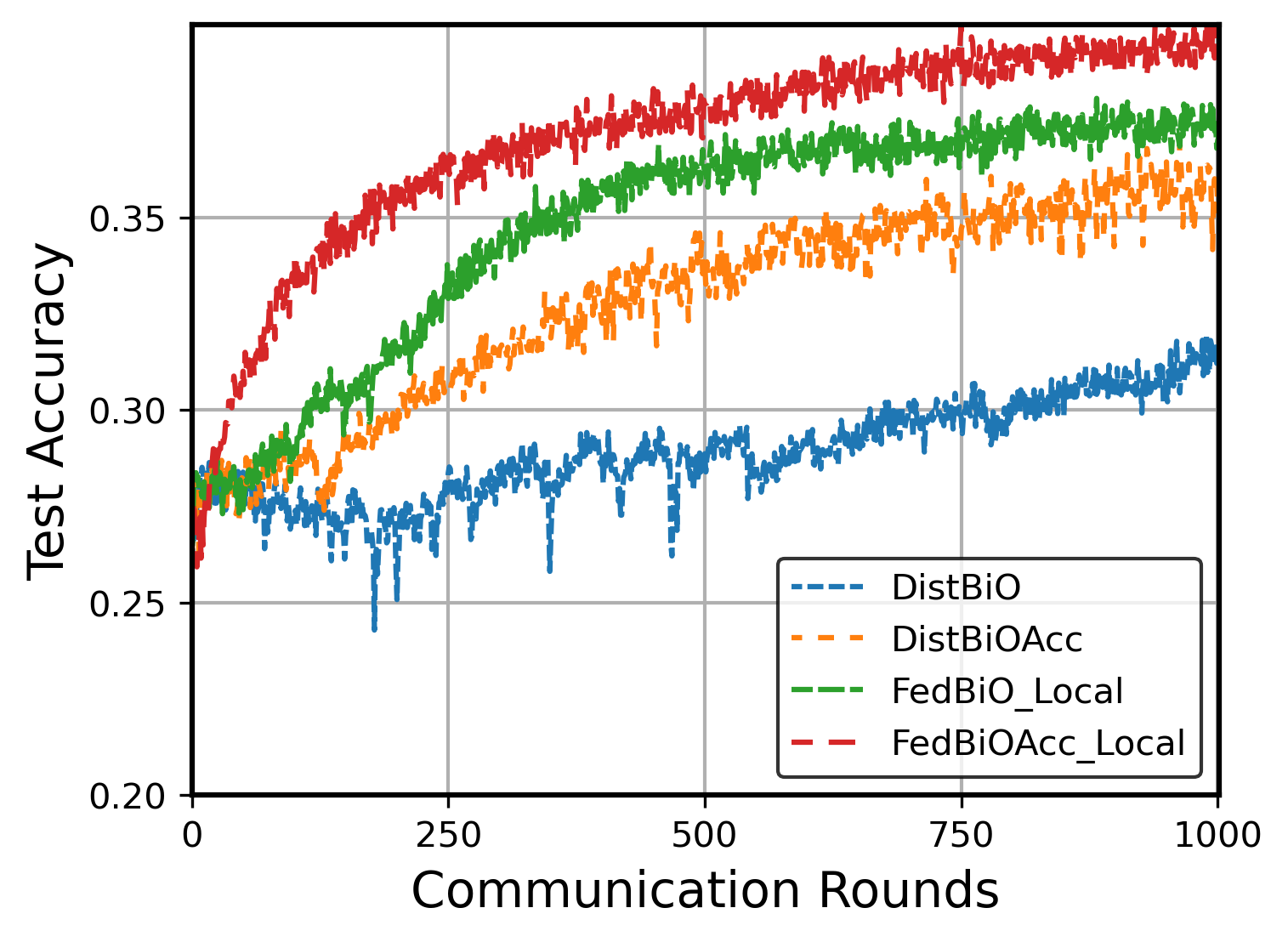}
    \includegraphics[width=0.24\columnwidth]{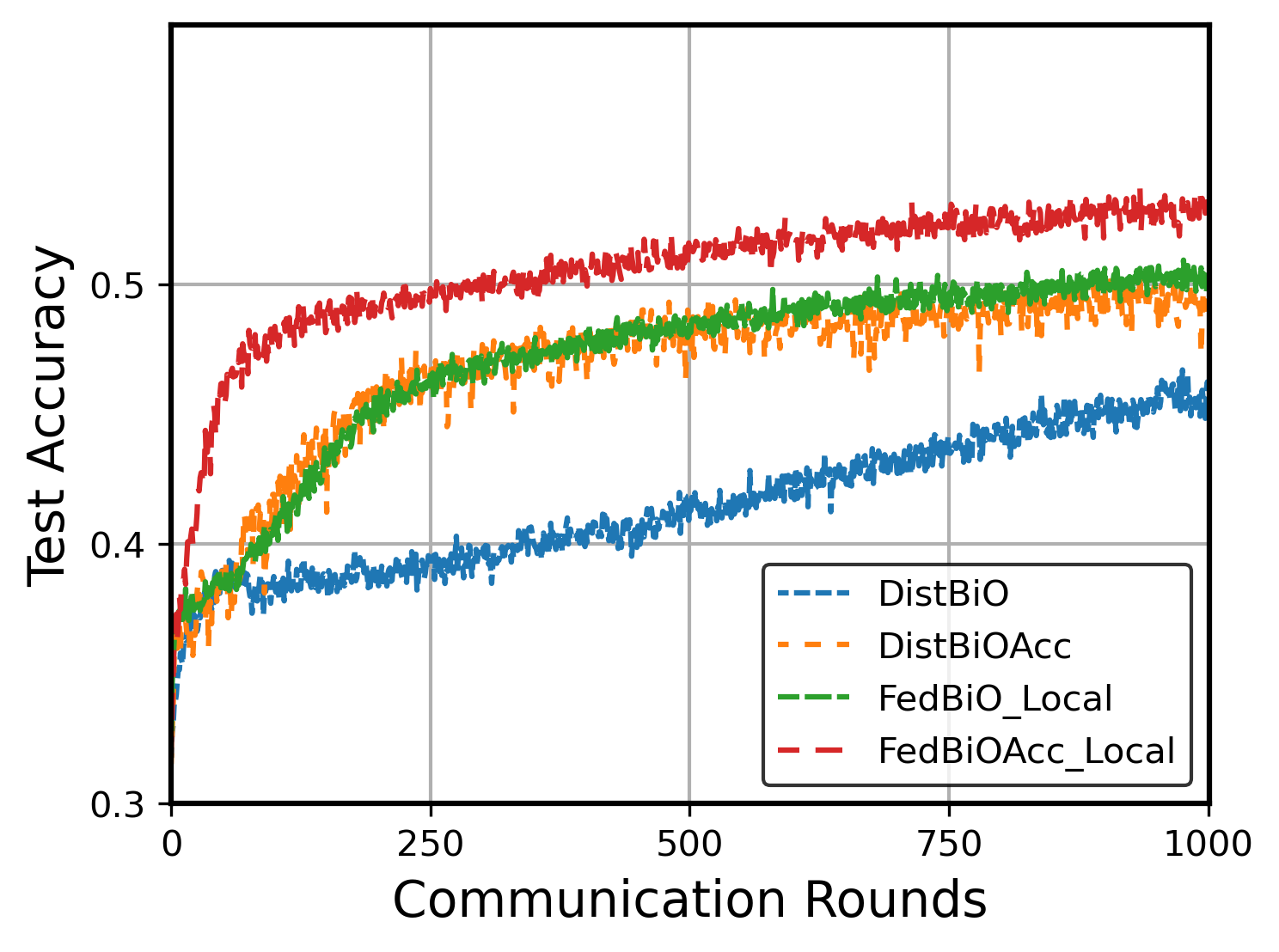}
    \includegraphics[width=0.24\columnwidth]{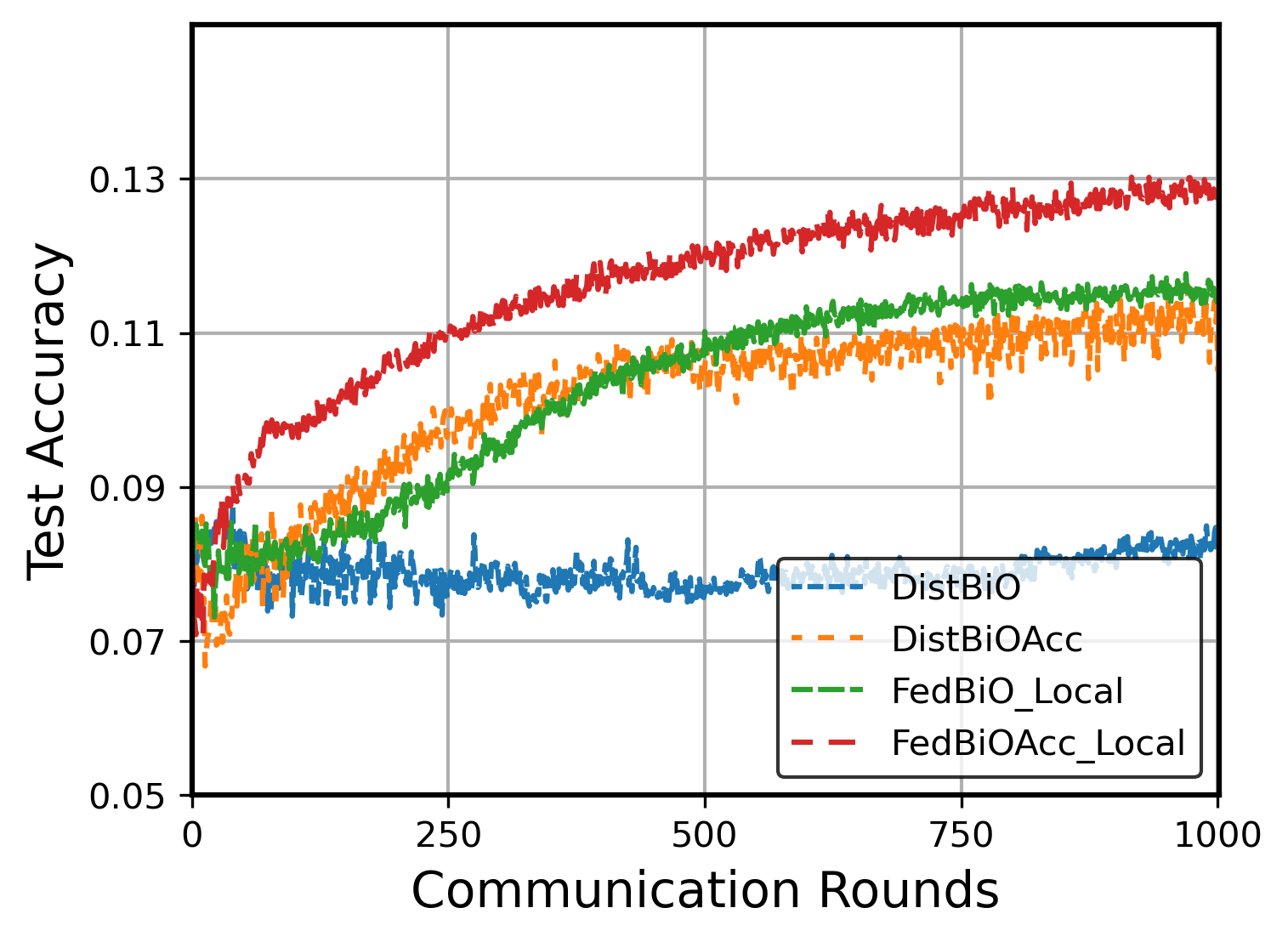}
    \includegraphics[width=0.24\columnwidth]{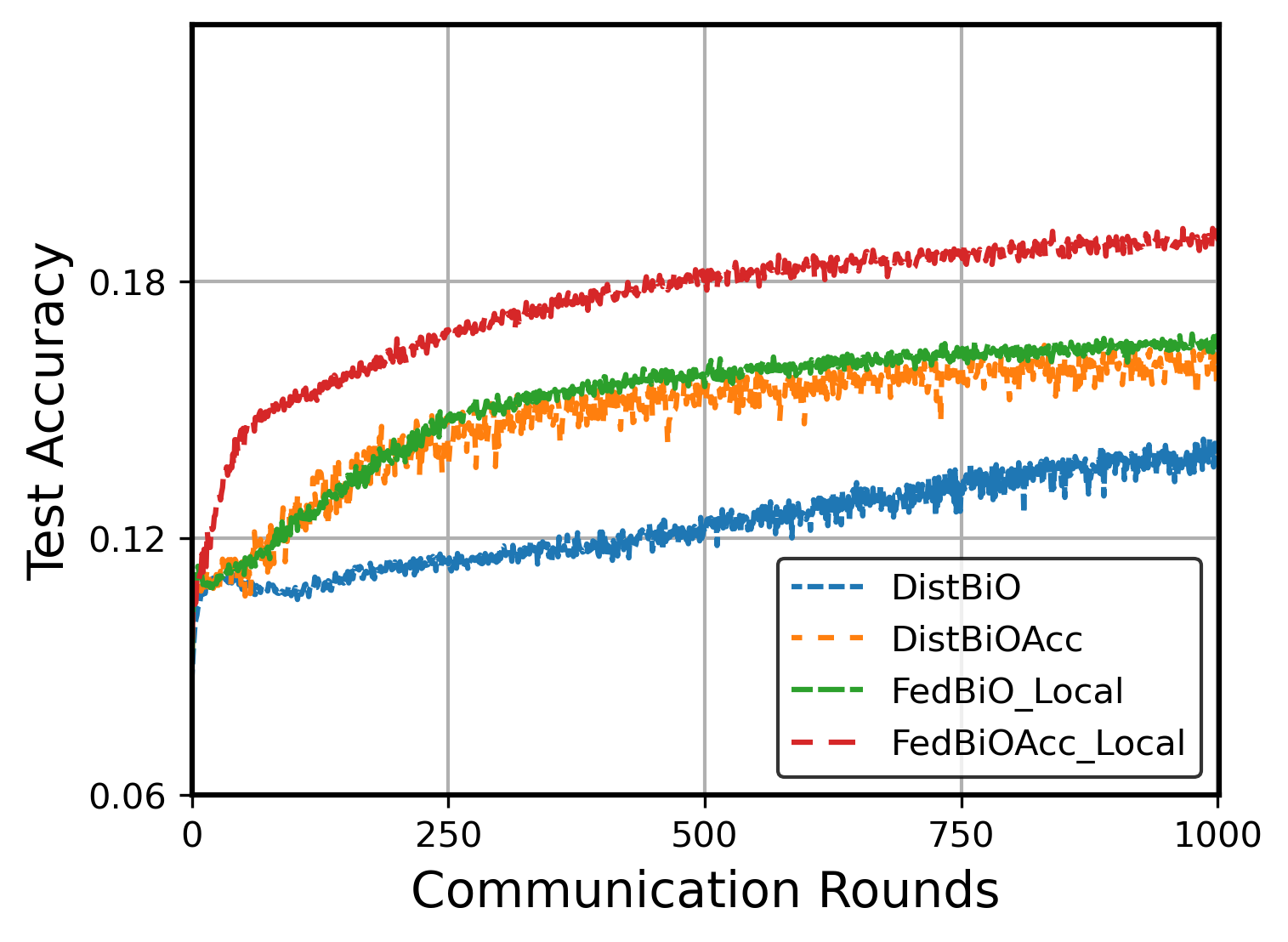}
    \includegraphics[width=0.24\columnwidth]{mini-5-1-val-err.png}
    \includegraphics[width=0.24\columnwidth]{mini-5-5-val-err.png}
    \includegraphics[width=0.24\columnwidth]{mini-20-1-val-err.png}
    \includegraphics[width=0.24\columnwidth]{mini-20-5-val-err.png}
    \includegraphics[width=0.24\columnwidth]{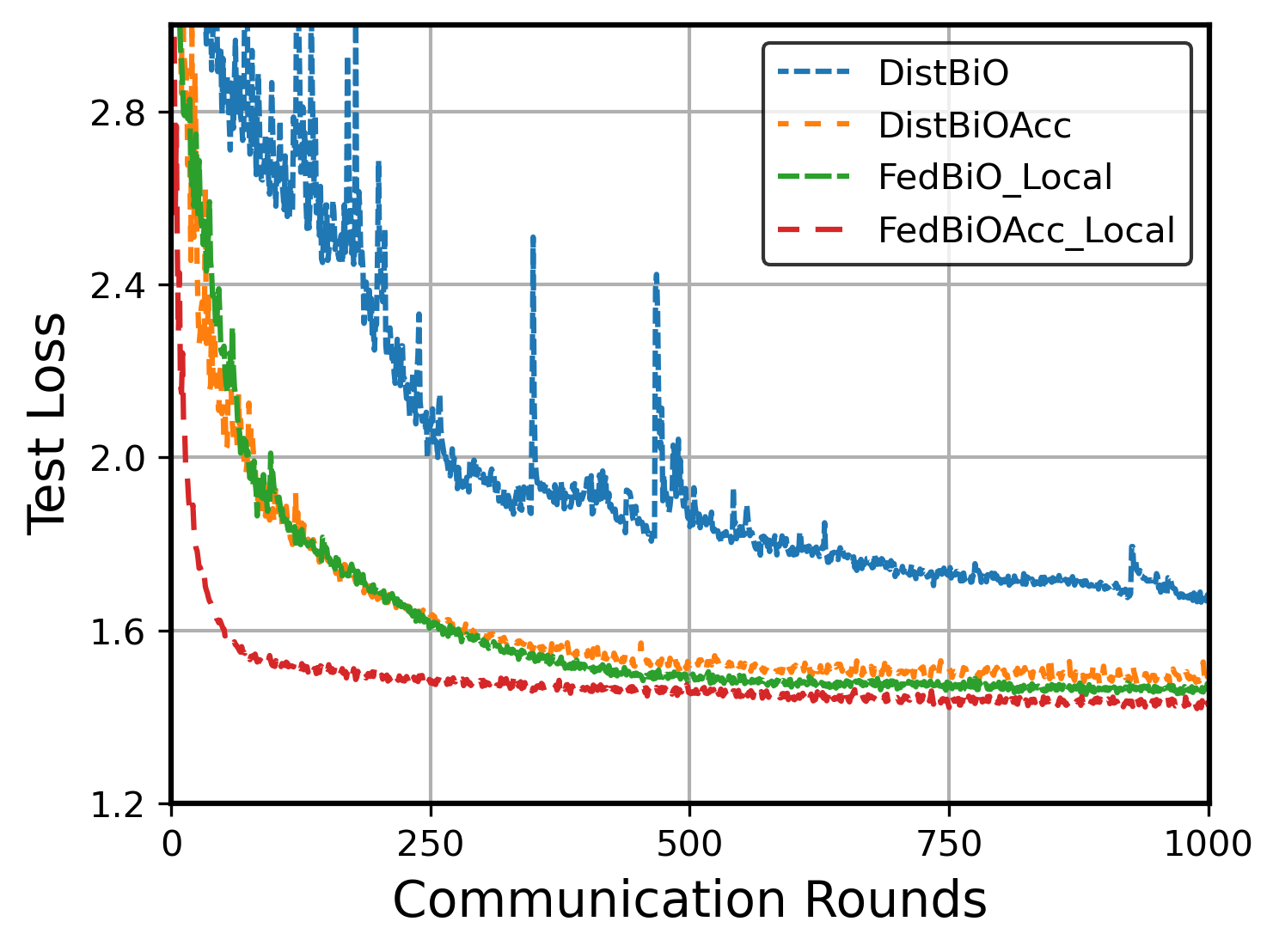}
    \includegraphics[width=0.24\columnwidth]{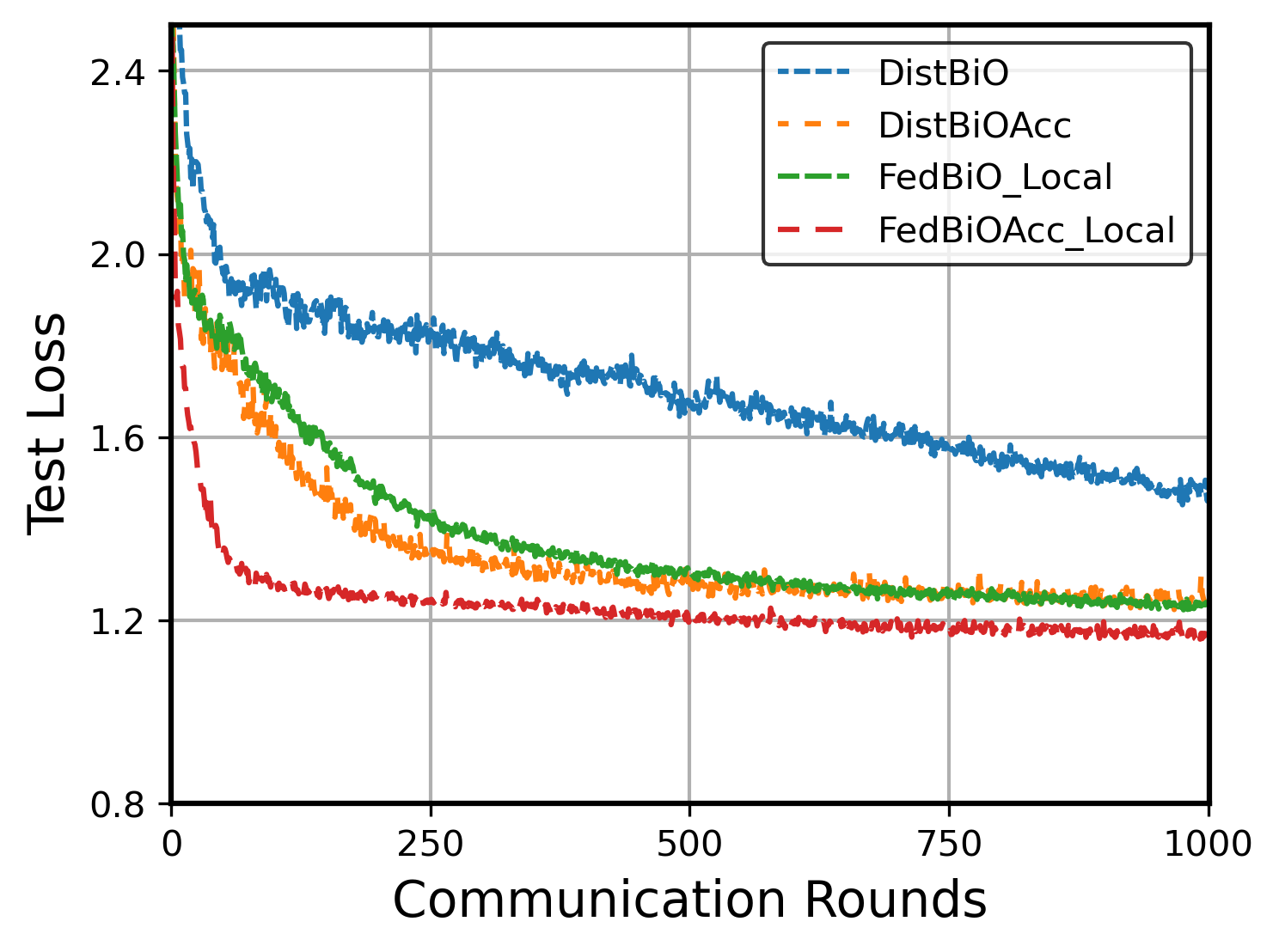}
    \includegraphics[width=0.24\columnwidth]{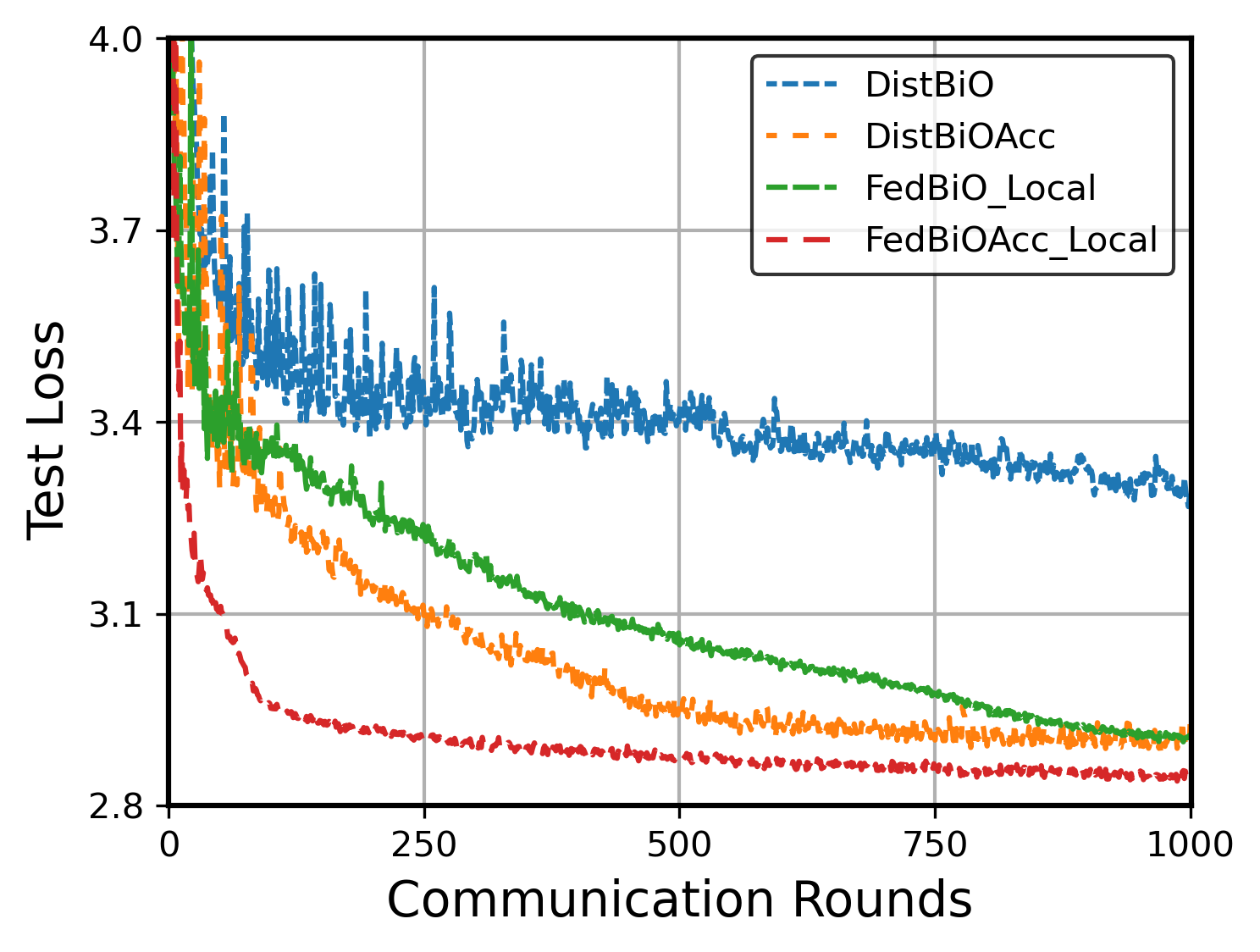}
    \includegraphics[width=0.24\columnwidth]{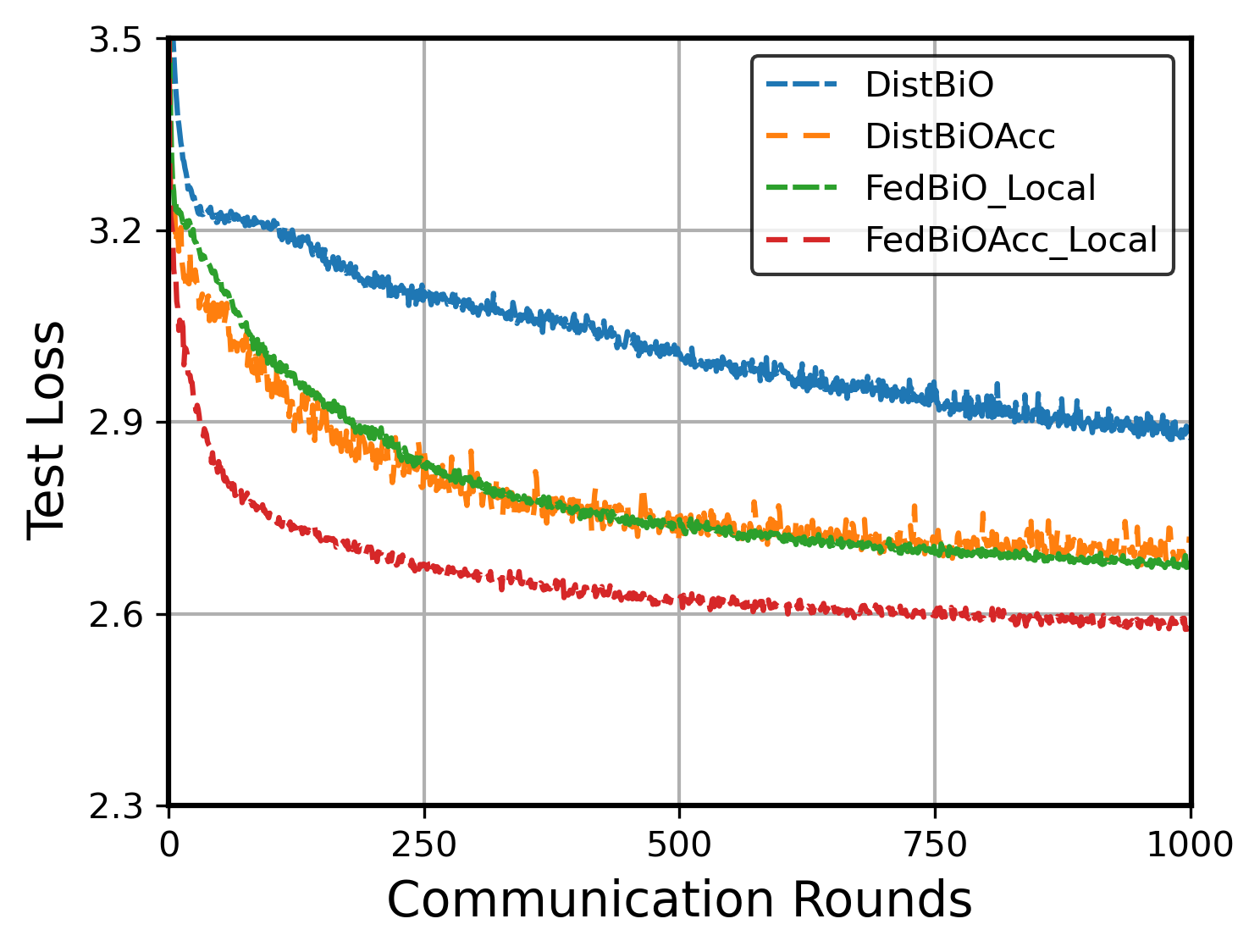}
\end{center}
\caption{Results for the MiniImageNet Dataset. From Left to Right: 5-way-1-shot, 5-way-5-shot, 20-way-1-shot, 20-way-5-shot.}
\label{fig:hyper-rep-miniimagenet}
\end{figure}

In the Experiments for Omniglot, for FedBiO, we choose learning rate 0.4, hyper learning rate 1, $\tau$ 0.5, for FedBiOAcc, we choose $\delta$ as 2, $u$ as 10000,  $C_{\eta}$ as 100, $\tau$ as 0.5, $eta$ as 1 and $\gamma$ as 0.4. For MiniImageNet, for FedBiO, we choose learning rate 0.05, hyper learning rate 0.1, $\tau$ 0.01, for FedBiOAcc, we choose $\delta$ as 2, $u$ as 10000,  $C_{\eta}$ as 100, $\tau$ as 0.01, $eta$ as 1 and $\gamma$ as 0.05.

\newpage

\section{Proof for Global Lower Level Problem}
This section includes proofs related to the Federated Bilevel Optimization problems with global lower level problems (Eq.~\ref{eq:fed-bi-multi}).  First, we have the global and local hyper-gradient $\nabla h(x) =  \Phi(x,y_x)$, $\nabla h^{(m)}(x) = \Phi^{(m)}(x,y_x)$ as defined in Eq.~\ref{eq:outer_grad_other} and Eq.~\ref{eq:outer_grad_other_local}, and the following proposition:

\begin{proposition}
\label{some smoothness}
Suppose Assumptions~\ref{assumption:f_smoothness} and ~\ref{assumption:g_smoothness} hold, the following statements hold:
\begin{itemize}
\item [a)] $y_x$ is Lipschitz continuous in $x$ with constant $\rho = \kappa$, where $\kappa = \frac{L}{\mu}$ is the condition number of $g(x,y)$.

\item [b)] $\|\Phi(x_1; y_1) - \Phi(x_2; y_2)\|^2 \leq \hat{L}^2 (\|x_1- x_2\|^2 + \|y_1- y_2\|^2)$, where $\hat{L} = O(\kappa^2)$.

\item [c)] $h(x)$ is Lipschitz continuous in $x$ with constant $\bar{L}$ i.e., for any given $x_1, x_2 \in X$, we have
$\|\nabla h(x_2) - \nabla h(x_1)\| \le \bar{L} \|x_2 - x_1\|$
where $\bar{L} =O(\kappa^3)$.
\end{itemize}
\end{proposition}
This is a standard results in bilevel optimization and we omit the proof here. 

\subsection{Proof for the FedBiOAcc Algorithm}
In this section, we prove the convergence of the FedBiOAcc Algorithm. To simplify the notation, we denote \[\mu_{t, \xi}^{(m)} = \nabla_x f^{(m)}(x^{(m)}_t, y^{(m)}_t; \xi_{f,1}) - \nabla_{xy}g^{(m)}(x^{(m)}_t, y^{(m)}_t; \xi_{g,1})u_{t}^{(m)},\]
and we have:
\[\mathbb{E}_{\xi}[\mu_{t, \xi}^{(m)}] = \nabla_x f^{(m)}(x^{(m)}_t, y^{(m)}_t) - \nabla_{xy}g^{(m)}(x^{(m)}_t, y^{(m)}_t)u_{t}^{(m)}\]
where the expectation is \emph{w.r.t} $\{\xi_{f,1}, \xi_{g,1}\}$ at iteration $t$, we denote $\mu_{t}^{(m)} = \mathbb{E}_{\xi}[\mu_{t, \xi}^{(m)}]$ for short. Similarly, we denote \[p_{t, \xi}^{(m)} = \nabla_{y^2} g^{(m)}(x^{(m)}_t, y^{(m)}_t; \xi_{g,2})u^{(m)}_{t} + \nabla_y f^{(m)}(x^{(m)}_t, y^{(m)}_t; \xi_{f,2}),\] and we have:
\[\mathbb{E}_{\xi}[p_{t, \xi}^{(m)}] = \nabla_{y^2} g^{(m)}(x^{(m)}_t, y^{(m)}_t)u^{(m)}_{t} + \nabla_y f^{(m)}(x^{(m)}_t, y^{(m)}_t).\] where the expectation is \emph{w.r.t} $\{\xi_{f,2}, \xi_{g,2}\}$ at iteration $t$, we denote $p_{t}^{(m)} = \mathbb{E}_{\xi}[p_{t, \xi}^{(m)}]$ for short.

\subsubsection{Hyper-Gradient Bias and Inner-Gradient Bias}
\begin{lemma}
\label{lemma:u_bound_storm_multi}
Suppose we have $c_u\alpha_t^2 < 1$, then we have:
\begin{align*}
  \mathbb{E} \big[ \big\| \bar{q}_{t} -  \bar{p}_{t}\big\|^2 \big] &\leq ( 1 - c_{u}\alpha_{t-1}^2)\mathbb{E} \big[ \big\| \bar{q}_{t-1} - \bar{p}_{t-1}\big\|^2 \big] + \frac{2(c_{u}\alpha_{t-1}^2)^2}{b_x M}\sigma^2\\
  &\qquad  +  \frac{4\tilde{L}_2^2}{b_xM^2}\sum_{m=1}^M\mathbb{E}\big[\big\|x^{(m)}_t - x^{(m)}_{t-1}\big\|^2 + \big\|y^{(m)}_t - y^{(m)}_{t-1}\big\|^2\big] \nonumber\\
  &\qquad + \frac{8L^2}{b_xM^2}\sum_{m=1}^M\mathbb{E}\big[\big\|u^{(m)}_t - u^{(m)}_{t-1}\big\|^2\big]
\end{align*}
where $\tilde{L}_2^2 = \big(L^2 + \frac{2L_{y^2}^2C_f^2}{\mu^2}\big)$ and the expectation outside is \emph{w.r.t} all the stochasity of the algorithm.
\end{lemma}

\begin{proof}
First, we have:
\begin{align*}
    &\mathbb{E} \big[ \big\| \bar{q}_{t} -\bar{p}_{t}] \big\|^2 \big] =  \mathbb{E} \big[ \big\|\bar{p}_{t, \mathcal{B}_x} + ( 1 -  c_{u}\alpha_{t-1}^2) (\bar{q}_{t-1} -\bar{p}_{t-1, \mathcal{B}_x}) -\bar{p}_{t}\big\|^2 \big] \nonumber \\
    & = \mathbb{E} \big[ \big\| ( 1 - c_{u}\alpha_{t-1}^2) \big(\bar{q}_{t-1} -\bar{p}_{t-1}\big) + \big(\bar{p}_{t,\mathcal{B}_x} -\bar{p}_{t} +  (1 - c_{u}\alpha_{t-1}^2)(\bar{p}_{t-1} -\bar{p}_{t-1,\mathcal{B}_x}) \big) \big\|^2  \big] \nonumber \\
    & \leq ( 1 - c_{u}\alpha_{t-1}^2)\mathbb{E} \big[ \big\| \bar{q}_{t-1} -\bar{p}_{t-1} \big\|^2 \big] + \mathbb{E} \big[ \big\|\bar{p}_{t,\mathcal{B}_x} -\bar{p}_{t} +  (1 - c_{u}\alpha_{t-1}^2)(\bar{p}_{t-1} -\bar{p}_{t-1,\mathcal{B}_x})\big\|^2 \big] \nonumber\\
    & \leq ( 1 - c_{u}\alpha_{t-1}^2)\mathbb{E} \big[ \big\| \bar{q}_{t-1} -\bar{p}_{t-1} \big\|^2 \big] \nonumber\\
    &\qquad + \frac{1}{b_x^2M^2}\sum_{m=1}^M\sum_{\xi_x \in \mathcal{B}_x}\mathbb{E} \big[ \big\|p^{(m)}_{t,\xi_x} - p^{(m)}_{t} +  (1 - c_{u}\alpha_{t-1}^2)(p^{(m)}_{t-1} - p^{(m)}_{t-1,\xi_x})\big\|^2 \big]
\end{align*}
where the first inequality uses the fact that the cross product term is zero in expectation, the condition that $c_\nu\alpha_t^2 < 1$ and the second inequality follows that samples are independent among clients. We denote the second term of above as $T_1$, then we have:
\begin{align*}
  T_1 & \overset{(a)}{\leq} 2(c_{u}\alpha_{t-1}^2)^2\mathbb{E} \big[ \big\| p^{(m)}_{t,\xi_x} - p^{(m)}_{t} \big\|^2 \big]  +  2(1 - c_{u}\alpha_{t-1}^2)^2 \mathbb{E} \big[ \big\|p^{(m)}_{t,\xi_x} - p^{(m)}_{t-1,\xi_x} - (p^{(m)}_{t} - p^{(m)}_{t-1}) \big\|^2 \big] \nonumber \\
  & \overset{(b)}{\leq} 2(c_{u}\alpha_{t-1}^2)^2\sigma^2  +   2\mathbb{E} \big[ \big\|p^{(m)}_{t,\xi_x} - p^{(m)}_{t-1,\xi_x}\big\|^2 \big] 
\end{align*}
where inequality (a) follows the generalized triangle inequality; (b) and the bounded variance assumption. We denote the second term above as $T_{1,2}$, we have:
\begin{align*}
    T_{1,2} &=  2\mathbb{E}\big\|\nabla_{y^2} g^{(m)}(x^{(m)}_t, y^{(m)}_t; \xi_{g,2})u^{(m)}_{t} + \nabla_y f^{(m)}(x^{(m)}_t, y^{(m)}_t; \xi_{f,2})\nonumber\\
    &\qquad - \big( \nabla_{y^2} g^{(m)}(x^{(m)}_{t-1}, y^{(m)}_{t-1}; \xi_{g,2})u^{(m)}_{t-1} + \nabla_y f^{(m)}(x^{(m)}_t, y^{(m)}_{t-1}; \xi_{f,2})\big)\big\|^2 \nonumber\\
    &\leq 4\mathbb{E}\big\|\nabla_y f^{(m)}(x^{(m)}_t, y^{(m)}_t; \mathcal{B}_{f,1}) - \nabla_y f^{(m)}(x^{(m)}_{t-1}, y^{(m)}_{t-1}; \mathcal{B}_{f,1})\big\|^2 \nonumber\\
    &\qquad + 4\mathbb{E}\big\|\nabla_{y^2}g^{(m)}(x^{(m)}_t, y^{(m)}_t; \mathcal{B}_{g,1})u_{t}^{(m)}  - \nabla_{y^2}g^{(m)}(x^{(m)}_{t-1}, y^{(m)}_{t-1}; \mathcal{B}_{g,1})u_{t-1}^{(m)}\big\|^2 \nonumber\\
    &\leq 4\big(L^2 + \frac{2L_{y^2}^2C_f^2}{\mu^2}\big)\mathbb{E}\big[\big\|x^{(m)}_t - x^{(m)}_{t-1}\big\|^2 + \big\|y^{(m)}_t - y^{(m)}_{t-1}\big\|^2\big] + 8L^2\mathbb{E}\big[\big\|u^{(m)}_t - u^{(m)}_{t-1}\big\|^2\big]
\end{align*}
Combine everything together finishes the proof.
\end{proof}

\begin{lemma}
\label{lemma:hg_bound_storm_multi}
Suppose we have $c_\nu\alpha_t^2 < 1$, then we have:
\begin{align*}
  \mathbb{E} \big[ \big\| \bar{\nu}_{t} -  \bar{\mu}_{t}\big\|^2 \big] &\leq ( 1 - c_{\nu}\alpha_{t-1}^2)\mathbb{E} \big[ \big\| \bar{\nu}_{t-1} - \bar{\mu}_{t-1}\big\|^2 \big] + \frac{2(c_{\nu}\alpha_{t-1}^2)^2}{b_x M}\sigma^2\\
  &\qquad  +  \frac{4\tilde{L}_1^2}{b_xM^2}\sum_{m=1}^M\mathbb{E}\big[\big\|x^{(m)}_t - x^{(m)}_{t-1}\big\|^2 + \big\|y^{(m)}_t - y^{(m)}_{t-1}\big\|^2\big] \nonumber\\
  &\qquad + \frac{8L^2}{b_xM^2}\sum_{m=1}^M\mathbb{E}\big[\big\|u^{(m)}_t - u^{(m)}_{t-1}\big\|^2\big]
\end{align*}
where $\tilde{L}_1^2 = \big(L^2 + \frac{2L_{xy}^2C_f^2}{\mu^2}\big)$ and the expectation outside is \emph{w.r.t} all the stochasity of the algorithm.
\end{lemma}

\begin{proof}
First, we have:
\begin{align*}
    &\mathbb{E} \big[ \big\| \bar{\nu}_{t} - \bar{\mu}_{t}] \big\|^2 \big] =  \mathbb{E} \big[ \big\|\bar{\mu}_{t, \mathcal{B}_x} + ( 1 -  c_{\nu}\alpha_{t-1}^2) (\bar{\nu}_{t-1} - \bar{\mu}_{t-1, \mathcal{B}_x}) - \bar{\mu}_{t}\big\|^2 \big] \nonumber \\
    & = \mathbb{E} \big[ \big\| ( 1 - c_{\nu}\alpha_{t-1}^2) \big(\bar{\nu}_{t-1} - \bar{\mu}_{t-1}\big) + \big(\bar{\mu}_{t,\mathcal{B}_x} - \bar{\mu}_{t} +  (1 - c_{\nu}\alpha_{t-1}^2)(\bar{\mu}_{t-1} - \bar{\mu}_{t-1,\mathcal{B}_x}) \big) \big\|^2  \big] \nonumber \\
    & \leq ( 1 - c_{\nu}\alpha_{t-1}^2)\mathbb{E} \big[ \big\| \bar{\nu}_{t-1} - \bar{\mu}_{t-1} \big\|^2 \big] + \mathbb{E} \big[ \big\|\bar{\mu}_{t,\mathcal{B}_x} - \bar{\mu}_{t} +  (1 - c_{\nu}\alpha_{t-1}^2)(\bar{\mu}_{t-1} - \bar{\mu}_{t-1,\mathcal{B}_x})\big\|^2 \big] \nonumber\\
    & \leq ( 1 - c_{\nu}\alpha_{t-1}^2)\mathbb{E} \big[ \big\| \bar{\nu}_{t-1} - \bar{\mu}_{t-1} \big\|^2 \big] + \frac{1}{b_x^2M^2}\sum_{m=1}^M\sum_{\xi_x \in \mathcal{B}_x}\mathbb{E} \big[ \big\|\mu^{(m)}_{t,\xi_x} - \mu^{(m)}_{t} +  (1 - c_{\nu}\alpha_{t-1}^2)(\mu^{(m)}_{t-1} - \mu^{(m)}_{t-1,\xi_x})\big\|^2 \big]
\end{align*}
where the first inequality uses the fact that the cross product term is zero in expectation, the condition that $c_\nu\alpha_t^2 < 1$ and the second inequality follows that samples are independent among clients. We denote the second term of above as $T_1$, then we have:
\begin{align*}
  T_1 & \overset{(a)}{\leq} 2(c_{\nu}\alpha_{t-1}^2)^2\mathbb{E} \big[ \big\| \mu^{(m)}_{t,\xi_x} - \mu^{(m)}_{t} \big\|^2 \big]  +  2(1 - c_{\nu}\alpha_{t-1}^2)^2 \mathbb{E} \big[ \big\|\mu^{(m)}_{t,\xi_x} - \mu^{(m)}_{t-1,\xi_x} - (\mu^{(m)}_{t} - \mu^{(m)}_{t-1}) \big\|^2 \big] \nonumber \\
  & \overset{(b)}{\leq} 2(c_{\nu}\alpha_{t-1}^2)^2\sigma^2  +   2\mathbb{E} \big[ \big\|\mu^{(m)}_{t,\xi_x} - \mu^{(m)}_{t-1,\xi_x}\big\|^2 \big] 
\end{align*}
where inequality (a) follows the generalized triangle inequality; (b) and the bounded variance assumption. 
We denote the second term above as $T_{1,2}$, we have:
\begin{align*}
    T_{1,2} &=  2\mathbb{E}\big\|\nabla_x f^{(m)}(x^{(m)}_t, y^{(m)}_t; \mathcal{B}_{f,1}) - \nabla_{xy}g^{(m)}(x^{(m)}_t, y^{(m)}_t; \mathcal{B}_{g,1})u_{t}^{(m)} \nonumber\\
    &\qquad - \big( \nabla_x f^{(m)}(x^{(m)}_{t-1}, y^{(m)}_{t-1}; \mathcal{B}_{f,1}) - \nabla_{xy}g^{(m)}(x^{(m)}_{t-1}, y^{(m)}_{t-1}; \mathcal{B}_{g,1})u_{t-1}^{(m)}\big)\big\|^2 \nonumber\\
    &\leq 4\mathbb{E}\big\|\nabla_x f^{(m)}(x^{(m)}_t, y^{(m)}_t; \mathcal{B}_{f,1}) - \nabla_x f^{(m)}(x^{(m)}_{t-1}, y^{(m)}_{t-1}; \mathcal{B}_{f,1})\big\|^2 \nonumber\\
    &\qquad + 4\mathbb{E}\big\|\nabla_{xy}g^{(m)}(x^{(m)}_t, y^{(m)}_t; \mathcal{B}_{g,1})u_{t}^{(m)}  - \nabla_{xy}g^{(m)}(x^{(m)}_{t-1}, y^{(m)}_{t-1}; \mathcal{B}_{g,1})u_{t-1}^{(m)}\big\|^2 \nonumber\\
    &\leq 4\big(L^2 + \frac{2L_{xy}^2C_f^2}{\mu^2}\big)\mathbb{E}\big[\big\|x^{(m)}_t - x^{(m)}_{t-1}\big\|^2 + \big\|y^{(m)}_t - y^{(m)}_{t-1}\big\|^2\big] + 8L^2\mathbb{E}\big[\big\|u^{(m)}_t - u^{(m)}_{t-1}\big\|^2\big]
\end{align*}
Combine everything together finishes the proof.
\end{proof}

\begin{lemma}
\label{lemma: inner_est_error_storm_multi}
Suppose we have $c_{\omega}\alpha_{t-1}^2 < 1$, then for $t \neq  \bar{t}_s$, with $s \in [S]$, we have:
\begin{align*}
    &\mathbb{E} \big[ \big\|\bar{\omega}_t - \frac{1}{M}\sum_{m=1}^M\nabla_y g^{(m)}(x^{(m)}_{t}, y^{(m)}_{t} ) \big\|^2 \big]\nonumber\\
    & \leq ( 1 - c_{\omega}\alpha_{t-1}^2)\mathbb{E} \big[ \big\| \bar{\omega}_{t-1} - \frac{1}{M}\sum_{m=1}^M\nabla_y g^{(m)}(x^{(m)}_{t-1}, y^{(m)}_{t-1})  \big\|^2 \big] + \frac{2(c_{\omega}\alpha_{t-1}^2)^2\sigma^2}{b_yM} \nonumber\\
    & \qquad + \frac{2L^2}{b_yM^2}\sum_{m=1}^M\mathbb{E}\big[\big\|x^{(m)}_t - x^{(m)}_{t-1}\big\|^2 + \big\|y^{(m)}_t - y^{(m)}_{t-1}\big\|^2\big]
\end{align*}
where the expectation is w.r.t the stochasticity of the algorithm.
\end{lemma}

\begin{proof}
First, we have:
\begin{align*}
    &\mathbb{E} \big[ \big\|\bar{\omega}_t - \frac{1}{M}\sum_{m=1}^M\nabla_y g^{(m)}(x^{(m)}_{t}, y^{(m)}_{t} ) \big\|^2 \big] \nonumber \\
    & =  \mathbb{E} \big[ \big\|\frac{1}{M}\sum_{m=1}^M \big(\nabla_y g^{(m)} (x^{(m)}_{t}, y^{(m)}_{t} ,\mathcal{B}_{y}) \nonumber\\
    &\qquad + ( 1 -  c_{\omega}\alpha_{t-1}^2) (\omega_{t-1}^{(m)} - \nabla_y g^{(m)} (x^{(m)}_{t-1}, y^{(m)}_{t-1} ,\mathcal{B}_{y})) - \nabla_y g^{(m)}(x^{(m)}_{t}, y^{(m)}_{t} ) \big)\big\|^2 \big] \nonumber \\
    & = \mathbb{E} \big[ \big\| ( 1 - c_{\omega}\alpha_{t-1}^2) (\bar{\omega}_{t-1} - \frac{1}{M}\sum_{m=1}^M\nabla_y g^{(m)}(x^{(m)}_{t-1}, y^{(m)}_{t-1}) \nonumber\\
    &\qquad + \frac{1}{M}\sum_{m=1}^M\big(\nabla_y g^{(m)} (x^{(m)}_{t}, y^{(m)}_{t} ,\mathcal{B}_{y}) - \nabla_y g^{(m)}(x^{(m)}_{t}, y^{(m)}_{t} ) \nonumber \\
    & \qquad \qquad +  (1 - c_{\omega}\alpha_{t-1}^2)(\nabla_y g^{(m)}(x^{(m)}_{t-1}, y^{(m)}_{t-1}) - \nabla_y g^{(m)} (x^{(m)}_{t-1}, y^{(m)}_{t-1} ,\mathcal{B}_{y}))\big) \big\|^2 \big] \nonumber \\
    & \overset{(a)}{\leq} ( 1 - c_{\omega}\alpha_{t-1}^2)\mathbb{E} \big[ \big\| \bar{\omega}_{t-1} - \frac{1}{M}\sum_{m=1}^M\nabla_y g^{(m)}(x^{(m)}_{t-1}, y^{(m)}_{t-1})  \big\|^2 \big] \nonumber \\
    & \qquad + \frac{1}{b_y^2M^2}\sum_{m=1}^M\mathbb{E}\sum_{\xi_y \in \mathcal{B}_y} \big[ \big\|\big( \nabla_y g^{(m)} (x^{(m)}_{t}, y^{(m)}_{t} ,\xi_{y}) - \nabla_y g^{(m)}(x^{(m)}_{t}, y^{(m)}_{t} ) \nonumber \\
    & \qquad +  (1 - c_{\omega}\alpha_{t-1}^2)(\nabla_y g^{(m)}(x^{(m)}_{t-1}, y^{(m)}_{t-1}) - \nabla_y g^{(m)} (x^{(m)}_{t-1}, y^{(m)}_{t-1} ,\xi_{y}))\big) \big\|^2 \big]
\end{align*}
where inequality (a) uses the fact that the cross product term is zero in expectation and the condition that $c_{\omega}\alpha_t^2 < 1, t \in [T]$, furthermore, the samples are sampled independently on clients.

We denote the second term in the above inequality as $T_1$, we have:
\begin{align*}
    T_1 & \overset{(b)}{\leq} 2(c_{\omega}\alpha_{t-1}^2)^2\mathbb{E} \big[ \big\| \nabla_y g^{(m)} (x^{(m)}_{t}, y^{(m)}_{t} ,\xi_{y}) - \nabla_y g^{(m)}(x^{(m)}_{t}, y^{(m)}_{t} )\big\|^2\big] \nonumber \\
    & \qquad +  2(1 - c_{\omega}\alpha_{t-1}^2)^2\mathbb{E} \big[\big\| -\nabla_y g^{(m)}(x^{(m)}_{t}, y^{(m)}_{t}) \nonumber \\
    & \qquad + \nabla_y g^{(m)} (x^{(m)}_{t}, y^{(m)}_{t} ,\xi_{y}) +  \nabla_y g^{(m)}(x^{(m)}_{t-1}, y^{(m)}_{t-1}) - \nabla_y g^{(m)} (x^{(m)}_{t-1}, y^{(m)}_{t-1} ,\xi_{y})\big\|^2 \big] \nonumber\\
    & \overset{(c)}{\leq} 2(c_{\omega}\alpha_{t-1}^2)^2\sigma^2 + 2\mathbb{E} \big[ \big\|\nabla_y g^{(m)} (x^{(m)}_{t}, y^{(m)}_{t} ,\xi_{y}) - \nabla_y g^{(m)} (x^{(m)}_{t-1}, y^{(m)}_{t-1} ,\xi_{y})\big\|^2 \big] \nonumber \\
    & \overset{(d)}{\leq} 2(c_{\omega}\alpha_{t-1}^2)^2\sigma^2 + 2L^2\mathbb{E}\big[\big\|x^{(m)}_t - x^{(m)}_{t-1}\big\|^2 + \big\|y^{(m)}_t - y^{(m)}_{t-1}\big\|^2\big]
\end{align*}
inequality (b) uses the generalized triangle inequality; inequality (c) follows the bounded variance assumption~\ref{assumption:noise_assumption}, Proposition~\ref{prop: Sum_Mean_Kron}; inequality (d) uses the smoothness assumption~\ref{assumption:g_smoothness}. 
\end{proof}

\subsubsection{Lower Problem Solution Error}
\begin{lemma}
\label{lemma: inner_drift_storm_multi}
Suppose we choose $\gamma \leq \frac{1}{2L}$ and $\alpha_t < 1$. Then for $t \in [T]$, we have:
\begin{align*}
 \|\bar{y}_{t+1} - y_{\bar{x}_{t+1}}\|^2 &\leq \big(1-\frac{\mu\gamma\alpha_t}{4}\big)\|\bar{y}_t - y_{\bar{x}_t}\|^2 - \frac{\gamma^2\alpha_t}{4} \|\bar{\omega}_t\|^2 +  \frac{9\kappa^2\eta^2\alpha_t}{2\mu\gamma}\|\bar{\nu}_t\|^2 \nonumber\\
  &\qquad +  \frac{9\gamma\alpha_t L^2}{\mu M}\sum_{m=1}^M\big[\|x^{(m)}_t - \bar{x}_t\|^2  + \|y^{(m)}_t - \bar{y}_t\|^2 \big] + \frac{9\gamma\alpha_t}{\mu}\|\frac{1}{M}\sum_{m=1}^M \nabla_y g^{(m)}(x^{(m)}_t,y^{(m)}_t)-\bar{w}_t\|^2
\end{align*}
\end{lemma}

\begin{proof}
First, we exploit Proposition~\ref{prop:strong-prog}, and choose the function $g(\bar{x}_t, \cdot)$, by assumption it is $L$ smooth and $\mu$ strongly convex, and we choose $\gamma < \frac{1}{2L}$ and $\alpha_t < 1$, thus:
\begin{align} \label{eq:E8_multi}
\|\bar{y}_{t+1}-y_{\bar{x}_t}\|^2 & \leq ( 1-\frac{\mu\gamma\alpha_t}{2})\|\bar{y}_t - y_{\bar{x}_t}\|^2 - \frac{\gamma^2\alpha_t}{4} \|\bar{\omega}_t\|^2 + \frac{4\gamma\alpha_t}{\mu}\|\nabla_y g(\bar{x}_t,\bar{y}_t)-\bar{w}_t\|^2.
\end{align}
Next, we decompose the term $\|\bar{y}_{t+1} - y_{\bar{x}_{t+1}}\|^2$ as follows:
\begin{align} \label{eq:E9_multi}
  \|\bar{y}_{t+1} - y_{\bar{x}_{t+1}}\|^2 & \leq (1+\frac{\mu\gamma\alpha_t}{4})\|\bar{y}_{t+1} - y_{\bar{x}_t}\|^2  + (1+\frac{4}{\mu\gamma\alpha_t})\|y_{\bar{x}_t} - y_{\bar{x}_{t+1}}\|^2 \nonumber \\
  & \leq (1+\frac{\mu\gamma\alpha_t}{4})\|\bar{y}_{t+1} - y_{\bar{x}_t}\|^2  + (1+\frac{4}{\mu\gamma\alpha_t})\kappa^2\|\bar{x}_t - \bar{x}_{t+1}\|^2
\end{align}
where the second inequality is due to case a) of Proposition 3.9. Combining the above inequalities \ref{eq:E8_multi} and \ref{eq:E9_multi}, we have
\begin{align*}
 \|\bar{y}_{t+1} - y_{\bar{x}_{t+1}}\|^2 & \leq (1+\frac{\mu\gamma\alpha_t}{4})( 1-\frac{\mu\gamma\alpha_t}{2})\|\bar{y}_t - y_{\bar{x}_t}\|^2
 - (1+\frac{\mu\gamma\alpha_t}{4})\frac{\gamma^2\alpha_t}{4} \|\bar{\omega}_t\|^2    \nonumber \\
 &\quad + (1+\frac{\mu\gamma\alpha_t}{4})\frac{4\gamma\alpha_t}{\mu}\|\nabla_y g(\bar{x}_t,\bar{y}_t)-\bar{w}_t\|^2 + (1+\frac{4}{\mu\gamma\alpha_t})\kappa^2\eta^2\alpha_t^2\|\bar{\nu}_t\|^2
\end{align*}
Since we choose $\gamma \leq \frac{1}{2L}$, $\alpha_t < 1$, we have:
\begin{align*}
  (1+\frac{\mu\gamma\alpha_t}{4})( 1-\frac{\mu\gamma\alpha_t}{2})&= 1-\frac{\mu\gamma\alpha_t}{4} - \frac{\mu^2\gamma^2\alpha_t^2}{8} \leq 1-\frac{\mu\gamma\alpha_t}{4}
 \end{align*}
and $ - (1+\frac{\mu\gamma\alpha_t}{4})\leq -1, (1+\frac{\mu\gamma\alpha_t}{4})\leq\frac{9}{8}$, $\mu\gamma\alpha_t < \frac{1}{2}$.
Thus, we have
\begin{align*}
 \|\bar{y}_{t+1} - y_{\bar{x}_{t+1}}\|^2 & \leq \big(1-\frac{\mu\gamma\alpha_t}{4}\big)\|\bar{y}_t - y_{\bar{x}_t}\|^2
 - \frac{\gamma^2\alpha_t}{4} \|\bar{\omega}_t\|^2 + \frac{9\gamma\alpha_t}{2\mu}\underbrace{\|\nabla_y g(\bar{x}_t,\bar{y}_t)-\bar{w}_t\|^2}_{T_1}
  +  \frac{9\kappa^2\eta^2\alpha_t}{2\mu\gamma}\|\bar{\nu}_t\|^2
\end{align*}
For the term $T_1$ in the inequality above, we have:
\begin{align*}
    \|\nabla_y g(\bar{x}_t,\bar{y}_t)-\bar{w}_t\|^2 &\leq 2\|\nabla_y g(\bar{x}_t,\bar{y}_t)-\frac{1}{M}\sum_{m=1}^M\nabla_y g^{(m)}(x^{(m)}_t,y^{(m)}_t)\|^2 \nonumber\\
    &\qquad + 2\|\frac{1}{M}\sum_{m=1}^M \nabla_y g^{(m)}(x^{(m)}_t,y^{(m)}_t)-\bar{w}_t\|^2 \nonumber\\
    &\leq \frac{2L^2}{M}\sum_{m=1}^M\big[\|x^{(m)}_t - \bar{x}_t\|^2 + \|y^{(m)}_t - \bar{y}_t\|^2 \big] \nonumber\\
    &\qquad+ 2\|\frac{1}{M}\sum_{m=1}^M \nabla_y g^{(m)}(x^{(m)}_t,y^{(m)}_t)-\bar{w}_t\|^2 
\end{align*}
This completes the proof.
\end{proof}

\begin{lemma}
\label{lemma: u_inner_drift_storm_multi}
Suppose we choose $\tau \leq \frac{1}{2L}$ and $\alpha_t < 1$, $r=\frac{C_f}{\mu}$. Then for $t \in [T]$, we have:
\begin{align*}
 \|\bar{u}_{t+1} - u_{\bar{x}_{t+1}}\|^2 &\leq \big(1-\frac{\mu\tau\alpha_t}{4}\big)\|\bar{u}_t - u_{\bar{x}_t}\|^2
 - \frac{\tau^2\alpha_t}{4} \|\bar{q}_t\|^2
  +  \frac{9\kappa^2\eta^2\alpha_t}{2\mu\tau}\|\bar{\nu}_t\|^2 + \frac{9\tau\alpha_t}{\mu}\|\bar{p}-\bar{q}_t\|^2 \nonumber\\
  &+  \frac{18\tau\alpha_t \tilde{L}_2^2}{\mu M}\sum_{m=1}^M\big[\|x^{(m)}_t - \bar{x}_t\|^2 + \|y^{(m)}_t - \bar{y}_t\|^2 \big] + \frac{18\tau\alpha_t L^2}{M}\sum_{m=1}^M\big\|u^{(m)}_t - \bar{u}_t\big\|^2
\end{align*}
where $\tilde{L}_2^2 = (L^2 + \frac{2L_{y^2}^2C_f^2}{\mu^2})$ is a constant.
\end{lemma}

\begin{proof}
First, we exploit Proposition~\ref{prop:strong-prog}, and choose the function $\frac{1}{2} x^T\nabla_{y^2} g(\bar{x}, y_{\bar{x}})x - \nabla_y f(\bar{x}, y_{\bar{x}})^Tx$, by assumption it is $L$ smooth and $\mu$ strongly convex, and we choose $\tau < \frac{1}{2L}$ and $\alpha_t < 1$, thus:
\begin{align*}
\|\bar{u}_{t+1}-u_{\bar{x}_t}\|^2 & \leq ( 1-\frac{\mu\tau\alpha_t}{2})\|\bar{u}_t - u_{\bar{x}_t}\|^2 - \frac{\tau^2\alpha_t}{4} \|\bar{q}_t\|^2 + \frac{4\tau\alpha_t}{\mu}\|\nabla_{y^2} g(\bar{x}, y_{\bar{x}})\bar{u}_t - \nabla_y f(\bar{x}, y_{\bar{x}}) -\bar{q}_t\|^2.
\end{align*}
where we also use the fact that \[\|\bar{u}_{t+1} - u_{\bar{x}_t}\|^2 \leq \|\bar{u}_t - \tau\alpha_t\bar{q}_t - u_{\bar{x}_t}\|^2\]
for $r = \frac{C_f}{\mu} \geq \|u_{\bar{x}_t}\|$.
Next, we decompose the term $\|\bar{u}_{t+1} - u_{\bar{x}_{t+1}}\|^2$ as follows:
\begin{align*}
  \|\bar{u}_{t+1} - u_{\bar{x}_{t+1}}\|^2 & \leq (1+\frac{\mu\tau\alpha_t}{4})\|\bar{u}_{t+1} - u_{\bar{x}_t}\|^2  + (1+\frac{4}{\mu\tau\alpha_t})\|u_{\bar{x}_t} - u_{\bar{x}_{t+1}}\|^2 \nonumber \\
  & \leq (1+\frac{\mu\tau\alpha_t}{4})\|\bar{u}_{t+1} - u_{\bar{x}_t}\|^2  + (1+\frac{4}{\mu\tau\alpha_t})\bar{L}^2\|\bar{x}_t - \bar{x}_{t+1}\|^2
\end{align*}
where the second inequality is due to case a) of Proposition 3.9. Combining the above inequalities \ref{eq:E8_multi} and \ref{eq:E9_multi}, we have:
\begin{align*}
 \|\bar{u}_{t+1} - u_{\bar{x}_{t+1}}\|^2 & \leq (1-\frac{\mu\tau\alpha_t}{4})\|\bar{u}_t - u_{\bar{x}_t}\|^2
 - \frac{\tau^2\alpha_t}{4} \|\bar{q}_t\|^2    \nonumber \\
 &\quad + \frac{9\tau\alpha_t}{2\mu}\underbrace{\|\nabla_{y^2} g(\bar{x}, y_{\bar{x}})\bar{u}_t - \nabla_y f(\bar{x}, y_{\bar{x}}) -\bar{q}_t\|^2}_{T_1} + \frac{9\bar{L}^2\eta^2\alpha_t}{2\mu\tau}\|\bar{\nu}_t\|^2
\end{align*}
where we use the fact that $\tau \leq \frac{1}{2L}$, $\alpha_t < 1$
For the term $T_1$ in the inequality above, we have:
\begin{align*}
    T_1 &\leq 2\|\nabla_{y^2} g(\bar{x}, y_{\bar{x}})\bar{u}_t - \nabla_y f(\bar{x}, y_{\bar{x}})  - \bar{p}_t\|^2 + 2\|\bar{p}_t - \bar{q}_t\|^2 \nonumber\\
    &\leq 2\big\|\nabla_{y^2} g(\bar{x}, y_{\bar{x}})\bar{u}_t - \nabla_y f(\bar{x}, y_{\bar{x}}) \nonumber\\
    &\qquad - \frac{1}{M}\sum_{m=1}^M \big(\nabla_{y^2} g^{(m)}(x^{(m)}_t, y^{(m)}_t)u^{(m)}_{t} + \nabla_y f^{(m)}(x^{(m)}_t, y^{(m)}_t)\big)\big\|^2 + 2\|\bar{p}_t - \bar{q}_t\|^2 
\end{align*}
We denote the first term of the above inequality as $T_{1,1}$, we have:
\begin{align*}
    T_{1,1} &\leq 4\big\|\nabla_y f(\bar{x}, y_{\bar{x}})  - \frac{1}{M}\sum_{m=1}^M \big(\nabla_y f^{(m)}(x^{(m)}_t, y^{(m)}_t)\big)\big\|^2 \nonumber\\
    &\qquad + 4\big\|\nabla_{y^2} g(\bar{x}, y_{\bar{x}})\bar{u}_t - \frac{1}{M}\sum_{m=1}^M \big(\nabla_{y^2} g^{(m)}(x^{(m)}_t, y^{(m)}_t)u^{(m)}_{t}\big)\big\|^2 \nonumber\\
    &\leq \big(\frac{4L^2}{M} + \frac{8L_{y^2}^2C_f^2}{\mu^2 M}\big)\sum_{m=1}^M\big[\big\|x^{(m)}_t - \bar{x}_t\big\|^2 + \big\|y^{(m)}_t - \bar{u}_t\big\|^2 \big] + \frac{4L^2}{M}\sum_{m=1}^M\big\|u^{(m)}_t - \bar{u}_t\big\|^2
\end{align*}
Combine everything completes the proof.
\end{proof}

\subsubsection{Upper Variable Drift}
\begin{lemma}
\label{lem: x_drift_Storm_Multi}
For any $t \neq \bar{t}_s, s\in[S]$, we have:
\begin{align*}
\|x_t^{(m)}-  \bar{x}_t \|^2 &\leq I\eta^2 \sum_{\ell = \bar{t}_{s-1}}^{t-1} \alpha_{\ell}^2\big\|   \nu_\ell^{(m)} - \bar{\nu}_\ell\big\|^2 \nonumber\\
\|y_t^{(m)}-  \bar{y}_t \|^2 &\leq I\gamma^2 \sum_{\ell = \bar{t}_{s-1}}^{t-1} \alpha_{\ell}^2\big\|   \omega_\ell^{(m)} - \bar{\omega}_\ell\big\|^2 \nonumber\\
\|u_t^{(m)}-  \bar{u}_t \|^2 &\leq I\tau^2 \sum_{\ell = \bar{t}_{s-1}}^{t-1} \alpha_{\ell}^2\big\|   q_\ell^{(m)} - \bar{q}_\ell\big\|^2 \nonumber\\
\end{align*}
\end{lemma}

\begin{proof}
Note from Algorithm and the definition of $\bar{t}_s$ that at $t = \bar{t}_{s}$ with $s \in [S]$, $x_{t}^{(m)} = \bar{x}_{t}$, for all $k$. For $t \neq \bar{t}_s$, with $s \in [S]$, we have: $x_{t}^{(m)} = x_{t-1}^{(m)} - \eta\alpha_{t-1}  \nu_{t-1}^{(m)}$, this implies that: $x_t^{(m)} = x_{\bar{t}_{s-1}}^{(m)} - \sum_{\ell = \bar{t}_{s-1}}^{t-1} \eta\alpha_{\ell}  \nu_\ell^{(m)} \quad \text{and} \quad \bar{x}_{t}  = \bar{x}_{\bar{t}_{s-1}}  - \sum_{\ell = \bar{t}_{s-1}}^{t-1} \eta\alpha_{\ell}  \bar{\nu}_\ell.$
So for $t \neq \bar{t}_s$, with $s \in [S]$ we have:
\begin{align*}
\|x_t^{(m)}-  \bar{x}_t \|^2 & =  \big\| x_{\bar{t}_{s-1}}^{(m)} - \bar{x}_{\bar{t}_{s-1}}  - \big( \sum_{\ell = \bar{t}_{s-1}}^{t-1} \eta\alpha_{\ell}  \nu_\ell^{(m)} -   \sum_{\ell =  \bar{t}_{s-1}}^{t-1} \eta\alpha_{\ell}  \bar{\nu}_\ell  \big) \big\|^2 =  \big\|  \sum_{\ell = \bar{t}_{s-1}}^{t-1} \eta\alpha_{\ell}\big(  \nu_\ell^{(m)} -      \bar{\nu}_\ell  \big) \big\|^2 \nonumber\\
&\leq I\eta^2 \sum_{\ell = \bar{t}_{s-1}}^{t-1} \alpha_{\ell}^2\big\|   \nu_\ell^{(m)} -      \bar{\nu}_\ell\big\|^2
\end{align*}
We can derive the bound for $\|y_t^{(m)}-  \bar{y}_t \|^2$ and$\|u_t^{(m)}-  \bar{u}_t \|^2$  similarly. This completes the proof.
\end{proof}

\begin{lemma}\label{lem: ErrorAccumulation_Iterates_storm_mu_multi}
Suppose $\eta\alpha_t < \frac{1}{16I\tilde{L}_1}$, then for $t \neq \bar{t}_s, s\in[S]$, we have:
\begin{align*}
    & \sum_{m=1}^M \mathbb{E} \| \hat{\nu}_{t}^{(m)} - \bar{\nu}_{t} \|^2 \nonumber \\
    &\leq \left(1 + \frac{17}{16I}\right) \sum_{m=1}^M \mathbb{E} \|  \nu_{t-1}^{(m)}  - \bar{\nu}_{t-1} \|^2  + 8 I \tilde{L}_1^2\alpha_{t-1}^2 \sum_{m=1}^M \mathbb{E}\big[2\|\eta\bar{\nu}_{t-1} \|^2 + \| \gamma \omega^{(m)}_{t-1} \|^2 \big] + 16IL^2\alpha_{t-1}^2\sum_{m=1}^M\mathbb{E}\| \tau q_{t-1}^{(m)} \|^2 \nonumber \\
    & \qquad  + 128I(c_{\nu}\alpha_{t-1}^2)^2\tilde{L}_1^2\sum_{m=1}^M\mathbb{E}\big[\big\| x^{(m)}_t - \bar{x}_t\big\|^2 + \big\|y^{(m)}_t - \bar{y}_t\big\|^2\big]  + 32I(c_{\nu}\alpha_{t-1}^2)^2L^2\sum_{m=1}^M\mathbb{E}\big\|u_{t}^{(m)}  - \bar{u}_{t}\big\|^2 \nonumber\\
    &\qquad + 8I M (c_{\nu}\alpha_{t-1}^2)^2\frac{\sigma^2}{b_x} + 32IM(c_{\nu}\alpha_{t-1}^2)^2\zeta_f^2 + 64I(c_{\nu}\alpha_{t-1}^2)^2M\frac{C_f^2\zeta_{g, xy}^2}{\mu^2}
\end{align*}
where the expectation is w.r.t the stochasticity of the algorithm.
\end{lemma}

\begin{proof}
For $t \neq \bar{t}_s$, we have:
\begin{align}\label{eq:nu_drift_storm_multi}
\mathbb{E} \| \hat{\nu}_{t}^{(m)} - \bar{\nu}_{t} \|^2
& = \mathbb{E} \big\| \mu^{(m)}_{t,\mathcal{B}_x}+ (1 - c_{\nu}\alpha_{t-1}^2) \big( \nu_{t-1}^{(m)} -   \mu^{(m)}_{t-1,\mathcal{B}_x}\big) - \big(\bar{\mu}_{t,\mathcal{B}_x} + (1 - c_{\nu}\alpha_{t-1}^2) \big( \bar{\nu}_{t-1} -  \bar{\mu}_{t-1,\mathcal{B}_x}\big) \big)   \big\|^2 \nonumber \\
& = \mathbb{E} \big\| (1 - c_{\nu}\alpha_{t-1}^2) \big( \nu_{t-1}^{(m)}  - \bar{\nu}_{t-1} \big) +  \mu^{(m)}_{t,\mathcal{B}_x} -   \bar{\mu}_{t,\mathcal{B}_x} - (1- c_{\nu}\alpha_{t-1}^2) \big(  \mu^{(m)}_{t-1,\mathcal{B}_x} -  \bar{\mu}_{t-1,\mathcal{B}_x} \big)  \big\|^2 \nonumber \\
& \overset{(a)}{\leq} (1 + \frac{1}{I}) (1 - c_{\nu}\alpha_{t-1}^2)^2 \mathbb{E} \|  \nu_{t-1}^{(m)}  - \bar{\nu}_{t-1} \|^2 \nonumber \\
& \qquad + \big( 1 + I \big) \mathbb{E} \big\|  \mu^{(m)}_{t,\mathcal{B}_x}-  \bar{\mu}_{t,\mathcal{B}_x}  - (1- c_{\nu}\alpha_{t-1}^2) \big(  \mu^{(m)}_{t-1,\mathcal{B}_x} - \bar{\mu}_{t-1,\mathcal{B}_x} \big)  \big\|^2 \nonumber\\
& \leq \left(1 + \frac{1}{I}\right) \mathbb{E} \|  \nu_{t-1}^{(m)}  - \bar{\nu}_{t-1} \|^2 + \big( 1 + I \big) \mathbb{E} \big\|  \mu^{(m)}_{t,\mathcal{B}_x}-  \bar{\mu}_{t,\mathcal{B}_x}  - (1- c_{\nu}\alpha_{t-1}^2) \big(  \mu^{(m)}_{t-1,\mathcal{B}_x} -  \bar{\mu}_{t-1,\mathcal{B}_x} \big)  \big\|^2
\end{align}
where $(a)$ follows from the the generalized triangle inequality. 

\noindent Next we bound the second term of the above inequality:
\begin{align*}
& \sum_{m=1}^M\mathbb{E}\big\|  \mu^{(m)}_{t,\mathcal{B}_x} -  \bar{\mu}_{t,\mathcal{B}_x}  - (1- c_{\nu}\alpha_{t-1}^2) \big(  \mu^{(m)}_{t-1,\mathcal{B}_x} -  \bar{\mu}_{t-1,\mathcal{B}_x} \big)  \big\|^2 \nonumber\\
& \leq  2 \sum_{m=1}^M\mathbb{E}\big\| \mu^{(m)}_{t,\mathcal{B}_x} - \bar{\mu}_{t,\mathcal{B}_x} -  \big( \mu^{(m)}_{t-1,\mathcal{B}_x} - \bar{\mu}_{t-1,\mathcal{B}_x} \big)\big\|^2 + 2 (c_{\nu}\alpha_{t-1}^2)^2 \sum_{m=1}^M\mathbb{E} \big\|  \mu^{(m)}_{t-1,\mathcal{B}_x} - \bar{\mu}_{t-1,\mathcal{B}_x}  \big\|^2
\end{align*}
where the inequality follows the triangle inequality. We bound the two terms separately, for the first term, we have:
\begin{align}\label{eq:mu_ave_drift_storm_multi}
&\sum_{m=1}^M\mathbb{E}\big\| \mu^{(m)}_{t,\mathcal{B}_x} -  \bar{\mu}_{t,\mathcal{B}_x} -  \big( \mu^{(m)}_{t-1,\mathcal{B}_x} -  \bar{\mu}_{t-1,\mathcal{B}_x} \big)\big\|^2 \overset{(a)}{\leq}  \sum_{m=1}^M\mathbb{E} \big\| \mu^{(m)}_{t,\mathcal{B}_x} - \mu^{(m)}_{t-1,\mathcal{B}_x} \big\|^2 \nonumber \\
&\leq \sum_{m=1}^M\mathbb{E}\big\| \nabla_x f^{(m)}(x^{(m)}_t, y^{(m)}_t; \xi_{f,1}) - \nabla_{xy}g^{(m)}(x^{(m)}_t, y^{(m)}_t; \xi_{g,1})u_{t}^{(m)} \nonumber\\
&\qquad - \big(\nabla_x f^{(m)}(x^{(m)}_{t-1}, y^{(m)}_{t-1}; \xi_{f,1}) - \nabla_{xy}g^{(m)}(x^{(m)}_{t-1}, y^{(m)}_{t-1}; \xi_{g,1})u_{t-1}^{(m)}\big)\big\|^2 \nonumber\\
& \overset{(b)}{\leq}  2\big(L^2 + \frac{2L_{xy}^2C_f^2}{\mu^2}\big) \sum_{m=1}^M\mathbb{E}\big[ \| x_t^{(m)} - x_{t-1}^{(m)} \|^2 + \| y_t^{(m)} - y_{t-1}^{(m)} \|^2 \big] + 4L^2\sum_{m=1}^M\mathbb{E}\| u_t^{(m)} - u_{t-1}^{(m)} \|^2 \nonumber\\
&\leq 2\tilde{L}_1^2\alpha_{t-1}^2 \sum_{m=1}^M \mathbb{E}\big[ \| \eta\nu^{(m)}_{t-1} \|^2 + \| \gamma \omega^{(m)}_{t-1} \|^2 \big] + 4L^2\alpha_{t-1}^2\sum_{m=1}^M\mathbb{E}\| \tau q_{t-1}^{(m)} \|^2
\end{align}
where $(a)$ follows Proposition~\ref{prop: Sum_Mean_Kron}; $(b)$ follows Proposition~\ref{some smoothness} and the fact that $\hat{x}_t^{(m)} = x_t^{(m)}$ when $t \neq \bar{t}_s$; Next for the second term, we have:
\begin{align}\label{eq:mu_drift_storm_multi}
&\sum_{m=1}^M \mathbb{E} \big\|  \mu^{(m)}_{t-1,\mathcal{B}_x} -  \bar{\mu}_{t-1,\mathcal{B}_x}  \big\|^2 = \sum_{m=1}^M \mathbb{E}\big\| \mu^{(m)}_{t-1,\mathcal{B}_x} - \mu^{(m)}_{t-1}  - \big(  \bar{\mu}_{t-1,\mathcal{B}_x} - \bar{\mu}_{t-1}  \big)   + \mu^{(m)}_{t-1} - \bar{\mu}_{t-1} \big\|^2 
\nonumber \\
& \overset{(a)}{\leq}  2 \sum_{m=1}^M\mathbb{E} \big\|  \mu^{(m)}_{t-1,\mathcal{B}_x} - \mu^{(m)}_{t-1}  - \big(  \bar{\mu}_{t-1,\mathcal{B}_x} - \bar{\mu}_{t-1}  \big) \big\|^2 + 2\sum_{m=1}^M\mathbb{E} \big\| \mu^{(m)}_{t-1} -  \bar{\mu}_{t-1}\big\|^2 \nonumber \\
& \overset{(b)}{\leq}   2\underbrace{\sum_{m=1}^M \mathbb{E} \big\|  \mu^{(m)}_{t-1,\mathcal{B}_x} - \mu^{(m)}_{t-1} \big\|^2}_{T_1} + 2\underbrace{\sum_{m=1}^M\mathbb{E} \big\| \mu^{(m)}_{t-1} -  \bar{\mu}_{t-1}\big\|^2}_{T_2}
\end{align}
Note for the term $T_1$ of Eq.~\ref{eq:mu_drift_storm_multi}, we have $\mathbb{E} \big\|  \mu^{(m)}_{t-1,\mathcal{B}_x} - \mu^{(m)}_{t-1} \big\|^2 \leq \frac{\sigma^2}{b_x}$ by the bounded variance assumption; Next for the term $T_2$, we have:
\begin{align*}
    T_2 &= \sum_{m=1}^M\big\|\nabla_x f^{(m)}(x^{(m)}_{t-1}, y^{(m)}_{t-1}) - \nabla_{xy}g^{(m)}(x^{(m)}_{t-1}, y^{(m)}_{t-1})u_{t-1}^{(m)} \nonumber\\
    &\qquad -\frac{1}{M}\sum_{j=1}^M\big(\nabla_x f^{(j)}(x^{(j)}_{t-1}, y^{(j)}_{t-1}) - \nabla_{xy}g^{(j)}(x^{(j)}_{t-1}, y^{(j)}_{t-1})u_{t-1}^{(j)} \big)\big\|^2 \nonumber\\
    &\leq  16\big(L^2 + \frac{2L_{xy}^2C_f^2}{\mu^2}\big)\sum_{m=1}^M\mathbb{E}\big[\big\| x^{(m)}_t - \bar{x}_t\big\|^2 + \big\|y^{(m)}_t - \bar{y}_t\big\|^2\big] + 4L^2\sum_{m=1}^M\mathbb{E}\big\|u_{t}^{(m)}  - \bar{u}_{t}\big\|^2 + 4M\zeta_f^2 + \frac{8MC_f^2\zeta_{g, xy}^2}{\mu^2}
\end{align*}
Finally, combine Eq.~\ref{eq:mu_ave_drift_storm_multi}, Eq.~\ref{eq:mu_drift_storm_multi} with Eq.~\ref{eq:nu_drift_storm_multi} and use the fact that $I\geq 1$, we have:
\begin{align*}
    & \sum_{m=1}^M \mathbb{E} \| \hat{\nu}_{t}^{(m)} - \bar{\nu}_{t} \|^2 \nonumber \\
    & \leq \big(1 + \frac{1}{I}\big) \sum_{m=1}^M \mathbb{E} \|  \nu_{t-1}^{(m)}  - \bar{\nu}_{t-1} \|^2  + 8 I \tilde{L}_1^2\alpha_{t-1}^2 \sum_{m=1}^M \mathbb{E}\big[ \underbrace{\| \eta\nu^{(m)}_{t-1} \|^2}_{T_1} + \| \gamma \omega^{(m)}_{t-1} \|^2 \big] + 16IL^2\alpha_{t-1}^2\sum_{m=1}^M\mathbb{E}\| \tau q_{t-1}^{(m)} \|^2 \nonumber \\
    & \qquad  + 128I(c_{\nu}\alpha_{t-1}^2)^2\tilde{L}_1^2\sum_{m=1}^M\mathbb{E}\big[\big\| x^{(m)}_t - \bar{x}_t\big\|^2 + \big\|y^{(m)}_t - \bar{y}_t\big\|^2\big]  + 32I(c_{\nu}\alpha_{t-1}^2)^2L^2\sum_{m=1}^M\mathbb{E}\big\|u_{t}^{(m)}  - \bar{u}_{t}\big\|^2 \nonumber\\
    &\qquad + 8I M (c_{\nu}\alpha_{t-1}^2)^2\frac{\sigma^2}{b_x} + 32IM(c_{\nu}\alpha_{t-1}^2)^2\zeta_f^2 + 64I(c_{\nu}\alpha_{t-1}^2)^2M\frac{C_f^2\zeta_{g, xy}^2}{\mu^2}
\end{align*}
We separate the term $T_1$ with triangle inequality to get:
\begin{align*}
    & \sum_{m=1}^M \mathbb{E} \| \hat{\nu}_{t}^{(m)} - \bar{\nu}_{t} \|^2 \nonumber \\
    &\leq \left(1 + \frac{1}{I} + 16 I \tilde{L}_1^2\eta^2\alpha_{t-1}^2\right) \sum_{m=1}^M \mathbb{E} \|  \nu_{t-1}^{(m)}  - \bar{\nu}_{t-1} \|^2 \nonumber\\&\qquad  + 8 I \tilde{L}_1^2\alpha_{t-1}^2 \sum_{m=1}^M \mathbb{E}\big[2\|\eta\bar{\nu}_{t-1} \|^2 + \| \gamma \omega^{(m)}_{t-1} \|^2 \big] + 16IL^2\alpha_{t-1}^2\sum_{m=1}^M\mathbb{E}\| \tau q_{t-1}^{(m)} \|^2 \nonumber \\
    & \qquad  + 128I(c_{\nu}\alpha_{t-1}^2)^2\tilde{L}_1^2\sum_{m=1}^M\mathbb{E}\big[\big\| x^{(m)}_t - \bar{x}_t\big\|^2 + \big\|y^{(m)}_t - \bar{y}_t\big\|^2\big]  + 32I(c_{\nu}\alpha_{t-1}^2)^2L^2\sum_{m=1}^M\mathbb{E}\big\|u_{t}^{(m)}  - \bar{u}_{t}\big\|^2 \nonumber\\
    &\qquad + 8I M (c_{\nu}\alpha_{t-1}^2)^2\frac{\sigma^2}{b_x} + 32IM(c_{\nu}\alpha_{t-1}^2)^2\zeta_f^2 + 64I(c_{\nu}\alpha_{t-1}^2)^2M\frac{C_f^2\zeta_{g, xy}^2}{\mu^2}
\end{align*}
This completes the proof.
\end{proof}

\begin{lemma}\label{lem: ErrorAccumulation_Iterates_storm_omega_multi}
Suppose $\gamma\alpha_t < \frac{1}{16IL}$, then for $t \neq \bar{t}_s, s\in[S]$, we have:
\begin{align*}
    \sum_{m=1}^M \mathbb{E} \|\omega_{t}^{(m)} - \bar{\omega}_{t} \|^2
    &\leq \left(1 + \frac{33}{32I}\right) \sum_{m=1}^M \mathbb{E} \|  \omega_{t-1}^{(m)}  - \bar{\omega}_{t-1} \|^2  + 4 IL^2\alpha_{t-1}^2 \sum_{m=1}^M \mathbb{E}\big[2\| \gamma\bar{\omega}_{t-1} \|^2 + \| \eta \nu^{(m)}_{t-1} \|^2 \big] \nonumber \\
    & \qquad + 8I M (c_{\omega}\alpha_{t-1}^2)^2\frac{\sigma^2}{b_y} + 16I M (c_{\omega}\alpha_{t-1}^2)^2\zeta_g^2  +  16IL^2(c_{\omega}\alpha_{t-1}^2)^2\sum_{m = 1}^M \mathbb{E}\big[\| x_{t - 1}^{(m)} - \bar{x}_{t-1}\|^2\big]\nonumber\\ 
    & \qquad  + 16IL^2 (c_{\omega}\alpha_{t-1}^2)^2\sum_{m = 1}^M \mathbb{E}\big[\| y^{(m)}_{t - 1} - \bar{y}_{t-1}\|^2 \big]
\end{align*}
where the expectation is w.r.t the stochasticity of the algorithm.
\end{lemma}

\begin{proof}
By the update step in Line 7 of Algorithm~\ref{alg:FedBiOAcc}, for $t \neq \bar{t}_s$, we have:
\begin{align}\label{eq:omega_drift_storm_multi1}
\mathbb{E} \| \hat{\omega}_{t}^{(m)} - \bar{\omega}_{t} \|^2
& = \mathbb{E} \big\| (1 - c_{\omega}\alpha_{t-1}^2) \big( \omega_{t-1}^{(m)}  - \bar{\omega}_{t-1} \big) +  \nabla_y g^{(m)} (x^{(m)}_{t}, y^{(m)}_{t} ,\mathcal{B}_{y}) - \frac{1}{M} \sum_{j=1}^M  \nabla_y g^{(j)} (x^{(j)}_{t}, y^{(j)}_{t} ,\mathcal{B}_{y}) \nonumber\\
&\qquad - (1- c_{\omega}\alpha_{t-1}^2) \big(  \nabla_y g^{(m)} (x^{(m)}_{t-1}, y^{(m)}_{t-1} ,\mathcal{B}_{y}) - \frac{1}{M} \sum_{j=1}^M  \nabla_y g^{(j)} (x^{(j)}_{t-1}, y^{(j)}_{t-1} ,\mathcal{B}_{y}) \big)  \big\|^2 \nonumber \\
& \leq (1 + \frac{1}{I}) (1 - c_{\omega}\alpha_{t-1}^2)^2 \mathbb{E} \|  \omega_{t-1}^{(m)}  - \bar{\omega}_{t-1} \|^2 \nonumber \\
& \qquad + ( 1 + I) \mathbb{E} \big\|  \nabla_y g^{(m)} (x^{(m)}_{t}, y^{(m)}_{t} ,\mathcal{B}_{y}) - \frac{1}{M} \sum_{j=1}^M  \nabla_y g^{(j)} (x^{(j)}_{t}, y^{(j)}_{t} ,\mathcal{B}_{y}) \nonumber\\
&\qquad - (1- c_{\omega}\alpha_{t-1}^2) \big(  \nabla_y g^{(m)} (x^{(m)}_{t-1}, y^{(m)}_{t-1} ,\mathcal{B}_{y}) - \frac{1}{M} \sum_{j=1}^M  \nabla_y g^{(j)} (x^{(j)}_{t-1}, y^{(j)}_{t-1} ,\mathcal{B}_{y}) \big) \big\|^2
\end{align}
where the inequality follows from the the generalized triangle inequality and the condition that $c_{\omega}\alpha_{t}^2 < 1$. 

\noindent Next we denote the second term in Eq.~\ref{eq:omega_drift_storm_multi1} as $T_1$, then we have:
\begin{align*}
T_1 & \leq  2 \sum_{m=1}^M\mathbb{E}\big\| \nabla_y g^{(m)} (x^{(m)}_{t}, y^{(m)}_{t} ,\mathcal{B}_{y})- \frac{1}{M} \sum_{j=1}^M \nabla_y g^{(m)} (x^{(j)}_{t}, y^{(j)}_{t} ,\mathcal{B}_{y}) \nonumber\\
&\qquad -  \big( \nabla_y g^{(m)} (x^{(m)}_{t-1}, y^{(m)}_{t-1} ,\mathcal{B}_{y}) - \frac{1}{M} \sum_{j=1}^M  \nabla_y g^{(j)} (x^{(j)}_{t-1}, y^{(m)}_{t-1} ,\mathcal{B}_{y}) \big)\big\|^2 \nonumber\\
&\qquad + 2 (c_{\omega}\alpha_{t-1}^2)^2 \sum_{m=1}^M\mathbb{E} \big\|  \nabla_y g^{(m)} (x^{(m)}_{t-1}, y^{(m)}_{t-1} ,\mathcal{B}_{y}) - \frac{1}{M} \sum_{j=1}^M  \nabla_y g^{(j)} (x^{(j)}_{t-1}, y^{(j)}_{t-1} ,\mathcal{B}_{y}) \big\|^2
\end{align*}
We bound the two terms separately, we denote them as $T_{1,1}$ and $T_{1,2}$ separately, then we have:
\begin{align}\label{eq:omega_ave_drift_storm_multi}
T_{1,1} &\overset{(a)}{\leq}  \sum_{m=1}^M\mathbb{E} \big\| \nabla_y g^{(m)} (x^{(m)}_{t}, y^{(m)}_{t} ,\mathcal{B}_{y}) - \nabla_y g^{(m)} (x^{(m)}_{t-1}, y^{(m)}_{t-1} ,\mathcal{B}_{y}) \big\|^2 \nonumber \\
& \overset{(b)}{\leq}  L^2 \sum_{m=1}^M\mathbb{E}\big[ \| x_t^{(m)} - x_{t-1}^{(m)} \|^2 + \| y_t^{(m)} - y_{t-1}^{(m)} \|^2 \big] 
\leq L^2\alpha_{t-1}^2 \sum_{m=1}^M \mathbb{E}\big[ \| \eta\nu^{(m)}_{t-1} \|^2 + \| \gamma \omega^{(m)}_{t-1} \|^2 \big]
\end{align}
where $(a)$ follows Proposition~\ref{prop: Sum_Mean_Kron}; $(b)$ follows Proposition~\ref{some smoothness}.b) and the fact that $\hat{x}_t^{(m)} = x_t^{(m)}$ and $\hat{y}^{(m)}_{t} = y^{(m)}_t$ when $t \neq \bar{t}_s$; Next for the second term, we have:
\begin{align}\label{eq:omega_drift_storm_multi2}
T_{1,2} &= \sum_{m=1}^M \mathbb{E}\big\| \nabla_y g^{(m)} (x^{(m)}_{t-1}, y^{(m)}_{t-1} ,\mathcal{B}_{y}) - \nabla_y g^{(m)} (x^{(m)}_{t-1}, y^{(m)}_{t-1}) \nonumber\\
&\qquad - \frac{1}{M} \sum_{j=1}^M \big( \nabla_y g^{(j)} (x^{(j)}_{t-1}, y^{(j)}_{t-1} ,\mathcal{B}_{y}) - \nabla_y g^{(j)} (x^{(j)}_{t-1}, y^{(j)}_{t-1}) \big) \nonumber\\
&\qquad + \nabla_y g^{(m)} (x^{(m)}_{t-1}, y^{(m)}_{t-1}) - \frac{1}{M}\sum_{j=1}^M\nabla_y g^{(j)} (x^{(j)}_{t-1}, y^{(j)}_{t-1})\big\|^2 
\nonumber \\
& \overset{(b)}{\leq}   2\sum_{m=1}^M \mathbb{E} \big\|  \nabla_y g^{(m)} (x^{(m)}_{t-1}, y^{(m)}_{t-1} ,\mathcal{B}_{y}) - \nabla_y g^{(m)} (x^{(m)}_{t-1}, y^{(m)}_{t-1})  \big\|^2 \nonumber\\
&\qquad + 4\sum_{m=1}^M \frac{1}{M} \sum_{j=1}^M  \mathbb{E}  \|  \nabla g^{(m)} (\bar{x}_{t-1}, \bar{y}_{t-1}) - \nabla_y g^{(j)} (\bar{x}_{t-1}, \bar{y}_{t-1})  \|^2 \nonumber\\
&\qquad  + 4 \sum_{m=1}^M \mathbb{E} \big\|\nabla_y g^{(m)} (x^{(m)}_{t-1}, y^{(m)}_{t-1}) - \nabla_y g^{(m)} (\bar{x}_{t-1}, \bar{y}_{t-1}) \nonumber\\
&\qquad + \frac{1}{M}\sum_{j=1}^M\nabla_y g^{(j)} (\bar{x}_{t-1}, \bar{y}_{t-1}) - \nabla_y g^{(j)} (x^{(j)}_{t-1}, y^{(j)}_{t-1})  \big\|^2
\end{align}
We denote the three terms above as $T_{1,2,1} - T_{1,2,3}$ respectively. For the term $T_{1,2,1}$ of Eq.~\ref{eq:omega_drift_storm_multi2}, we have $T_{1,2,1} \leq 2M\sigma^2/b_y$ by the bounded variance assumption; For the term $T_{1,2,2}$ of Eq.~\ref{eq:omega_drift_storm_multi2}, by the bounded intra-node heterogeneity assumption we have $T_{1,2,2} \leq 4M\zeta_g^2$. Finally, For the term $T_{1,2,3}$ of Eq.~\ref{eq:omega_drift_storm_multi2}:
\begin{align*}
T_{1,2,3} &\leq  4\sum_{m=1}^M \mathbb{E} \big\| \nabla_y g^{(m)} (x^{(m)}_{t-1}, y^{(m)}_{t-1}) - \nabla_y g^{(m)} (\bar{x}_{t-1}, \bar{y}_{t-1})  \big\|^2  \nonumber\\
& \leq 4L^2 \sum_{m = 1}^M \mathbb{E}\big[\| x_{t - 1}^{(m)} - \bar{x}_{t-1}\|^2\big] + 4L^2\sum_{m = 1}^M \mathbb{E}\big[\| y^{(m)}_{t - 1} - \bar{y}_{t-1}\|^2 \big]
\end{align*}
Finally, combine Eq.~\ref{eq:omega_drift_storm_multi1}, Eq.~\ref{eq:omega_ave_drift_storm_multi} with Eq.~\ref{eq:omega_drift_storm_multi2} and use the fact that $I\geq 1$, we have:
\begin{align*}
    \sum_{m=1}^M \mathbb{E} \| \hat{\omega}_{t}^{(m)} - \bar{\omega}_{t} \|^2
    &\leq \big(1 + \frac{1}{I}\big) \sum_{m=1}^M \mathbb{E} \|  \omega_{t-1}^{(m)}  - \bar{\omega}_{t-1} \|^2  + 4 I L^2\alpha_{t-1}^2 \sum_{m=1}^M \mathbb{E}\big[ \underbrace{\| \gamma\omega^{(m)}_{t-1} \|^2}_{T_1} + \| \eta \nu^{(m)}_{t-1} \|^2 \big] \nonumber \\
    & \qquad + 8I M (c_{\omega}\alpha_{t-1}^2)^2\frac{\sigma^2}{b_y} \nonumber\\
    &\qquad + 16I M (c_{\omega}\alpha_{t-1}^2)^2\zeta_g^2  +  16IL^2(c_{\omega}\alpha_{t-1}^2)^2\sum_{m = 1}^M \mathbb{E}\big[\| x_{t - 1}^{(m)} - \bar{x}_{t-1}\|^2\big]\nonumber\\ 
    & \qquad  + 16IL^2 (c_{\omega}\alpha_{t-1}^2)^2\sum_{m = 1}^M \mathbb{E}\big[\| y^{(m)}_{t - 1} - \bar{y}_{t-1}\|^2 \big]
\end{align*}
We separate the term $T_1$ with triangle inequality to get:
\begin{align*}
    \sum_{m=1}^M \mathbb{E} \| \hat{\omega}_{t}^{(m)} - \bar{\omega}_{t} \|^2
    &\leq \left(1 + \frac{1}{I} + 8 IL^2\gamma^2\alpha_{t-1}^2\right) \sum_{m=1}^M \mathbb{E} \|  \omega_{t-1}^{(m)}  - \bar{\omega}_{t-1} \|^2 \nonumber\\
    &\qquad + 4 I L^2\alpha_{t-1}^2 \sum_{m=1}^M \mathbb{E}\big[2\| \gamma\bar{\omega}_{t-1} \|^2 + \| \eta \nu^{(m)}_{t-1} \|^2 \big] \nonumber \\
    & \qquad + 8I M (c_{\omega}\alpha_{t-1}^2)^2\frac{\sigma^2}{b_y} \nonumber\\
    &\qquad + 16I M (c_{\omega}\alpha_{t-1}^2)^2\zeta_g^2  +  16IL^2(c_{\omega}\alpha_{t-1}^2)^2\sum_{m = 1}^M \mathbb{E}\big[\| x_{t - 1}^{(m)} - \bar{x}_{t-1}\|^2\big]\nonumber\\ 
    &\qquad  + 16IL^2 (c_{\omega}\alpha_{t-1}^2)^2\sum_{m = 1}^M \mathbb{E}\big[\| y^{(m)}_{t - 1} - \bar{y}_{t-1}\|^2 \big]
\end{align*}
This completes the proof.
\end{proof}

\begin{lemma}\label{lem: ErrorAccumulation_Iterates_storm_u_multi}
Suppose $\tau\alpha_t < \frac{1}{32IL}$, then for $t \neq \bar{t}_s, s\in[S]$, we have:
\begin{align*}
    \sum_{m=1}^M \mathbb{E} \| \hat{q}_{t}^{(m)} - \bar{q}_{t} \|^2
    &\leq \left(1 + \frac{33}{32I}\right) \sum_{m=1}^M \mathbb{E} \|  q_{t-1}^{(m)}  - \bar{q}_{t-1} \|^2  +  8I \tilde{L}_2^2\alpha_{t-1}^2 \sum_{m=1}^M \mathbb{E}\big[\| \gamma \omega^{(m)}_{t-1} \|^2 + \| \eta \nu^{(m)}_{t-1} \|^2 \big] \nonumber \\
    & \qquad + 32IL^2\alpha_{t-1}^2\sum_{m=1}^M\mathbb{E} \| \tau^2 \bar{q}_{t-1} \|^2  + 8I M (c_{u}\alpha_{t-1}^2)^2\frac{\sigma^2}{b_x} \nonumber\\
    &\qquad + 16I M (c_{u}\alpha_{t-1}^2)^2\zeta_f^2 + 32IM(c_{u}\alpha_{t-1}^2)^2\frac{C_f^2\zeta_{g, yy}^2}{\mu^2} \nonumber\\ 
    & \qquad  +  64I(c_{u}\alpha_{t-1}^2)^2\tilde{L}_2^2\sum_{m=1}^M\mathbb{E}\big[\big\| x^{(m)}_t - \bar{x}_t\big\|^2 + \big\|y^{(m)}_t - \bar{y}_t\big\|^2\big] \nonumber\\
    &\qquad + 16I(c_{u}\alpha_{t-1}^2)^2L^2\sum_{m=1}^M\mathbb{E}\big\|u_{t}^{(m)}  - \bar{u}_{t}\big\|^2
\end{align*}
where the expectation is w.r.t the stochasticity of the algorithm.
\end{lemma}

\begin{proof}
For $t \neq \bar{t}_s$, we have:
\begin{align*}
\mathbb{E} \| \hat{q}_{t}^{(m)} - \bar{q}_{t} \|^2
& = \mathbb{E} \big\| (1 - c_{u}\alpha_{t-1}^2) \big( q_{t-1}^{(m)}  - \bar{q}_{t-1} \big) +  p_{t, \mathcal{B}_x}^{(m)} -  \bar{p}_{t, \mathcal{B}_x}  - (1- c_{u}\alpha_{t-1}^2) \big(  p_{t-1, \mathcal{B}_x}^{(m)}- \bar{p}_{t-1, \mathcal{B}_x} \big)  \big\|^2 \nonumber \\
& \leq (1 + \frac{1}{I}) (1 - c_{u}\alpha_{t-1}^2)^2 \mathbb{E} \|  q_{t-1}^{(m)}  - \bar{q}_{t-1} \|^2 \nonumber \\
& \qquad + ( 1 + I) \mathbb{E} \big\|   p_{t, \mathcal{B}_x}^{(m)} -  \bar{p}_{t, \mathcal{B}_x}  - (1- c_{u}\alpha_{t-1}^2) \big(  p_{t-1, \mathcal{B}_x}^{(m)}- \bar{p}_{t-1, \mathcal{B}_x} \big)  \big\|^2
\end{align*}
where the inequality follows from the the generalized triangle inequality and the condition that $c_{u}\alpha_{t}^2 < 1$. 

\noindent Next we sum over $M$ for the second term in Eq.~\ref{eq:omega_drift_storm_multi1} and denote it as $T_1$, then we have:
\begin{align*}
T_1 & \leq  2 \sum_{m=1}^M\mathbb{E}\big\|p_{t, \mathcal{B}_x}^{(m)} -  \bar{p}_{t, \mathcal{B}_x}  - \big(  p_{t-1, \mathcal{B}_x}^{(m)}- \bar{p}_{t-1, \mathcal{B}_x} \big) \big\|^2  + 2 (c_{u}\alpha_{t-1}^2)^2 \sum_{m=1}^M\mathbb{E} \big\| p_{t-1, \mathcal{B}_x}^{(m)}- \bar{p}_{t-1, \mathcal{B}_x} \big\|^2
\end{align*}
We bound the two terms separately, we denote them as $T_{1,1}$ and $T_{1,2}$ separately, then we have:
\begin{align*}
T_{1,1} &\overset{(a)}{\leq}  \sum_{m=1}^M\mathbb{E} \big\| p_{t, \mathcal{B}_x}^{(m)} - p_{t-1, \mathcal{B}_x}^{(m)}\big\|^2 \nonumber\\
&\overset{(b)}{\leq} 2\big(L^2 + \frac{2L_{y^2}^2C_f^2}{\mu^2}\big)\sum_{m=1}^M\mathbb{E}\big[ \| x_t^{(m)} - x_{t-1}^{(m)} \|^2 + \| y_t^{(m)} - y_{t-1}^{(m)} \|^2 \big]  + 4L^2\sum_{m=1}^M\mathbb{E} \| u_t^{(m)} - u_{t-1}^{(m)} \|^2 \nonumber\\
&\leq 2\tilde{L}_2^2\alpha_{t-1}^2 \sum_{m=1}^M \mathbb{E}\big[ \| \eta\nu^{(m)}_{t-1} \|^2 + \| \gamma \omega^{(m)}_{t-1} \|^2 \big] + 4L^2\alpha_{t-1}^2\sum_{m=1}^M\mathbb{E} \| \tau^2 q^{(m)}_{t-1} \|^2
\end{align*}
where $(a)$ follows Proposition~\ref{prop: Sum_Mean_Kron}; $(b)$ follows Proposition~\ref{some smoothness} and the fact that $\hat{x}_t^{(m)} = x_t^{(m)}$ and $\hat{y}^{(m)}_{t} = y^{(m)}_t$ when $t \neq \bar{t}_s$; Next for the second term, we have:
\begin{align*}
T_{1,2} &= \sum_{m=1}^M \mathbb{E}\big\| p_{t-1, \mathcal{B}_x}^{(m)} - p_{t-1}^{(m)} -  (\bar{p}_{t-1, \mathcal{B}_x} - \bar{p}_{t-1}\big)  + p_{t-1}^{(m)} - \bar{p}_{t-1} \big\|^2 
\nonumber \\
& \overset{(b)}{\leq}   2\sum_{m=1}^M \mathbb{E} \big\|  p_{t-1, \mathcal{B}_x}^{(m)} - p_{t-1}^{(m)} \big\|^2  + 2\sum_{m=1}^M  \mathbb{E}  \big\| p_{t-1}^{(m)} - \bar{p}_{t-1}  \big\|^2 
\end{align*}
We denote the two terms above as $T_{1,2,1}, T_{1,2,2}$ respectively. For the term $T_{1,2,1}$ of Eq.~\ref{eq:omega_drift_storm_multi2}, we have $T_{1,2,1} \leq 2M\sigma^2/b_x$ by the bounded variance assumption; For the term $T_{1,2,2}$ of Eq.~\ref{eq:omega_drift_storm_multi2}, we have
\begin{align*}
T_{1,2,2} &\leq  16\big(L^2 + \frac{2L_{y^2}^2C_f^2}{\mu^2}\big)\sum_{m=1}^M\mathbb{E}\big[\big\| x^{(m)}_t - \bar{x}_t\big\|^2 + \big\|y^{(m)}_t - \bar{y}_t\big\|^2\big] \nonumber\\
&\qquad + 4L^2\sum_{m=1}^M\mathbb{E}\big\|u_{t}^{(m)}  - \bar{u}_{t}\big\|^2 + 4M\zeta_f^2 + \frac{8MC_f^2\zeta_{g, yy}^2}{\mu^2}
\end{align*}
Finally, combine everythin together and use the fact that $I\geq 1$, we have:
\begin{align*}
    \sum_{m=1}^M \mathbb{E} \| \hat{q}_{t}^{(m)} - \bar{q}_{t} \|^2
    &\leq \big(1 + \frac{1}{I}\big) \sum_{m=1}^M \mathbb{E} \|  q_{t-1}^{(m)}  - \bar{q}_{t-1} \|^2  + 8I \tilde{L}_2^2\alpha_{t-1}^2 \sum_{m=1}^M \mathbb{E}\big[\| \gamma \omega^{(m)}_{t-1} \|^2 + \| \eta \nu^{(m)}_{t-1} \|^2 \big] \nonumber\\
    &\qquad + 16IL^2\alpha_{t-1}^2\sum_{m=1}^M\mathbb{E} \underbrace{\| \tau^2 q^{(m)}_{t-1} \|^2}_{T_1} \nonumber \\
    & \qquad + 8I M (c_{u}\alpha_{t-1}^2)^2\frac{\sigma^2}{b_x} + 16I M (c_{u}\alpha_{t-1}^2)^2\zeta_f^2 + 32IM(c_{u}\alpha_{t-1}^2)^2\frac{C_f^2\zeta_{g, yy}^2}{\mu^2} \nonumber\\ 
    & \qquad  +  64I(c_{u}\alpha_{t-1}^2)^2\tilde{L}_2^2\sum_{m=1}^M\mathbb{E}\big[\big\| x^{(m)}_t - \bar{x}_t\big\|^2 \nonumber\\
    &\qquad + \big\|y^{(m)}_t - \bar{y}_t\big\|^2\big] + 16I(c_{u}\alpha_{t-1}^2)^2L^2\sum_{m=1}^M\mathbb{E}\big\|u_{t}^{(m)}  - \bar{u}_{t}\big\|^2
\end{align*}
We separate the term $T_1$ with triangle inequality to get:
\begin{align*}
    \sum_{m=1}^M \mathbb{E} \| \hat{q}_{t}^{(m)} - \bar{q}_{t} \|^2
    &\leq \left(1 + \frac{1}{I} + 32IL^2\tau^2\alpha_{t-1}^2\right) \sum_{m=1}^M \mathbb{E} \|  q_{t-1}^{(m)}  - \bar{q}_{t-1} \|^2 \nonumber\\
    &\qquad +  8I \tilde{L}_2^2\alpha_{t-1}^2 \sum_{m=1}^M \mathbb{E}\big[\| \gamma \omega^{(m)}_{t-1} \|^2 + \| \eta \nu^{(m)}_{t-1} \|^2 \big] \nonumber \\
    &\qquad + 32IL^2\alpha_{t-1}^2\sum_{m=1}^M\mathbb{E} \| \tau^2 \bar{q}_{t-1} \|^2  + 8I M (c_{u}\alpha_{t-1}^2)^2\frac{\sigma^2}{b_x} \nonumber\\
    &\qquad + 16I M (c_{u}\alpha_{t-1}^2)^2\zeta_f^2 + 32IM(c_{u}\alpha_{t-1}^2)^2\frac{C_f^2\zeta_{g, yy}^2}{\mu^2} \nonumber\\ 
    & \qquad  +  64I(c_{u}\alpha_{t-1}^2)^2\tilde{L}_2^2\sum_{m=1}^M\mathbb{E}\big[\big\| x^{(m)}_t - \bar{x}_t\big\|^2 + \big\|y^{(m)}_t - \bar{y}_t\big\|^2\big] \nonumber\\
    &\qquad + 16I(c_{u}\alpha_{t-1}^2)^2L^2\sum_{m=1}^M\mathbb{E}\big\|u_{t}^{(m)}  - \bar{u}_{t}\big\|^2
\end{align*}
This completes the proof.
\end{proof}


Next, to simply the notation, we denote $A_t = \mathbb{E} \| \bar{\nu}_{t} - \bar{\mu}_{t} \|^2$, $B_t = \mathbb{E} \|\bar{y}_t - y_{\bar{x}_{t}} \|^2 $, $C_t = \mathbb{E}\|\bar{\omega}_t - \frac{1}{M}\sum_{m=1}^M\nabla_y g^{(m)}(x^{(m)}_{t}, y^{(m)}_{t} ) \|^2$, $D_t = \frac{1}{M}\sum_{m=1}^M \mathbb{E}\|\nu^{(m)}_{t} - \bar{\nu}_{t}\|^2$, $E_t = \mathbb{E}\|\bar{\nu}_{t}\|^2$, $F_t =  \mathbb{E}\|\bar{\omega}_{t}\|^2 $, $G_t = \frac{1}{M}\sum_{m=1}^M \mathbb{E}\|\omega^{(m)}_{t} - \bar{\omega}_{t}\|^2$, $H_t = \mathbb{E} [\| \bar{q}_{t} -  \bar{p}_{t}\|^2]$, $I_t = \mathbb{E}[\|\bar{u}_{t} - u_{\bar{x}_{t}}\|^2]$, $J_t = \mathbb{E}\|q^{(m)}_{t} - \bar{q}_{t}\|^2$, $Q_t = \mathbb{E}\|\bar{q}_{t}\|^2$.

\begin{lemma}
\label{lemma:d_bound_storm_multi}
For $\eta < \min(\frac{\tilde{L}^2}{c_{\nu}}, \frac{\tilde{L}^2}{c_{\omega}}, \frac{\tilde{L}^2}{c_{u}}, 1)$, $\gamma < \min(\frac{\tilde{L}^2}{c_{\nu}}, \frac{\tilde{L}^2}{c_{\omega}}, \frac{\tilde{L}^2}{c_{u}}, 1)$, $\tau < \min(\frac{\tilde{L}^2}{c_{\nu}}, \frac{\tilde{L}^2}{c_{u}}, \frac{1}{2})$ and $\alpha_t < \frac{1}{16\tilde{L}I}$, where $\tilde{L}= max(\tilde{L}_1, \tilde{L}_2)$ , we have:
\begin{align*}
    \sum_{t = \bar{t}_{s-1}}^{\bar{t}_s-1} \alpha_{t} D_{t} &\leq \sum_{t=\bar{t}_{s-1}}^{\bar{t}_s-1} \big(\alpha_{t}E_{t} + \alpha_{t}F_{t}  + \alpha_{t}Q_{t} + \frac{c_{\omega}^2\alpha_{t}^3}{\tilde{L}^2}\frac{\sigma^2}{b_y} + \frac{c_{\omega}^2\alpha_{t}^3}{\tilde{L}^2} \zeta_g^2\nonumber\\
    & \qquad +  \frac{c_{\nu}^2\alpha_{t}^3}{\tilde{L}^2} \frac{\sigma^2}{b_x} + \frac{c_{u}^2\alpha_{t}^3}{\tilde{L}^2}\frac{\sigma^2}{b_x} + \frac{c_{\nu}^2\alpha_{t}^3\zeta_f^2}{\tilde{L}^2} + \frac{c_{u}^2\alpha_{t}^3\zeta_f^2}{\tilde{L}^2} + \frac{2c_{\nu}^2\alpha_{t}^3}{\tilde{L}^2}\frac{C_f^2\zeta_{g, xy}^2}{\mu^2}    +\frac{4c_{u}^2\alpha_{t}^3}{\tilde{L}^2}\frac{C_f^2\zeta_{g, yy}^2}{\mu^2}\big) \nonumber\\
    \sum_{t=\bar{t}_{s-1}}^{\bar{t}_s-1} \alpha_t G_t &\leq \sum_{t=\bar{t}_{s-1}}^{\bar{t}_s-1} \big(\alpha_{t}E_{t} + \alpha_{t}F_{t} + \alpha_{t}Q_{t} + \frac{2c_{\omega}^2\alpha_{t}^3}{\tilde{L}^2}\frac{\sigma^2}{b_y} + \frac{2c_{\omega}^2\alpha_{t}^3}{\tilde{L}^2} \zeta_g^2 \nonumber\\
    &\qquad +  \frac{c_{\nu}^2\alpha_{t}^3}{\tilde{L}^2} \frac{2\sigma^2}{b_x} + \frac{c_{u}^2\alpha_{t}^3}{\tilde{L}^2}\frac{\sigma^2}{b_x} + \frac{c_{\nu}^2\alpha_{t}^3\zeta_f^2}{\tilde{L}^2} +  \frac{c_{u}^2\alpha_{t}^3\zeta_f^2}{\tilde{L}^2} + \frac{c_{\nu}^2\alpha_{t}^3}{\tilde{L}^2}\frac{C_f^2\zeta_{g, xy}^2}{\mu^2}  +\frac{2c_{u}^2\alpha_{t}^3}{\tilde{L}^2}\frac{C_f^2\zeta_{g, yy}^2}{\mu^2}\big) \nonumber\\
    \sum_{t=\bar{t}_{s-1}+1}^{\bar{t}_s} \alpha_t J_t &\leq  \sum_{t=\bar{t}_{s-1}}^{\bar{t}_s-1} \big(\alpha_{t}F_{t} + \alpha_{t}E_{t} + \alpha_{t}Q_{t} + \frac{c_{\omega}^2\alpha_{t}^3}{\tilde{L}^2}\frac{\sigma^2}{b_y} + \frac{c_{\omega}^2\alpha_{t}^3}{\tilde{L}^2} \zeta_g^2\nonumber\\
    &\qquad + \frac{c_{u}^2\alpha_{t}^3}{\tilde{L}^2}\frac{\sigma^2}{b_x} + \frac{c_{\nu}^2\alpha_{t}^3}{\tilde{L}^2} \frac{\sigma^2}{b_x} + \frac{c_{u}^2\alpha_{t}^3\zeta_f^2}{\tilde{L}^2} + \frac{c_{\nu}^2\alpha_{t}^3\zeta_f^2}{2\tilde{L}^2} + \frac{c_{\nu}^2\alpha_{t}^3}{\tilde{L}^2}\frac{C_f^2\zeta_{g, xy}^2}{\mu^2} +\frac{40c_{u}^2\alpha_{t}^3}{\tilde{L}^2}\frac{C_f^2\zeta_{g, yy}^2}{\mu^2}\big)
\end{align*}
\end{lemma}

\begin{proof}
Based on Lemma~\ref{lem: ErrorAccumulation_Iterates_storm_mu_multi}, for $t \neq \bar{t}_s$, we have:
\begin{align*}
    D_t & \leq \big(1 + \frac{17}{16I}\big) D_{t-1} +  16I \tilde{L}_1^2\alpha_{t-1}^2\eta^2E_{t-1} + 16I\tilde{L}_1^2\alpha_{t-1}^2\gamma^2F_{t-1} + 16I \tilde{L}_1^2\alpha_{t-1}^2\gamma^2G_{t-1} + 32IL^2\tau^2\alpha_{t-1}^2J_{t-1} \nonumber\\
    & \qquad + 32IL^2\tau^2\alpha_{t-1}^2Q_{t-1} +  8Ic_{\nu}^2\alpha_{t-1}^4 \frac{\sigma^2}{b_x} + 32Ic_{\nu}^2\alpha_{t-1}^4\zeta_f^2 + 64Ic_{\nu}^2\alpha_{t-1}^4\frac{C_f^2\zeta_{g, xy}^2}{\mu^2}\nonumber\\
    &\qquad + 128I^2\tilde{L}_1^2\eta^2c_{\nu}^2\alpha_{t-1}^4\sum_{\ell = \bar{t}_{s-1}}^{t-2} \alpha_l^2 D_l + 128I^2\tilde{L}_1^2\gamma^2c_{\nu}^2\alpha_{t-1}^4\sum_{\ell = \bar{t}_{s-1}}^{t-2} \alpha_l^2 G_l + 32I^2L^2\tau^2c_{\nu}^2\alpha_{t-1}^4\sum_{\ell = \bar{t}_{s-1}}^{t-2} \alpha_l^2 J_l
\end{align*}
while for $t = \bar{t}_s$, we have $D_{\bar{t}_{s}} = 1/M\sum_{m=1}^M \mathbb{E} \| \nu_{\bar{t}_{s}}^{(m)} - \bar{\nu}_{\bar{t}_{s}} \|^2 = 0$. Apply the above equation recursively from $\bar{t}_{s-1} + 1$ to $t$. so we have:
\begin{align*}
    D_t & \leq \sum_{\ell=\bar{t}_{s-1}}^{\ell}\big(1 + \frac{17}{16I}\big)^{t - \ell} \big(16I \tilde{L}_1^2\alpha_{\ell}^2\eta^2E_{\ell} + 16I\tilde{L}_1^2\alpha_{\ell}^2\gamma^2F_{\ell} + 16I \tilde{L}_1^2\alpha_{\ell}^2\gamma^2G_{\ell} + 32IL^2\tau^2\alpha_{\ell}^2J_{\ell} \nonumber\\
    & \qquad + 32IL^2\tau^2\alpha_{\ell}^2Q_{\ell} +  8Ic_{\nu}^2\alpha_{\ell}^4 \frac{\sigma^2}{b_x} + 32Ic_{\nu}^2\alpha_{\ell}^4\zeta_f^2 + 64Ic_{\nu}^2\alpha_{\ell}^4\frac{C_f^2\zeta_{g, xy}^2}{\mu^2}\nonumber\\
    &\qquad + 128I^2\tilde{L}_1^2\eta^2c_{\nu}^2\alpha_{\ell}^4\sum_{\bar{\ell} = \bar{t}_{s-1}}^{\ell - 1} \alpha_{\bar{l}}^2 D_{\bar{\ell}} + 128I^2\tilde{L}_1^2\gamma^2c_{\nu}^2\alpha_{\ell}^4\sum_{\bar{\ell} = \bar{t}_{s-1}}^{\ell - 1} \alpha_{\bar{\ell}}^2 G_{\bar{\ell}} + 32I^2L^2\tau^2c_{\nu}^2\alpha_{\ell}^4\sum_{\bar{\ell} = \bar{t}_{s-1}}^{\ell - 1} \alpha_{\bar{\ell}}^2 J_{\bar{\ell}}\big) \nonumber\\
    & \leq  \sum_{\ell=\bar{t}_{s-1}}^{t-1} \big(48I \tilde{L}_1^2\alpha_{\ell}^2\eta^2E_{\ell} + 48I\tilde{L}_1^2\alpha_{\ell}^2\gamma^2F_{\ell} + 48I \tilde{L}_1^2\alpha_{\ell}^2\gamma^2G_{\ell} + 96IL^2\tau^2\alpha_{\ell}^2J_{\ell} \nonumber\\
    & \qquad + 96IL^2\tau^2\alpha_{\ell}^2Q_{\ell} +  24Ic_{\nu}^2\alpha_{\ell}^4 \frac{\sigma^2}{b_x} + 96Ic_{\nu}^2\alpha_{\ell}^4\zeta_f^2 + 192Ic_{\nu}^2\alpha_{\ell}^4\frac{C_f^2\zeta_{g, xy}^2}{\mu^2}\nonumber\\
    &\qquad + 384I^2\tilde{L}_1^2\eta^2c_{\nu}^2\alpha_{\ell}^4\sum_{\bar{\ell} = \bar{t}_{s-1}}^{\ell - 1} \alpha_{\bar{l}}^2 D_{\bar{\ell}} + 384I^2\tilde{L}_1^2\gamma^2c_{\nu}^2\alpha_{\ell}^4\sum_{\bar{\ell} = \bar{t}_{s-1}}^{\ell - 1} \alpha_{\bar{\ell}}^2 G_{\bar{\ell}} + 96I^2L^2\tau^2c_{\nu}^2\alpha_{\ell}^4\sum_{\bar{\ell} = \bar{t}_{s-1}}^{\ell - 1} \alpha_{\bar{\ell}}^2 J_{\bar{\ell}}\big)
\end{align*}
The second inequality uses the fact that $t -l \le I$ and the inequality $log(1+ a/x) \leq a/x$ for $x > -a$, so we have $(1+a/x)^x \leq e^{a}$, Then we choose $a = 17/16$ and $x = I$. Finally, we use the fact that $e^{17/16} \leq 3$. 

\noindent Next we multiply $\alpha_t$ over both sides and take sum from $\bar{t}_{s-1} + 1$ to $\bar{t}_{s}$, we have:
\begin{align*}
    &\sum_{t=\bar{t}_{s-1}+1}^{\bar{t}_s} \alpha_tD_t \nonumber\\
    &\leq \sum_{t=\bar{t}_{s-1}}^{\bar{t}_s-1} \alpha_t\sum_{\ell=\bar{t}_{s-1}}^{t-1} \big(48I \tilde{L}_1^2\alpha_{\ell}^2\eta^2E_{\ell} + 48I\tilde{L}_1^2\alpha_{\ell}^2\gamma^2F_{\ell} + 48I \tilde{L}_1^2\alpha_{\ell}^2\gamma^2G_{\ell} + 96IL^2\tau^2\alpha_{\ell}^2J_{\ell} \nonumber\\
    & \qquad + 96IL^2\tau^2\alpha_{\ell}^2Q_{\ell} +  24Ic_{\nu}^2\alpha_{\ell}^4 \frac{\sigma^2}{b_x} + 96Ic_{\nu}^2\alpha_{\ell}^4\zeta_f^2 + 192Ic_{\nu}^2\alpha_{\ell}^4\frac{C_f^2\zeta_{g, xy}^2}{\mu^2}\nonumber\\
    &\qquad + 384I^2\tilde{L}_1^2\eta^2c_{\nu}^2\alpha_{\ell}^4\sum_{\bar{\ell} = \bar{t}_{s-1}}^{\ell - 1} \alpha_{\bar{l}}^2 D_{\bar{\ell}} + 384I^2\tilde{L}_1^2\gamma^2c_{\nu}^2\alpha_{\ell}^4\sum_{\bar{\ell} = \bar{t}_{s-1}}^{\ell - 1} \alpha_{\bar{\ell}}^2 G_{\bar{\ell}} + 96I^2L^2\tau^2c_{\nu}^2\alpha_{\ell}^4\sum_{\bar{\ell} = \bar{t}_{s-1}}^{\ell - 1} \alpha_{\bar{\ell}}^2 J_{\bar{\ell}}\big) \nonumber\\
    & \overset{(a)}{\leq} \sum_{t=\bar{t}_{s-1}}^{\bar{t}_s-1} \big(3I\tilde{L}_1\alpha_{t}^2\eta^2E_{t} + 3I\tilde{L}_1\alpha_{t}^2\gamma^2F_{t} + 3I \tilde{L}_1\alpha_{t}^2\gamma^2G_{t} + 6IL\tau^2\alpha_{t}^2J_{t} \nonumber\\
    & \qquad + 6IL\tau^2\alpha_{t}^2Q_{t} +  \frac{3Ic_{\nu}^2\alpha_{t}^4}{2\tilde{L}} \frac{\sigma^2}{b_x} + \frac{6Ic_{\nu}^2\alpha_{t}^4\zeta_f^2}{\tilde{L}} + \frac{12Ic_{\nu}^2\alpha_{t}^4}{\tilde{L}}\frac{C_f^2\zeta_{g, xy}^2}{\mu^2}\nonumber\\
    &\qquad + 32I^{2}\tilde{L}_1\eta^2c_{\nu}^2\alpha_{t}^4\sum_{\ell = \bar{t}_{s-1}}^{t - 1} \alpha_{\ell}^2 D_{\ell} + 32I^{2}\tilde{L}_1\gamma^2c_{\nu}^2\alpha_{t}^4\sum_{\ell = \bar{t}_{s-1}}^{t - 1} \alpha_{\ell}^2 G_{\ell} + 6I^{2}L\tau^2c_{\nu}^2\alpha_{t}^4\sum_{\ell = \bar{t}_{s-1}}^{t - 1} \alpha_{\ell}^2 J_{\ell}\big) \nonumber\\
    &\overset{(b)}{\leq} \sum_{t=\bar{t}_{s-1}}^{\bar{t}_s-1} \big( \frac{3\eta^2}{16}\alpha_{t}E_{t} + \frac{3\gamma^2}{16}\alpha_{t}F_{t} + \frac{3\gamma^2}{16}\alpha_{t}G_{t} + \frac{3\tau^2}{8}\alpha_{t}J_{t} + \frac{3\tau^2}{8}\alpha_{t}Q_{t} \nonumber\\
    & \qquad +  \frac{3c_{\nu}^2\alpha_{t}^3}{32\tilde{L}^2} \frac{\sigma^2}{b_x} + \frac{3c_{\nu}^2\alpha_{t}^3\zeta_f^2}{8\tilde{L}^2} + \frac{3c_{\nu}^2\alpha_{t}^3}{4\tilde{L}^2}\frac{C_f^2\zeta_{g, xy}^2}{\mu^2} + \frac{\eta^2c_{\nu}^2}{8*16^3I^2\tilde{L}^4} \alpha_{t} D_{t} \nonumber\\
    &\qquad + \frac{\gamma^2c_{\nu}^2}{8*16^3I^2\tilde{L}^4} \alpha_{t} G_{t} + \frac{3\tau^2c_{\nu}^2}{8*16^4I^2\tilde{L}^4} \alpha_{t} J_{t}\big)
\end{align*}
In inequalities $(a)$ and $(b)$, we use $\alpha_t < \frac{1}{16\tilde{L}I} \leq \frac{1}{16\tilde{L}_1I}$. Note that $\sum_{t=\bar{t}_{s-1}+1}^{\bar{t}_s} \alpha_tD_t = \sum_{t = \bar{t}_{s-1}}^{\bar{t}_s-1} \alpha_{t} D_{t}$ as $D_{\bar{t}_s} = D_{\bar{t}_{s-1}} =0$. 

Then if we choose $\eta < \frac{\tilde{L}^2}{c_{\nu}}$ and $\gamma < \frac{\tilde{L}^2}{c_{\nu}}$, $\tau < \frac{\tilde{L}^2}{c_{\nu}}$, we have
\begin{align}\label{eq:d_bound_storm_multi}
    \sum_{t = \bar{t}_{s-1}}^{\bar{t}_s-1} \alpha_{t} D_{t} &\leq \sum_{t=\bar{t}_{s-1}}^{\bar{t}_s-1} \big( \frac{\eta^2}{4}\alpha_{t}E_{t} + \frac{\gamma^2}{4}\alpha_{t}F_{t} + \frac{\gamma^2}{2}\alpha_{t}G_{t} + \tau^2\alpha_{t}J_{t} + \frac{\tau^2}{2}\alpha_{t}Q_{t} \nonumber\\
    & \qquad +  \frac{c_{\nu}^2\alpha_{t}^3}{8\tilde{L}^2} \frac{\sigma^2}{b_x} + \frac{c_{\nu}^2\alpha_{t}^3\zeta_f^2}{2\tilde{L}^2} + \frac{c_{\nu}^2\alpha_{t}^3}{\tilde{L}^2}\frac{C_f^2\zeta_{g, xy}^2}{\mu^2}\big) 
\end{align}
Based on Lemma~\ref{lem: ErrorAccumulation_Iterates_storm_omega_multi}, for $t \neq \bar{t}_s$, we have:
\begin{align*}
    G_t
    &\leq \left(1 + \frac{33}{32I}\right) G_{t-1}  + 8IL^2\eta^2\alpha_{t-1}^2D_{t-1} + 8IL^2\eta^2\alpha_{t-1}^2E_{t-1} + 8 IL^2\gamma^2\alpha_{t-1}^2F_{t-1} \nonumber \\
    & \qquad + 8I c_{\omega}^2\alpha_{t-1}^4\frac{\sigma^2}{b_y} + 16I c_{\omega}^2\alpha_{t-1}^4\zeta_g^2  +  16I^2L^2\eta^2c_{\omega}^2\alpha_{t-1}^4\sum_{\ell = \bar{t}_{s-1}}^{t-2} \alpha_l^2 D_l + 16I^2L^2\gamma^2c_{\omega}^2\alpha_{t-1}^4\sum_{\ell = \bar{t}_{s-1}}^{t-2} \alpha_l^2 G_l
\end{align*}
Follow similar derivation, by recursively applying the above inequality, we have:
\begin{align*}
    G_t &\leq \sum_{\ell=\bar{t}_{s-1}}^{t-1} \big(24IL^2\eta^2\alpha_{\ell}^2D_{\ell} + 24IL^2\eta^2\alpha_{\ell}^2E_{\ell} + 24 IL^2\gamma^2\alpha_{\ell}^2F_{\ell} \nonumber \\
    & \qquad + 24I c_{\omega}^2\alpha_{\ell}^4\frac{\sigma^2}{b_y} + 48I c_{\omega}^2\alpha_{\ell}^4\zeta_g^2  +  48I^2L^2\eta^2c_{\omega}^2\alpha_{\ell}^4\sum_{\bar{\ell} = \bar{t}_{s-1}}^{\ell-1} \alpha_{\bar{\ell}}^2 D_{\bar{\ell}} + 48I^2L^2\gamma^2c_{\omega}^2\alpha_{\ell}^4\sum_{\bar{\ell} = \bar{t}_{s-1}}^{\ell-1} \alpha_{\bar{\ell}}^2 G_{\bar{\ell}}\big)
\end{align*}
Next we multiply $\alpha_t$ over both sides and take sum from $\bar{t}_{s-1} + 1$ to $\bar{t}_{s}$, use the condition that $\alpha_{t} < \frac{1}{16I\tilde{L}} < \frac{1}{16IL}$, $\eta < \frac{\tilde{L}^2}{c_{\omega}}$ and $\gamma < \frac{\tilde{L}^2}{c_{\omega}}$, we have:


\begin{align}\label{eq:g_bound_storm_multi}
    \sum_{t=\bar{t}_{s-1}}^{\bar{t}_s-1} \alpha_t G_t &\leq \sum_{t=\bar{t}_{s-1}}^{\bar{t}_s-1} \big(\frac{1}{2}\eta^2\alpha_{t}D_{t} + \frac{1}{8}\eta^2\alpha_{t}E_{t} + \frac{1}{8}\gamma^2\alpha_{t}F_{t} + \frac{c_{\omega}^2\alpha_{t}^3}{8\tilde{L}^2}\frac{\sigma^2}{b_y} + \frac{c_{\omega}^2\alpha_{t}^3}{8\tilde{L}^2} \zeta_g^2\big)
\end{align}
Based on Lemma~\ref{lem: ErrorAccumulation_Iterates_storm_u_multi}, we have:
\begin{align*}
    J_t &\leq \left(1 + \frac{33I}{32I}\right) J_{t-1}  +  16I \tilde{L}_2^2\tau^2\alpha_{t-1}^2G_{t-1} + 16I \tilde{L}_2^2\tau^2\alpha_{t-1}^2F_{t-1} + 16I \tilde{L}_2^2\eta^2\alpha_{t-1}^2D_{t-1} + 16I \tilde{L}_2^2\eta^2\alpha_{t-1}^2E_{t-1} \nonumber \\
    & \qquad + 16IL^2\tau^2\alpha_{t-1}^2Q_{t-1}  + 8Ic_{u}^2\alpha_{t-1}^4\frac{\sigma^2}{b_x} + 16Ic_{u}^2\alpha_{t-1}^4\zeta_f^2 + 32Ic_{u}^2\alpha_{t-1}^4\frac{C_f^2\zeta_{g, yy}^2}{\mu^2} \nonumber\\ 
    & \qquad  +  64I^2c_{u}^2\eta^2\alpha_{t-1}^4\tilde{L}_2^2\sum_{\ell = \bar{t}_{s-1}}^{t-2} \alpha_l^2 D_l + 64I^2c_{u}^2\gamma^2\alpha_{t-1}^4\tilde{L}_2^2\sum_{\ell = \bar{t}_{s-1}}^{t-2} \alpha_l^2 G_l  + 16I^2c_{u}^2\tau^2\alpha_{t-1}^4L^2\sum_{\ell = \bar{t}_{s-1}}^{t-2} \alpha_l^2 J_l
\end{align*}
Suppose we have $\alpha_t < \frac{1}{16\tilde{L}I}$, $\eta < \frac{\tilde{L}^2}{c_u}$, $\gamma < \frac{\tilde{L}^2}{c_u}$, $\tau < \frac{\tilde{L}^2}{c_u}$
\begin{align}\label{eq:j_bound_storm_multi}
    \sum_{t=\bar{t}_{s-1}+1}^{\bar{t}_s} \alpha_t J_t &\leq  \sum_{t=\bar{t}_{s-1}}^{\bar{t}_s-1} \big(\frac{\tau^2}{2}\alpha_{}G_{t}  +  \frac{\tau^2}{4}\alpha_{t}F_{t} + \frac{\eta^2}{2}\alpha_{t}D_{t} + \frac{\eta^2}{4}\alpha_{t-1}E_{t} \nonumber \\
    & \qquad + \frac{\tau^2}{4}\alpha_{t}Q_{t}  + \frac{c_{u}^2\alpha_{t}^3}{8\tilde{L}^2}\frac{\sigma^2}{b_x} + \frac{c_{u}^2\alpha_{t}^3\zeta_f^2}{4\tilde{L}^2} +\frac{3c_{u}^2\alpha_{t}^3}{\tilde{L}^2}\frac{C_f^2\zeta_{g, yy}^2}{\mu^2}\big)
\end{align}
Next, we combine Eq.~\ref{eq:d_bound_storm_multi}, Eq.~\ref{eq:g_bound_storm_multi} and Eq.~\ref{eq:j_bound_storm_multi} to have the result in the lemma.
\end{proof}

\subsubsection{Descent Lemma}
\begin{lemma}
\label{lemma:hg_error_storm_multi}
For all $t \in [\bar{t}_{s-1}, \bar{t}_s - 1]$, the iterates generated satisfy:
\begin{align*}
    \mathbb{E}\big\|  \nabla h(\bar{x}_{t})  - \bar{\mu}_t \big\|^2 &\leq \frac{2\tilde{L}_1^2}{M}\sum_{m=1}^M \mathbb{E} \big[\big\|\bar{x}_t -  x^{(m)}_t\big\|^2 + 2\big\|\bar{y}_t  - y^{(m)}_t \big\|^2 + 2\big\|y_{\bar{x}_t}  - \bar{y}_t\big\|^2 \big]  + 4L^2\mathbb{E} \big\|u_{\bar{x}_t}  - \bar{u}_{t}\big\|^2
\end{align*}
where we denote $u_{\bar{x}_t} = [\nabla_{y^2} g(\bar{x}_t, y_{\bar{x}_t})]^{-1}\nabla_y f(\bar{x}_t, y_{\bar{x}_t})$ and $\tilde{L}_1^2 = \big(L^2 + \frac{2L_{xy}^2C_f^2}{\mu^2}\big)$ is a constant.
\end{lemma}

\begin{proof}
This lemma follows the same derivation as Lemma~\ref{lemma:hg_bound_fedavg_multi}.
\end{proof}

\begin{lemma}
\label{lemma:desent_storm_multi}
Suppose $\eta\alpha_t < \frac{1}{2\bar{L}}$, for all $t \in [\bar{t}_{s-1}, \bar{t}_s - 1]$ and $s \in [S]$, the iterates generated satisfy:
\begin{align*}
\mathbb{E}\big[  h(\bar{x}_{t + 1}) \big] & \leq \mathbb{E} \big[    h(\bar{x}_{t }) \big]-  \frac{\eta\alpha_t}{4}  \mathbb{E} \big[\big\| \bar{\nu}_t  \big\|^2 \big] - \frac{\eta\alpha_t}{2} \mathbb{E} \big[\|\nabla h(\bar{x}_t) \|^2 \big] + \eta\alpha_t\mathbb{E}\big[ \big\| \bar{u}_t   - \bar{\nu}_{t}   \big\|^2 \big]\\
& \qquad + \frac{2\tilde{L}_1^2\eta\alpha_t}{M}\sum_{m=1}^M \mathbb{E} \big[\big\|\bar{x}_t -  x^{(m)}_t\big\|^2 + 2\big\|\bar{y}_t  - y^{(m)}_t \big\|^2 + 2\big\|y_{\bar{x}_t}  - \bar{y}_t\big\|^2 \big]  + 4L^2\eta\alpha_t\mathbb{E} \big\|u_{\bar{x}_t}  - \bar{u}_{t}\big\|^2
\end{align*}
where the expectation is w.r.t the stochasticity of the algorithm.
\end{lemma}
\begin{proof}
By the smoothness of $h(x)$ we have:
\begin{align*}
    \mathbb{E}[  h(\bar{x}_{t + 1}) ] 
    & \leq \mathbb{E} \big[ h(\bar{x}_{t }) + \langle \nabla h(\bar{x}_{t}),  \bar{x}_{t + 1} - \bar{x}_{t}\rangle + \frac{\bar{L}}{2} \| \bar{x}_{t + 1} - \bar{x}_{t } \|^2 \big] \nonumber\\
    &  \overset{(a)}{=}\mathbb{E} \big[ h(\bar{x}_{t}) - \eta\alpha_t \langle \nabla h(\bar{x}_{t}),  \bar{\nu}_t \rangle + \frac{\eta^2\alpha_t^2 \bar{L}}{2} \| \bar{\nu}_{t}  \|^2 \big] \nonumber\\
    & \overset{(b)}{=}    \mathbb{E} \big[ h(\bar{x}_{t}) - \frac{\eta\alpha_t}{2}  \big\| \bar{\nu}_{t}  \big\|^2  - \frac{\eta\alpha_t}{2} \| \nabla h(\bar{x}_{t}) \|^2   + \frac{\eta\alpha_t}{2} \big\|  \nabla h(\bar{x}_{t})  - \bar{\nu}_{t}   \big\|^2 + \frac{\eta\alpha_t^2 \bar{L}}{2} \big\| \bar{\nu}_{t} \big\|^2  \big] \nonumber \\  
    & = \mathbb{E} \big[    h(\bar{x}_{t }) -  \frac{\eta\alpha_t}{4}\big\| \bar{\nu}_t  \big\|^2 - \frac{\eta\alpha_t}{2} \|\nabla h(\bar{x}_t) \|^2  + \frac{\eta\alpha_t}{2} \underbrace{\big\|  \nabla h(\bar{x}_{t})  - \bar{\nu}_{t}   \big\|^2}_{T_1} \big]
\end{align*}
where equality $(a)$ follows from the iterate update given in Algorithm~\ref{alg:FedBiOAcc}; $(b)$ uses $\langle a , b \rangle = \frac{1}{2} [\|a\|^2 + \|b\|^2 - \|a - b \|^2]$ and $\eta\alpha_t < \frac{1}{2\bar{L}}$; For the term $T_1$, we have:
\begin{align*}
    \mathbb{E} \big[ \big\|  \nabla h(\bar{x}_{t})  - \bar{\nu}_{t}   \big\|^2  \big] & \leq 2\mathbb{E} \big[ \big\| \nabla h(\bar{x}_{t})  - \bar{u}_t  \big\|^2 \big] + 2\mathbb{E} \big[ \big\| \bar{u}_t   - \bar{\nu}_{t}   \big\|^2 \big]
\end{align*}
Use Lemma~\ref{lemma:hg_bound_storm_multi} for the first term and combine everything together finishes the proof.
\end{proof}

\subsubsection{Proof of Convergence Theorem}
\label{sec:fedbioacc-multi}
We first denote the following potential function $\mathcal{G}(t)$:
\begin{align*}
    \mathcal{G}_t &= h(\bar{x}_{t}) + \frac{9bM\eta}{64\alpha_{t}}\big\| \bar{\nu}_{t} - \bar{\mu}_t \big\|^2 + \frac{18\eta\tilde{L}^2}{\mu\gamma}\big\|\bar{y}_t - y_{\bar{x}_{t}} \big\|^2 + \frac{9bM\eta}{64\alpha_{t}}\big\|\bar{q}_{t} - \bar{p}_{t}\big\|^2 \nonumber \\
    & \qquad \qquad + \frac{9bM\eta}{64\alpha_{t}} \big\|\bar{\omega}_t - \frac{1}{M} \sum_{m=1}^M\nabla_y g^{(m)}(x^{(m)}_{t}, y^{(m)}_{t} ) \big\|^2 + \frac{18\eta L^2}{\mu\tau}\big\|\bar{u}_{t} - u_{\bar{x}_{t}}\big\|^2
\end{align*}
Furthermore, we have constants $\tilde{L}_1^2 = \big(L^2 + \frac{2L_{xy}^2C_f^2}{\mu^2}\big)$ and $\tilde{L}_2^2 = \big(L^2 + \frac{2L_{y^2}^2C_f^2}{\mu^2}\big)$, to ease the writing, without loss of generality, we assume the second order Lipschitz constants $L_{xy} = L_{y^2}$, as a result $\tilde{L}_1^2 = \tilde{L}_2^2$, we denote it as $\tilde{L}$ in the subsequent proof.
\begin{theorem}\label{theorem:FedBiOAcc_multi}
Suppose we choose $c_{\nu} = \frac{64}{9bM} + \frac{2}{3b^2M^2}$, $c_{\omega} = \frac{48^2}{bM\mu^2} + \frac{2}{3b^2M^2}$, $c_{u} = \frac{48^2}{bM\mu^2} + \frac{2}{3b^2M^2}$
$u = (bM\sigma)^2\bar{u}$, where $\bar{u} = \max\big(2,16^2I^{3}\tilde{L}^2, c_{\nu}^{3/2},c_{\omega}^{3/2}\big)$, $\delta = \frac{(bM\sigma)^{2/3}}{(16\tilde{L})^{1/3}}$,  $\alpha_t = \frac{\delta}{(u + t)^{1/3}}, t \in [T] $, $\gamma < \min\big(\frac{1}{8C_1^{1/2}}, \frac{\tilde{L}}{4C_1^{1/2}}, \frac{\tilde{L}^2}{c_{\nu}}, \frac{\tilde{L}^2}{c_{\omega}}, \frac{\tilde{L}^2}{c_{u}}, \frac{1}{2L},1\big)$, $\eta < \min\big(\frac{\mu\gamma}{36\kappa\tilde{L}}, \frac{1}{8C_1^{1/2}}, \frac{\tilde{L}}{4C_1^{1/2}}, \frac{\tilde{L}^2}{c_{\nu}}, \frac{\tilde{L}^2}{c_{\omega}}, \frac{\tilde{L}^2}{c_{u}}, \frac{1}{2\bar{L}}, 1\big)$, $\tau < \min\big(\frac{1}{8C_1^{1/2}}, \frac{\tilde{L}}{4C_1^{1/2}}, \frac{\tilde{L}^2}{c_{\nu}}, \frac{\tilde{L}^2}{c_{u}}, \frac{1}{2L}, \frac{1}{2}\big)$ where $C_1$ is a constant, we set the mini-batch size $b_x = b_y = b$ and the first batch with size $b_1 = O(Ib)$, $r = \frac{C_f}{\mu}$, then we have:
\begin{align*}
    \frac{1}{T}\sum_{t = 1}^{T-1} \mathbb{E} \big[ \|\nabla h(\bar{x}_t) \|^2 \big] = O\big(\frac{\kappa^{19/3}I}{T} + \frac{\kappa^{16/3}}{(bMT)^{2/3}}\big)
\end{align*}
To reach an  $\epsilon$-stationary point, we need $T = O(\kappa^{8}(bM)^{-1}\epsilon^{-1.5})$, $I = O(\kappa^{5/3}(bM)^{-1}\epsilon^{-0.5})$.
\end{theorem}

\begin{proof}
By the condition that $u \ge c_{\nu}^{3/2}\delta^3$, it is straightforward to verify that $c_{\nu}\alpha_t^2 < 1$. By Lemma~\ref{lemma:hg_bound_storm_multi}, we have:
\begin{align*}
   \frac{A_t}{\alpha_{t-1}} - \frac{A_{t-1}}{\alpha_{t-2}} &\leq \left(\alpha_{t-1}^{-1} - \alpha_{t-2}^{-1} - c_{\nu}\alpha_{t-1}\right) A_{t-1} + \frac{2c_{\nu}^2\alpha_{t-1}^3\sigma^2}{bM} + \frac{16L^2\tau^2\alpha_{t-1}}{bM}(J_{t-1} + Q_{t-1})\nonumber\\
   &\qquad + \frac{8\tilde{L}^2\eta^2\alpha_{t-1}}{bM}(D_{t-1} + E_{t-1}) + \frac{8\tilde{L}^2\gamma^2\alpha_{t-1}}{bM}(F_{t-1} + G_{t-1}) 
\end{align*}
where we choose $b_x = b_y = b$. For $\alpha_{t-1}^{-1} - \alpha_{t-2}^{-1}$, we have:
\begin{align*}
    \alpha_{t}^{-1} - \alpha_{t-1}^{-1} & =  \frac{(u + \sigma^2 t)^{1/3}}{\delta} -  \frac{(u + \sigma^2 (t-1))^{1/3}}{\delta} 
    \overset{(a)}{\leq}  \frac{\sigma^2}{3 \delta (u + \sigma^2 (t-1))^{2/3}} \nonumber \\
    & \overset{(b)}{\leq} \frac{2^{2/3} \sigma^2 \delta^2}{3 \delta^3 (u + \sigma^2 t)^{2/3}} \overset{(c)}{=} \frac{2^{2/3} \sigma^2}{3 \delta^3 } \alpha_{t}^2 \leq \frac{2}{3Ib^2M^2} \alpha_{t} \leq \frac{2^{2/3} \sigma^2}{3 \delta^3 } \alpha_{t}^2 \leq \frac{2}{3b^2M^2} \alpha_{t}
\end{align*}
where inequality $(a)$ results from the concavity of $x^{1/3}$ as: $(x + y)^{1/3} - x^{1/3} \leq y/3x^{2/3}$, inequality $(b)$ used the fact that $u_t \geq 2\sigma^2$, inequality $(c)$ uses the definition of $\alpha_t$, 
By choosing $c_{\nu} = \frac{64}{9bM} + \frac{2}{3b^2M^2}$, we have:
\begin{align*}
   \frac{A_t}{\alpha_{t-1}} - \frac{A_{t-1}}{\alpha_{t-2}} & \leq  - \frac{64}{9bM}\alpha_{t-1} A_{t-1} + \frac{2c_{\nu}^2\alpha_{t-1}^3\sigma^2}{bM} + \frac{16L^2\tau^2\alpha_{t-1}}{bM}(J_{t-1} + Q_{t-1}) \nonumber\\ &\qquad + \frac{8\tilde{L}^2\eta^2\alpha_{t-1}}{bM}(D_{t-1} + E_{t-1}) + \frac{8\tilde{L}^2\gamma^2\alpha_{t-1}}{bM}(F_{t-1} + G_{t-1})
\end{align*}
Next, we telescope from $\bar{t}_{s-1} + 1$ to $\bar{t}_{s}$:
\begin{align}\label{eq:A_tele_storm_multi}
    \big(\frac{A_{\bar{t}_s}}{\alpha_{\bar{t}_{s}-1}} - \frac{A_{\bar{t}_{s-1}}}{\alpha_{\bar{t}_{s-1} - 1}}\big) & \leq  - \frac{64}{9bM}\sum_{t=\bar{t}_{s-1}}^{\bar{t}_s-1}\alpha_{t} A_{t} + \frac{2c_{\nu}^2\sigma^2}{bM} \sum_{t=\bar{t}_{s-1}}^{\bar{t}_s-1}\alpha_{t}^3 + \frac{16\tilde{L}^2\eta^2}{bM} \sum_{t=\bar{t}_{s-1}}^{\bar{t}_s-1}\alpha_{t}D_{t} \nonumber \\
    & \qquad+ \frac{8\tilde{L}^2\eta^2}{bM}\sum_{t=\bar{t}_{s-1}}^{\bar{t}_s-1}\alpha_{t}E_{t}  + \frac{8\tilde{L}^2\gamma^2}{bM}\sum_{t=\bar{t}_{s-1}}^{\bar{t}_s-1}\alpha_{t}F_{t} + \frac{16\tilde{L}^2\gamma^2}{bM}\sum_{t=\bar{t}_{s-1}}^{\bar{t}_s-1}\alpha_{t}G_{t} \nonumber\\
    &\qquad + \frac{32L^2\tau^2}{bM}\sum_{t=\bar{t}_{s-1}}^{\bar{t}_s-1}\alpha_{t}J_{t} + \frac{16L^2\tau^2}{bM}\sum_{t=\bar{t}_{s-1}}^{\bar{t}_s-1}\alpha_{t}Q_{t}
\end{align}
Next, we follow similar derivation as $A_t/\alpha_{t-1}  - A_{t-1}/\alpha_{t-2}$. By Lemma~\ref{lemma: inner_est_error_storm_multi}.
we choose $c_{\omega} = \frac{48^2}{bM\mu^2} + \frac{2}{3b^2M^2}$, 
to obtain:
\begin{align*}
    \frac{C_t}{\alpha_{t-1}} - \frac{C_{t-1}}{\alpha_{t-2}} & \leq -\frac{48^2 \alpha_{t-1}}{bM\mu^2}C_{t-1} + \frac{2c_{\omega}^2\alpha_{t-1}^3\sigma^2}{bM} + \frac{4L^2\eta^2\alpha_{t-1}}{bM}(D_{t-1} + E_{t-1}) + \frac{4L^2\gamma^2\alpha_{t-1}}{bM}(F_{t-1} + G_{t-1})
\end{align*}
Then telescope from $\bar{t}_{s-1} + 1$ to $\bar{t}_{s}$, we have:
\begin{align}
   &\frac{C_{\bar{t}_s}}{\alpha_{\bar{t}_{s}-1}} - \frac{C_{\bar{t}_{s-1}}}{\alpha_{\bar{t}_{s-1} - 1}}\nonumber\\
   & \leq -\frac{48^2}{bM\mu^2}\sum_{t=\bar{t}_{s-1}}^{\bar{t}_s-1}\alpha_{t}C_{t} + \frac{2c_{\omega}^2\sigma^2}{bM}\sum_{t=\bar{t}_{s-1}}^{\bar{t}_s-1}\alpha_{t}^3 + \frac{16L^2\eta^2}{bM}\sum_{t=\bar{t}_{s-1}}^{\bar{t}_s-1}\alpha_{t}D_{t}  \nonumber \\
    & \qquad + \frac{8L^2\eta^2}{bM}\sum_{t=\bar{t}_{s-1}}^{\bar{t}_s-1}\alpha_{t}E_{t} + \frac{8L^2\gamma^2}{bM}\sum_{t=\bar{t}_{s-1}}^{\bar{t}_s-1}\alpha_{t}F_{t} + \frac{16L^2\gamma^2}{bM}\sum_{t=\bar{t}_{s-1}}^{\bar{t}_s-1}\alpha_{t}G_{t}
\label{eq:C_tele_storm_multi}
\end{align}
Next from Lemma~~\ref{lemma:u_bound_storm_multi}, we choose $c_{u} = \frac{48^2}{bM\mu^2} + \frac{2}{3b^2M^2}$, to obtain:
\begin{align*}
  \frac{H_t}{\alpha_{t-1}} - \frac{H_{t-1}}{\alpha_{t-2}} &\leq -\frac{48^2\alpha_{t-1}}{bM\mu^2}H_{t-1} + \frac{2c_{u}^2\alpha_{t-1}^3}{b M}\sigma^2  +  \frac{8\eta^2\alpha_{t-1}\tilde{L}^2}{bM} (D_{t-1} + E_{t-1}) \nonumber\\
  &\qquad + \frac{8\gamma^2\alpha_{t-1}\tilde{L}^2}{bM} (F_{t-1} + G_{t-1}) + \frac{8\tau^2\alpha_{t-1}L^2}{bM}(J_{t-1} + Q_{t-1})
\end{align*}
Then telescope from $\bar{t}_{s-1} + 1$ to $\bar{t}_{s}$, we have:
\begin{align}
   \frac{H_{\bar{t}_s}}{\alpha_{\bar{t}_{s}-1}} - \frac{H_{\bar{t}_{s-1}}}{\alpha_{\bar{t}_{s-1} - 1}}& \leq -\frac{48^2}{bM\mu^2}\sum_{t=\bar{t}_{s-1}}^{\bar{t}_s-1}\alpha_{t}H_{t} + \frac{2c_{u}^2}{b M}\sum_{t=\bar{t}_{s-1}}^{\bar{t}_s-1}\alpha_{t}^3\sigma^2  +  \frac{8\eta^2\tilde{L}^2}{bM} \sum_{t=\bar{t}_{s-1}}^{\bar{t}_s-1}\alpha_{t-1}(D_{t} + E_{t}) \nonumber\\
  &\qquad + \frac{8\gamma^2\tilde{L}^2}{bM} \sum_{t=\bar{t}_{s-1}}^{\bar{t}_s-1}\alpha_{t-1}(F_{t} + G_{t}) + \frac{8\tau^2L^2}{bM} \sum_{t=\bar{t}_{s-1}}^{\bar{t}_s-1}\alpha_{t}(J_{t} + Q_{t-1})
\label{eq:H_tele_storm_multi}
\end{align}
Next from Lemma~\ref{lemma: inner_drift_storm_multi}, for $t \neq \bar{t}_s$, we have:
\begin{align*}
    B_{t+1} - B_{t}  & \leq  - \frac{\mu\gamma\alpha_{t}B_{t}}{4} - \frac{\gamma^2\alpha_{t}F_{t}}{4}  + \frac{9\gamma\alpha_{t}C_{t}}{\mu} +  \frac{9\kappa^2\eta^2\alpha_{t}E_{t}}{2\mu\gamma} \nonumber\\
    &\qquad + \frac{9\gamma\alpha_{t}L^2}{\mu}\sum_{\ell=\bar{t}_{s-1}}^{t-1}I\eta^2\alpha_{\ell}^2D_{\ell}+ \frac{9\gamma\alpha_{t}L^2}{\mu}\sum_{\ell=\bar{t}_{s-1}}^{t-1}I\gamma^2\alpha_{\ell}^2G_{\ell}
\end{align*}
When $t = \bar{t}_s$, we do not have the last two terms in the above inequality. Next, we telescope from $\bar{t}_{s-1} + 1$ to $\bar{t}_{s}$ and have:
\begin{align}\label{eq:B_tele_storm_multi}
    B_{\bar{t}_s} - B_{\bar{t}_{s-1}}  & \leq  - \frac{\mu\gamma}{4} \sum_{t=\bar{t}_{s-1}}^{\bar{t}_s-1}\alpha_{t}B_{t} - \frac{\gamma^2}{4} \sum_{t=\bar{t}_{s-1}}^{\bar{t}_s-1}\alpha_{t}F_{t}  + \frac{9\gamma}{\mu} \sum_{t=\bar{t}_{s-1}}^{\bar{t}_s-1}\alpha_{t}C_{t} + \frac{9\kappa^2\eta^2}{2\mu\gamma}\sum_{t=\bar{t}_{s-1}}^{\bar{t}_s-1}\alpha_{t}E_{t}\nonumber\\
    &\qquad + \frac{9I\eta^2\gamma L^2}{\mu} \sum_{t=\bar{t}_{s-1}+1}^{\bar{t}_s-1}\alpha_t\sum_{\ell=\bar{t}_{s-1}}^{t-1}\alpha_{\ell}^2D_{\ell} + \frac{9I\gamma^3 L^2}{\mu} \sum_{t=\bar{t}_{s-1}+1}^{\bar{t}_s-1}\alpha_t\sum_{\ell=\bar{t}_{s-1}}^{t-1}\alpha_{\ell}^2G_{\ell} \nonumber\\
    &\leq - \frac{\mu\gamma}{4} \sum_{t=\bar{t}_{s-1}}^{\bar{t}_s-1}\alpha_{t}B_{t} - \frac{\gamma^2}{4} \sum_{t=\bar{t}_{s-1}}^{\bar{t}_s-1}\alpha_{t}F_{t}  + \frac{9\gamma}{\mu} \sum_{t=\bar{t}_{s-1}}^{\bar{t}_s-1}\alpha_{t}C_{t} + \frac{9\kappa^2\eta^2}{2\mu\gamma}\sum_{t=\bar{t}_{s-1}}^{\bar{t}_s-1}\alpha_{t}E_{t}\nonumber\\
    &\qquad + \frac{9L^2\eta^2\gamma}{16^2\hat{L}^2\mu}\sum_{\ell=\bar{t}_{s-1}}^{\bar{t}_s-1}\alpha_{\ell}D_{\ell} + \frac{9L^2\gamma^3}{16^2\hat{L}^2\mu}\sum_{\ell=\bar{t}_{s-1}}^{\bar{t}_s-1}\alpha_{\ell}G_{\ell} 
\end{align}
where we use the fact that $\alpha_t < \frac{1}{16\hat{L}I}$. Next, from Lemma~\ref{lemma: u_inner_drift_storm_multi}, we have:
\begin{align*}
 I_{t+1} - I_{t} &\leq -\frac{\mu\tau\alpha_t}{4}I_t
 - \frac{\tau^2\alpha_t}{4} Q_t
  +  \frac{9\kappa^2\eta^2\alpha_t}{2\mu\tau}E_t + \frac{9\tau\alpha_t}{\mu}H_t \nonumber\\
  &+  \frac{18I\eta^2\tau\alpha_t \tilde{L}^2}{\mu}\sum_{\ell=\bar{t}_{s-1}}^{t-1}\alpha_{\ell}^2D_{\ell} + \frac{18I\gamma^2\tau\alpha_t \tilde{L}^2}{\mu}\sum_{\ell=\bar{t}_{s-1}}^{t-1}\alpha_{\ell}^2G_{\ell} + 18I\tau^3\alpha_t L^2\sum_{\ell=\bar{t}_{s-1}}^{t-1}\alpha_{\ell}^2J_{\ell}
\end{align*}
when $t = \bar{t}_s$, we do not have the last three terms in the above inequality.  Next, we telescope from $\bar{t}_{s-1} + 1$ to $\bar{t}_{s}$ and have:
\begin{align}\label{eq:I_tele_storm_multi}
  I_{\bar{t}_s} - I_{\bar{t}_{s-1}}&\leq -\frac{\mu\tau}{4}\sum_{t=\bar{t}_{s-1}}^{\bar{t}_s-1}\alpha_tI_t
 - \frac{\tau^2}{4} \sum_{t=\bar{t}_{s-1}}^{\bar{t}_s-1}\alpha_tQ_t
  +  \frac{9\kappa^2\eta^2}{2\mu\tau} \sum_{t=\bar{t}_{s-1}}^{\bar{t}_s-1}\alpha_tE_t + \frac{9\tau}{\mu}\sum_{t=\bar{t}_{s-1}}^{\bar{t}_s-1}\alpha_t H_t \nonumber\\
  &\qquad +  \frac{18I\eta^2\tau \tilde{L}^2}{\mu}\sum_{t=\bar{t}_{s-1}}^{\bar{t}_s-1}\alpha_t\sum_{\ell=\bar{t}_{s-1}}^{t-1}\alpha_{\ell}^2D_{\ell} + \frac{18I\gamma^2\tau \tilde{L}^2}{\mu}\sum_{t=\bar{t}_{s-1}}^{\bar{t}_s-1}\alpha_t\sum_{\ell=\bar{t}_{s-1}}^{t-1}\alpha_{\ell}^2G_{\ell} \nonumber\\
  &\qquad + 18I\tau^3 L^2\sum_{t=\bar{t}_{s-1}}^{\bar{t}_s-1}\alpha_t\sum_{\ell=\bar{t}_{s-1}}^{t-1}\alpha_{\ell}^2J_{\ell} \nonumber\\
  &\leq -\frac{\mu\tau}{4}\sum_{t=\bar{t}_{s-1}}^{\bar{t}_s-1}\alpha_tI_t
 - \frac{\tau^2}{4} \sum_{t=\bar{t}_{s-1}}^{\bar{t}_s-1}\alpha_tQ_t
  +  \frac{9\kappa^2\eta^2}{2\mu\tau} \sum_{t=\bar{t}_{s-1}}^{\bar{t}_s-1}\alpha_tE_t + \frac{9\tau}{\mu}\sum_{t=\bar{t}_{s-1}}^{\bar{t}_s-1}\alpha_t H_t \nonumber\\
  &\qquad +  \frac{18\eta^2\tau}{16^2\mu}\sum_{t=\bar{t}_{s-1}}^{\bar{t}_s-1}\alpha_{t}D_{t} + \frac{18\gamma^2\tau}{16^2\mu}\sum_{t=\bar{t}_{s-1}}^{\bar{t}_s-1}\alpha_{t}G_{t} + \frac{18\tau^3 L^2}{16^2\tilde{L}^2}\sum_{t=\bar{t}_{s-1}}^{\bar{t}_s-1}\alpha_{t}J_{t}
\end{align}
Next, by Lemma~\ref{lemma:desent_storm_multi}, when $t + 1 \neq \bar{t}_s$, we have:
\begin{align*}
    \mathbb{E} [  h(\bar{x}_{t + 1})]  &\leq \mathbb{E} [    h(\bar{x}_{t })] -  \frac{\eta\alpha_t}{4}  E_t - \frac{\eta\alpha_t}{2} \mathbb{E} [\|\nabla h(\bar{x}_t) \|^2 ] + \eta\alpha_t A_t + 4\tilde{L}^2\eta\alpha_t B_t + 4L^2\eta\alpha_t I_t \nonumber\\
    &\qquad + 2\tilde{L}^2I\eta^3\alpha_t \sum_{\ell = \bar{t}_{s-1}}^{t-1}  \alpha_l^2D_l + 4\tilde{L}^2I\gamma^2\eta\alpha_t \sum_{\ell = \bar{t}_{s-1}}^{t-1}  \alpha_l^2G_l
\end{align*}
When $t = \bar{t}_s$, we do not have the last two terms. Next, we telescope from $\bar{t}_{s-1}$ to $\bar{t}_{s} - 1$ to have:
\begin{align}\label{eq:h_tele_storm_multi}
    &\mathbb{E} [  h(\bar{x}_{\bar{t}_{s}}) - h(\bar{x}_{\bar{t}_{s - 1} }) ] \nonumber\\
    & \leq - \sum_{t = \bar{t}_{s-1}}^{\bar{t}_s-1}\frac{\eta\alpha_t}{4} E_t -  \sum_{t = \bar{t}_{s-1}}^{\bar{t}_s-1}\frac{\eta\alpha_t}{2} \mathbb{E} [ \|\nabla h(\bar{x}_t) \|^2 ] + 4L^2\eta \sum_{t = \bar{t}_{s-1}}^{\bar{t}_s-1}\alpha_tI_t + \sum_{t = \bar{t}_{s-1}}^{\bar{t}_s-1}\eta\alpha_t A_t \nonumber\\
    & \qquad + \sum_{t = \bar{t}_{s-1}}^{\bar{t}_s-1}4\tilde{L}^2\eta\alpha_t B_t + 2\tilde{L}^2I\eta^3\sum_{t = \bar{t}_{s-1}}^{\bar{t}_s-1}\alpha_t \sum_{\ell = \bar{t}_{s-1}}^{t-1}  \alpha_l^2D_l + 4\tilde{L}^2I\gamma^2\eta\sum_{t = \bar{t}_{s-1}}^{\bar{t}_s-1}\alpha_t \sum_{\ell = \bar{t}_{s-1}}^{t-1}  \alpha_l^2G_l \nonumber\\
    & \leq - \sum_{t = \bar{t}_{s-1}}^{\bar{t}_s-1}\frac{\eta\alpha_t}{4} E_t -  \sum_{t = \bar{t}_{s-1}}^{\bar{t}_s-1}\frac{\eta\alpha_t}{2} \mathbb{E} [ \|\nabla h(\bar{x}_t) \|^2 ] + 4L^2\eta \sum_{t = \bar{t}_{s-1}}^{\bar{t}_s-1}\alpha_tI_t + \sum_{t = \bar{t}_{s-1}}^{\bar{t}_s-1}\eta\alpha_t A_t \nonumber\\
    & \qquad + \sum_{t = \bar{t}_{s-1}}^{\bar{t}_s-1}4\tilde{L}^2\eta\alpha_t B_t + \frac{\eta^3}{128} \sum_{t = \bar{t}_{s-1}}^{\bar{t}_s-1}  \alpha_tD_t + \frac{\gamma^2\eta}{64} \sum_{t = \bar{t}_{s-1}}^{\bar{t}_s-1}  \alpha_t G_t
\end{align}
In the inequality, we use the fact that $\bar{t}_s - \bar{t}_{s-1}  \leq I$, $\alpha_t < \frac{1}{16\tilde{L}I}$.

Combine Eq.~(\ref{eq:A_tele_storm_multi}),Eq.~(\ref{eq:C_tele_storm_multi}), Eq.~(\ref{eq:B_tele_storm_multi}) and Eq.~(\ref{eq:h_tele_storm_multi}) and we have:
\begin{align*}
    &\mathbb{E}[\mathcal{G}_{\bar{t}_s}] - \mathbb{E}[\mathcal{G}_{\bar{t}_{s-1}}] \nonumber\\
    & \leq - \sum_{t = \bar{t}_{s-1}}^{\bar{t}_s-1}\frac{\eta\alpha_t}{2} \mathbb{E} [ \|\nabla h(\bar{x}_t) \|^2] + \big(\frac{9\eta c_{\omega}^2\sigma^2}{32} +  \frac{9\eta c_{\nu}^2\sigma^2}{32} + \frac{9\eta c_{u}^2\sigma^2}{32}\big) \sum_{t=\bar{t}_{s-1}}^{\bar{t}_s-1}\alpha_{t}^3 \nonumber\\
    & \qquad - \frac{\tilde{L}^2\eta}{2}\sum_{t=\bar{t}_{s-1}}^{\bar{t}_s-1} \alpha_{t}B_{t}  - \frac{L^2\eta}{2}\sum_{t=\bar{t}_{s-1}}^{\bar{t}_s-1} \alpha_{t}I_{t}  -\sum_{t=\bar{t}_{s-1}}^{\bar{t}_s-1}\big(\frac{9\eta\gamma \tilde{L}^2}{2\mu}  - \frac{9\eta\gamma^2 \tilde{L}^2}{4}  - \frac{9\eta\gamma^2L^2}{8}\big)\alpha_{t}F_{t} \nonumber \\
    & \qquad -\sum_{t=\bar{t}_{s-1}}^{\bar{t}_s-1}\big(\frac{1}{4} - \frac{81\kappa^2\tilde{L}^2\eta^2}{\mu^2\gamma^2} - \frac{81\kappa^2L^2\eta^2}{\mu^2\tau^2}  -  \frac{9L^2\eta^2}{8} - \frac{9\tilde{L}^2\eta^2}{4} \big)\eta\alpha_{t}E_{t}\nonumber\\
    & \qquad -\sum_{t=\bar{t}_{s-1}}^{\bar{t}_s-1}\big(\frac{9\tau\eta L^2}{2\mu} - \frac{9\eta\tau^2L^2}{4}\big)\alpha_{t}Q_{t} + \big(\frac{81\kappa^2}{64} + \frac{9L^2}{4} + 9\tilde{L}^2 \big)\tau^2\eta \sum_{t=\bar{t}_{s-1}}^{\bar{t}_s-1}\alpha_{t}J_{t}\nonumber\\
    & \qquad +  \big(\frac{1}{128} + \frac{81\kappa^2}{128} + \frac{81\kappa^2}{64} + \frac{9L^2}{4} + \frac{9\tilde{L}^2}{2} \big)\eta^3 \sum_{t=\bar{t}_{s-1}}^{\bar{t}_s-1}\alpha_{t}D_{t} \nonumber\\
    &\qquad + \big(\frac{1}{64} + \frac{81\kappa^2}{128} + \frac{81\kappa^2}{64} +\frac{9\tilde{L}^2}{4} + \frac{9L^2}{2} \big)\gamma^2\eta \sum_{t=\bar{t}_{s-1}}^{\bar{t}_s-1}\alpha_{t}G_{t}
\end{align*}
By the condition that $\eta < \frac{\mu\gamma}{36\kappa\tilde{L}}$ and $\gamma \leq \frac{1}{2L} < \frac{1}{2\mu}$. Next, we denote: \[C_1 = \frac{1}{64} + \frac{81\kappa^2}{32} + 9\tilde{L}^2 = O(\kappa^2)\] Then, we have:
\begin{align}\label{eq:phi_bound_storm_multi}
    &\mathbb{E}[\mathcal{G}_{\bar{t}_s}] - \mathbb{E}[\mathcal{G}_{\bar{t}_{s - 1}}] \nonumber\\
    & \leq - \sum_{t = \bar{t}_{s-1}}^{\bar{t}_s-1}\frac{\eta\alpha_t}{2} \mathbb{E} \big[ \|\nabla h(\bar{x}_t) \|^2 \big] + \big(\frac{9\eta c_{\omega}^2\sigma^2}{32} +  \frac{9\eta c_{\nu}^2\sigma^2}{32} + \frac{9\eta c_{u}^2\sigma^2}{32}\big) \sum_{t=\bar{t}_{s-1}}^{\bar{t}_s-1}\alpha_{t}^3 \nonumber\\
    & \qquad  -\frac{9\eta\gamma^2\tilde{L}^2}{8}\sum_{t=\bar{t}_{s-1}}^{\bar{t}_s-1}\alpha_{t}F_{t} -\frac{\eta}{8} \sum_{t=\bar{t}_{s-1}}^{\bar{t}_s-1}\alpha_{t}E_{t}   - \frac{\tilde{L}^2\eta}{2}\sum_{t=\bar{t}_{s-1}}^{\bar{t}_s-1} \alpha_{t}B_{t} - \frac{L^2\eta}{2}\sum_{t=\bar{t}_{s-1}}^{\bar{t}_s-1} \alpha_{t}I_{t} \nonumber\\
    &\qquad -\frac{9\eta\tau^2L^2}{4}\sum_{t=\bar{t}_{s-1}}^{\bar{t}_s-1}\alpha_{t}Q_{t} +  C_1\eta^3\sum_{t=\bar{t}_{s-1}}^{\bar{t}_s-1}\alpha_{t}D_{t} +  C_1\gamma^2\eta\sum_{t=\bar{t}_{s-1}}^{\bar{t}_s-1}\alpha_{t}G_{t} +  C_1\tau^2\eta\sum_{t=\bar{t}_{s-1}}^{\bar{t}_s-1}\alpha_{t}J_{t}
\end{align}
Combine Eq.~(\ref{eq:phi_bound_storm_multi}) with Lemma~\ref{lemma:d_bound_storm_multi}, and use the condition that $\eta < \min\big(\frac{1}{8C_1^{1/2}},\frac{\tilde{L}}{4C_1^{1/2}}, 1\big)$, $\gamma < \min\big(\frac{1}{8C_1^{1/2}}, \frac{\tilde{L}}{4C_1^{1/2}}, 1\big)$ and $\tau < \min\big(\frac{1}{8C_1^{1/2}}, \frac{\tilde{L}}{4C_1^{1/2}}, 1\big)$ we have:
\begin{align*}
    \mathbb{E}[\mathcal{G}_{\bar{t}_s}] - \mathbb{E}[\mathcal{G}_{\bar{t}_{s - 1}}] & \leq - \sum_{t = \bar{t}_{s-1}}^{\bar{t}_s-1}\frac{\eta\alpha_t}{2} \mathbb{E} \big[ \|\nabla h(\bar{x}_t) \|^2 \big]  + C_{\sigma,\zeta}\eta \sum_{t=\bar{t}_{s-1}}^{\bar{t}_s-1}\alpha_t^3
\end{align*}
For ease of notation, we denote 
\[
C_{\sigma,\zeta} = \big(4c_{\omega}^2\sigma^2 +  4c_{u}^2\sigma^2 +  4c_{\nu}^2\sigma^2 + 3c_{u}^2\zeta_f^2 + 3c_{\nu}^2\zeta_f^2  + 3c_{\omega}^2\zeta_g^2 + \frac{3c_{\nu}^2 C_f^2\zeta_{g, xy}^2}{\mu^2} +\frac{120c_{u}^2C_f^2\zeta_{g, yy}^2}{\mu^2}\big).
\]
Next, sum over all $s \in [S]$ (assume $T = SI + 1$ without loss of generality), we have:
\begin{align*}
    \mathbb{E}[\mathcal{G}_{T}] - \mathbb{E}[\mathcal{G}_{1}] 
    & \leq - \sum_{t = 1}^{T-1}\frac{\eta\alpha_t}{2} \mathbb{E} \big[ \|\nabla h(\bar{x}_t) \|^2 \big] + \eta C_{\sigma, \zeta}\sum_{t=1}^{T-1}\alpha_{t}^3
\end{align*}
Rearranging the terms and use the fact that $\alpha_t$ is non-increasing, we have:
\begin{align*}
    \frac{\eta\alpha_T}{2}\sum_{t = 1}^{T-1}\mathbb{E} \big[ \|\nabla h(\bar{x}_t) \|^2 \big] 
    &\leq \mathbb{E}[\mathcal{G}_{1}] - \mathbb{E}[\mathcal{G}_{T}]  + \eta C_{\sigma, \zeta} \sum_{t=1}^{T-1}\alpha_{t}^3\nonumber\\
    &\leq h(x_{1}) - h^{\ast} + \frac{9bM\eta A_1}{64\alpha_{1}} + \frac{18\eta\tilde{L}^2B_1}{\mu\gamma} \nonumber\\
    &\qquad + \frac{ 9bM\eta C_1}{64\alpha_{1}} + \frac{ 9bM\eta H_1}{64\alpha_{1}} + \frac{18\eta L^2I_1}{\mu\tau} + \eta C_{\sigma, \zeta}\sum_{t=1}^{T-1}\alpha_{t}^3
\end{align*}
where we use $\mathcal{G}_T \ge h^{\ast}$ ($h^{\ast}$ is the optimal value of $h$), and for the last term, we use the following fact:
\begin{align*}
     \sum_{t=1}^T \alpha_t^3 & =    \sum_{t = 1}^{T} \frac{\delta^3 }{u + \sigma^2 t} \leq  \sum_{t = 1}^{T} \frac{\delta^3  }{\sigma^2 + \sigma^2 t} = \frac{ \delta^3}{\sigma^2}   \sum_{t = 1}^{T} \frac{1}{1 +   t} \leq \frac{  \delta^3 }{\sigma^2}   \ln(T+1) = \frac{b^2M^2\ln(T+1)}{16\tilde{L}}
\end{align*}
the first inequality follows $u_t > \sigma^2$, the last inequality follows Proposition~\ref{Lem: AD_Sum_1overT}. 

Next, we denote the initial sub-optimality as $\Delta = h(\bar{x}_{1}) - h^{\ast}$, initial inner variable estimation error \emph{i.e.} $B_1 = \|y_1 - y_{x_{1}}\|^2 \leq \Delta_y$ and the initial hyper-gradient computation error $I_1 = \|u_1 -  [\nabla_{y^2} g(x_1, y_{x_1})]^{-1}\nabla_y f(x_1, y_{x_1})\|^2 \leq \Delta_u$.

Furthermore, we have $A_1 = \leq \frac{\sigma^2}{b_1M}$, $C_1 \leq \frac{\sigma^2}{b_1M}$, $H_1 \leq \frac{\sigma^2}{b_1M}$ where $b_1$ be the size of the first batch. Then, we divide both sides by $\eta\alpha_T T/2$ to have:
\begin{align*}
    \frac{1}{T}\sum_{t = 1}^{T-1} \mathbb{E} \big[ \|\nabla h(\bar{x}_t) \|^2 \big] & \leq \big(\frac{2\Delta}{\eta} + \frac{27b\sigma^2}{32b_1\alpha_{1}} + \frac{36\tilde{L}^2\Delta_y}{\mu\gamma} + \frac{36L^2\Delta_u}{\mu\tau} +\frac{b^2M^2C_{\sigma,\zeta}\ln(T)}{8\tilde{L}}\big)\frac{1}{T\alpha_T}
\end{align*}
Note that we have:
\begin{align*}
    \frac{1}{{\alpha_t t}} = \frac{(u + \sigma^2t)^{1/3}}{\delta t} \leq \frac{u^{1/3}}{\delta t} + \frac{\sigma^{2/3}}{\delta t^{2/3}}
\end{align*}
where the inequality uses the fact that $(x + y)^{1/3} \leq x^{1/3} + y^{1/3}$. In particular, when $t=1$, we have 
\begin{align}\label{eq:alpha1_bound}
    \frac{1}{{\alpha_1}} \leq \frac{u^{1/3} + \sigma^{2/3}}{\delta} = \frac{(16\tilde{L})^{1/3}((bM)^{2/3}\bar{u}^{1/3} + 1)}{(bM)^{2/3}}
\end{align}
when $t=T$, we have:
\begin{align}\label{eq:alphaT_bound}
    \frac{1}{{\alpha_T T}} \leq \frac{u^{1/3}}{\delta T} + \frac{\sigma^{2/3}}{\delta T^{2/3}} = (16\tilde{L})^{1/3}\left(\frac{\bar{u}^{1/3}}{T} + \frac{1}{(bMT)^{2/3}}\right)
\end{align}
In summary, we have:
\begin{align*}
    &\frac{1}{T}\sum_{t = 1}^{T-1} \mathbb{E} \big[ \|\nabla h(\bar{x}_t) \|^2 \big] \nonumber\\
    & \leq \big(\frac{2\Delta}{\eta} + \frac{27b\sigma^2}{32 b_1\alpha_{1}} + \frac{36\tilde{L}^2\Delta_y}{\mu\gamma} + \frac{36L^2\Delta_u}{\mu\tau} +\frac{b^2M^2 C_{\sigma,\zeta}\ln(T)}{8\tilde{L}}\big)\big(\frac{(16\tilde{L}\bar{u})^{1/3}}{T} + \frac{(16\tilde{L})^{1/3}}{(bMT)^{2/3}}\big)
\end{align*}
Recall that $\tilde{L} = O(\kappa)$, $\bar{L} = O(\kappa^3)$,  therefore we have $c_{\nu} =\Theta((bM)^{-1})$, $c_{\omega} = \Theta(\kappa^2(bM)^{-1})$, $c_{u} = \Theta(\kappa^2(bM)^{-1})$ $\bar{u} = \Theta(I^{3}\kappa^{3})$, then for $\eta, \gamma, \tau$, we have:
\[
\gamma < \min\big(\frac{1}{8C_1^{1/2}}, \frac{\tilde{L}}{4C_1^{1/2}}, \frac{\tilde{L}^2}{c_{\nu}}, \frac{\tilde{L}^2}{c_{\omega}}, \frac{\tilde{L}^2}{c_{u}}, \frac{1}{2L},1\big)
\]

\[
\tau < \min\big(\frac{1}{8C_1^{1/2}}, \frac{\tilde{L}}{4C_1^{1/2}}, \frac{\tilde{L}^2}{c_{\nu}}, \frac{\tilde{L}^2}{c_{u}}, \frac{1}{2L}, \frac{1}{2}\big)
\]

\[
\eta < \min\big(\frac{\mu\gamma}{36\kappa\tilde{L}}, \frac{1}{8C_1^{1/2}}, \frac{\tilde{L}}{4C_1^{1/2}}, \frac{\tilde{L}^2}{c_{\nu}}, \frac{\tilde{L}^2}{c_{\omega}}, \frac{\tilde{L}^2}{c_{u}}, \frac{1}{2\bar{L}}, 1\big)
\]
where $C_1 = O(\kappa^2)$, so we have $\gamma^{-1} = O(\kappa)$, $\eta^{-1} = O(\kappa^{3})$, $\tau^{-1} = O(\kappa)$, furthermore, $\alpha_1^{-1} = O(I\kappa^{4/3})$,  $C_{\sigma,\zeta} =O(\kappa^6(bM)^{-2})$, assume we choose the size of the first batch to be $b_1 = Ib$. 
 
Combine everything together, we have:
\begin{align*}
    \frac{1}{T}\sum_{t = 1}^{T-1} \mathbb{E} \big[ \|\nabla h(\bar{x}_t) \|^2 \big] = O\big(\frac{\kappa^{19/3}I}{T} + \frac{\kappa^{16/3}}{(bMT)^{2/3}}\big)
\end{align*}
To reach an  $\epsilon$-stationary point, we need $T = O(\kappa^{8}(bM)^{-1}\epsilon^{-1.5})$, $I = O(\kappa^{5/3}(bM)^{-1}\epsilon^{-0.5})$. The communication cost is $E = T/I \geq \kappa^{19/3}\epsilon^{-1}$, the sample complexity is $Gc(f, \epsilon) = O(M^{-1}\kappa^{8}\epsilon^{-1.5})$, $Gc(g, \epsilon) = O(M^{-1}\kappa^{8}\epsilon^{-1.5})$, $Jv(g, \epsilon) = O(M^{-1}\kappa^{8}\epsilon^{-1.5})$, $Hv(g, \epsilon) = O(M^{-1}\kappa^{8}\epsilon^{-1.5})$
\end{proof}

\newpage

\subsection{Proof for the FedBiO Algorithm}
Algorithm~\ref{alg:FedBiO} follows Eq.~\ref{eq:fedbio_alg}, and we discuss its convergence property in this subsection.

\begin{algorithm}[t]
\caption{Federated Bilevel Optimization (\textbf{FedBiO})}
\label{alg:FedBiO}
\begin{algorithmic}[1]
\STATE {\bfseries Input:} Initial states $x_1$, $y_1$ and $u_1$; learning rates $\{\gamma_t, \eta_t, \tau_t\}_{t=1}^T$
\STATE {\bfseries Initialization:} Set $x^{(m)}_1 = x_1$, $y^{(m)}_1 = y_1$, $u^{(m)}_1 = u_1$;
\FOR{$t=1$ \textbf{to} $T$}
\STATE Randomly sample mutually independent minibatch of samples $\mathcal{B}_{y}$ and $\mathcal{B}_{x} = \{\mathcal{B}_{g,1}, \mathcal{B}_{g,2}, \mathcal{B}_{f, 1}, \mathcal{B}_{f,2}\}$ of size b;
\STATE $\omega_{t}^{(m)} = \nabla_y g^{(m)} (x^{(m)}_{t}, y^{(m)}_{t}, \mathcal{B}_y)$
\STATE $\nu^{(m)}_t = \nabla_x f^{(m)}(x^{(m)}_t, y^{(m)}_t; \mathcal{B}_{f,1}) - \nabla_{xy}g^{(m)}(x^{(m)}_t, y^{(m)}_t; \mathcal{B}_{g,1})u_{t}^{(m)}$;
\STATE $\hat{y}^{(m)}_{t+1} = y^{(m)}_{t} - \gamma_t  \omega_{t}^{(m)}$, $\hat{x}^{(m)}_{t+1} = x^{(m)}_{t} - \eta_t \nu_{t}^{(m)}$;
\IF{$t$ mod I $ = 0$}
\STATE $y^{(m)}_{t+1} = \frac{1}{M}\sum_{j=1}^{M} \hat{y}^{(j)}_{t+1}$;  $x^{(m)}_{t+1} = \frac{1}{M}\sum_{j=1}^{M} \hat{x}^{(j)}_{t+1}$
\ELSE
\STATE $y^{(m)}_{t+1} = \hat{y}^{(m)}_{t+1}$, $x^{(m)}_{t+1} = \hat{x}^{(m)}_{t+1}$
\ENDIF
\STATE $\hat{u}_{t+1}^{(m)} = \mathcal{P}_r(\tau_t\nabla_y f^{(m)}(x^{(m)}_t, y^{(m)}_t; \mathcal{B}_{f,2}) + (I-\tau_t\nabla_{y^2} g^{(m)}(x^{(m)}_t, y^{(m)}_t; \mathcal{B}_{g,2}))u^{(m)}_t)$;
\IF{$t$ mod I $ = 0$}
\STATE $u^{(m)}_{t+1} = \frac{1}{M}\sum_{j=1}^{M} \hat{u}^{(j)}_{t+1}$
\ELSE
\STATE $u^{(m)}_{t+1} = \hat{u}^{(m)}_{t+1}$
\ENDIF
\ENDFOR
\end{algorithmic}
\end{algorithm}

\subsubsection{Lower Problem Solution Error and Hyper-gradient Estimation Error}
\begin{lemma}
\label{lemma: inner_drift_fedavg_multi}
When $\gamma < \frac{1}{2L}$, we have:
\begin{align*}
\mathbb{E}\big\|\bar{y}_t - y_{\bar{x}_{t}} \big\|^2  &\leq (1 - \frac{\mu\gamma}{4})\mathbb{E}\big\|\bar{y}_{t-1} - y_{\bar{x}_{t-1}}\big\|^2 + \frac{9\kappa^2\eta^2}{2\mu\gamma}\mathbb{E}\big\|\bar{\nu}_{t-1}\big\|^2 -\frac{\gamma^2}{4} \mathbb{E}\|\bar{\omega}_{t-1}\|^2 \nonumber\\
&\qquad + \frac{9L^2\gamma}{2\mu M}\sum_{m=1}^M\mathbb{E}\big[\|x^{(m)}_{t-1} - \bar{x}_{t-1}\|^2 + \|y^{(m)}_{t-1} - \bar{y}_{t-1})\|^2\big] + \frac{4\gamma\sigma^2}{\mu b_y M}
\end{align*}
\end{lemma}

\begin{lemma}\label{lemma: hyper_drift_fedavg_multi}
Suppose we choose $\tau < \frac{1}{L}$, then we have:
\begin{align*}
    \mathbb{E} \big\| \bar{u}_{t+1} -  u_{\bar{x}_{t+1}}\big\|^2 
    &\leq (1 - \frac{\mu\tau}{4})\mathbb{E}\big\| \bar{u}_t - u_{\bar{x}_t}\big\|^2 + \frac{5\tau^2\sigma^2}{4b_x M} +  \frac{5\eta^2\bar{L}^2}{\mu\tau}\mathbb{E} \big\|\bar{\nu}_{t}\big\|^2 \nonumber\\
    &\qquad + \frac{5}{4}\big(\frac{3\tau L^2}{\mu M} + \frac{\tau L_{y^2}^2C_f^2}{2\mu^3M}\big)\sum_{m=1}^M\mathbb{E} \big[\big\|\bar{x}_t -  x^{(m)}_t\big\|^2 + 2\big\|\bar{y}_t  - y^{(m)}_t \big\|^2 + 2\big\|y_{\bar{x}_t}  - \bar{y}_t\big\|^2\big] 
\end{align*}
\end{lemma}

We provide the proof for Lemma~\ref{lemma: hyper_drift_fedavg_multi} here and Lemma~\ref{lemma: inner_drift_fedavg_multi} can be derived similarly.

\begin{proof}
First, by proposition~\ref{prop:strong-prog} (set $\alpha = 1$) and choose $\gamma < \frac{1}{2L}$, we have:
\begin{align*}
\mathbb{E}\|\bar{y}_{t}-y_{\bar{x}_{t-1}}\|^2 &\leq ( 1-\frac{\mu\gamma}{2})\mathbb{E}\|\bar{y}_{t-1} - y_{\bar{x}_{t-1}}\|^2 - \frac{\gamma^2}{4} \mathbb{E}\|\bar{\omega}_{t-1}\|^2 \nonumber\\
&\qquad + \frac{4\gamma}{\mu}\mathbb{E}\|\nabla_y g(\bar{x}_{t-1},\bar{y}_{t-1})- \frac{1}{M} \sum_{m=1}^M \nabla_y g(x^{(m)}_{t-1},y^{(m)}_{t-1})\|^2 + \frac{4\gamma\sigma^2}{\mu b_y M}\nonumber\\
&\leq ( 1-\frac{\mu\gamma}{2})\mathbb{E}\|\bar{y}_{t-1} - y_{\bar{x}_{t-1}}\|^2 - \frac{\gamma^2}{4} \mathbb{E}\|\bar{\omega}_{t-1}\|^2 \nonumber\\
&\qquad + \frac{4\gamma}{\mu M}\sum_{m=1}^M\mathbb{E}\|\nabla_y^{(m)} g(\bar{x}_{t-1},\bar{y}_{t-1})-\nabla_y g(x^{(m)}_{t-1},y^{(m)}_{t-1})\|^2 + \frac{4\gamma\sigma^2}{\mu b_y M}\nonumber\\
&\leq ( 1-\frac{\mu\gamma}{2})\mathbb{E}\|\bar{y}_{t-1} - y_{\bar{x}_{t-1}}\|^2 - \frac{\gamma^2}{4} \mathbb{E}\|\bar{\omega}_{t-1}\|^2 \nonumber\\
&\qquad + \frac{4L^2\gamma}{\mu M}\sum_{m=1}^M\mathbb{E}\big[\|x^{(m)}_{t-1} - \bar{x}_{t-1}\|^2 + \|y^{(m)}_{t-1} - \bar{y}_{t-1})\|^2\big] + \frac{4\gamma\sigma^2}{\mu b_y M}\nonumber\\
\end{align*}
Furthermore, by the generalized triangle inequality, we have:
\begin{align*}
\mathbb{E}\big\|\bar{y}_t - y_{\bar{x}_{t}} \big\|^2  &\leq (1 - \frac{\mu\gamma}{4})\mathbb{E}\big\|\bar{y}_{t-1} - y_{\bar{x}_{t-1}}\big\|^2 + (1 + \frac{4}{\mu\gamma})\mathbb{E}\big\|y_{\bar{x}_{t}} - y_{\bar{x}_{t-1}}\big\|^2 - (1 + \frac{\mu\gamma}{4})\frac{\gamma^2}{4} \mathbb{E}\|\bar{\omega}_{t-1}\|^2 \nonumber\\
&\qquad + (1 + \frac{\mu\gamma}{4})\frac{4L^2\gamma}{\mu M}\sum_{m=1}^M\mathbb{E}\big[\|x^{(m)}_{t-1} - \bar{x}_{t-1}\|^2 + \|y^{(m)}_{t-1} - \bar{y}_{t-1})\|^2\big] + \frac{4\gamma\sigma^2}{\mu b_y M} \nonumber \\
&\leq (1 - \frac{\mu\gamma}{4})\mathbb{E}\big\|\bar{y}_{t-1} - y_{\bar{x}_{t-1}}\big\|^2 + \frac{9\kappa^2\eta^2}{2\mu\gamma}\mathbb{E}\big\|\bar{\nu}_{t-1}\big\|^2 -\frac{\gamma^2}{4} \mathbb{E}\|\bar{\omega}_{t-1}\|^2 \nonumber\\
&\qquad + \frac{9L^2\gamma}{2\mu M}\sum_{m=1}^M\mathbb{E}\big[\|x^{(m)}_{t-1} - \bar{x}_{t-1}\|^2 + \|y^{(m)}_{t-1} - \bar{y}_{t-1})\|^2\big] + \frac{4\gamma\sigma^2}{\mu b_y M}
\end{align*}
where the second inequality is due to $\gamma < 1/2L$. This completes the proof.
\end{proof}

\subsubsection{Local Variable Drift}
\begin{lemma}
\label{lem: x_drift_FedAvg_Multi}
For any $t \neq \bar{t}_s, s\in[S]$, we have:
\begin{align*}
\|x_t^{(m)}-  \bar{x}_t \|^2 &\leq I\eta^2 \sum_{\ell = \bar{t}_{s-1}}^{t-1} \big\|   \nu_\ell^{(m)} - \bar{\nu}_\ell\big\|^2,\; \|y_t^{(m)}-  \bar{y}_t \|^2 \leq I\gamma^2 \sum_{\ell = \bar{t}_{s-1}}^{t-1} \big\|   \omega_\ell^{(m)} - \bar{\omega}_\ell\big\|^2
\end{align*}
\end{lemma}

\begin{proof}
Note from Algorithm and the definition of $\bar{t}_s$ that at $t = \bar{t}_{s}$ with $s \in [S]$, $x_{t}^{(m)} = \bar{x}_{t}$, for all $k$. For $t \neq \bar{t}_s$, with $s \in [S]$, we have: $x_{t}^{(m)} = x_{t-1}^{(m)} - \eta  \nu_{t-1}^{(m)}$, this implies that: $x_t^{(m)} = x_{\bar{t}_{s-1}}^{(m)} - \sum_{\ell = \bar{t}_{s-1}}^{t-1} \eta  \nu_\ell^{(m)} \quad \text{and} \quad \bar{x}_{t}  = \bar{x}_{\bar{t}_{s-1}}  - \sum_{\ell = \bar{t}_{s-1}}^{t-1} \eta  \bar{\nu}_\ell.$
So for $t \neq \bar{t}_s$, with $s \in [S]$ we have:
\begin{align*}
\|x_t^{(m)}-  \bar{x}_t \|^2 & =  \big\| x_{\bar{t}_{s-1}}^{(m)} - \bar{x}_{\bar{t}_{s-1}}  - \big( \sum_{\ell = \bar{t}_{s-1}}^{t-1} \eta  \nu_\ell^{(m)} -   \sum_{\ell =  \bar{t}_{s-1}}^{t-1} \eta  \bar{\nu}_\ell  \big) \big\|^2 =  \big\|  \sum_{\ell = \bar{t}_{s-1}}^{t-1} \eta\big(  \nu_\ell^{(m)} -      \bar{\nu}_\ell  \big) \big\|^2 \nonumber\\
&\leq I\eta^2 \sum_{\ell = \bar{t}_{s-1}}^{t-1} \big\|   \nu_\ell^{(m)} -      \bar{\nu}_\ell\big\|^2
\end{align*}
We can derive the bound for $\|y_t^{(m)}-  \bar{y}_t \|^2$ similarly. This completes the proof.
\end{proof}

\begin{lemma}
\label{lem: ErrorAccumulation_Iterates_FedAvg_Multi_Nu}
For any $t \in [T]$, we have:
\begin{align*}
  \frac{1}{M}\sum_{m=1}^M\mathbb{E}\big\|\big(\nu_t^{(m)} -\bar{\nu}_t  \big) \big\|^2 &\leq \frac{4L^2}{M}\sum_{m=1}^M\mathbb{E}\big\|u_{t}^{(m)}  - \bar{u}_{t}\big\|^2 +4\zeta_f^2 + \frac{8C_f^2\zeta_{g, xy}^2}{\mu^2} + \frac{2\sigma^2}{b_x} \nonumber\\
  &\qquad + \big(\frac{16L^2}{M} + \frac{32L_{xy}^2C_f^2}{\mu^2 M}\big)\sum_{m=1}^M \mathbb{E}\big[\big\| x^{(m)}_t - \bar{x}_t\big\|^2 + \big\|y^{(m)}_t - \bar{y}_t\big\|^2\big]
\end{align*}
\end{lemma}

\begin{lemma}
\label{lem: ErrorAccumulation_Iterates_FedAvg_Multi_Omega}
For $t \in T$, we have:
\begin{align*}
\frac{1}{M} \sum_{m = 1}^M  \mathbb{E}\big\|\big(\omega_t^{(m)} -\bar{\omega}_t  \big) \big\|^2 &\leq \frac{2L^2}{M}\sum_{m=1}^M\mathbb{E}\big\| x^{(m)}_t - \bar{x}_{t} \big\|^2 + \frac{2L^2}{M}\sum_{m=1}^M\mathbb{E}\big\| y^{(m)}_t - \bar{y}_{t} \big\|^2  + \frac{2\sigma^2}{b_y} + 2\zeta_g^2
\end{align*}
\end{lemma}

\begin{lemma}\label{lem: ErrorAccumulation_Iterates_FedAvg_Multi_U}
For $t \in  [T]$, we have:
\begin{align*}
\frac{1}{M}\sum_{m=1}^M\mathbb{E}\big\|\big(  u_{t+1}^{(m)} -\bar{u}_{t+1}  \big) \big\|^2 &\leq (1 + \frac{1}{I})\frac{1}{M}\sum_{m=1}^M\mathbb{E}\big\|u_{t}^{(m)}  - \bar{u}_{t}\big\|^2 + \frac{64I\tau^2C_f^2\zeta_{g, yy}^2}{\mu^2} + 32I\tau^2\zeta_f^2 + \frac{2\tau^2\sigma^2}{b_x} \nonumber\\
&\qquad + \big(\frac{128IL^2\tau^2}{M} + \frac{256I\tau^2L_{y^2}^2C_f^2}{\mu^2 M}\big)\sum_{m=1}^M \mathbb{E}\big[\big\| x^{(m)}_t - \bar{x}_t\big\|^2 + \big\|y^{(m)}_t - \bar{y}_t\big\|^2\big] 
\end{align*}
\end{lemma}

Lemma~\ref{lem: ErrorAccumulation_Iterates_FedAvg_Multi_Nu}-Lemma~\ref{lem: ErrorAccumulation_Iterates_FedAvg_Multi_U} bounds the local drift of $\nu_t^{(m)}$, $\omega_t^{(m)}$ and $u_{t+1}^{(m)}$. We provide the proof for Lemma~\ref{lem: ErrorAccumulation_Iterates_FedAvg_Multi_Nu} here and the other two bounds can be derived similarly.

\begin{proof}
We have:
\begin{align*}
\frac{1}{M}\sum_{m=1}^M\mathbb{E}\big\|\big(\nu_t^{(m)} -\bar{\nu}_t  \big) \big\|^2
&\leq \frac{1}{M}\sum_{m=1}^M\mathbb{E}\big\| \nabla_x f^{(m)}(x^{(m)}_t, y^{(m)}_t) - \nabla_{xy}g^{(m)}(x^{(m)}_t, y^{(m)}_t)u_{t}^{(m)} \nonumber\\
&\qquad - \frac{1}{M}\sum_{j=1}^M\nabla_x f^{(j)}(x^{(j)}_t, y^{(j)}_t) - \nabla_{xy}g^{(j)}(x^{(j)}_t, y^{(j)}_t)u_{t}^{(j)}\big\|^2 + \frac{2\sigma^2}{b_x} \nonumber\\
&\leq \underbrace{\frac{2}{M}\sum_{m=1}^M\mathbb{E}\big\| \nabla_x f^{(m)}(x^{(m)}_t, y^{(m)}_t) - \frac{1}{M}\sum_{j=1}^M\nabla_x f^{(j)}(x^{(j)}_t, y^{(j)}_t)\big\|^2}_{T_1} \nonumber\\
&\qquad + \underbrace{\frac{2}{M}\sum_{m=1}^M\mathbb{E}\big\| \nabla_{xy}g^{(m)}(x^{(m)}_t, y^{(m)}_t)u_{t}^{(m)}  - \frac{1}{M}\sum_{j=1}^M \nabla_{xy}g^{(j)}(x^{(j)}_t, y^{(j)}_t)u_{t}^{(j)}\big\|^2}_{T_2} + \frac{2\sigma^2}{b_x}
\end{align*}
For the term $T_1$, we have:
\begin{align*}
    T_1 &\leq \frac{16}{M}\sum_{m=1}^M\mathbb{E}\big\| \nabla_x f^{(m)}(x^{(m)}_t, y^{(m)}_t) - \nabla_x f^{(m)}(\bar{x}_t, \bar{y}_t)\big\|^2 + \frac{4}{M}\sum_{m=1}^M\mathbb{E}\big\| \nabla_x f^{(m)}(\bar{x}_t, \bar{y}_t) - \nabla_x f(\bar{x}_t, \bar{y}_t)\big\|^2 \nonumber\\
    &\leq \frac{16L^2}{M}\sum_{m=1}^M\mathbb{E}\big[\big\| x^{(m)}_t - \bar{x}_t\big\|^2 + \big\|y^{(m)}_t - \bar{y}_t\big\|^2\big] + 4\zeta_f^2
\end{align*}
Next for the term $T_2$, we have:
\begin{align*}
    T_2 &\leq \frac{4L^2}{M}\sum_{m=1}^M\mathbb{E}\big\|u_{t}^{(m)}  - \bar{u}_{t}\big\|^2 + \frac{4C_f^2}{\mu^2 M^2}\sum_{m=1}^M\sum_{j=1}^M \mathbb{E}\big\| \nabla_{xy}g^{(m)}(x^{(m)}_t, y^{(m)}_t) -  \nabla_{xy}g^{(j)}(x^{(j)}_t, y^{(j)}_t)\big\|^2 \nonumber\\
    &\leq \frac{4L^2}{M}\sum_{m=1}^M\mathbb{E}\big\|u_{t}^{(m)}  - \bar{u}_{t}\big\|^2 + \frac{32L_{xy}^2C_f^2}{\mu^2 M}\sum_{m=1}^M \mathbb{E}\big[\big\| x^{(m)}_t - \bar{x}_t\big\|^2 + \big\|y^{(m)}_t - \bar{y}_t\big\|^2\big] + \frac{8C_f^2\zeta_{g, xy}^2}{\mu^2}
\end{align*}
Combine everything together, we get the claim in the lemma.
\end{proof}

Lemma~\ref{lem: ErrorAccumulation_Iterates_FedAvg_Multi_Nu}-Lemma~\ref{lem: ErrorAccumulation_Iterates_FedAvg_Multi_U} have recursive dependence of each other. Next, we provide an un-intertwined bound for each of them. For ease of notation, we denote $D_t = \frac{1}{M} \sum_{m = 1}^M \mathbb{E}\| \nu_t^{(m)} -\bar{\nu}_t\|^2$, $B_t = \frac{1}{M} \sum_{m=1}^M \mathbb{E}\|\omega^{(m)}_t - \bar{\omega}_t\|^2$, $A_t = \frac{1}{M} \sum_{m=1}^M \mathbb{E}\|u^{(m)}_t - \bar{u}_t\|^2$ and $C_t = \mathbb{E}\big\| \bar{y}_t -  y_{\bar{x}_t}\big\|^2$.

\begin{lemma}
\label{lemma:d_bound_fedavg_multi}
For $\gamma \leq \frac{1}{8I\tilde{L}_1}$ and $\eta < \frac{1}{8I\tilde{L}_2}$, $\tau < \frac{1}{128I\tilde{L}_2}$, where $\tilde{L}_1^2 = \big(L^2 + \frac{2L_{xy}^2C_f^2}{\mu^2}\big)$ and $\tilde{L}_2^2 = \big(L^2 + \frac{2L_{y^2}^2C_f^2}{\mu^2}\big)$ are constants, then we have:
\begin{align*}
\sum_{t=\bar{t}_{s-1}}^{\bar{t}_s-1} D_t &\leq  96I\zeta_f^2 + 16I\zeta_g^2 + \frac{16IC_f^2\zeta_{g, yy}^2}{\mu^2} +  \frac{32IC_f^2\zeta_{g,xy}^2}{\mu^2} + \frac{16I\sigma^2}{b_y} + \frac{20I\sigma^2}{b_x} \nonumber\\
\sum_{t=\bar{t}_{s-1}}^{\bar{t}_s-1} B_t &\leq 24I\zeta_f^2  + 8I\zeta_g^2 + \frac{4IC_f^2\zeta_{g, yy}^2}{\mu^2} +  \frac{8IC_f^2\zeta_{g,xy}^2}{\mu^2} + \frac{8I\sigma^2}{b_y} + \frac{5I\sigma^2}{b_x}
\end{align*}
\end{lemma}

\begin{proof}
Based on Lemma~\ref{lem: ErrorAccumulation_Iterates_FedAvg_Multi_Nu}, and sum from $\bar{t}_{s-1} + 1$ to $\bar{t}_{s} - 1$, we have:
\begin{align*}
\sum_{t=\bar{t}_{s-1}+1}^{\bar{t}_s-1} D_t&\leq 16I\eta^2\tilde{L}_1^2 \sum_{t=\bar{t}_{s-1}+1}^{\bar{t}_s-1}\sum_{\ell = \bar{t}_{s-1}}^{t-1} D_{\ell} + 16I\gamma^2\tilde{L}_1^2 \sum_{t=\bar{t}_{s-1}+1}^{\bar{t}_s-1}\sum_{\ell = \bar{t}_{s-1}}^{t-1} B_{\ell} \nonumber\\
&\qquad + 4L^2\sum_{t=\bar{t}_{s-1}+1}^{\bar{t}_s-1} A_t + 4(I-1)\zeta_f^2 + \frac{8(I-1)C_f^2\zeta_{g, xy}^2}{\mu^2} + \frac{2(I-1)\sigma^2}{b_x} \nonumber\\
&\leq 16I^2\eta^2\tilde{L}_1^2 \sum_{t=\bar{t}_{s-1}+1}^{\bar{t}_s-1}D_{t} + 16I^2\gamma^2\tilde{L}_1^2 \sum_{t=\bar{t}_{s-1}+1}^{\bar{t}_s-1} B_{t} \nonumber\\
&\qquad + 4L^2\sum_{t=\bar{t}_{s-1}+1}^{\bar{t}_s-1} A_t + 4(I-1)\zeta_f^2 + \frac{8(I-1)C_f^2\zeta_{g, xy}^2}{\mu^2} + \frac{2(I-1)\sigma^2}{b_x}
\end{align*}
where we denote $\tilde{L}_1^2 = (L^2 + \frac{2L_{xy}^2C_f^2}{\mu^2})$. Combine with the case of $t = \bar{t}_s$ in Lemma~\ref{lem: ErrorAccumulation_Iterates_FedAvg_Multi_Nu}, we have:
\begin{align}\label{eq:nu_drift_fedavg_multi}
\sum_{t=\bar{t}_{s-1}}^{\bar{t}_s-1} D_t &\leq 16I^2\tilde{L}_1^2\eta^2 \sum_{t=\bar{t}_{s-1}}^{\bar{t}_s-1}D_{t} + 16I^2\tilde{L}_1^2\gamma^2\sum_{t=\bar{t}_{s-1}}^{\bar{t}_s-1} B_{t} + 4L^2\sum_{t=\bar{t}_{s-1}}^{\bar{t}_s-1} A_t + 4I\zeta_f^2 +\frac{8IC_f^2\zeta_{g, xy}^2}{\mu^2} + \frac{2I\sigma^2}{b_x}
\end{align}
Based on Lemma~\ref{lem: ErrorAccumulation_Iterates_FedAvg_Multi_Omega}, and sum from $\bar{t}_{s-1} + 1$ to $\bar{t}_{s} - 1$, we have:
\begin{align*}
\sum_{t=\bar{t}_{s-1}+1}^{\bar{t}_s-1} B_t &\leq 2I\eta^2L^2 \sum_{t=\bar{t}_{s-1}+1}^{\bar{t}_s-1}\sum_{\ell = \bar{t}_{s-1}}^{t-1} D_{\ell} + 2I\gamma^2L^2 \sum_{t=\bar{t}_{s-1}+1}^{\bar{t}_s-1}\sum_{\ell = \bar{t}_{s-1}}^{t-1} B_{\ell}  + 2(I-1)\sigma^2 + 2(I-1)\zeta_g^2 \nonumber\\
&\leq 2I^2\eta^2L^2 \sum_{\ell = \bar{t}_{s-1}}^{\bar{t}_s-1} D_{\ell} + 2I^2\gamma^2L^2 \sum_{\ell = \bar{t}_{s-1}}^{\bar{t}_s-1} B_{\ell}  + \frac{2(I-1)\sigma^2}{b_y} + 2(I-1)\zeta_g^2
\end{align*}
Combine with the case of $t = \bar{t}_s$ in Lemma~\ref{lem: ErrorAccumulation_Iterates_FedAvg_Multi_Omega}, we have:
\begin{align}\label{eq:omega_drift_fedavg_multi}
\sum_{t=\bar{t}_{s-1}}^{\bar{t}_s-1} B_t &\leq 2I^2\eta^2L^2 \sum_{\ell = \bar{t}_{s-1}}^{\bar{t}_s-1} D_{\ell} + 2I^2\gamma^2L^2 \sum_{\ell = \bar{t}_{s-1}}^{\bar{t}_s-1} B_{\ell}  + \frac{2I\sigma^2}{b_y} + 2I\zeta_g^2
\end{align}
Apply Lemma~\ref{lem: ErrorAccumulation_Iterates_FedAvg_Multi_U} recursively, we have:
\begin{align*}
A_{t} &\leq \sum_{\ell=\bar{t}_{s-1}}^{t-1} (1 + \frac{1}{I})^{t-\ell} \big(\frac{64I\tau^2C_f^2\zeta_{g, yy}^2}{\mu^2} + 32I\tau^2\zeta_f^2 + \frac{2\tau^2\sigma^2}{b_x}  + 128I^2\eta^2\tau^2\tilde{L}_2^2\sum_{\bar{\ell} = \bar{t}_{s-1}}^{\ell-1} D_{\bar{\ell}} + 128I^2\gamma^2\tau^2\tilde{L}_2^2\sum_{\bar{\ell} = \bar{t}_{s-1}}^{\ell-1} B_{\bar{\ell}}\big) \nonumber\\
&\leq \sum_{\ell=\bar{t}_{s-1}}^{t-1}\big(\frac{192I\tau^2C_f^2\zeta_{g, yy}^2}{\mu^2} + 96I\tau^2\zeta_f^2 + \frac{6\tau^2\sigma^2}{b_x}  + 384I^2\eta^2\tau^2\tilde{L}_2^2\sum_{\bar{\ell} = \bar{t}_{s-1}}^{\ell-1} D_{\bar{\ell}} + 384I^2\gamma^2\tau^2\tilde{L}_2^2\sum_{\bar{\ell} = \bar{t}_{s-1}}^{\ell-1} B_{\bar{\ell}}\big)
\end{align*}
where we denote $\tilde{L}_2^2 = (L^2 + \frac{2L_{y^2}^2C_f^2}{\mu^2})$, and the second inequality uses the fact that $t -l \le I$ and the inequality $log(1+ a/x) \leq a/x$ for $x > -a$, so we have $(1+a/x)^x \leq e^{a}$, Then we choose $a = 1$ and $x = I$. Finally, we use the fact that $e^{1} \leq 3$. Next, we sum from $\bar{t}_{s-1}$ to $\bar{t}_{s} - 1$ to have:
\begin{align}\label{eq:u_drift_fedavg_multi}
\sum_{t=\bar{t}_{s-1}}^{\bar{t}_s - 1} A_{t} &\leq \sum_{\ell=\bar{t}_{s-1}}^{\bar{t}_s-1}\big(\frac{48I^2\tau^2C_f^2\zeta_{g, yy}^2}{\mu^2} + 96I^2\tau^2\zeta_f^2 + \frac{6I\tau^2\sigma^2}{b_x}  + 24I^3\eta^2\tau^2\tilde{L}_2^2\sum_{\bar{\ell} = \bar{t}_{s-1}}^{\ell-1} D_{\bar{\ell}} + 24I^3\gamma^2\tau^2\tilde{L}_2^2\sum_{\bar{\ell} = \bar{t}_{s-1}}^{\ell-1} B_{\bar{\ell}}\big) \nonumber\\
&\leq \frac{192I^3\tau^2C_f^2\zeta_{g, yy}^2}{\mu^2} + 96I^3\tau^2\zeta_f^2 + \frac{6I^2\tau^2\sigma^2}{b_x}  + 384I^4\eta^2\tau^2\tilde{L}_2^2 \sum_{\ell=\bar{t}_{s-1}}^{\bar{t}_s-1}D_{\ell} + 384I^4\gamma^2\tau^2\tilde{L}_2^2 \sum_{\ell=\bar{t}_{s-1}}^{\bar{t}_s-1}B_{\ell}
\end{align}
Combine Eq.~\ref{eq:nu_drift_fedavg_multi}, Eq.~\ref{eq:omega_drift_fedavg_multi} and Eq.~\ref{eq:u_drift_fedavg_multi}, and we choose $\eta$, $\gamma$ and $\tau$ such that $I^2\gamma^2L^2 < \frac{1}{4}$ , $I^2\gamma^2\tilde{L}_1^2 < \frac{1}{64}$, $I^2\eta^2\tilde{L}_1^2 < \frac{1}{32}$, $I^2\eta^2L^2 < \frac{1}{16}$, $I^2\tau^2\tilde{L}_2^2 < \frac{1}{128^2}$, $I^2\tau^2L^2 < \frac{1}{48}$, we get the claim in the lemma, by using the fact that $\tilde{L}_1 > L$ and $\tilde{L}_2 > L$, we get the simplified condition in the lemma.

\end{proof}

\subsubsection{Descent Lemma}
\begin{lemma}
\label{lemma:hg_bound_fedavg_multi}
For all $t \in [\bar{t}_{s-1}, \bar{t}_s - 1]$, the iterates generated satisfy:
\begin{align*}
    \mathbb{E}\big\|  \nabla h(\bar{x}_{t})  - \mathbb{E}_{\xi}[\bar{\nu}_t]  \big\|^2 &\leq \frac{2\tilde{L}_1^2}{M}\sum_{m=1}^M \mathbb{E} \big[\big\|\bar{x}_t -  x^{(m)}_t\big\|^2 + 2\big\|\bar{y}_t  - y^{(m)}_t \big\|^2 + 2\big\|y_{\bar{x}_t}  - \bar{y}_t\big\|^2 \big]  + 4L^2\mathbb{E} \big\|u_{\bar{x}_t}  - \bar{u}_{t}\big\|^2
\end{align*}
where we denote $u_{\bar{x}_t} = [\nabla_{y^2} g(\bar{x}_t, y_{\bar{x}_t})]^{-1}\nabla_y f(\bar{x}_t, y_{\bar{x}_t})$ and $\tilde{L}_1^2 = \big(L^2 + \frac{2L_{xy}^2C_f^2}{\mu^2}\big)$ is a constant.
\end{lemma}

\begin{proof}
By $\nabla h(\bar{x}_{t})  = \Phi(\bar{x},y_{\bar{x}})$, we have:
\begin{align*}
      \mathbb{E}\big\|  \nabla h(\bar{x}_{t})  - \mathbb{E}_{\xi}[\bar{\nu}_t]  \big\|^2  & \leq \mathbb{E}\big\|  \nabla_x f(\bar{x}_t, y_{\bar{x}_t}) -  \nabla_{xy} g(\bar{x}_t, y_{\bar{x}_t})\times [\nabla_{y^2} g(\bar{x}_t, y_{\bar{x}_t})]^{-1} \nabla_y f(\bar{x}_t, y_{\bar{x}_t}) \nonumber\\
      &\qquad -  \frac{1}{M}\sum_{m=1}^M\big(\nabla_x f^{(m)}(x^{(m)}_t, y^{(m)}_t) - \nabla_{xy}g^{(m)}(x^{(m)}_t, y^{(m)}_t)u_{t}^{(m)}\big) \big\|^2 \nonumber\\
      &\leq 2\mathbb{E}\big\|  \nabla_x f(\bar{x}_t, y_{\bar{x}_t})  -  \frac{1}{M}\sum_{m=1}^M\nabla_x f^{(m)}(x^{(m)}_t, y^{(m)}_t) \big\|^2  \nonumber\\ &\qquad +  2\mathbb{E}\big\| \nabla_{xy} g(\bar{x}_t, y_{\bar{x}_t})\times [\nabla_{y^2} g(\bar{x}_t, y_{\bar{x}_t})]^{-1} \nabla_y f(\bar{x}_t, y_{\bar{x}_t}) \nonumber\\
      &\qquad\qquad\qquad -  \frac{1}{M}\sum_{m=1}^M\nabla_{xy}g^{(m)}(x^{(m)}_t, y^{(m)}_t)u_{t}^{(m)} \big\|^2
\end{align*}
We denote the two terms above as $T_1$, $T_2$ respectively. For the first term $T_1$, we have:
\begin{align*}
    T_1 &\leq \frac{2L^2}{M}\sum_{m=1}^M \mathbb{E} \big[\big\|\bar{x}_t -  x^{(m)}_t\big\|^2 + \big\|y_{\bar{x}_t}  - y^{(m)}_t \big\|^2 \big] \leq \frac{2L^2}{M}\sum_{m=1}^M \mathbb{E} \big[\big\|\bar{x}_t -  x^{(m)}_t\big\|^2 + 2\big\|\bar{y}_t  - y^{(m)}_t \big\|^2 + 2\big\|y_{\bar{x}_t}  - \bar{y}_t\big\|^2 \big]
\end{align*}
For the second term $T_2$, we have:
\begin{align*}
    T_2 &\leq \frac{4C_f^2}{\mu^2 M}\sum_{m=1}^M \mathbb{E}\big\| \nabla_{xy} g^{(m)}(\bar{x}_t, y_{\bar{x}_t}) -  \nabla_{xy}g^{(m)}(x^{(m)}_t, y^{(m)}_t)\big\|^2 \nonumber\\
    &\qquad + 4L^2\mathbb{E} \big\| [\nabla_{y^2} g(\bar{x}_t, y_{\bar{x}})]^{-1}\nabla_y f(\bar{x}_t, y_{\bar{x}_t}) - \bar{u}_{t} \big\|^2
\end{align*}
We denote the first term above as $T_{2,1}$. For the term $T_{2,1}$, we have:
\begin{align*}
    T_{2,1} \leq \frac{4L_{xy}^2C_f^2}{\mu^2 M}\sum_{m=1}^M \mathbb{E} \big[\big\|\bar{x}_t -  x^{(m)}_t\big\|^2 + 2\big\|\bar{y}_t  - y^{(m)}_t \big\|^2 + 2\big\|y_{\bar{x}_t}  - \bar{y}_t\big\|^2\big]
\end{align*}
Combine everything together, we get the claim in the lemma.
\end{proof}

\begin{lemma}
\label{lemma:desent_fedavg_multi}
For all $t \in [\bar{t}_{s-1} + 1, \bar{t}_s - 1]$ and $s \in [S]$, suppose $\eta < \frac{1}{2\bar{L}}$, the iterates generated satisfy:
\begin{align*}
     \mathbb{E}[h(\bar{x}_{t + 1})] & \leq \mathbb{E}[h(\bar{x}_{t })] - \frac{\eta}{2} \mathbb{E}\|\nabla h(\bar{x}_{t})\|^2 - \frac{\eta}{4}\| \mathbb{E}_{\xi}[\nu^{(m)}_t]  \|^2 + \frac{\eta^2 \bar{L}\sigma^2}{2b_xM} + 2\eta L^2\mathbb{E} \big\|u_{\bar{x}_t}  - \bar{u}_{t}\big\|^2 \nonumber\\
    &\qquad + \frac{\eta\tilde{L}_1^2}{M}\sum_{m=1}^M \mathbb{E} \big[\big\|\bar{x}_t -  x^{(m)}_t\big\|^2 + 2\big\|\bar{y}_t  - y^{(m)}_t \big\|^2 + 2\big\|y_{\bar{x}_t}  - \bar{y}_t\big\|^2 \big]
\end{align*}
where the expectation is w.r.t the stochasticity of the algorithm.
\end{lemma}
\begin{proof}
Using the smoothness of $f$ we have:
\begin{align*}
    \mathbb{E} [h(\bar{x}_{t + 1})]
    & \leq   \mathbb{E} [h(\bar{x}_{t })] + \mathbb{E}\langle \nabla h(\bar{x}_{t}),  \bar{x}_{t + 1} - \bar{x}_{t}\rangle + \frac{\bar{L}}{2} \mathbb{E}\| \bar{x}_{t + 1} - \bar{x}_{t } \|^2  \nonumber\\
    &  =  \mathbb{E} [h(\bar{x}_{t})] - \eta \mathbb{E}\langle \nabla h(\bar{x}_{t}),  \mathbb{E}_{\xi}[\bar{\nu}_t] \rangle + \frac{\eta^2 \bar{L}}{2} \mathbb{E}\| \mathbb{E}_{\xi}[\bar{\nu}_{t}]  \|^2 +  \frac{\eta^2 \bar{L}\sigma^2}{2b_xM} \nonumber\\
    &  \overset{(a)}{=}  \mathbb{E}[h(\bar{x}_{t})] - \frac{\eta}{2} \mathbb{E}\|\nabla h(\bar{x}_{t})\|^2  + \frac{\eta}{2}\mathbb{E}\|\nabla h(\bar{x}_{t}) - \mathbb{E}_{\xi}[\bar{\nu}_t]\|^2   - \left(\frac{\eta}{2} - \frac{\eta^2 \bar{L}}{2}\right) \| \mathbb{E}_{\xi}[\bar{\nu}_t]  \|^2 +  \frac{\eta^2 \bar{L}\sigma^2}{2b_xM}\nonumber\\
    & \overset{(b)}{\leq} \mathbb{E}[h(\bar{x}_{t })] - \frac{\eta}{2} \mathbb{E}\|\nabla h(\bar{x}_{t})\|^2 - \frac{\eta}{4}\| \mathbb{E}_{\xi}[\nu^{(m)}_t]  \|^2 + \frac{\eta^2 \bar{L}\sigma^2}{2b_xM} \nonumber\\
    &\qquad +\frac{\eta\tilde{L}_1^2}{M}\sum_{m=1}^M \mathbb{E} \big[\big\|\bar{x}_t -  x^{(m)}_t\big\|^2 + 2\big\|\bar{y}_t  - y^{(m)}_t \big\|^2 + 2\big\|y_{\bar{x}_t}  - \bar{y}_t\big\|^2 \big]  + 2\eta L^2\mathbb{E} \big\|u_{\bar{x}_t}  - \bar{u}_{t}\big\|^2
\end{align*}
where equality $(a)$ uses $\langle a , b \rangle = \frac{1}{2} [\|a\|^2 + \|b\|^2 - \|a - b \|^2]$;  (b) follows the assumption that  $\eta < 1/2\bar{L}$ and Lemma~\ref{lemma:hg_bound_fedavg_multi}.
\end{proof}

\subsubsection{Proof of Convergence Theorem}\label{appendix:fedbio-multi}
We first denote the following potential function $\mathcal{G}(t)$:
\begin{align*}
    \mathcal{G}_t &= \mathbb{E}[h(\bar{x}_{t})] + \frac{9\eta\tilde{L}_1^2}{\mu\gamma} \mathbb{E}\big\|\bar{y}_t - y_{\bar{x}_{t}} \big\|^2 + \frac{9\eta L^2}{\mu\tau} \mathbb{E}\big\|\bar{u}_t - u_{\bar{x}_t} \big\|^2
\end{align*}

\begin{theorem}
\label{theorem:FedBiO-multi}
Suppose we choose $\tau = \min(\frac{1}{128I\tilde{L}_2}, \frac{1}{144\kappa L})$, then denote $\bar{\gamma} = \min(\frac{1}{8I\tilde{L}_2}, \frac{\tau}{36\kappa\bar{L}}, \frac{1}{4\bar{L}}, \frac{1}{8I\tilde{L}_1})$, if we choose $\eta = \frac{\mu\gamma}{36\kappa \tilde{L}_1}$,  and $\gamma = \min\big(\bar{\gamma}, \left(\frac{\Delta^{'}}{C_{\gamma}^{'}T}\right)^{1/3}\big)$ and $r =\frac{C_f}{\mu}$ where $\Delta^{'}$ and $C_\gamma^{'}$ are constants denoted in Eq.~\ref{eq:constant-multi}, then we have:
\begin{align*}
   \frac{1}{T} \sum_{t = 1}^{T} \mathbb{E}\|\nabla h(\bar{x}_{t})\|^2  =  O\left(\frac{\kappa^8}{T} + \left(\frac{\kappa^{12}}{T^{2}}\right)^{1/3} + \frac{\kappa^4\sigma^2}{b_y M} + \frac{\sigma^2}{b_x M}\right)
\end{align*}
and to reach an $\epsilon$ stationary point, we choose the inner batch size $b_y = O(M^{-1}\kappa^4\epsilon^{-1})$, upper batch size $b_x = O(M^{-1}\epsilon^{-1})$, and $T = O(\kappa^6\epsilon^{-1.5})$ number of iterations.
\end{theorem}

\begin{proof}
Similar to Lemma~\ref{lemma:d_bound_fedavg_multi}, we denote $D_t = \frac{1}{M} \sum_{m = 1}^M \| \nu_t^{(m)} -\bar{\nu}_t\|^2$, $B_t = \frac{1}{M} \sum_{m=1}^M \|\omega^{(m)}_t - \bar{\omega}_{t}\|^2 $, $C_t =  \mathbb{E}\| \bar{y}_t -  y_{\bar{x}_t}\|^2$ and $A_t = \frac{1}{M} \sum_{m=1}^M \mathbb{E}\|u^{(m)}_t - \bar{u}_t\|^2$, additionally, we denote $E_t = \|\mathbb{E}_{\xi}[\bar{\nu}_{t}]\|^2$ and $F_t = \mathbb{E}\| \bar{u}_t - u_{\bar{x}_t}\|^2$. Combine Lemma~\ref{lemma: inner_drift_fedavg_multi}, Lemma~\ref{lemma: hyper_drift_fedavg_multi} and the definition of the potential function we have:
\begin{align*}
    \mathcal{G}_{\bar{t}_s} - \mathcal{G}_{\bar{t}_{s-1}} &\leq -\frac{\eta}{2} \sum_{t = \bar{t}_{s-1}}^{\bar{t}_s-1} \mathbb{E}\|\nabla h(\bar{x}_{t})\|^2 - \frac{\eta L^2}{4} \sum_{t=\bar{t}_{s-1}}^{\bar{t}_s-1}F_{t} -\frac{\eta\tilde{L}_1^2}{4} \sum_{t=\bar{t}_{s-1}}^{\bar{t}_s-1}C_{t} \nonumber\\
    &\qquad - \frac{\eta}{4}\left(1 - \frac{162\kappa^2\eta^2\tilde{L}_1^2}{\mu^2\gamma^2} - \frac{180\kappa^2\eta^2\bar{L}^2}{\tau^2}\right)\sum_{t=\bar{t}_{s-1}}^{\bar{t}_s-1}E_{t}  + \big(\tilde{L}_1^2 + \frac{81\kappa^2\tilde{L}_1^2}{2}  + \frac{45\kappa^2\tilde{L}_2^2}{16}\big)I^2\eta^3\sum_{t=\bar{t}_{s-1}}^{\bar{t}_s-1}D_{t} \nonumber\\
    &\qquad + \big(2\tilde{L}_1^2 + \frac{81\kappa^2\tilde{L}_1^2}{2}  + \frac{45\kappa^2\tilde{L}_2^2}{16}\big)I^2\gamma^2\eta\sum_{t = \bar{t}_{s-1}}^{\bar{t}_s-1}B_{t}  + \frac{81I\kappa^2\tilde{L}_1^2\eta^3\sigma^2}{2\mu^2\gamma^2 b_x M} + \frac{36I\tilde{L}_1^2\eta\sigma^2}{\mu^2 b_yM} \nonumber\\
    &\qquad + \frac{45I\kappa^2\bar{L}^2\eta^3\sigma^2}{\tau^2 b_x M} + \frac{45I\kappa L\tau\eta\sigma^2}{4 b_x M}  +  \frac{\eta^2 I\bar{L}\sigma^2}{2b_x M}
\end{align*}
to bound the coefficients above, we choose $\eta \leq \min\big(\frac{\mu\gamma}{36\kappa\tilde{L}_1}, \frac{\tau}{36\kappa\bar{L}}, \frac{1}{4\bar{L}}\big)$, $\tau < \frac{1}{144\kappa L}$ and we denote $C_1 = \big(2\tilde{L}_1^2 + \frac{81\kappa^2\tilde{L}_1^2}{2}  + \frac{45\kappa^2\tilde{L}_2^2}{16}\big)I^2$. Then we have:
\begin{align*}
    \mathcal{G}_{\bar{t}_s} - \mathcal{G}_{\bar{t}_{s-1}} &\leq -\frac{\eta}{2} \sum_{t = \bar{t}_{s-1}}^{\bar{t}_s-1} \mathbb{E}\|\nabla h(\bar{x}_{t})\|^2 - \frac{\eta L^2}{4} \sum_{t=\bar{t}_{s-1}}^{\bar{t}_s-1}F_{t} -\frac{\eta\tilde{L}_1^2}{4} \sum_{t=\bar{t}_{s-1}}^{\bar{t}_s-1}C_{t}  - \frac{\eta}{8}\sum_{t=\bar{t}_{s-1}}^{\bar{t}_s-1}E_{t} \nonumber\\
    &\qquad + C_1\eta^3\sum_{t=\bar{t}_{s-1}}^{\bar{t}_s-1}D_{t}  + C_1\gamma^2\eta\sum_{t = \bar{t}_{s-1}}^{\bar{t}_s-1}B_{t}  + \frac{I\eta\sigma^2}{2b_x M} + \frac{36I\tilde{L}_1^2\eta\sigma^2}{\mu^2 b_yM} 
\end{align*}
Next, we combine with lemma~\ref{lemma:d_bound_fedavg_multi} to have:
\begin{align*}
    \mathcal{G}_{\bar{t}_s} - \mathcal{G}_{\bar{t}_{s-1}} &\leq -\frac{\eta}{2} \sum_{t = \bar{t}_{s-1}}^{\bar{t}_s-1} \mathbb{E}\|\nabla h(\bar{x}_{t})\|^2 - \frac{\eta L^2}{4} \sum_{t=\bar{t}_{s-1}}^{\bar{t}_s-1}F_{t} -\frac{\eta\tilde{L}_1^2}{4} \sum_{t=\bar{t}_{s-1}}^{\bar{t}_s-1}C_{t}  - \frac{\eta}{8}\sum_{t=\bar{t}_{s-1}}^{\bar{t}_s-1}E_{t} + \frac{I\eta\sigma^2}{2b_x M} + \frac{36I\tilde{L}_1^2\eta\sigma^2}{\mu^2 b_yM} \nonumber\\
    & \qquad + C_1\eta^3\big(96I\zeta_f^2 + 16I\zeta_g^2 + \frac{16IC_f^2\zeta_{g, yy}^2}{\mu^2} +  \frac{32IC_f^2\zeta_{g,xy}^2}{\mu^2} + \frac{16I\sigma^2}{b_y} + \frac{20I\sigma^2}{b_x}\big) \nonumber\\
    &\qquad + C_1\gamma^2\eta\big( 24I\zeta_f^2  + 8I\zeta_g^2 + \frac{4IC_f^2\zeta_{g, yy}^2}{\mu^2} +  \frac{8IC_f^2\zeta_{g,xy}^2}{\mu^2} + \frac{8I\sigma^2}{b_y} + \frac{5I\sigma^2}{b_x}\big)
\end{align*}
Sum over all $s \in [S]$ (assume $T = SI + 1$ without loss of generality) to obtain:
\begin{align*}
    \frac{\eta}{2} \sum_{t = 1}^{T} \mathbb{E}\|\nabla h(\bar{x}_{t})\|^2 &\leq \mathcal{G}_{1} - \mathcal{G}_{T} + \frac{T\eta\sigma^2}{2b_x M} + \frac{36T\tilde{L}_1^2\eta\sigma^2}{\mu^2 b_yM} \nonumber\\
    & \qquad + C_1\eta^3\big(96T\zeta_f^2 + 16T\zeta_g^2 + \frac{16TC_f^2\zeta_{g, yy}^2}{\mu^2} +  \frac{32TC_f^2\zeta_{g,xy}^2}{\mu^2} + \frac{16T\sigma^2}{b_y} + \frac{20T\sigma^2}{b_x}\big) \nonumber\\
    &\qquad + C_1\gamma^2\eta\big( 24T\zeta_f^2  + 8T\zeta_g^2 + \frac{4TC_f^2\zeta_{g, yy}^2}{\mu^2} +  \frac{8TC_f^2\zeta_{g,xy}^2}{\mu^2} + \frac{8T\sigma^2}{b_y} + \frac{5T\sigma^2}{b_x}\big) \nonumber\\
    &\leq \Delta + \frac{9\eta\tilde{L}_1^2\Delta_y}{\mu\gamma} + \frac{9\eta L^2\Delta_u}{\mu\tau} + \frac{T\eta\sigma^2}{2b_x M} + \frac{36T\tilde{L}_1^2\eta\sigma^2}{\mu^2 b_yM} \nonumber\\
    & \qquad + C_1\eta^3\big(96T\zeta_f^2 + 16T\zeta_g^2 + \frac{16TC_f^2\zeta_{g, yy}^2}{\mu^2} +  \frac{32TC_f^2\zeta_{g,xy}^2}{\mu^2} + \frac{16T\sigma^2}{b_y} + \frac{20T\sigma^2}{b_x}\big) \nonumber\\
    &\qquad + C_1\gamma^2\eta\big( 24T\zeta_f^2  + 8T\zeta_g^2 + \frac{4TC_f^2\zeta_{g, yy}^2}{\mu^2} +  \frac{8TC_f^2\zeta_{g,xy}^2}{\mu^2} + \frac{8T\sigma^2}{b_y} + \frac{5T\sigma^2}{b_x}\big)
\end{align*}
we define $\Delta = h(x_{1}) - h^*$ as the initial sub-optimality of the function, $\Delta_y = \|y_1 - y_{x_{1}}\|^2$ as the initial sub-optimality of the inner variable estimation, $\Delta_u =  \|u_1 - u_{x_{1}}\|^2$ as the initial sub-optimality of the hyper-gradient estimation. Then we divide by $\eta T/2$ on both sides and have:
\begin{align*}
    \frac{1}{T} \sum_{t = 1}^{T} \mathbb{E}\|\nabla h(\bar{x}_{t})\|^2 &\leq \frac{2\Delta}{\eta T} + \frac{18\tilde{L}_1^2\Delta_y}{\mu\gamma T} + \frac{18 L^2\Delta_u}{\mu\tau T} + \frac{\sigma^2}{b_x M} + \frac{72\tilde{L}_1^2\sigma^2}{\mu^2 b_yM} \nonumber\\
    & \qquad + 2C_1\eta^2\big(96\zeta_f^2 + 16\zeta_g^2 + \frac{16C_f^2\zeta_{g, yy}^2}{\mu^2} +  \frac{32C_f^2\zeta_{g,xy}^2}{\mu^2} + \frac{16\sigma^2}{b_y} + \frac{20\sigma^2}{b_x}\big) \nonumber\\
    &\qquad + 2C_1\gamma^2\big( 24\zeta_f^2  + 8\zeta_g^2 + \frac{4C_f^2\zeta_{g, yy}^2}{\mu^2} +  \frac{8C_f^2\zeta_{g,xy}^2}{\mu^2} + \frac{8\sigma^2}{b_y} + \frac{5\sigma^2}{b_x}\big)
\end{align*}
For ease of notation, we denote constants $C_{\eta} = 2C_1(96\zeta_f^2 + 16\zeta_g^2 + \frac{16C_f^2\zeta_{g, yy}^2}{\mu^2} +  \frac{32C_f^2\zeta_{g,xy}^2}{\mu^2} + \frac{16\sigma^2}{b_y} + \frac{20\sigma^2}{b_x}) $ and $C_{\gamma} = 2C_1(24\zeta_f^2  + 8\zeta_g^2 + \frac{4C_f^2\zeta_{g, yy}^2}{\mu^2} +  \frac{8C_f^2\zeta_{g,xy}^2}{\mu^2} + \frac{8\sigma^2}{b_y} + \frac{5\sigma^2}{b_x})$, we have:
\begin{align}\label{eq:fedavg_multi}
\frac{1}{T} \sum_{t = 1}^{T} \mathbb{E}\|\nabla h(\bar{x}_{t})\|^2 &\leq \frac{2\Delta}{\eta T} + \frac{18\tilde{L}_1^2\Delta_y}{\mu\gamma T} + \frac{18 L^2\Delta_u}{\mu\tau T} + \frac{\sigma^2}{b_x M} + \frac{72\tilde{L}_1^2\sigma^2}{\mu^2 b_yM} + C_\eta\eta^2 + C_{\gamma}\gamma^2
\end{align}
Recall that, we have the condition that $\eta \leq \min\big(\frac{1}{8I\tilde{L}_2}, \frac{\mu\gamma}{36\kappa\tilde{L}_1}, \frac{\tau}{36\kappa\bar{L}}, \frac{1}{4\bar{L}}\big)$, $\gamma \leq \frac{1}{8I\tilde{L}_1}$, $\tau \leq \min(\frac{1}{128I\tilde{L}_2}, \frac{1}{144\kappa L})$. Suppose we choose $\tau = \min( \frac{1}{128I\tilde{L}_2}, \frac{1}{144\kappa L})$, then denote \[\bar{\gamma} = \min( \frac{1}{8I\tilde{L}_2}, \frac{\tau}{36\kappa\bar{L}}, \frac{1}{4\bar{L}}, \frac{1}{8I\tilde{L}_1}),\] and let $\gamma \leq \bar{\gamma}$, and $\eta = \frac{\mu\gamma}{36\kappa \tilde{L}_1}$, then we have:
\begin{align*}
    \frac{1}{T} \sum_{t = 1}^{T} \mathbb{E}\|\nabla h(\bar{x}_{t})\|^2 &\leq \frac{72\kappa \tilde{L}_1\Delta + 18\tilde{L}_1^2\Delta_y}{ \mu\gamma T} +  \big(\frac{C_{\eta}\mu^2}{36^2\kappa^2 \tilde{L}_1^2} + C_{\gamma}\big)\gamma^2 + \frac{18 L^2\Delta_u}{\mu\tau T} + \frac{\sigma^2}{b_x M} + \frac{72\tilde{L}_1^2\sigma^2}{\mu^2 b_yM}
\end{align*}
We denote 
\begin{align}\label{eq:constant-multi}
C_{\gamma}^{'} = \big(\frac{C_{\eta}\mu^2}{36^2\kappa^2 \tilde{L}_1^2}  + C_{\gamma}\big),\; \Delta^{'} = \frac{72\kappa \tilde{L}_1\Delta + 18\tilde{L}_1^2\Delta_y}{\mu},
\end{align}
then we choose $\gamma$ as:
\[
\gamma = \min\big(\bar{\gamma}, \left(\frac{\Delta^{'}}{C_{\gamma}^{'}T}\right)^{1/3}\big)
\] 
and obtain:
\begin{align*}
    \frac{1}{T} \sum_{t = 1}^{T} \mathbb{E}\|\nabla h(\bar{x}_{t})\|^2 
    &\leq \frac{\Delta^{'}}{\bar{\gamma}T} + \left(\frac{C_{\gamma}^{'}(\Delta^{'})^2}{T^2}\right)^{1/3} + \frac{18 L^2\Delta_u}{\mu\tau T} + \frac{\sigma^2}{b_x M} + \frac{72\tilde{L}_1^2\sigma^2}{\mu^2 b_yM}
\end{align*}
Finally, since $\tilde{L}_1 = O(\kappa)$ , $\bar{L} = O(\kappa^3)$, suppose we choose $I = O(1)$, then we have $\tau^{-1} = O(\kappa)$, $\bar{\gamma}^{-1} = O(\kappa^{5})$, $\Delta^{'} = O(\kappa^3)$, $C_1 = O(\kappa^4)$, $C_\eta = O(\kappa^{6})$, $C_{\gamma} = O(\kappa^6)$, $C_{\gamma}^{'} = O(\kappa^6)$ then we have:
\begin{align*}
   \frac{1}{T} \sum_{t = 1}^{T} \mathbb{E}\|\nabla h(\bar{x}_{t})\|^2  =  O\left(\frac{\kappa^8}{T} + \left(\frac{\kappa^{12}}{T^{2}}\right)^{1/3} + \frac{\kappa^4\sigma^2}{b_y M} + \frac{\sigma^2}{b_x M}\right)
\end{align*}
and to reach an $\epsilon$ stationary point, we choose the inner batch size $b_y = O(M^{-1}\kappa^4\epsilon^{-1})$, upper batch size $b_x = O(M^{-1}\epsilon^{-1})$, and $T = O(\kappa^6\epsilon^{-1.5})$ number of iterations.
\end{proof}

\newpage

\section{Proof for Local Lower Level Problem}
The FedBiOAcc-Local and FedBiO-Local are presented in Algorithm~\ref{alg:FedBiOAcc-local} and Algorithm~\ref{alg:FedBiO-Local}, respectively. Then in this section, we discuss the convergence rate of the two algorithms. Please see Theorem~\ref{theorem:FedBiOAcc} and Theorem~\ref{theorem:FedBiO} for the convergence rates.

For Eq.~\eqref{eq:fed-bi}, we also assume Assumptions~\ref{assumption:function} -\ref{assumption:noise_assumption}, with a slightly different assumption to the heterogeneity as follows:

\begin{algorithm}[t]
\caption{FedBiO- Local Lower Level Problem}
\label{alg:FedBiO-Local}
\begin{algorithmic}[1]
\STATE {\bfseries Input:} Initial states $x_1$, $y_1$; learning rates $\{\gamma_t, \eta_t\}_{t=1}^T$
\STATE {\bfseries Initialization:} Set $x^{(m)}_1 = x_1$, $y^{(m)}_1 = y_1$;
\FOR{$t=1$ \textbf{to} $T$}
\STATE Randomly sample mutually independent minibatch of samples $\mathcal{B}_{y}$ and $\mathcal{B}_{x}$ of size b;
\STATE $\omega_{t}^{(m)} = \nabla_y g^{(m)} (x^{(m)}_{t}, y^{(m)}_{t}, \mathcal{B}_y)$
and $\nu^{(m)}_t = \Phi^{(m)}(x^{(m)}, y^{(m)}; \mathcal{B}_x)$;
\STATE $\hat{y}^{(m)}_{t+1} = y^{(m)}_{t} - \gamma_t  \omega_{t}^{(m)}$, $\hat{x}^{(m)}_{t+1} = x^{(m)}_{t} - \eta_t \nu_{t}^{(m)}$;
\IF{$t$ mod I $ = 0$}
\STATE $y^{(m)}_{t+1} = \hat{y}^{(m)}_{t+1}$, $x^{(m)}_{t+1} = \frac{1}{M}\sum_{j=1}^{M} \hat{x}^{(j)}_{t+1}$
\ELSE
\STATE $y^{(m)}_{t+1} = \hat{y}^{(m)}_{t+1}$, $x^{(m)}_{t+1} = \hat{x}^{(m)}_{t+1}$
\ENDIF
\ENDFOR
\end{algorithmic}
\end{algorithm}

\begin{algorithm}[t]
\caption{FedBiOAcc - Local Lower Level Problem}
\label{alg:FedBiOAcc-local}
\begin{algorithmic}[1]
\STATE {\bfseries Input:} Constants $c_{\omega}$, $c_{\nu}$, $\gamma$, $\eta$; learning rate schedule $\{\alpha_t\}$, $t \in [T]$, initial state ($x_1$, $y_1$);
\STATE{\bfseries Initialization:} Set $y^{(m)}_{1} = y_{1}$, $x^{(m)}_{1} = x_{1}$, $\omega_{1}^{(m)} = \nabla_y g^{(m)} (x_{1}, y_{1}, \mathcal{B}_{y})$, $\nu_{1}^{(m)} = \Phi^{(m)}(x_1, y_1; \mathcal{B}_x)$ for $m \in [M]$
\FOR{$t=1$ \textbf{to} $T$}
\STATE $\hat{y}^{(m)}_{t+1} = y^{(m)}_{t} - \gamma\alpha_{t}  \omega_{t}^{(m)}$, $\hat{x}^{(m)}_{t+1} = x^{(m)}_{t} - \eta\alpha_{t} \nu_{t}^{(m)}$, $\hat{u}_{t+1}^{(m)} = u^{(m)}_t - \tau\alpha_tq^{(m)}_t$
\IF{$t$ mod I $ = 0$}
\STATE $y^{(m)}_{t+1} = \hat{y}^{(m)}_{t+1}$, $x^{(m)}_{t+1} = \frac{1}{M}\sum_{j=1}^{M} \hat{x}^{(j)}_{t+1}$
\ELSE
\STATE $y^{(m)}_{t+1} = \hat{y}^{(m)}_{t+1}$, $x^{(m)}_{t+1} = \hat{x}^{(m)}_{t+1}$, 
\ENDIF
\STATE Randomly sample minibatches $\mathcal{B}_{y}$ and $\mathcal{B}_{x}$
\STATE $\hat{\omega}_{t+1}^{(m)} = \nabla_y g^{(m)} (x^{(m)}_{t+1}, y^{(m)}_{t+1}, \mathcal{B}_{y}) + (1 - c_{\omega}\alpha_{t}^2) (\omega_{t}^{(m)} - \nabla_y g^{(m)} (x^{(m)}_{t}, y^{(m)}_{t}, \mathcal{B}_{y}))$
\STATE $\hat{\nu}_{t+1}^{(m)} = \Phi^{(m)}(x^{(m)}_{t+1}, y^{(m)}_{t+1}; \mathcal{B}_x) + (1 - c_{\nu}\alpha_{t}^2) (\nu_{t}^{(m)} - \Phi^{(m)}(x^{(m)}_{t}, y^{(m)}_{t}; \mathcal{B}_x))$
\IF{$t$ mod I $ = 0$}
\STATE $\omega^{(m)}_{t+1} = \hat{\omega}^{(m)}_{t+1}$, $\nu^{(m)}_{t+1} =  \frac{1}{M}\sum_{j=1}^{M} \hat{\nu}^{(j)}_{t+1} $,
\ELSE
\STATE  $\omega^{(m)}_{t+1} = \hat{\omega}^{(m)}_{t+1}$, $\nu^{(m)}_{t+1} = \hat{\nu}^{(m)}_{t+1}$
\ENDIF
\ENDFOR
\end{algorithmic}
\end{algorithm}

\begin{assumption} \label{assumption:hetero-local}
For any $m, j \in [M]$ and $z = (x,y)$, we have: $ \| \nabla f^{(m)} (z) -  \nabla f^{(j)} (z) \| \leq \zeta_f$, $ \| \nabla_{xy} g^{(m)} (z) -  \nabla_{xy} g^{(j)} (z) \| \leq \zeta_{g,xy}$, $ \| \nabla_{y^2} g^{(m)} (z) -  \nabla_{y^2} g^{(j)} (z) \| \leq \zeta_{g,yy}$, $ \| y^{(m)}_x -  y^{(j)}_x\| \leq \zeta_{g^{\ast}}$, where $\zeta_f$, $\zeta_{g,xy}$, $\zeta_{g,yy}$, $\zeta_{g^{\ast}}$ are constants.
\end{assumption}
Note that we remove the requirement of gradient dissimilarity $\zeta_g$ in Assumption~\ref{assumption:hetero} and add the dissimilarity bound $\zeta_{g^{\ast}}$ for the minimizer of the lower level problem. Note that Assumption~\ref{assumption:hetero-local} is a sufficient condition such that the dissimilarity of local hyper-gradient is bounded by some constant $\zeta$.

\begin{proposition} (Lemma 4 and~7 in~\cite{yang2021provably})
\label{prop:hg_var_app}
Suppose Assumptions \ref{assumption:f_smoothness}, \ref{assumption:g_smoothness} and \ref{assumption:noise_assumption} hold and $\tau < \frac{1}{L}$, the hypergradient estimator $\Phi(x, y;\mathcal{B}_x)$ w.r.t.\ x based on a minibatch $\mathcal{B}_x$ has bounded variance and bias:
\begin{itemize}
    \item [a)] $\|\mathbb{E}[\Phi^{(m)}(x, y;\mathcal{B}_x)] - \Phi^{(m)}(x, y)\| \leq G_1$, where $G_1 = \kappa(1 - \tau\mu)^{Q+1}C_f$
    \item [b)] $\mathbb{E} \|\Phi^{(m)}(x, y;\mathcal{B}_x)- \mathbb{E}[\Phi^{(m)}(x, y;\mathcal{B}_x)]\|^2 \leq G_2^2$, where $G_2^2 = (2C_f^2+ 12C_f^2L^2\tau^2(Q+1)^2+ 4C_f^2L^2(Q+2)(Q+1)^2\tau^4\sigma^2)/b_{x}$ 
\end{itemize}
\end{proposition}

\begin{proposition}
\label{some smoothness local}
Suppose Assumptions~\ref{assumption:f_smoothness} and ~\ref{assumption:g_smoothness} hold, the following statements hold:
\begin{itemize}
\item [a)] $y^{(m)}_x$ is Lipschitz continuous in $x$ with constant $\rho = \kappa$, where $\kappa = \frac{L}{\mu}$ is the condition number of $g^{(m)}(x,y)$.

\item [b)] $\|\Phi^{(m)}(x_1; y_1) - \Phi^{(m)}(x_2; y_2)\|^2 \leq \hat{L}^2 (\|x_1- x_2\|^2 + \|y_1- y_2\|^2)$, where $\hat{L} = O(\kappa^2)$.

\item [d)] $h^{(m)}(x)$ is Lipschitz continuous in $x$ with constant $\bar{L}$ i.e., for any given $x_1, x_2 \in X$, we have
$\|\nabla h^{(m)}(x_2) - \nabla h^{(m)}(x_1)\| \le \bar{L} \|x_2 - x_1\|$
where $\bar{L} =O(\kappa^3)$.
\end{itemize}
\end{proposition}
This is a standard results in bilevel optimization and we omit the proof here. 

\begin{proposition} \label{prop:7}
In Eq.~\ref{eq:fed-bi}, suppose Assumption~\ref{assumption:function},~\ref{assumption:f_smoothness},~\ref{assumption:g_smoothness},~\ref{assumption:hetero-local} hold, we have:
\begin{align*}
    \|  \nabla h^{(m)}(x) - \nabla h^{(j)}(x)  \|  & \leq (1 + \kappa)\zeta_{f} + \frac{C_f}{\mu}\zeta_{g,xy} + \frac{\kappa C_f}{\mu}\zeta_{g, yy} + \big((1 + \kappa)L + \frac{C_f L_{xy}}{\mu} + \frac{\kappa C_fL_{y^2}}{\mu}\big)\zeta_{g^{\ast}} \coloneqq \zeta
\end{align*}
\end{proposition}

\begin{proof}
For $h^{(m)} (x) = f^{(m)}(x , y_{x}^{(m)}), m \in [M]$ in Eq.~\ref{eq:fed-bi}, we have:
\begin{align*}
\|  \nabla h^{(m)}(x) - \nabla h^{(j)}(x)  \| &= \|\nabla_x f^{(m)}(x, y^{(m)}_x) -  \nabla_{xy} g^{(m)}(x, y^{(m)}_x)[\nabla_{y^2} g^{(m)}(x, y^{(m)}_x)]^{-1} \nabla_y f^{(m)}(x, y^{(m)}_x) \nonumber \\
& \qquad-  \left(\nabla_x f^{(j)}(x, y^{(j)}_x) -  \nabla_{xy} g^{(j)}(x, y^{(j)}_x) [\nabla_{y^2} g^{(j)}(x, y^{(j)}_x)]^{-1} \nabla_y f^{(j)}(x, y^{(j)}_x)\right)\| \nonumber \\
& \leq  \big\|\nabla_x f^{(m)} (x, y_x^{(m)}) - \nabla_x f^{(j)} (x, y^{(j)}_{x}) \big\|  + \big\| \nabla_{xy} g^{(m)}(x, y_x^{(m)}) \\
& \quad - \nabla_{xy} g^{(j)}(x, y^{(j)}_{x}) \big\| \big\|\big(\nabla_{yy} g^{(m)}(x, y^{(m)}_{x})\big)^{-1}\nabla_{y} f^{(m)}(x, y_x^{(m)})\big\| \nonumber \\
& \quad + \big\| \nabla_{xy} g^{(j)}(x, y^{(j)}_{x}) \big\|\big\| \big(\nabla_{yy} g^{(m)}(x, y^{(m)}_{x})\big)^{-1}\nabla_{y} f^{(m)}(x, y_x^{(m)}) \\ & \quad - \big(\nabla_{yy} g^{(j)}(x, y^{(j)}_{x})\big)^{-1}\nabla_{y} f^{(j)}(x, y^{(j)}_{x}) \big\|
\end{align*}
Next we bound the three terms separately. For the first term:
\begin{align}
    \big\|\nabla_x f^{(m)} (x, y_x^{(m)}) - \nabla_x f^{(j)} (x, y^{(j)}_{x}) \big\| 
    & \leq \big\|\nabla_x f^{(m)} (x, y_x^{(m)}) - \nabla_x f^{(j)} (x, y^{(m)}_{x}) \big\| \nonumber \\
    & \qquad + \big\|\nabla_x f^{(j)} (x, y_x^{(m)}) - \nabla_x f^{(j)} (x, y^{(j)}_{x}) \big\| \nonumber \\
    & \leq \zeta_{f} + L\big\| y_x^{(m)} - y^{(j)}_{x} \big\| \leq \zeta_{f} + L\zeta_{g^{\ast}}
\label{eq:diff1}
\end{align}
where the second inequality is due to Assumption~\ref{assumption:f_smoothness} and Assumption~\ref{assumption:hetero-local}. The last inequality also follows the Assumption~\ref{assumption:hetero-local}. Next, for the second term, we have:
\begin{align*}
  &\big\| \nabla_{xy} g^{(m)}(x, y_x^{(m)}) - \nabla_{xy} g^{(j)}(x, y^{(j)}_{x}) \big\| \big\|\big(\nabla_{yy} g^{(m)}(x, y^{(m)}_{x})\big)^{-1}\nabla_{y} f^{(m)}(x, y_x^{(m)})\big\| \nonumber \\
  & \leq \frac{C_f}{\mu}\big\| \nabla_{xy} g^{(m)}(x, y_x^{(m)}) - \nabla_{xy} g^{(j)}(x, y^{(j)}_{x}) \big\| \nonumber \\
  & \leq \frac{C_f}{\mu}\big\| \nabla_{xy} g^{(m)}(x, y_x^{(m)}) - \nabla_{xy} g^{(j)}(x, y^{(m)}_{x}) \big\| + \frac{C_f}{\mu}\big\| \nabla_{xy} g^{(j)}(x, y_x^{(m)}) - \nabla_{xy} g^{(j)}(x, y^{(j)}_{x}) \big\|  \nonumber \\
  & \leq \frac{C_f\zeta_{g,xy}}{\mu}+ \frac{C_fL_{xy}}{\mu}\big\| y_x^{(m)} -  y^{(j)}_{x}) \big\| \leq \frac{C_f\zeta_{g,xy}}{\mu}+ \frac{C_f L_{xy}\zeta_{g^\ast}}{\mu}
\end{align*}
where the first inequality follows from the Assumption~\ref{assumption:function},~\ref{assumption:f_smoothness}; the third inequality follows from Assumption~\ref{assumption:hetero-local},~\ref{assumption:g_smoothness}, the last inequality follows from Assumption~\ref{assumption:hetero-local}.  Next, for the third term, we have:
\begin{align*}
  &\big\| \nabla_{xy} g^{(j)}(x, y^{(j)}_{x}) \big\|\big\| \big(\nabla_{yy} g^{(m)}(x, y^{(m)}_{x})\big)^{-1}\nabla_{y} f^{(m)}(x, y_x^{(m)}) - \big(\nabla_{yy} g^{(j)}(x, y^{(j)}_{x})\big)^{-1}\nabla_{y} f^{(j)}(x, y^{(j)}_{x}) \big\| \nonumber \\ 
  & \leq L\big\| \big(\nabla_{yy} g^{(m)}(x, y^{(m)}_{x})\big)^{-1}\nabla_{y} f^{(m)}(x, y_x^{(m)}) - \big(\nabla_{yy} g^{(j)}(x, y^{(j)}_{x})\big)^{-1}\nabla_{y} f^{(j)}(x, y^{(j)}_{x}) \big\| \nonumber \\
  & \leq L\big\|\big(\nabla_{yy} g^{(m)}(x, y^{(m)}_{x})\big)^{-1}\big\|\big\|\nabla_{y} f^{(m)}(x, y_x^{(m)}) - \nabla_{y} f^{(j)}(x, y^{(j)}_{x}) \big\| \nonumber \\
  & \qquad + L\big\| \big(\nabla_{yy} g^{(m)}(x, y^{(m)}_{x})\big)^{-1} - \big(\nabla_{yy} g^{(j)}(x, y^{(j)}_{x})\big)^{-1} \big\|\big\|\nabla_{y} f^{(j)}(x, y^{(j)}_{x})\big\| \nonumber \\
  & \leq \frac{L}{\mu}\big\|\nabla_{y} f^{(m)}(x, y_x^{(m)}) - \nabla_{y} f^{(j)}(x, y^{(j)}_{x}) \big\| \nonumber \\
  & \qquad + C_fL\big\| \big(\nabla_{yy} g^{(m)}(x, y^{(m)}_{x})\big)^{-1} - \big(\nabla_{yy} g^{(j)}(x, y^{(j)}_{x})\big)^{-1} \big\|\nonumber \\
  & \leq \frac{L(\zeta_{f} + L\zeta_{g^{\ast}})}{\mu} + C_fL\big\|\big(\nabla_{yy} g^{(m)}(x, y^{(m)}_{x})\big)^{-1}\big\|\times \nonumber \\
  & \qquad \qquad \qquad \big\| \nabla_{yy} g^{(m)}(x, y^{(m)}_{x}) - \nabla_{yy} g^{(j)}(x, y^{(j)}_{x}) \big\|\big\|\big(\nabla_{yy} g^{(j)}(x, y^{(j)}_{x})\big)^{-1}\big\| \\
  & \leq \frac{L(\zeta_{f} + L\zeta_{g^{\ast}})}{\mu} + \frac{C_fL(\zeta_{g, yy} + L_{y^2}\zeta_{g^\ast})}{\mu^2}
\end{align*}
where the first inequality is by Assumption~\ref{assumption:g_smoothness}; the third inequality is by Assumption~\ref{assumption:g_smoothness},~\ref{assumption:f_smoothness}; the fourth inequality is by Cauchy Schwartz inequality; the last inequality is by Assumption~\ref{assumption:function},~\ref{assumption:g_smoothness} and the result in Eq.~\ref{eq:diff1}. Combine everything together, we have:
\begin{align*}
    \big\|\nabla_x f^{(m)} (x, y_x^{(m)}) - \nabla_x f^{(j)} (x, y^{(j)}_{x}) \big\| & \leq \zeta_{f} + L\zeta_{g^{\ast}} + \frac{C_f\zeta_{g,xy}}{\mu}+ \frac{C_fL_{xy}\zeta_{g^\ast}}{\mu} + \frac{L(\zeta_{f} + L\zeta_{g^{\ast}})}{\mu} \nonumber \\
    & \qquad \qquad+ \frac{C_fL(\zeta_{g, yy} + L_{y^2}\zeta_{g^\ast})}{\mu^2}
\end{align*}
which completes the proof.
\end{proof}

\subsection{Proof for the FedBiOAcc-Local Algorithm}
\subsubsection{Hyper-Gradient Bias and Inner-Gradient Bias}
\begin{lemma}
\label{lemma:hg_bound_storm}
Suppose we have $c_\nu\alpha_t^2 < 1$, then:
\begin{align*}
  \mathbb{E} \big[ \big\| \bar{\nu}_{t} -  \mathbb{E}_{\xi}[\bar{\mu}_{t,\mathcal{B}_x}]\big\|^2 \big] &\leq ( 1 - c_{\nu}\alpha_{t-1}^2)\mathbb{E} \big[ \big\| \bar{\nu}_{t-1} - \mathbb{E}_{\xi}[\bar{\mu}_{t-1,\mathcal{B}_x}]\big\|^2 \big] + \frac{2c_{\nu}^2\alpha_{t-1}^4}{b_xM}G_2^2\\
  &\qquad  +  \frac{2\hat{L}^2}{b_xM^2}\sum_{m=1}^M\mathbb{E}\big[\big\|x^{(m)}_{t} - x^{(m)}_{t-1}\big\|^2 + \big\|y^{(m)}_{t} - y^{(m)}_{t-1}\big\|^2\big]
\end{align*}
where $\mu_{t, \xi}^{(m)} = \Phi^{(m)} (x^{(m)}_{t}, y^{(m)}_{t}; \mathcal{B}_{x})$ and the expectation outside is \emph{w.r.t} all the stochasity of the algorithm.
\end{lemma}

\begin{proof}
For ease of notation, we denote $\mu_{t, \xi}^{(m)} = \Phi^{(m)} (x^{(m)}_{t}, y^{(m)}_{t}; \xi_{x})$, and $\mu_{t}^{(m)} = \Phi^{(m)} (x^{(m)}_{t}, y^{(m)}_{t})$, then by the definition of $\bar{\nu}_t$ we have:
\begin{align*}
    &\mathbb{E} \big[ \big\| \bar{\nu}_{t} - \mathbb{E}_{\xi}[\bar{\mu}_{t, \mathcal{B}_x}] \big\|^2 \big]
    = \mathbb{E}  \big[\big\| \frac{1}{M}\sum_{m=1}^M \big(\hat{\nu}^{(m)}_t -  \mathbb{E}_{\xi}[\mu_{t, \mathcal{B}_x}^{(m)}] \big) \big\|^2\big] \nonumber \\
    & =  \mathbb{E} \big[ \big\|\frac{1}{M}\sum_{m=1}^M \big(\mu_{t, \mathcal{B}_x}^{(m)} + ( 1 -  c_{\nu}\alpha_{t-1}^2) (\nu_{t-1}^{(m)} - \mu_{t-1, \mathcal{B}_x}^{(m)}) - \mathbb{E}_{\xi}[\mu_{t, \mathcal{B}_x}^{(m)}]\big) \big\|^2 \big] \nonumber \\
    & = \mathbb{E} \big[ \big\| ( 1 - c_{\nu}\alpha_{t-1}^2) \big(\bar{\nu}_{t-1} - \mathbb{E}_{\xi}[\bar{\mu}_{t-1, \mathcal{B}_x}]\big) + \big(\bar{\mu}_{t, \mathcal{B}_x} - \mathbb{E}_{\xi}[\bar{\mu}_{t, \mathcal{B}_x}] +  (1 - c_{\nu}\alpha_{t-1}^2)(\mathbb{E}_{\xi}[\bar{\mu}_{t-1, \mathcal{B}_x}] - \bar{\mu}_{t-1, \mathcal{B}_x}) \big) \big\|^2  \big] \nonumber \\
    & \leq ( 1 - c_{\nu}\alpha_{t-1}^2)\mathbb{E} \big[ \big\| \bar{\nu}_{t-1} - \mathbb{E}_{\xi}[\bar{\mu}_{t-1, \mathcal{B}_x}] \big\|^2 \big] \nonumber\\
    &\qquad + \frac{1}{b_x^2M^2}\sum_{m=1}^M\sum_{\xi_x \in \mathcal{B}_x} \mathbb{E}\big[ \big\|\mu_{t, \xi_x}^{(m)} - \mathbb{E}_{\xi}[\mu_{t, \xi_x}^{(m)}] +  (1 - c_{\nu}\alpha_{t-1}^2)(\mathbb{E}_{\xi}[\mu_{t-1, \xi_x}^{(m)}] - \mu_{t-1, \xi_x}^{(m)})\big\|^2 \big]
\end{align*}
where inequality $(a)$ uses the fact that the cross product term is zero in expectation, the condition that $c_\nu\alpha_t^2 < 1$ and the fact that clients independently choose samples. 

We denote the second term above as $T_1$, then we have:
\begin{align*}
  T_1 & \overset{(a)}{\leq} 2(c_{\nu}\alpha_{t-1}^2)^2\mathbb{E} \big[ \big\| \mu_{t, \xi_x}^{(m)} - \mathbb{E}_{\xi}[\mu_{t, \xi_x}^{(m)}] \big\|^2 \big]  +  2(1 - c_{\nu}\alpha_{t-1}^2)^2 \mathbb{E} \big[ \big\|\mu_{t, \xi_x}^{(m)} - \mu_{t-1, \xi_x}^{(m)} - (\mathbb{E}_\xi[\mu_{t}^{(m)}] - \mathbb{E}_\xi[\mu_{t-1}^{(m)}]) \big\|^2 \big] \nonumber \\
  & \overset{(b)}{\leq} 2(c_{\nu}\alpha_{t-1}^2)^2\mathbb{E} \big[ \big\| \mu_{t, \xi_x}^{(m)} - \mathbb{E}_\xi[\mu_{t}^{(m)}]\big\|^2 \big]  +   2\mathbb{E} \big[ \big\|\mu_{t, \xi_x}^{(m)} - \mu_{t-1, \xi_x}^{(m)}\big\|^2 \big] 
  \nonumber \\
  & \overset{(c)}{\leq} 2(c_{\nu}\alpha_{t-1}^2)^2G_2^2  +  2\hat{L}^2\mathbb{E}\big[\big\|x^{(m)}_{t} - x^{(m)}_{t-1}\big\|^2 + \big\|y^{(m)}_{t} - y^{(m)}_{t-1}\big\|^2\big]
\end{align*}
where inequality (a) follows the generalized triangle inequality; (b) follows Proposition~\ref{prop: Sum_Mean_Kron} due to the definition of $\mu^{(m)}_t$; (c) follows the smoothness property of $\hat{L}$ and the bounded variance assumption; 
This completes the proof.
\end{proof}


\begin{lemma}
\label{lemma: inner_est_error_storm}
Suppose we have $c_{\omega}\alpha_{t-1}^2 < 1$, then we have:
\begin{align*}
    &\frac{1}{M}\sum_{m=1}^M\mathbb{E} \big[ \big\|\omega^{(m)}_t - \nabla_y g^{(m)}(x^{(m)}_{t}, y^{(m)}_{t} ) \big\|^2 \big]\nonumber\\
    & \leq \frac{( 1 - c_{\omega}\alpha_{t-1}^2)}{M}\sum_{m=1}^M\mathbb{E} \big[ \big\| \omega_{t-1}^{(m)} - \nabla_y g^{(m)}(x^{(m)}_{t-1}, y^{(m)}_{t-1})  \big\|^2 \big] + \frac{2(c_{\omega}\alpha_{t-1}^2)^2\sigma^2}{b_y} \nonumber\\
    & \qquad + \frac{2L^2}{b_yM}\sum_{m=1}^M\mathbb{E}\big[\big\|x^{(m)}_t - x^{(m)}_{t-1}\big\|^2 +  \big\|y^{(m)}_{t} - y^{(m)}_{t-1}\big\|^2\big]
\end{align*}
where the expectation is w.r.t the stochasticity of the algorithm.
\end{lemma}

The proof of Lemma~\ref{lemma: inner_est_error_storm} can be derived similar as Lemma~\ref{lemma:hg_bound_storm}

\subsubsection{Lower Problem Solution Error}

\begin{lemma}
\label{lemma: inner_drift_storm}
Suppose we choose $\gamma \leq \frac{1}{2L}$ and $\alpha_t < 1$. Then for $t \neq \bar{t}_s$, we have:
\begin{align*}
\mathbb{E}\big[\big\|y^{(m)}_{t} - y^{(m)}_{x^{(m)}_{t}} \big\|^2\big] & \leq (1 - \frac{\mu\gamma\alpha_{t-1}}{4})\mathbb{E}\big[\big\|y^{(m)}_{t-1} - y^{(m)}_{x^{(m)}_{t-1}}\big\|^2\big] - \frac{\gamma^2\alpha_{t-1}}{4}\mathbb{E}\big[\big\|\omega^{(m)}_{t-1}\big\|^2\big]\nonumber \\
& \qquad + \frac{9\gamma\alpha_{t-1}}{2\mu}\mathbb{E}\big[\big\|\omega^{(m)}_{t-1} - \nabla_y g^{(m)}(x^{(m)}_{t-1}, y^{(m)}_{t-1} )\big\|^2\big] + \frac{9\kappa^2\eta^2\alpha_{t-1}}{2\mu\gamma}\mathbb{E}\big[\|\nu^{(m)}_{t-1}\|^2\big]
\end{align*}
for $t = \bar{t}_s$, we have:
\begin{align*}
\mathbb{E}\big[\big\|y^{(m)}_{t} - y^{(m)}_{x^{(m)}_{t}} \big\|^2\big] & \leq (1 - \frac{\mu\gamma\alpha_{t-1}}{4})\mathbb{E}\big[\big\|y^{(m)}_{t-1} - y^{(m)}_{x^{(m)}_{t-1}}\big\|^2\big] - \frac{\gamma^2\alpha_{t-1}}{4}\mathbb{E}\big[\big\|\omega^{(m)}_{t-1}\big\|^2\big]\nonumber \\
& \qquad +\frac{9\gamma\alpha_{t-1}}{2\mu}\mathbb{E}\big[\big\|\omega^{(m)}_{t-1} - \nabla_y g^{(m)}(x^{(m)}_{t-1}, y^{(m)}_{t-1} )\big\|^2\big]\nonumber\\
&\qquad + \frac{9\kappa^2\eta^2\alpha_{t-1}}{\mu\gamma}\mathbb{E}\big[\|\nu^{(m)}_{t-1}\|^2\big] + \frac{9\kappa^2}{\mu\gamma\alpha_{t-1}}\mathbb{E}\big[\|\hat{x}_{t}^{(m)}) - \bar{x}_{t}\|^2\big]
\end{align*}
\end{lemma}

\begin{proof}
First, we exploit Proposition~\ref{prop:strong-prog}, and choose the function $g^{(m)}(x^{(m)}_t, \cdot)$, by assumption it is $L$ smooth and $\mu$ strongly convex, and we choose $\gamma < \frac{1}{2L}$ and $\alpha_t < 1$, thus:
\begin{align} \label{eq:E8}
\|y_{t+1}^{(m)}-y^{(m)}_{x_t^{(m)}}\|^2 & \leq ( 1-\frac{\mu\gamma\alpha_t}{2})\|y^{(m)}_{x_t^{(m)}}-y_t^{(m)}\|^2 - \frac{\gamma^2\alpha_t}{4} \|\omega_t^{(m)}\|^2 + \frac{4\gamma\alpha_t}{\mu}\|\nabla_y g(x_t^{(m)},y_t^{(m)})-w_t^{(m)}\|^2.
\end{align}
Next, we decompose the term $\|y_{t+1}^{(m)} - y^{(m)}_{x_{t+1}^{(m)}}\|^2$ as follows:
\begin{align} \label{eq:E9}
  \|y_{t+1}^{(m)} - y^{(m)}_{x_{t+1}^{(m)}}\|^2 & \leq (1+\frac{\mu\gamma\alpha_t}{4})\|y_{t+1}^{(m)} - y^{(m)}_{x_t^{(m)}}\|^2  + (1+\frac{4}{\mu\gamma\alpha_t})\|y^{(m)}_{x_t^{(m)}} - y^{(m)}_{x_{t+1}^{(m)}})\|^2 \nonumber \\
  & \leq (1+\frac{\mu\gamma\alpha_t}{4})\|y_{t+1}^{(m)} - y^{(m)}_{x_t^{(m)}}\|^2  + (1+\frac{4}{\mu\gamma\alpha_t})\kappa^2\|x_t^{(m)} - x_{t+1}^{(m)}\|^2
\end{align}
where the first inequality holds by the generalized triangle inequality, and the second inequality is due to case a) of Proposition 3.9. Combining the above inequalities \ref{eq:E8} and \ref{eq:E9}, we have
\begin{align*}
 \|y_{t+1}^{(m)} -y^{(m)}_{\hat{x}_{t+1}^{(m)}})\|^2 & \leq (1+\frac{\mu\gamma\alpha_t}{4})( 1-\frac{\mu\gamma\alpha_t}{2})\|y_{t}^{(m)} -y^{(m)}_{x_{t}^{(m)}})\|^2
 - (1+\frac{\mu\gamma\alpha_t}{4})\frac{\gamma^2\alpha_t}{4} \|\omega_t^{(m)}\|^2     \nonumber \\
 &\quad + (1+\frac{\mu\gamma\alpha_t}{4})\frac{4\gamma\alpha_t}{\mu}\|\nabla_y g(x_t^{(m)},y_{t}^{(m)})-w_t^{(m)}\|^2 + (1+\frac{4}{\mu\gamma\alpha_t})\kappa^2\|x_t^{(m)} - x_{t+1}^{(m)}\|^2
\end{align*}
Since we choose $\gamma \leq \frac{1}{2L}$, $\alpha_t < 1$, we have:
\begin{align*}
  (1+\frac{\mu\gamma\alpha_t}{4})( 1-\frac{\mu\gamma\alpha_t}{2})&= 1-\frac{\mu\gamma\alpha_t}{4} - \frac{\mu^2\gamma^2\alpha_t^2}{8} \leq 1-\frac{\mu\gamma\alpha_t}{4}
 \end{align*}
and $ - (1+\frac{\mu\gamma\alpha_t}{4})\leq -1, (1+\frac{\mu\gamma\alpha_t}{4})\leq\frac{9}{8}$, $\mu\gamma\alpha_t < \frac{1}{2}$.
Thus, we have
\begin{align*}
 \|y_{t+1}^{(m)} -y^{(m)}_{\hat{x}_{t+1}^{(m)}})\|^2 & \leq \big(1-\frac{\mu\gamma\alpha_t}{4}\big)\|y_{t}^{(m)} -y^{(m)}_{x_{t}^{(m)}})\|^2
 - \frac{\gamma^2\alpha_t}{4} \|\omega_t^{(m)}\|^2     \nonumber \\
 & \quad + \frac{9\gamma\alpha_t}{2\mu}\|\nabla_y g(x_t^{(m)},y_{t}^{(m)})-w_t\|^2
  +  \frac{9\kappa^2}{2\mu\gamma\alpha_t}\underbrace{\|x_t^{(m)} - x_{t+1}^{(m)}\|^2}_{T_1}
\end{align*}
Note for the term $T_1$ we have $T_1 = \big\|\eta\alpha_{t}\nu^{(m)}_{t}\big\|^2$ for $ t+1 \neq \bar{t}_s$ and $T_1 = \big\|\bar{x}_{t+1} - x^{(m)}_{t}\big\|^2 \leq 2\big\|\hat{x}^{(m)}_{t+1} - \bar{x}_{t+1}\big\|^2+2\big\|\eta\alpha_{t}\nu^{(m)}_{t}\big\|^2$ for $t+1 = \bar{t}_s$. This completes the proof.
\end{proof}

\subsubsection{Upper Variable Drift}
\begin{lemma}\label{lem:x_drift}
For $t \in [\bar{t}_{s-1} + 1,  \bar{t}_s]$, with $s \in [S]$ we have:
\begin{align*}
\| \hat{x}_t^{(m)}-  \bar{x}_t \|^2 \leq \sum_{\ell = \bar{t}_{s-1}}^{t-1} I\eta^2\alpha_l^2 \big\| \big(\nu_\ell^{(m)} -  \bar{\nu}_\ell  \big) \big\|^2
\end{align*}
\end{lemma}

\begin{proof}
Since we have $\hat{x}_{t}^{(m)} = x_{t-1}^{(m)} - \eta\alpha_{t-1}\nu_{t-1}^{(m)}$, this implies that:
\begin{align*}
    \hat{x}_t^{(m)} = x_{\bar{t}_{s-1}}^{(m)} - \sum_{\ell = \bar{t}_{s-1}}^{t-1} \eta  \alpha_\ell\nu_\ell^{(m)} \quad \text{and} \quad \bar{x}_{t}  = \bar{x}_{\bar{t}_{s-1}}  - \sum_{\ell = \bar{t}_{s-1}}^{t-1} \eta\alpha_\ell  \bar{\nu}_\ell.
\end{align*}
So for $t \in [\bar{t}_{s-1} + 1,  \bar{t}_s]$, with $s \in [S]$ we have:
\begin{align*}
\| \hat{x}_t^{(m)}-  \bar{x}_t \|^2 & = \big\| x_{\bar{t}_{s-1}}^{(m)} - \bar{x}_{\bar{t}_{s-1}}  - \big( \sum_{\ell = \bar{t}_{s-1}}^{t-1} \eta\alpha_\ell  \nu_\ell^{(m)} -   \sum_{\ell =  \bar{t}_{s-1}}^{t-1} \eta \alpha_\ell  \bar{\nu}_\ell  \big) \big\|^2 \nonumber \\ & \overset{(a)}{=} \big\|  \sum_{\ell = \bar{t}_{s-1}}^{t-1} \eta\alpha_\ell\big( \nu_\ell^{(m)} - \bar{\nu}_\ell  \big) \big\|^2 \overset{(b)}{\leq} \sum_{\ell = \bar{t}_{s-1}}^{t-1} I\eta^2\alpha_l^2 \big\| \big(\nu_\ell^{(m)} -  \bar{\nu}_\ell  \big) \big\|^2
\end{align*}
where the equality $(a)$ follows from the fact that $x_{\bar{t}_{s-1}}^{(m)} = \bar{x}_{\bar{t}_{s-1}}$; inequality (b) is due to $t - \bar{t}_{s-1} \leq I$ and the generalized triangle inequality. 
\end{proof}

\begin{lemma}\label{lem: ErrorAccumulation_Iterates_FedAvg_storm}
Suppose $\alpha_t < \frac{1}{16I\hat{L}}$, $\eta < 1$, then for $t \neq \bar{t}_s, s\in[S]$, we have:
\begin{align*}
    \sum_{m=1}^M \mathbb{E} \| \nu_{t}^{(m)} - \bar{\nu}_{t} \|^2
    &\leq \big(1 + \frac{33}{32I}\big) \sum_{m=1}^M \mathbb{E} \|  \nu_{t-1}^{(m)}  - \bar{\nu}_{t-1} \|^2  + 4 I \hat{L}^2\alpha_{t-1}^2 \sum_{m=1}^M \mathbb{E}\big[2\| \eta\bar{\nu}_{t-1} \|^2 + \| \gamma \omega^{(m)}_{t-1} \|^2 \big] \nonumber \\
    & \qquad + 8I M (c_{\nu}\alpha_{t-1}^2)^2G_1^2 + \frac{8I M (c_{\nu}\alpha_{t-1}^2)^2G_2^2}{b_x} + 16I M (c_{\nu}\alpha_{t-1}^2)^2\zeta^2\nonumber \\ 
    & \qquad + 128I\hat{L}^2 (c_{\nu}\alpha_{t-1}^2)^2\sum_{m = 1}^M \mathbb{E}\big[\| y^{(m)}_{t - 1} - y^{(m)}_{x^{(m)}_{t-1}}\|^2 \big] \nonumber \\ 
    & \qquad +  128I\bar{L}^2(c_{\nu}\alpha_{t-1}^2)^2\sum_{m = 1}^M \sum_{\ell = \bar{t}_{s-1}}^{t-2} I\eta^2\alpha_l^2 \mathbb{E}\big\| \big(\nu_\ell^{(m)} -  \bar{\nu}_\ell  \big) \big\|^2
\end{align*}
where the expectation is w.r.t the stochasticity of the algorithm.
\end{lemma}

\begin{proof}
By the update step in Line 7 of Algorithm~\ref{alg:FedBiOAcc}, for $t \neq \bar{t}_s$, we have:
\begin{align}\label{eq:nu_drift}
\mathbb{E} \| \hat{\nu}_{t}^{(m)} - \bar{\nu}_{t} \|^2
& = \mathbb{E} \big\| \mu^{(m)}_{t,\mathcal{B}_x}+ (1 - c_{\nu}\alpha_{t-1}^2) \big( \nu_{t-1}^{(m)} -   \mu^{(m)}_{t-1,\mathcal{B}_x}\big) - \big( \bar{\mu}_{t,\mathcal{B}_x} + (1 - c_{\nu}\alpha_{t-1}^2) \big( \bar{\nu}_{t-1} -  \bar{\mu}_{t-1,\mathcal{B}_x}\big) \big)   \big\|^2 \nonumber \\
& = \mathbb{E} \big\| (1 - c_{\nu}\alpha_{t-1}^2) \big( \nu_{t-1}^{(m)}  - \bar{\nu}_{t-1} \big) +  \mu^{(m)}_{t,\mathcal{B}_x} -  \bar{\mu}_{t,\mathcal{B}_x} - (1- c_{\nu}\alpha_{t-1}^2) \big(  \mu^{(m)}_{t-1,\mathcal{B}_x} -  \bar{\mu}_{t-1,\mathcal{B}_x} \big)  \big\|^2 \nonumber \\
& \overset{(a)}{\leq} (1 + \frac{1}{I}) (1 - c_{\nu}\alpha_{t-1}^2)^2 \mathbb{E} \|  \nu_{t-1}^{(m)}  - \bar{\nu}_{t-1} \|^2 \nonumber \\
& \qquad + \big( 1 + I \big) \mathbb{E} \big\|  \mu^{(m)}_{t,\mathcal{B}_x}-   \bar{\mu}_{t,\mathcal{B}_x}  - (1- c_{\nu}\alpha_{t-1}^2) \big(  \mu^{(m)}_{t-1,\mathcal{B}_x} -  \bar{\mu}_{t-1,\mathcal{B}_x} \big)  \big\|^2 \nonumber\\
& \leq \left(1 + \frac{1}{I}\right) \mathbb{E} \|  \nu_{t-1}^{(m)}  - \bar{\nu}_{t-1} \|^2 + \big( 1 + I \big) \mathbb{E} \big\|  \mu^{(m)}_{t,\mathcal{B}_x}-  \bar{\mu}_{t,\mathcal{B}_x}  - (1- c_{\nu}\alpha_{t-1}^2) \big(  \mu^{(m)}_{t-1,\mathcal{B}_x} -  \bar{\mu}_{t-1,\mathcal{B}_x} \big)  \big\|^2
\end{align}
where $(a)$ follows from the the generalized triangle inequality. 

\noindent Next we bound the second term of the above inequality (denoted as $T_1$):
\begin{align*}
& \sum_{m=1}^M\mathbb{E}\big\|  \mu^{(m)}_{t,\mathcal{B}_x} -  \bar{\mu}_{t,\mathcal{B}_x}  - (1- c_{\nu}\alpha_{t-1}^2) \big(  \mu^{(m)}_{t-1,\mathcal{B}_x} -  \bar{\mu}_{t-1,\mathcal{B}_x} \big)  \big\|^2 \nonumber\\
& \leq  2 \sum_{m=1}^M\mathbb{E}\big\| \mu^{(m)}_{t,\mathcal{B}_x} -  \bar{\mu}_{t,\mathcal{B}_x} -  \big( \mu^{(m)}_{t-1,\mathcal{B}_x} - \bar{\mu}_{t-1,\mathcal{B}_x} \big)\big\|^2 + 2 (c_{\nu}\alpha_{t-1}^2)^2 \sum_{m=1}^M\mathbb{E} \big\|  \mu^{(m)}_{t-1,\mathcal{B}_x} -  \bar{\mu}_{t-1,\mathcal{B}_x}  \big\|^2 \nonumber\\
& \leq  2 \sum_{m=1}^M\mathbb{E}\big\| \mu^{(m)}_{t,\mathcal{B}_x} -  \mu^{(m)}_{t-1,\mathcal{B}_x}\big\|^2 + 2 (c_{\nu}\alpha_{t-1}^2)^2 \sum_{m=1}^M\mathbb{E} \big\|  \mu^{(m)}_{t-1,\mathcal{B}_x} -  \bar{\mu}_{t-1,\mathcal{B}_x}  \big\|^2
\end{align*}
where the second inequality follows Proposition~\ref{prop: Sum_Mean_Kron}. We bound the two terms separately, for the first term, we have:
\begin{align}\label{eq:mu_ave_drift}
&\sum_{m=1}^M\mathbb{E} \big\| \mu^{(m)}_{t,\mathcal{B}_x} - \mu^{(m)}_{t-1,\mathcal{B}_x} \big\|^2  \leq  \hat{L}^2 \sum_{m=1}^M\mathbb{E}\big[ \| x_t^{(m)} - x_{t-1}^{(m)} \|^2 + \| y_t^{(m)} - y_{t-1}^{(m)} \|^2 \big] \nonumber\\
&\leq \hat{L}^2\alpha_{t-1}^2 \sum_{m=1}^M \mathbb{E}\big[ \| \eta\nu^{(m)}_{t-1} \|^2 + \| \gamma \omega^{(m)}_{t-1} \|^2 \big]
\end{align}
where the inequalities follow Proposition~\ref{some smoothness local}.b) and the fact that $\hat{x}_t^{(m)} = x_t^{(m)}$ when $t \neq \bar{t}_s$; 

Next for the second term, we have:
\begin{align}\label{eq:mu_drift}
&\sum_{m=1}^M \mathbb{E} \big\|  \mu^{(m)}_{t-1,\mathcal{B}_x} -  \bar{\mu}_{t-1,\mathcal{B}_x}  \big\|^2 = \sum_{m=1}^M \mathbb{E}\big\| \mu^{(m)}_{t-1,\mathcal{B}_x} - \mu^{(m)}_{t-1}  - \big(  \bar{\mu}_{t-1,\mathcal{B}_x} - \bar{\mu}_{t-1}  \big)   + \mu^{(m)}_{t-1} - \bar{\mu}_{t-1} \big\|^2 
\nonumber \\
& \overset{(a)}{\leq}  2 \sum_{m=1}^M\mathbb{E} \big\|  \mu^{(m)}_{t-1,\mathcal{B}_x} - \mu^{(m)}_{t-1}  -  \big(  \bar{\mu}_{t-1,\mathcal{B}_x} - \bar{\mu}_{t-1}  \big) \big\|^2 + 2\sum_{m=1}^M\mathbb{E} \big\| \mu^{(m)}_{t-1} -  \bar{\mu}_{t-1}\big\|^2 \nonumber \\
& \overset{(b)}{\leq}   2\underbrace{\sum_{m=1}^M \mathbb{E} \big\|  \mu^{(m)}_{t-1,\mathcal{B}_x} - \mu^{(m)}_{t-1} \big\|^2}_{T_1} + 4 \underbrace{\sum_{m=1}^M\mathbb{E}\big\|\nabla h^{(m)}(\bar{x}_{t-1}) - \nabla h(\bar{x}_{t-1}) \big\|^2}_{T_2} \nonumber \\
&\qquad \qquad + 4 \underbrace{\sum_{m=1}^M \mathbb{E} \big\|\mu^{(m)}_{t-1} - \nabla h^{(m)}(\bar{x}_{t-1}) + \nabla h(\bar{x}_{t-1}) -  \bar{\mu}_{t-1}   \big\|^2}_{T_3}
\end{align}
Note for the term $T_1$ of Eq.~\ref{eq:mu_drift}, we have $\mathbb{E} \big\|  \mu^{(m)}_{t-1,\mathcal{B}_x} - \mu^{(m)}_{t-1} \big\|^2 \leq G_1^2 + \frac{G_2^2}{b_x}$; For the term $T_2$ of Eq.~\ref{eq:mu_drift}, by the bounded intra-node heterogeneity assumption we have:
\begin{align*}
    T_2 \leq 4\sum_{m=1}^M \frac{1}{M} \sum_{j=1}^M  \mathbb{E}  \|  \nabla h^{(m)}(\bar{x}_{t-1}) - \nabla h^{(j)}(\bar{x}_{t-1})  \|^2 \leq 4 M \zeta^2
\end{align*}
Finally, For the term $T_3$ of Eq.~\ref{eq:mu_drift}
\begin{align*}
T_3 &\leq  8\sum_{m=1}^M\mathbb{E} \big\|\mu^{(m)}_{t-1} - \nabla h^{(m)}(\bar{x}_{t-1})   \big\|^2 + 8\sum_{m=1}^M\mathbb{E} \big\| \nabla h(\bar{x}_{t-1}) -   \bar{\mu}_{t-1}   \big\|^2 \leq  16 \sum_{m=1}^M \mathbb{E} \big\| \mu^{(m)}_{t-1} - \nabla h^{(m)}(\bar{x}_{t-1})   \big\|^2  \nonumber\\
& \leq 32 \sum_{m=1}^M \mathbb{E} \big[\big\| \mu^{(m)}_{t-1} - \nabla h^{(m)}(x^{(m)}_{t-1})   \big\|^2 + \big\| \nabla h^{(m)}(x^{(m)}_{t-1}) - \nabla h^{(m)}(\bar{x}_{t-1})   \big\|^2 \big] \nonumber\\
& \overset{(a)}{\leq} 32\bar{L}^2 \sum_{m = 1}^M \mathbb{E}\big[\| x_{t - 1}^{(m)} - \bar{x}_{t-1}\|^2\big] + 32\hat{L}^2\sum_{m = 1}^M \mathbb{E}\big[\| y^{(m)}_{t - 1} - y^{(m)}_{x^{(m)}_{t-1}}\|^2 \big]
\end{align*}
where inequality (b) follows Proposition~\ref{some smoothness local}.c) and d).

\noindent Combine Eq.~\ref{eq:nu_drift}, Eq.~\ref{eq:mu_ave_drift} and Eq.~\ref{eq:mu_drift}, use the fact that $I\geq 1$, we have:
\begin{align*}
    & \sum_{m=1}^M \mathbb{E} \|\nu_{t}^{(m)} - \bar{\nu}_{t} \|^2 \nonumber \\
    & \leq \big(1 + \frac{1}{I}\big) \sum_{m=1}^M \mathbb{E} \|  \nu_{t-1}^{(m)}  - \bar{\nu}_{t-1} \|^2  + 4 I \hat{L}^2\alpha_{t-1}^2 \sum_{m=1}^M \mathbb{E}\big[ \underbrace{\| \eta\nu^{(m)}_{t-1} \|^2}_{T_1} + \| \gamma \omega^{(m)}_{t-1} \|^2 \big] \nonumber \\
    & \qquad + 8I M (c_{\nu}\alpha_{t-1}^2)^2G_1^2 + \frac{8I M (c_{\nu}\alpha_{t-1}^2)^2G_2^2}{b_x} + 16I M (c_{\nu}\alpha_{t-1}^2)^2\zeta^2 \nonumber\\ 
    & \qquad  +  128I\bar{L}^2(c_{\nu}\alpha_{t-1}^2)^2\sum_{m = 1}^M \mathbb{E}\big[\| x_{t - 1}^{(m)} - \bar{x}_{t-1}\|^2\big] + 128I\hat{L}^2 (c_{\nu}\alpha_{t-1}^2)^2\sum_{m = 1}^M \mathbb{E}\big[\| y^{(m)}_{t - 1} - y^{(m)}_{x^{(m)}_{t-1}}\|^2 \big]
\end{align*}
We separate the term $T_1$ with triangle inequality to get:
\begin{align*}
    & \sum_{m=1}^M \mathbb{E} \|\nu_{t}^{(m)} - \bar{\nu}_{t} \|^2 \nonumber \\
    &\leq \left(1 + \frac{1}{I} + 8 I \hat{L}^2\eta^2\alpha_{t-1}^2\right) \sum_{m=1}^M \mathbb{E} \|  \nu_{t-1}^{(m)}  - \bar{\nu}_{t-1} \|^2  + 4 I \hat{L}^2\alpha_{t-1}^2 \sum_{m=1}^M \mathbb{E}\big[2\| \eta\bar{\nu}_{t-1} \|^2 + \| \gamma \omega^{(m)}_{t-1} \|^2 \big] \nonumber \\
    & \qquad + 8I M (c_{\nu}\alpha_{t-1}^2)^2G_1^2 + \frac{8I M (c_{\nu}\alpha_{t-1}^2)^2G_2^2}{b_x} + 16I M (c_{\nu}\alpha_{t-1}^2)^2\zeta^2 \nonumber \\ 
    & \qquad  +  \underbrace{128I\bar{L}^2(c_{\nu}\alpha_{t-1}^2)^2\sum_{m = 1}^M \mathbb{E}\big[\| x_{t - 1}^{(m)} - \bar{x}_{t-1}\|^2\big]}_{T_1} + 128I\hat{L}^2 (c_{\nu}\alpha_{t-1}^2)^2\sum_{m = 1}^M \mathbb{E}\big[\| y^{(m)}_{t - 1} - y^{(m)}_{x^{(m)}_{t-1}}\|^2 \big]
\end{align*}
Finally, choose $\eta\alpha_t < \frac{1}{16\hat{L} I}$ and combine with Lemma~\ref{lem:x_drift} to bound the term $T_1$, we get the bound in the lemma. This completes the proof.
\end{proof}


Next, to simply the notation, we denote $A_t = \mathbb{E}\| \bar{\nu}_{t} - \mathbb{E}_{\xi}[\bar{\mu}_{t,\mathcal{B}_x}] \|^2$, $B_t =  \frac{1}{M} \sum_{m=1}^M \mathbb{E}\|y^{(m)}_t - y^{(m)}_{x^{(m)}_{t}} \|^2$, $C_t = \frac{1}{M} \sum_{m=1}^M \mathbb{E}\|\omega^{(m)}_t - \nabla_y g^{(m)}(x^{(m)}_{t}, y^{(m)}_{t} ) \|^2$, $D_t =\frac{1}{M}\sum_{m=1}^M \mathbb{E}\|\nu^{(m)}_{t} - \bar{\nu}_{t}\|^2$, $E_t = \mathbb{E}\|\bar{\nu}_{t}\|^2$, $F_t =\frac{1}{M}\sum_{m=1}^M\mathbb{E}\|\omega^{(m)}_{t}\|^2$. 

\begin{lemma}
\label{lemma:d_bound}
For $\alpha_t < \frac{1}{16\hat{L}I}$, we have:
\begin{align*}
    \big(1 - \frac{3\kappa^2\eta^2 c_{\nu}^2}{4*16^3I^5\hat{L}^4}\big)\sum_{t = \bar{t}_{s-1}}^{\bar{t}_s-1} \alpha_{t} D_{t} &\leq \frac{3 c_{\nu}^2}{2*16^2I^4\hat{L}^2}\sum_{\ell=\bar{t}_{s-1}}^{\bar{t}_s-1}\alpha_{\ell}B_{\ell} +  \frac{3\eta^2}{32I}\sum_{\ell=\bar{t}_{s-1}}^{\bar{t}_s-1}\alpha_{\ell}E_{\ell} + \frac{3\gamma^2}{64I}\sum_{\ell=\bar{t}_{s-1}}^{\bar{t}_s-1}\alpha_{\ell}F_{\ell} \nonumber\\
    &\qquad + \left(\frac{3c_{\nu}^2G_1^2}{32I\hat{L}^2}  + \frac{3c_{\nu}^2G_2^2}{32Ib_x\hat{L}^2}   + \frac{3c_{\nu}^2\zeta^2}{16I\hat{L}^2}\right)\sum_{\ell=\bar{t}_{s-1}}^{\bar{t}_s-1}\alpha_{\ell}^3
\end{align*}
where the terms $D_t$, $E_t$ and $F_t$ are denoted above.
\end{lemma}
\begin{proof}
Based on Lemma~\ref{lem: ErrorAccumulation_Iterates_FedAvg_storm}, for $t \neq \bar{t}_s$, we have:
\begin{align*}
    D_t & \leq \big(1 + \frac{33}{32I}\big) D_{t-1} + 128I\hat{L}^2c_{\nu}^2\alpha_{t-1}^4B_{t-1}+ 8 I \hat{L}^2\alpha_{t-1}^2\eta^2E_{t-1} + 4 I \hat{L}^2\alpha_{t-1}^2\gamma^2F_{t-1} \nonumber\\
    & \qquad + 8Ic_{\nu}^2\alpha_{t-1}^4G_1^2 + \frac{8Ic_{\nu}^2\alpha_{t-1}^4G_2^2}{b_x} + 16Ic_{\nu}^2\alpha_{t-1}^4\zeta^2 + 128I^2\bar{L}^2\eta^2c_{\nu}^2\alpha_{t-1}^4\sum_{\ell = \bar{t}_{s-1}}^{t-2} \alpha_l^2 D_l
\end{align*}
while for $t = \bar{t}_s$, we have $D_{\bar{t}_{s}} = 1/M\sum_{m=1}^M \mathbb{E} \| \nu_{\bar{t}_{s}}^{(m)} - \bar{\nu}_{\bar{t}_{s}} \|^2 = 0$. Apply the above equation recursively from $\bar{t}_{s-1} + 1$ to $t$. so we have:
\begin{align*}
    D_t & \leq \sum_{\ell=\bar{t}_{s-1}}^{t-1}\big(1 + \frac{33}{32I}\big)^{t - \ell} \big(128I\hat{L}^2c_{\nu}^2\alpha_{\ell}^4B_{\ell} + 8 I \hat{L}^2\eta^2\alpha_{\ell}^2E_{\ell} + 4 I \hat{L}^2\gamma^2\alpha_{\ell}^2F_{\ell}  + 8Ic_{\nu}^2G_1^2\alpha_{\ell}^4  + \frac{8Ic_{\nu}^2G_2^2\alpha_{\ell}^4}{b_x}\nonumber\\
    & \qquad \qquad \qquad \qquad \qquad +   16Ic_{\nu}^2\zeta^2\alpha_{\ell}^4 + 128I^2\bar{L}^2\eta^2 c_{\nu}^2\alpha_{\ell}^4\sum_{\bar{\ell} = \bar{t}_{s-1}}^{\ell} \alpha_{\bar{\ell}}^2 D_{\bar{\ell}}\big) \nonumber \\
    & \leq  \sum_{\ell=\bar{t}_{s-1}}^{t-1} \big(384I\hat{L}^2c_{\nu}^2\alpha_{\ell}^4B_{\ell} + 24 I \hat{L}^2\eta^2\alpha_{\ell}^2E_{\ell} + 12 I \hat{L}^2\gamma^2\alpha_{\ell}^2F_{\ell}  + 24Ic_{\nu}^2G_1^2\alpha_{\ell}^4  + \frac{24Ic_{\nu}^2G_2^2\alpha_{\ell}^4}{b_x} \nonumber\\
    &\qquad \qquad + 16Ic_{\nu}^2\zeta^2\alpha_{\ell}^4 + 384I^2\bar{L}^2\eta^2 c_{\nu}^2\alpha_{\ell}^4\sum_{\bar{\ell} = \bar{t}_{s-1}}^{\ell} \alpha_{\bar{\ell}}^2 D_{\bar{\ell}}\big)
\end{align*}
The second inequality uses the fact that $t -l \le I$ and the inequality $log(1+ a/x) \leq a/x$ for $x > -a$, so we have $(1+a/x)^x \leq e^{a}$, Then we choose $a = 33/32$ and $x = I$. Finally, we use the fact that $e^{33/32} \leq 3$. 

Next we multiply $\alpha_t$ over both sides and take sum from $\bar{t}_{s-1} + 1$ to $\bar{t}_{s}$, we have:
\begin{align*}
    \sum_{t=\bar{t}_{s-1}+1}^{\bar{t}_s} \alpha_tD_t  &\leq \sum_{t=\bar{t}_{s-1}}^{\bar{t}_s-1} \alpha_t\sum_{\ell=\bar{t}_{s-1}}^{t-1} \big(384I\hat{L}^2c_{\nu}^2\alpha_{\ell}^4B_{\ell} + 24 I \hat{L}^2\eta^2\alpha_{\ell}^2E_{\ell} + 12 I \hat{L}^2\gamma^2\alpha_{\ell}^2F_{\ell} \nonumber\\
    &\qquad + \left(24Ic_{\nu}^2G_1^2 + \frac{24Ic_{\nu}^2G_1^2}{b_x}  + 48Ic_{\nu}^2\zeta^2\right)\alpha_{\ell}^4 + 384I^2\bar{L}^2\eta^2 c_{\nu}^2\alpha_{\ell}^4\sum_{\bar{\ell} = \bar{t}_{s-1}}^{\ell} \alpha_{\bar{\ell}}^2 D_{\bar{\ell}}\big) \nonumber\\
    & \overset{(a)}{\leq} \sum_{\ell=\bar{t}_{s-1}}^{\bar{t}_s-1} \big(24I^{1/2}\hat{L}c_{\nu}^2\alpha_{\ell}^4B_{\ell} +  \frac{3I^{1/2} \hat{L}\eta^2}{2}\alpha_{\ell}^2E_{\ell} + \frac{3I^{1/2} \hat{L}\gamma^2}{4}\alpha_{\ell}^2F_{\ell} \nonumber\\
    &\qquad + \left(\frac{3I^{1/2}c_{\nu}^2G_1^2}{2\hat{L}} + \frac{3I^{1/2}c_{\nu}^2G_2^2}{2b_x\hat{L}}  + \frac{3I^{1/2}c_{\nu}^2\zeta^2}{\hat{L}}\right)\alpha_{\ell}^4 + \frac{24I^{3/2}\bar{L}^2\eta^2 c_{\nu}^2}{\hat{L}}\alpha_{\ell}^4\sum_{\bar{\ell} = \bar{t}_{s-1}}^{\ell} \alpha_{\bar{\ell}}^2 D_{\bar{\ell}}\big) \nonumber\\
    &\overset{(b)}{\leq} \frac{3 c_{\nu}^2}{2*16^2I^4\hat{L}^2}\sum_{\ell=\bar{t}_{s-1}}^{\bar{t}_s-1}\alpha_{\ell}B_{\ell} +  \frac{3\eta^2}{32I}\sum_{\ell=\bar{t}_{s-1}}^{\bar{t}_s-1}\alpha_{\ell}E_{\ell} + \frac{3\gamma^2}{64I}\sum_{\ell=\bar{t}_{s-1}}^{\bar{t}_s-1}\alpha_{\ell}F_{\ell} \nonumber\\
    &\qquad + \left(\frac{3c_{\nu}^2G_1^2}{32I\hat{L}^2}  + \frac{3c_{\nu}^2G_2^2}{32Ib_x\hat{L}^2} +
    \frac{3c_{\nu}^2\zeta^2}{16I\hat{L}^2}\right)\sum_{\ell=\bar{t}_{s-1}}^{\bar{t}_s-1}\alpha_{\ell}^3 + \frac{3\kappa^2\eta^2 c_{\nu}^2}{4*16^3I^5\hat{L}^4}\sum_{t = \bar{t}_{s-1}}^{\bar{t}_s-1} \alpha_{t} D_{t} 
\end{align*}
In inequalities $(a)$ and $(b)$, we use $\alpha_t < \frac{1}{16\hat{L}I^{3/2}}$. Note that $\sum_{t=\bar{t}_{s-1}+1}^{\bar{t}_s} \alpha_tD_t = \sum_{t = \bar{t}_{s-1}}^{\bar{t}_s-1} \alpha_{t} D_{t}$ as $D_{\bar{t}_s} = D_{\bar{t}_{s-1}} =0$, so we have:
\begin{align*}
    \big(1 - \frac{3\kappa^2\eta^2 c_{\nu}^2}{4*16^3I^5\hat{L}^4}\big)\sum_{t = \bar{t}_{s-1}}^{\bar{t}_s-1} \alpha_{t} D_{t} &\leq \frac{3 c_{\nu}^2}{2*16^2I^4\hat{L}^2}\sum_{\ell=\bar{t}_{s-1}}^{\bar{t}_s-1}\alpha_{\ell}B_{\ell} +  \frac{3\eta^2}{32I}\sum_{\ell=\bar{t}_{s-1}}^{\bar{t}_s-1}\alpha_{\ell}E_{\ell} + \frac{3\gamma^2}{64I}\sum_{\ell=\bar{t}_{s-1}}^{\bar{t}_s-1}\alpha_{\ell}F_{\ell} \nonumber\\
    &\qquad + \left(\frac{3c_{\nu}^2G_1^2}{32I\hat{L}^2}  + \frac{3c_{\nu}^2G_2^2}{32Ib_x\hat{L}^2}  +
    \frac{3c_{\nu}^2\zeta^2}{16I\hat{L}^2}\right)\sum_{\ell=\bar{t}_{s-1}}^{\bar{t}_s-1}\alpha_{\ell}^3
\end{align*}
This completes the proof.
\end{proof}


\subsubsection{Descent Lemma}
\begin{lemma}
\label{lemma:desent_storm}
Suppose $\eta < \frac{1}{2\bar{L}}$, $\alpha_t < 1$, for all $t \in [\bar{t}_{s-1}, \bar{t}_s - 1]$ and $s \in [S]$, the iterates generated satisfy:
\begin{align*}
\mathbb{E}[ h(\bar{x}_{t + 1}) ] & \leq \mathbb{E} [    h(\bar{x}_{t }) ]-  \frac{\eta\alpha_t}{4}  \mathbb{E} [\| \bar{\nu}_t  \|^2 ] - \frac{\eta\alpha_t}{2} \mathbb{E} [\|\nabla h(\bar{x}_t) \|^2 ] + 2\eta\alpha_t\mathbb{E} [ \| \mathbb{E}_{\xi}[\bar{\mu}_{t,\mathcal{B}_x}]  - \bar{\nu}_{t}   \|^2 ] + 4\eta\alpha_t G_1^2\\
& \qquad + \frac{\bar{L}^2I\eta^3\alpha_t}{M} \sum_{\ell = \bar{t}_{s-1}}^{t-1}  \alpha_l^2\sum_{m = 1}^M \mathbb{E}\| \big(  \nu_\ell^{(m)} -  \bar{\nu}_\ell  \big) \|^2 + \frac{4\hat{L}^2\eta\alpha_t}{M} \sum_{m=1}^M\mathbb{E} \|y^{(m)}_{x^{(m)}_t}  - y^{(m)}_{t}   \|^2
\end{align*}
where the expectation is w.r.t the stochasticity of the algorithm.
\end{lemma}
\begin{proof}
By the smoothness of $h(x)$ we have:
\begin{align*}
    \mathbb{E}[  h(\bar{x}_{t + 1}) ] 
    & \leq \mathbb{E} \big[ h(\bar{x}_{t }) + \langle \nabla h(\bar{x}_{t}),  \bar{x}_{t + 1} - \bar{x}_{t}\rangle + \frac{\bar{L}}{2} \| \bar{x}_{t + 1} - \bar{x}_{t } \|^2 \big] \nonumber\\
    &  \overset{(a)}{=}\mathbb{E} \big[ h(\bar{x}_{t}) - \eta\alpha_t \langle \nabla h(\bar{x}_{t}),  \bar{\nu}_t \rangle + \frac{\eta^2\alpha_t^2 \bar{L}}{2} \| \bar{\nu}_{t}  \|^2 \big] \nonumber\\
    & \overset{(b)}{=}    \mathbb{E} \big[ h(\bar{x}_{t}) - \frac{\eta\alpha_t}{2}  \big\| \bar{\nu}_{t}  \big\|^2  - \frac{\eta\alpha_t}{2} \| \nabla h(\bar{x}_{t}) \|^2   + \frac{\eta\alpha_t}{2} \big\|  \nabla h(\bar{x}_{t})  - \bar{\nu}_{t}   \big\|^2 + \frac{\eta\alpha_t^2 \bar{L}}{2} \big\| \bar{\nu}_{t} \big\|^2  \big] \nonumber \\  
    & = \mathbb{E} \big[    h(\bar{x}_{t }) -  \frac{\eta\alpha_t}{4}\big\| \bar{\nu}_t  \big\|^2 - \frac{\eta\alpha_t}{2} \|\nabla h(\bar{x}_t) \|^2  + \frac{\eta\alpha_t}{2} \underbrace{\big\|  \nabla h(\bar{x}_{t})  - \bar{\nu}_{t}   \big\|^2}_{T_1} \big]
\end{align*}
where equality $(a)$ follows from the iterate update given in Algorithm~\ref{alg:FedBiOAcc}; $(b)$ uses $\langle a , b \rangle = \frac{1}{2} [\|a\|^2 + \|b\|^2 - \|a - b \|^2]$ and $\eta\alpha_t < \frac{1}{2\bar{L}}$; For the term $T_1$, we have:
\begin{align*}
    \mathbb{E} \big[ \big\|  \nabla h(\bar{x}_{t})  - \bar{\nu}_{t}   \big\|^2  \big] & \leq 2\mathbb{E} \big[ \big\|  \nabla h(\bar{x}_{t})  - \frac{1}{M} \sum_{m=1}^M \nabla h(x^{(m)}_t)   \big\|^2 \big] + 4\mathbb{E} \big[ \big\| \frac{1}{M} \sum_{m=1}^M  \nabla h(x^{(m)}_t)  - \mathbb{E}_{\xi}[\bar{\mu}_{t,\mathcal{B}_x}]   \big\|^2 \big] \nonumber\\
    &\qquad + 4\mathbb{E} \big[ \big\| \mathbb{E}_{\xi}[\bar{\mu}_{t,\mathcal{B}_x}]  - \bar{\nu}_{t}   \big\|^2 \big]
\end{align*}
For the first term, we have:
\begin{align*}
    2\mathbb{E} \big[ \big\|  \nabla h(\bar{x}_{t})  - \frac{1}{M} \sum_{m=1}^M \nabla h(x^{(m)}_t)   \big\|^2 \big]
    & \leq \frac{2}{M} \sum_{m=1}^M\mathbb{E} \big[ \big\|  \nabla h(\bar{x}_{t})  -  \nabla h(x^{(m)}_t)   \big\|^2 \big] \leq \frac{2\bar{L}^2}{M} \sum_{m=1}^M\mathbb{E} \big[ \big\| \bar{x}_{t}  -  x^{(m)}_t  \big\|^2 \big] \nonumber \\
    & \leq \frac{2\bar{L}^2I\eta^2}{M} \sum_{\ell = \bar{t}_{s-1}}^{t-1}  \alpha_l^2\sum_{m = 1}^M \mathbb{E}\big\| \big(  \nu_\ell^{(m)} -  \bar{\nu}_\ell  \big) \big\|^2
\end{align*}
where the last inequality uses Lemma~\ref{lem:x_drift}. For the second term, we have:
\begin{align*}
   4\mathbb{E} \big[ \big\| \frac{1}{M} \sum_{m=1}^M  \nabla h(x^{(m)}_t)  - \mathbb{E}_\xi[\bar{\mu}_{t, \xi}]   \big\|^2 \big] &\leq \frac{8}{M} \sum_{m=1}^M\mathbb{E} \big[ \big\|  \nabla h(x^{(m)}_t)  - \mu^{(m)}_{t}   \big\|^2 \big] + \frac{8}{M}\sum_{m=1}^M\mathbb{E} \big[ \big\| \mu^{(m)}_{t}   - \mathbb{E}_\xi[\mu^{(m)}_{t, \mathcal{B}_x}]   \big\|^2 \big]\nonumber\\
   &\leq \frac{8\hat{L}^2}{M} \sum_{m=1}^M\mathbb{E} \big[ \big\|y^{(m)}_{x^{(m)}_t}  - y^{(m)}_{t}   \big\|^2 \big] + 8G_1^2
\end{align*}
Plug the bound for term $T_1$ back gets the claim in the lemma.
\end{proof}

\subsubsection{Proof of Convergence Theorem}
\label{sec:fedbioacc}
We first denote the following potential function $\mathcal{G}(t)$:
\begin{align*}
    \mathcal{G}_t &= h(\bar{x}_{t}) + \frac{9bM\eta}{16\alpha_{t}}\big\| \bar{\nu}_{t} - \frac{1}{M} \sum_{m=1}^M  \nabla h(x^{(m)}_t)  \big\|^2 + \frac{18\eta\hat{L}^2}{\mu\gamma}\times \frac{1}{M} \sum_{m=1}^M \big\|y^{(m)}_t - y^{(m)}_{x^{(m)}_{t}} \big\|^2 \nonumber \\
    & \qquad \qquad + \frac{9bM\hat{L}^2\eta}{16L^2\alpha_{t}}\times\frac{1}{M} \sum_{m=1}^M \big\|\omega^{(m)}_t - \nabla_y g^{(m)}(x^{(m)}_{t}, y^{(m)}_{t} ) \big\|^2
\end{align*}

\begin{theorem}
Suppose $\gamma \leq \frac{1}{2L}$, $\eta < \min\big(\frac{\mu\gamma}{144\kappa\hat{L}}, \frac{\hat{L}^2}{C_1^{1/2}c_\nu}, \frac{1}{(C_1I)^{1/2}}, \frac{\hat{L}}{(C_1I)^{1/2}}, \frac{I\hat{L}^2}{\kappa c_{\nu}}, 1\big)$, $c_{\nu} = \frac{32}{9bM} + \frac{\hat{L}}{24Ib^2M^2}$, $c_{\omega} = \frac{144L^2}{bM\mu^2} + \frac{\hat{L}}{24Ib^2M^2}$, $u = (bM\sigma)^2\bar{u}$, where $\bar{u} = \max\big(2, 16^3I^{9/2}\hat{L}, c_{\nu}^{3/2},  c_{\omega}^{3/2}\big)$, $\delta = \frac{(bM\sigma)^{2/3}}{\hat{L}^{2/3}}$, then we have:
\begin{align*}
    \frac{1}{T}\sum_{t = 1}^{T-1} \mathbb{E} \big[ \|\nabla h(\bar{x}_t) \|^2 \big] = O\big(\frac{\kappa^{19/3}I^{3/2}}{T} + \frac{\kappa^{16/3}}{(bMT)^{2/3}} + \kappa^3b^2M^2I^{9/2}G_1^2\big)
\end{align*}
To reach an  $\epsilon$-stationary point, we need $T = O(\kappa^{8}(bM)^{-1}\epsilon^{-1.5})$, $I = O(\kappa^{10/9}(bM)^{-2/3}\epsilon^{-1/3})$ and $Q =O(\kappa\log(\frac{\kappa}{bM\epsilon}))$.
\label{theorem:FedBiOAcc}
\end{theorem}

\begin{proof}
By the condition that $u \ge c_{\nu}^{3/2}\delta^3$, it is straightforward to verify that $c_{\nu}\alpha_t^2 < 1$. By Lemma~\ref{lemma:hg_bound_storm} (in new notation), when $t \neq \bar{t}_s$, we have:
\begin{align*}
   \frac{A_t}{\alpha_{t-1}} - \frac{A_{t-1}}{\alpha_{t-2}} &\leq \left(\alpha_{t-1}^{-1} - \alpha_{t-2}^{-1} - c_{\nu}\alpha_{t-1}\right) A_{t-1} + \frac{2c_{\nu}^2\alpha_{t-1}^3G_2^2}{bM} + \frac{4\hat{L}^2\eta^2\alpha_{t-1}}{bM}(D_{t-1} + E_{t-1}) + \frac{2\hat{L}^2\gamma^2\alpha_{t-1}F_{t-1}}{bM}
\end{align*}
Note we choose $b_x = b$. For $\alpha_{t-1}^{-1} - \alpha_{t-2}^{-1}$, we have:
\begin{align*}
    \alpha_{t}^{-1} - \alpha_{t-1}^{-1} & =  \frac{(u + \sigma^2 t)^{1/3}}{\delta} -  \frac{(u + \sigma^2 (t-1))^{1/3}}{\delta} 
    \overset{(a)}{\leq}  \frac{\sigma^2}{3 \delta (u + \sigma^2 (t-1))^{2/3}} \nonumber \\
    & \overset{(b)}{\leq} \frac{2^{2/3} \sigma^2 \delta^2}{3 \delta^3 (u + \sigma^2 t)^{2/3}} \overset{(c)}{=} \frac{2^{2/3} \sigma^2}{3 \delta^3 } \alpha_{t}^2 \leq \frac{2\hat{L}^2}{3M^2} \alpha_{t}^2 \leq \frac{\hat{L}}{24Ib^2M^2} \alpha_{t}
\end{align*}
where inequality $(a)$ results from the concavity of $x^{1/3}$ as: $(x + y)^{1/3} - x^{1/3} \leq y/3x^{2/3}$, inequality $(b)$ used the fact that $u_t \geq 2\sigma^2$, inequality $(c)$ uses the definition of $\alpha_t$. By choosing $c_{\nu} = \frac{32}{9bM} + \frac{\hat{L}}{24Ib^2M^2}$, we have:
\begin{align*}
   \frac{A_t}{\alpha_{t-1}} - \frac{A_{t-1}}{\alpha_{t-2}} & \leq  - \frac{32}{9bM}\alpha_{t-1} A_{t-1} + \frac{2c_{\nu}^2\alpha_{t-1}^3G_2^2}{bM} + \frac{4\hat{L}^2\eta^2\alpha_{t-1}}{bM}(D_{t-1} + E_{t-1}) + \frac{2\hat{L}^2\gamma^2\alpha_{t-1}F_{t-1}}{bM}
\end{align*}
When $t = \bar{t}_s$, by Lemma~\ref{lemma:hg_bound_storm} and Lemma~\ref{lem:x_drift}, we have:
\begin{align*}
  \frac{A_t}{\alpha_{t-1}} - \frac{A_{t-1}}{\alpha_{t-2}} & \leq  - \frac{32}{9bM}\alpha_{t-1} A_{t-1} + \frac{2c_{\nu}^2\alpha_{t-1}^3G_2^2}{bM} + \frac{8\hat{L}^2\eta^2\alpha_{t-1}}{bM}(D_{t-1} + E_{t-1}) \nonumber\\
  &\qquad + \frac{2\hat{L}^2\gamma^2\alpha_{t-1}}{bM}F_{t-1} + \frac{8\hat{L}^2}{bM}\sum_{\ell=\bar{t}_{s-1}}^{t-1}I\eta^2\alpha_{\ell}D_{\ell}
\end{align*}
Note we use the fact $\alpha_t/\alpha_{\bar{t}_s - 1} < 2$ in the last term, which is due to:
\begin{align*}
    \frac{\alpha_t}{\alpha_{\bar{t}_s - 1}} & = \frac{(u_{\bar{t}_s - 1} + \sigma^2 (\bar{t}_s -1))^{1/3}}{(u_t + \sigma^2 t)^{1/3}} = \big(1 + \frac{u_{\bar{t}_s - 1}  - u_t +   \sigma^2(\bar{t}_s -1-  t)}{u_t + \sigma^2t}\big)^{1/3} \\
    & \leq \big(1 + \frac{(I-1)\sigma^2}{u_t + \sigma^2 t}\big)^{1/3} \leq 1 + \frac{(I-1)}{3(t + I+1)}  \leq 2
\end{align*}
where we use the condition $u_t \geq (I+1)\sigma^2$.  Next, we telescope from $\bar{t}_{s-1} + 1$ to $\bar{t}_{s}$:
\begin{align}
    \big(\frac{A_{\bar{t}_s}}{\alpha_{\bar{t}_{s}-1}} - \frac{A_{\bar{t}_{s-1}}}{\alpha_{\bar{t}_{s-1} - 1}}\big) & \leq  - \frac{32}{9bM}\sum_{t=\bar{t}_{s-1}}^{\bar{t}_s-1}\alpha_{t} A_{t} + \frac{2c_{\nu}^2G_2^2}{bM} \sum_{t=\bar{t}_{s-1}}^{\bar{t}_s-1}\alpha_{t}^3 + \frac{16I\hat{L}^2\eta^2}{bM} \sum_{t=\bar{t}_{s-1}}^{\bar{t}_s-1}\alpha_{t}D_{t} \nonumber \\
    & \qquad+ \frac{8\hat{L}^2\eta^2}{bM}\sum_{t=\bar{t}_{s-1}}^{\bar{t}_s-1}\alpha_{t}E_{t}  + \frac{2\hat{L}^2\gamma^2}{bM}\sum_{t=\bar{t}_{s-1}}^{\bar{t}_s-1}\alpha_{t}F_{t}
\label{eq:A_tele}
\end{align}
Next, we follow similar derivation as $A_t/\alpha_{t-1}  - A_{t-1}/\alpha_{t-2}$. By Lemma~\ref{lemma: inner_est_error_storm}, For $t \neq \bar{t}_s$, we choose $c_{\omega} = \frac{144L^2}{bM\mu^2} + \frac{\hat{L}}{24Ib^2M^2}$, to obtain:
\begin{align*}
    \frac{C_t}{\alpha_{t-1}} - \frac{C_{t-1}}{\alpha_{t-2}} & \leq -\frac{144L^2\alpha_{t-1}}{bM\mu^2}C_{t-1} + \frac{2c_{\omega}^2\alpha_{t-1}^3\sigma^2}{bM} + \frac{4L^2\eta^2\alpha_{t-1}}{bM}(D_{t-1} + E_{t-1}) + \frac{2L^2\gamma^2}{bM}\alpha_{t-1}F_{t-1}
\end{align*}
Note we choose $b_y = bM$. When $t = \bar{t}_s$, by Lemma~\ref{lemma:hg_bound_storm} and Lemma~\ref{lem:x_drift}, we have:
\begin{align*}
    \frac{C_t}{\alpha_{t-1}} - \frac{C_{t-1}}{\alpha_{t-2}} & \leq -\frac{144L^2\alpha_{t-1}}{bM\mu^2}C_{t-1} + \frac{2c_{\omega}^2\alpha_{t-1}^3\sigma^2}{bM} + \frac{8L^2\eta^2\alpha_{t-1}}{bM}(D_{t-1} + E_{t-1}) \nonumber\\
    &\qquad +  \frac{2L^2\gamma^2\alpha_{t-1}}{bM}F_{t-1} + \frac{4L^2}{bM}\sum_{\ell=\bar{t}_{s-1}}^{t-1}I\eta^2\alpha_{\ell}D_{\ell}
\end{align*}
Divide $\hat{c}_{\omega}$ for both sides and then telescope from $\bar{t}_{s-1} + 1$ to $\bar{t}_{s}$, we have:
\begin{align}
   \big(\frac{C_{\bar{t}_s}}{\alpha_{\bar{t}_{s}-1}} - \frac{C_{\bar{t}_{s-1}}}{\alpha_{\bar{t}_{s-1} - 1}}\big)& \leq -\frac{144L^2}{2bM\mu^2}\sum_{t=\bar{t}_{s-1}}^{\bar{t}_s-1}\alpha_{t}C_{t} + \frac{2c_{\omega}^2\sigma^2}{bM}\sum_{t=\bar{t}_{s-1}}^{\bar{t}_s-1}\alpha_{t}^3 + \frac{16IL^2\eta^2}{bM}\sum_{t=\bar{t}_{s-1}}^{\bar{t}_s-1}\alpha_{t}D_{t}  \nonumber \\
    & \qquad + \frac{8L^2\eta^2}{bM}\sum_{t=\bar{t}_{s-1}}^{\bar{t}_s-1}\alpha_{t}E_{t} + \frac{2L^2\gamma^2}{bM}\sum_{t=\bar{t}_{s-1}}^{\bar{t}_s-1}\alpha_{t}F_{t}.
\label{eq:C_tele}
\end{align}
Next from Lemma~\ref{lemma: inner_drift_storm}, for $t \neq \bar{t}_s$, we have:
\begin{align*}
    B_t - B_{t-1}  & \leq  - \frac{\mu\gamma\alpha_{t-1}B_{t-1}}{4} - \frac{\gamma^2\alpha_{t-1}F_{t-1}}{4}  + \frac{9\gamma\alpha_{t-1}C_{t-1}}{2\mu} + \frac{9\kappa^2\eta^2\alpha_{t-1}D_{t-1}}{\mu\gamma} + \frac{9\kappa^2\eta^2\alpha_{t-1}E_{t-1}}{\mu\gamma}
\end{align*}
When $t = \bar{t}_s$, we have:
\begin{align*}
    B_t - B_{t-1}  &\leq  - \frac{\mu\gamma\alpha_{t-1}B_{t-1}}{4} - \frac{\gamma^2\alpha_{t-1}F_{t-1}}{4}  + \frac{9\gamma\alpha_{t-1}C_{t-1}}{2\mu} + \frac{18\kappa^2\eta^2\alpha_{t-1}D_{t-1}}{\mu\gamma} \nonumber\\
    &\qquad + \frac{18\kappa^2\eta^2\alpha_{t-1}E_{t-1}}{\mu\gamma} + \frac{9\kappa^2I\eta^2\alpha_{t-1}}{\mu\gamma}\sum_{\ell=\bar{t}_{s-1}}^{t-1}\alpha_{\ell}D_{\ell}
\end{align*}
For the coefficient of the last term, we use $\alpha_t/\alpha_{\bar{t}_s - 1} < 2$. We telescope from $\bar{t}_{s-1} + 1$ to $\bar{t}_{s}$ and have:
\begin{align}\label{eq:B_tele}
    B_{\bar{t}_s} - B_{\bar{t}_{s-1}}  & \leq  - \frac{\mu\gamma}{4} \sum_{t=\bar{t}_{s-1}}^{\bar{t}_s-1}\alpha_{t}B_{t} - \frac{\gamma^2}{4} \sum_{t=\bar{t}_{s-1}}^{\bar{t}_s-1}\alpha_{t}F_{t}  + \frac{9\gamma}{2\mu} \sum_{t=\bar{t}_{s-1}}^{\bar{t}_s-1}\alpha_{t}C_{t} \nonumber\\
    &\qquad + \frac{36I\kappa^2\eta^2}{\mu\gamma} \sum_{t=\bar{t}_{s-1}}^{\bar{t}_s-1}\alpha_{t}D_{t} + \frac{18\kappa^2\eta^2}{\mu\gamma} \sum_{t=\bar{t}_{s-1}}^{\bar{t}_s-1}\alpha_{t}E_{t}
\end{align}
Next, by Lemma~\ref{lemma:desent_storm}, we have:
\begin{align*}
    \mathbb{E} [  h(\bar{x}_{t + 1})]  &\leq \mathbb{E} [    h(\bar{x}_{t })] -  \frac{\eta\alpha_t}{4}  E_t - \frac{\eta\alpha_t}{2} \mathbb{E} [\|\nabla h(\bar{x}_t) \|^2 ] + \bar{L}^2I\eta^3\alpha_t \sum_{\ell = \bar{t}_{s-1}}^{t-1}  \alpha_l^2D_l + 2\eta\alpha_t A_t + 4\hat{L}^2\eta\alpha_t B_t + 4\eta\alpha_t G_1^2
\end{align*}
We telescope from $\bar{t}_{s-1}$ to $\bar{t}_{s}$ to have:
\begin{align}
    \mathbb{E} [  h(\bar{x}_{\bar{t}_{s}}) - h(\bar{x}_{\bar{t}_{s - 1} }) ] & \leq - \sum_{t = \bar{t}_{s-1}}^{\bar{t}_s-1}\frac{\eta\alpha_t}{4} E_t -  \sum_{t = \bar{t}_{s-1}}^{\bar{t}_s-1}\frac{\eta\alpha_t}{2} \mathbb{E} [ \|\nabla h(\bar{x}_t) \|^2 ] + 4\eta \sum_{t = \bar{t}_{s-1}}^{\bar{t}_s-1}\alpha_tG_1^2\nonumber\\
    & \qquad + \bar{L}^2I\eta^3\sum_{t = \bar{t}_{s-1}}^{\bar{t}_s-1}\alpha_t \sum_{\ell = \bar{t}_{s-1}}^{t-1}  \alpha_l^2D_l + \sum_{t = \bar{t}_{s-1}}^{\bar{t}_s-1}2\eta\alpha_t A_t + \sum_{t = \bar{t}_{s-1}}^{\bar{t}_s-1}4\hat{L}^2\eta\alpha_t B_t \nonumber\\
    & \leq -\sum_{t = \bar{t}_{s-1}}^{\bar{t}_s-1}\frac{\eta\alpha_t}{4} E_t -  \sum_{t = \bar{t}_{s-1}}^{\bar{t}_s-1}\frac{\eta\alpha_t}{2} \mathbb{E} [ \|\nabla h(\bar{x}_t) \|^2 ] + 4\eta\sum_{t = \bar{t}_{s-1}}^{\bar{t}_s-1} \alpha_tG_1^2\nonumber\\
    & \qquad + \frac{\kappa^2\eta^3}{64} \sum_{t = \bar{t}_{s-1}}^{\bar{t}_s-1}  \alpha_t D_t + \sum_{t = \bar{t}_{s-1}}^{\bar{t}_s-1}2\eta\alpha_t A_t + \sum_{t = \bar{t}_{s-1}}^{\bar{t}_s-1}4\hat{L}^2\eta\alpha_t B_t 
\label{eq:h_tele}
\end{align}
In the last inequality, we use the fact that $\bar{t}_s - \bar{t}_{s-1}  \leq I$, $\alpha_t < \frac{1}{16\hat{L}I}$ and $\bar{L} / \hat{L} = \kappa + 1 \leq 2\kappa$

Combine Eq.~(\ref{eq:A_tele}),Eq.~(\ref{eq:C_tele}), Eq.~(\ref{eq:B_tele}) and Eq.~(\ref{eq:h_tele}) and we have:
\begin{align*}
    \mathbb{E}[\mathcal{G}_{\bar{t}_s}] - \mathbb{E}[\mathcal{G}_{\bar{t}_{s-1}}] & \leq - \sum_{t = \bar{t}_{s-1}}^{\bar{t}_s-1}\frac{\eta\alpha_t}{2} \mathbb{E} [ \|\nabla h(\bar{x}_t) \|^2] + \big(\frac{9\hat{L}^2\eta c_{\omega}^2\sigma^2}{8L^2} +  \frac{9\eta c_{\nu}^2G_2^2}{8}\big) \sum_{t=\bar{t}_{s-1}}^{\bar{t}_s-1}\alpha_{t}^3 - \frac{\hat{L}^2\eta}{2}\sum_{t=\bar{t}_{s-1}}^{\bar{t}_s-1} \alpha_{t}B_{t}\nonumber\\
    & \qquad  -\sum_{t=\bar{t}_{s-1}}^{\bar{t}_s-1}\big(\frac{9\eta\gamma^2 \hat{L}^2}{2}  - \frac{9\eta\gamma^2 \hat{L}^2}{8}  - \frac{9\eta\gamma\hat{L}^2}{8\mu}\big)\alpha_{t}F_{t} \nonumber \\
    & \qquad -\sum_{t=\bar{t}_{s-1}}^{\bar{t}_s-1}\big(\frac{1}{4} - \frac{324\kappa^2\hat{L}^2\eta^2}{\mu^2\gamma^2}  -  \frac{9\hat{L}^2\eta^2}{2} - \frac{9\hat{L}^2\eta^2}{2} \big)\eta\alpha_{t}E_{t}\nonumber\\
    & \qquad +  \big(\frac{\kappa^2}{64} + \frac{648I\kappa^2\hat{L}^2}{\mu^2\gamma^2} + 9I\hat{L}^2 + 9I\hat{L}^2 \big)\eta^3 \sum_{t=\bar{t}_{s-1}}^{\bar{t}_s-1}\alpha_{t}D_{t} + 4\eta\sum_{t = \bar{t}_{s-1}}^{\bar{t}_s-1} \alpha_tG_1^2
\end{align*}
By the condition that $\eta < \frac{\mu\gamma}{144\kappa\hat{L}}$. So we have:
\begin{align}
    \mathbb{E}[\mathcal{G}_{\bar{t}_s}] - \mathbb{E}[\mathcal{G}_{\bar{t}_{s - 1}}] & \leq - \sum_{t = \bar{t}_{s-1}}^{\bar{t}_s-1}\frac{\eta\alpha_t}{2} \mathbb{E} \big[ \|\nabla h(\bar{x}_t) \|^2 \big] + \big(\frac{9\eta c_{\omega}^2\sigma^2}{8\mu\gamma} +  \frac{9\eta c_{\nu}^2G_2^2}{8\mu\gamma}\big) \sum_{t=\bar{t}_{s-1}}^{\bar{t}_s-1}\alpha_{t}^3 + 4\eta \sum_{t = \bar{t}_{s-1}}^{\bar{t}_s-1}\alpha_tG_1^2\nonumber\\
    & \qquad  -\frac{9\eta\gamma^2\hat{L}^2}{4}\sum_{t=\bar{t}_{s-1}}^{\bar{t}_s-1}\alpha_{t}F_{t} -\frac{3\eta}{16} \sum_{t=\bar{t}_{s-1}}^{\bar{t}_s-1}\alpha_{t}E_{t}   - \frac{\hat{L}^2\eta}{2}\sum_{t=\bar{t}_{s-1}}^{\bar{t}_s-1} \alpha_{t}B_{t} +  C_1I\eta^3\sum_{t=\bar{t}_{s-1}}^{\bar{t}_s-1}\alpha_{t}D_{t}
\label{eq:phi_bound}
\end{align}
where we denote $C_1 = \big(\frac{\kappa^2}{64} + \frac{648\kappa^2\hat{L}^2}{\mu^2\gamma^2} + 9\hat{L}^2 + 9\hat{L}^2 \big)$. By Lemma~\ref{lemma:d_bound} and choose $\eta < \frac{\hat{L}^2}{\kappa c_{\nu}}$, we have:
\begin{align}
    \sum_{t = \bar{t}_{s-1}}^{\bar{t}_s-1} \alpha_{t} D_{t} &\leq \frac{ c_{\nu}^2}{128I^4\hat{L}^2}\sum_{t=\bar{t}_{s-1}}^{\bar{t}_s-1}\alpha_{t}B_{t} + \frac{\eta^2}{8I}\sum_{t=\bar{t}_{s-1}}^{\bar{t}_s-1}\alpha_{t}E_{t} + \frac{\gamma^2}{16I}\sum_{t=\bar{t}_{s-1}}^{\bar{t}_s-1}\alpha_{t}F_{t} \nonumber\\
    &\qquad + \big(\frac{c_{\nu}^2G_1^2}{8I\hat{L}^2} + \frac{c_{\nu}^2G_2^2}{8Ib\hat{L}^2}
    +  \frac{ c_{\nu}^2\zeta^2}{4I\hat{L}^2}\big)\sum_{t=\bar{t}_{s-1}}^{\bar{t}_s-1}\alpha_{t}^3
\label{eq:d_bound}
\end{align}
Combine Eq.~(\ref{eq:phi_bound}) and Eq.~(\ref{eq:d_bound}), and use the condition that $\eta < \min\big(\frac{\hat{L}^2}{C_1^{1/2}c_\nu}, \frac{1}{C_1^{1/2}}, \frac{\hat{L}}{C_1^{1/2}}, 1\big)$, the fact that $I \geq 1$, we have:
\begin{align*}
    \mathbb{E}[\mathcal{G}_{\bar{t}_s}] - \mathbb{E}[\mathcal{G}_{\bar{t}_{s - 1}}] & \leq - \sum_{t = \bar{t}_{s-1}}^{\bar{t}_s-1}\frac{\eta\alpha_t}{2} \mathbb{E} \big[ \|\nabla h(\bar{x}_t) \|^2 \big] + 4\eta \sum_{t = \bar{t}_{s-1}}^{\bar{t}_s-1}\alpha_tG_1^2 \nonumber\\
    &\qquad + \eta\big(\frac{9\hat{L}^2 c_{\omega}^2\sigma^2}{8L^2} +  \frac{9 c_{\nu}^2G_2^2}{8} + \frac{c_{\nu}^2G_1^2}{8} + \frac{c_{\nu}^2G_2^2}{8b} +  \frac{ c_{\nu}^2\zeta^2}{4}\big) \sum_{t=\bar{t}_{s-1}}^{\bar{t}_s-1}\alpha_{t}^3
\end{align*}
Sum over all $s \in [S]$ (assume $T = SI + 1$ without loss of generality), we have:
\begin{align*}
    \mathbb{E}[\mathcal{G}_{T}] - \mathbb{E}[\mathcal{G}_{1}] 
    & \leq - \sum_{t = 1}^{T-1}\frac{\eta\alpha_t}{2} \mathbb{E} \big[ \|\nabla h(\bar{x}_t) \|^2 \big] + \big(\eta C_{\sigma,\zeta} + \frac{4\eta G_1^2}{\alpha_T^2}\big)\sum_{t=1}^{T-1}\alpha_t^3
\end{align*}
For ease of notation, we denote $C_{\sigma,\zeta} = \big(\frac{9\hat{L}^2 c_{\omega}^2\sigma^2}{8L^2} +  \frac{9c_{\nu}^2G_2^2}{8} + \frac{c_{\nu}^2G_1^2}{8} + \frac{c_{\nu}^2G_2^2}{8b} +  \frac{ c_{\nu}^2\zeta^2}{4}\big)$.
Rearranging the terms and use the fact that $\alpha_t$ is non-increasing, we have:
\begin{align*}
    \frac{\eta\alpha_T}{2}\sum_{t = 1}^{T-1}\mathbb{E} \big[ \|\nabla h(\bar{x}_t) \|^2 \big] 
    & \leq \mathbb{E}[\mathcal{G}_{1}] - \mathbb{E}[\mathcal{G}_{T}]  + \big(\eta C_{\sigma,\zeta} + \frac{4\eta G_1^2}{\alpha_T^2}\big)\sum_{t=1}^{T-1}\alpha_t^3\nonumber\\
    & \leq h(x_{1}) - h^{\ast} + \frac{9bM\eta A_1}{16\alpha_{1}} + \frac{18\eta\hat{L}^2B_1}{\mu\gamma} + \frac{ 9bM\eta C_1}{16\alpha_{1}} + \big(\eta C_{\sigma,\zeta} + \frac{4\eta G_1^2}{\alpha_T^2}\big)\sum_{t=1}^{T-1}\alpha_t^3
\end{align*}
where we use $\mathcal{G}_T \ge h^{\ast}$ ($h^{\ast}$ is the optimal value of $h$), and for the last term, we use the following fact:
\begin{align*}
     \sum_{t=1}^T \alpha_t^3 & =    \sum_{t = 1}^{T} \frac{\delta^3 }{u + \sigma^2 t} \leq  \sum_{t = 1}^{T} \frac{\delta^3  }{\sigma^2 + \sigma^2 t} = \frac{ \delta^3}{\sigma^2}   \sum_{t = 1}^{T} \frac{1}{1 +   t} \leq \frac{  \delta^3 }{\sigma^2}   \ln(T+1) = \frac{b^2M^2}{\hat{L}^2} \ln(T+1)
\end{align*}
the first inequality follows $u_t > \sigma^2$, the last inequality follows Proposition~\ref{Lem: AD_Sum_1overT}. Next, we denote the initial sub-optimality as $\Delta = h(\bar{x}_{1}) - h^{\ast}$, and initial inner variable estimation error \emph{i.e.} $B_1 = \frac{1}{M} \sum_{m=1}^M \|y^{(m)}_1 - y^{(m)}_{x^{(m)}_{1}} \|^2 \leq \Delta_y$, and we assume $\Delta_y = O(\kappa^{-1})$, furthermore, we have:
\begin{align*}
A_1 = \mathbb{E} \big[\big\| \frac{1}{M} \sum_{m=1}^M \big(\Phi^{(m)} (x^{(m)}_{1}, y^{(m)}_{1}; \mathcal{B}_{x}) -  \Phi^{(m)} (x^{(m)}_{1}, y^{(m)}_{1}) \big) \big\|^2\big] \leq \frac{\sigma^2}{b_1M}
\end{align*}
and \[C_1 =  \frac{1}{M} \sum_{m=1}^M \mathbb{E}\|\omega^{(m)}_1 - \nabla_y g^{(m)}(x^{(m)}_{1}, y^{(m)}_{1} )\|^2 \leq \frac{\sigma^2}{b_1}\] where we choose the size of the first minibatch to be $b_x = b_1$ and $b_y = b_1M$. Then, we divide both sides by $\eta\alpha_T T/2$ to have:
\begin{align*}
    \frac{1}{T}\sum_{t = 1}^{T-1} \mathbb{E} \big[ \|\nabla h(\bar{x}_t) \|^2 \big] & \leq \frac{2\Delta}{\eta\alpha_T T} + \frac{9b\sigma^2}{8 b_1T\alpha_{1}\alpha_T} + \frac{36\hat{L}^2\Delta_y}{\kappa\mu\gamma T\alpha_T} + \frac{ 9b\sigma^2}{8b_1T\alpha_{1}\alpha_T} + \frac{2b^2M^2C_{\sigma,\zeta}\ln(T)}{\hat{L}^2 T\alpha_T} + \frac{8b^2M^2 G_1^2\ln(T)}{\hat{L}^2T\alpha_T^3}
\end{align*}
Note that we have:
\begin{align*}
    \frac{1}{{\alpha_t t}} = \frac{(u + \sigma^2t)^{1/3}}{\delta t} \leq \frac{u^{1/3}}{\delta t} + \frac{\sigma^{2/3}}{\delta t^{2/3}}
\end{align*}
where the inequality uses the fact that $(x + y)^{1/3} \leq x^{1/3} + y^{1/3}$. In particular, when $t=1$, we have 
\begin{align*}
    \frac{1}{{\alpha_1}} \leq \frac{u^{1/3} + \sigma^{2/3}}{\delta} = \frac{\hat{L}^{2/3}(\bar{u}^{1/3}(bM)^{2/3} + 1)}{(bM)^{2/3}}
\end{align*}
when $t=T$, we have:
\begin{align*}
    \frac{1}{{\alpha_T T}} \leq \frac{u^{1/3}}{\delta T} + \frac{\sigma^{2/3}}{\delta T^{2/3}} = \frac{\hat{L}^{2/3}\bar{u}^{1/3}}{T} + \frac{\hat{L}^{2/3}}{(bMT)^{2/3}}
\end{align*}
In summary, we have:
\begin{align*}
    \frac{1}{T}\sum_{t = 1}^{T-1} \mathbb{E} \big[ \|\nabla h(\bar{x}_t) \|^2 \big]  & \leq \big(\frac{2\Delta}{\eta} +  \frac{9b\sigma^2}{8b_1\alpha_{1}} + \frac{36\hat{L}^2\Delta_y}{\kappa\mu\gamma } + \frac{ 9b\sigma^2}{8b_1\alpha_{1}} \nonumber\\
    & \qquad + 2\ln(T)\big(\frac{b^2M^2C_{\sigma,\zeta}}{\hat{L}^2} + \frac{4b^2M^2G_1^2}{\hat{L}^2\alpha_T^2}\big) \big)\big(\frac{\hat{L}^{2/3}\bar{u}^{1/3}}{T} + \frac{\hat{L}^{2/3}}{(bMT)^{2/3}}\big)
\end{align*}
Note that $\hat{L} = O(\kappa^2)$,  $\bar{L} = O(\kappa^3)$, $c_\nu = O((bM)^{-1}\kappa^2)$ and $c_{\omega} = O((bM)^{-1}\kappa^2)$, $\bar{u} = O(I^{9/2}\kappa^2 + (bM)^{-3/2}\kappa^3)$, $\alpha_1^{-1} = O(I^{3/2}\kappa^2 + (bM)^{-1/2}\kappa^{7/3})$, then for $\eta$, we have:
\[
\eta \leq \min\big(\frac{\mu\gamma}{144\kappa\hat{L}}, \frac{\hat{L}^2}{\kappa c_{\nu}}, \frac{\hat{L}^2}{C_1^{1/2}c_\nu}, \frac{1}{C_1^{1/2}}, \frac{\hat{L}}{C_1^{1/2}}, \frac{1}{2\bar{L}},1\big)
\]
Recall that $C_1 = \big(\frac{\kappa^2}{64} + \frac{648\kappa^2\hat{L}^2}{\mu^2\gamma^2} + 9\hat{L}^2 + 9\hat{L}^2 \big)$, suppose we choose $\gamma = \frac{1}{2L}$, then $C_1 = O(\kappa^8)$ and $\eta^{-1} = O(\kappa^4)$, $\mu\gamma = O(\kappa^{-1})$. recall that $C_{\sigma,\zeta} = \big(\frac{9\hat{L}^2 c_{\omega}^2\sigma^2}{8L^2} +  \frac{9c_{\nu}^2G_2^2}{8} + \frac{c_{\nu}^2G_1^2}{8} + \frac{c_{\nu}^2G_2^2}{8b} +  \frac{ c_{\nu}^2\zeta^2}{4}\big)$, so we have $C_{\sigma, \zeta} = O((bM)^{-2}\kappa^8)$, suppose we choose $b_1 =O(I^{3/2})$. Finally, for the coefficient of the hyper-gradient bias term $G_1^2$, we have:
\[
\frac{ 8b^2M^2\ln(T)}{\hat{L}^{1/3}\alpha_T^2}\big(\frac{\bar{u}^{1/3}}{T} + \frac{1}{(bMT)^{2/3}}\big) \leq \frac{16b^2M^2\bar{u}}{T} + 16\left(\frac{b^2M^2\bar{u}}{T}\right)^{2/3} + 16\left(\frac{b^2M^2\bar{u}}{T}\right)^{1/3} + 16 = O(\kappa^3b^2M^2I^{9/2})
\]
Then, we have:
\begin{align*}
    \frac{1}{T}\sum_{t = 1}^{T-1} \mathbb{E} \big[ \|\nabla h(\bar{x}_t) \|^2 \big] = O\big(\frac{\kappa^{19/3}I^{3/2}}{T} + \frac{\kappa^{16/3}}{(bMT)^{2/3}} + \kappa^3b^2M^2I^{9/2}G_1^2\big)
\end{align*}
To reach an  $\epsilon$-stationary point, we need $T = O(\kappa^{8}(bM)^{-1}\epsilon^{-1.5})$, $I = O(\kappa^{10/9}(bM)^{-2/3}\epsilon^{-1/3})$ and $Q =O(\kappa\log(\frac{\kappa}{bM\epsilon}))$. The communication cost is $E = T/I \geq \kappa^{62/9}(bM)^{-1/3}\epsilon^{-7/6}$, the sample complexity is $Gc(f, \epsilon) = O(M^{-1}\kappa^{8}\epsilon^{-1.5})$, $Gc(g, \epsilon) = O(\kappa^{8}\epsilon^{-1.5})$, $Jv(g, \epsilon) = O(\kappa^{8}\epsilon^{-1.5})$, $Hv(g, \epsilon) = O(\kappa^{9}\epsilon^{-1.5})$

Suppose we choose $b = O(\epsilon^{-0.5})$, we have $T = O(\kappa^{8}M^{-1}\epsilon^{-1}))$, $I = \kappa^{10/9}M^{-2/3}$, $Q=O(\kappa\log(\frac{\kappa}{M\epsilon}))$ and $E = \kappa^{62/9}M^{-1/3}\epsilon^{-1}$. If we instead choose $b=O(1)$, we have $T = O(\kappa^{8}M^{-1}\epsilon^{-1.5})$, $I = O(\kappa^{10/9}M^{-2/3}\epsilon^{-1/3})$ and $Q=O(\kappa\log(\frac{\kappa}{M\epsilon}))$. The communication cost is $E = O(\kappa^{62/9}M^{-1/3}\epsilon^{-7/6})$.
\end{proof}


\subsection{Proof for the FedBiO-Local Algorithm}
In this section, we investigate the convergence rate for the FedBiO-Local algorithm (Algorithm~\ref{alg:FedBiO-Local}).

\subsubsection{Lower Problem Solution Error}
\begin{lemma}
\label{lemma: inner_drift}
When $\gamma < \frac{1}{L}$, when $t \neq \bar{t}_s$, we have:
\begin{align*}
\frac{1}{M}\sum_{m=1}^M \mathbb{E}\big\|y^{(m)}_t - y^{(m)}_{x^{(m)}_{t}} \big\|^2  & \leq (1 - \frac{\mu\gamma}{2}) \frac{1}{M}\sum_{m=1}^M \mathbb{E}\big\|y^{(m)}_{t-1} - y^{(m)}_{x^{(m)}_{t-1}}\big\|^2 + \frac{5\kappa^2\eta^2}{\mu\gamma M}\sum_{m=1}^M \mathbb{E}\big\|\nu^{(m)}_{t-1}\big\|^2 + \frac{3\gamma^2\sigma^2}{b_y}
\end{align*}
when $t = \bar{t}_s$, we have:
\begin{align*}
\frac{1}{M}\sum_{m=1}^M \mathbb{E}\big\|y^{(m)}_t - y^{(m)}_{x^{(m)}_{t}} \big\|^2  & \leq (1 - \frac{\mu\gamma}{2}) \frac{1}{M}\sum_{m=1}^M \mathbb{E}\big\|y^{(m)}_{t-1} - y^{(m)}_{x^{(m)}_{t-1}}\big\|^2 + \frac{10\kappa^2\eta^2}{\mu\gamma M}\sum_{m=1}^M \mathbb{E}\big\|\nu^{(m)}_{t-1}\big\|^2 \nonumber\\
&\qquad + \frac{10\kappa^2}{\mu\gamma M}\sum_{m=1}^M \mathbb{E}\big\|x^{(m)}_{t} - \bar{x}_{t}\big\|^2 + \frac{3\gamma^2\sigma^2}{b_y}
\end{align*}
\end{lemma}

\begin{proof}
First, we have:
\begin{align*}
\mathbb{E}\big\|y^{(m)}_t - y^{(m)}_{x^{(m)}_{t}} \big\|^2  &\leq (1 + \frac{\mu\gamma}{2})\mathbb{E}\big\|y^{(m)}_{t} - y^{(m)}_{x^{(m)}_{t-1}}\big\|^2 + (1 + \frac{2}{\mu\gamma})\mathbb{E}\big\|y^{(m)}_{x^{(m)}_{t}} - y^{(m)}_{x^{(m)}_{t-1}}\big\|^2  \nonumber \\
& \leq (1 + \frac{\mu\gamma}{2})(1 - \mu\gamma)\mathbb{E}\big\|y^{(m)}_{t-1} - y^{(m)}_{x^{(m)}_{t-1}}\big\|^2 +  (1 + \frac{2}{\mu\gamma}) \mathbb{E}\big\|y^{(m)}_{x^{(m)}_{t}} - y^{(m)}_{x^{(m)}_{t-1}}\big\|^2 + (1 + \frac{\mu\gamma}{2})\frac{2\gamma^2\sigma^2}{b_y}  \nonumber \\
& \leq (1 - \frac{\mu\gamma}{2})\mathbb{E}\big\|y^{(m)}_{t-1} - y^{(m)}_{x^{(m)}_{t-1}}\big\|^2 + \frac{3\kappa^2}{\mu\gamma}\mathbb{E}\big\|x^{(m)}_{t} - x^{(m)}_{t-1}\big\|^2 + \frac{3\gamma^2\sigma^2}{b_y}
\end{align*}
where the second inequality is due to Proposition~\ref{prop:strong-prog1} where we choose $\gamma < 1/L$; in the last inequality, we use $\gamma < 1/(L)$ and $\mu \leq L$,  For the last term, when $t \neq \bar{t}_s$, we have:
\begin{align*}
\mathbb{E}\big\|y^{(m)}_t - y^{(m)}_{x^{(m)}_{t}} \big\|^2  &\leq (1 - \frac{\mu\gamma}{2})\mathbb{E}\big\|y^{(m)}_{t-1} - y^{(m)}_{x^{(m)}_{t-1}}\big\|^2 + \frac{5\kappa^2\eta^2}{\mu\gamma}\mathbb{E}\big\|v^{(m)}_{t-1}\big\|^2 + \frac{3\gamma^2\sigma^2}{b_y}
\end{align*}
Then when $t = \bar{t}_s$, we have 
\begin{align*}
\mathbb{E}\big\|y^{(m)}_t - y^{(m)}_{x^{(m)}_{t}} \big\|^2  &\leq (1 - \frac{\mu\gamma}{2})\mathbb{E}\big\|y^{(m)}_{t-1} - y^{(m)}_{x^{(m)}_{t-1}}\big\|^2 + \frac{10\kappa^2\eta^2}{\mu\gamma}\mathbb{E}\big\|v^{(m)}_{t-1}\big\|^2 + \frac{10\kappa^2}{\mu\gamma}\mathbb{E}\|x^{(m)}_t - \bar{x}_t\|^2 + \frac{3\gamma^2\sigma^2}{b_y}
\end{align*}
Average over all clients, we get the claim in the lemma.
\end{proof}

\subsubsection{Upper Variable Drift}
\begin{lemma}
\label{lem: x_drift_FedAvg}
For any $t \neq \bar{t}_s, s \in [S]$, we have:
\begin{align*}
\|x_t^{(m)}-  \bar{x}_t \|^2 &\leq I\eta^2 \sum_{\ell = \bar{t}_{s-1}}^{t-1} \big\|   \nu_\ell^{(m)} - \bar{\nu}_\ell\big\|^2
\end{align*}
\end{lemma}

\begin{proof}
Note from Algorithm and the definition of $\bar{t}_s$ that at $t = \bar{t}_{s}$ with $s \in [S]$, $x_{t}^{(m)} = \bar{x}_{t}$, for all $k$. For $t \neq \bar{t}_s$, with $s \in [S]$, we have: $x_{t}^{(m)} = x_{t-1}^{(m)} - \eta  \nu_{t-1}^{(m)}$, this implies that: $x_t^{(m)} = x_{\bar{t}_{s-1}}^{(m)} - \sum_{\ell = \bar{t}_{s-1}}^{t-1} \eta  \nu_\ell^{(m)} \quad \text{and} \quad \bar{x}_{t}  = \bar{x}_{\bar{t}_{s-1}}  - \sum_{\ell = \bar{t}_{s-1}}^{t-1} \eta  \bar{\nu}_\ell.$
So for $t \neq \bar{t}_s$, with $s \in [S]$ we have:
\begin{align*}
\|x_t^{(m)}-  \bar{x}_t \|^2 & =  \big\| x_{\bar{t}_{s-1}}^{(m)} - \bar{x}_{\bar{t}_{s-1}}  - \big( \sum_{\ell = \bar{t}_{s-1}}^{t-1} \eta  \nu_\ell^{(m)} -   \sum_{\ell =  \bar{t}_{s-1}}^{t-1} \eta  \bar{\nu}_\ell  \big) \big\|^2 \overset{(a)}{=}  \big\|  \sum_{\ell = \bar{t}_{s-1}}^{t-1} \eta\big(  \nu_\ell^{(m)} -      \bar{\nu}_\ell  \big) \big\|^2 \nonumber\\
&\leq I\eta^2 \sum_{\ell = \bar{t}_{s-1}}^{t-1} \big\|   \nu_\ell^{(m)} -      \bar{\nu}_\ell\big\|^2
\end{align*}
This completes the proof.
\end{proof}

\begin{lemma}
\label{lem: ErrorAccumulation_Iterates_FedAvg}
For $t \neq \bar{t}_s, s \in [S]$, we have:
\begin{align*}
\frac{1}{M} \sum_{m = 1}^M  \mathbb{E}\big\|\big(  \nu_t^{(m)} -\bar{\nu}_t  \big) \big\|^2 &\leq \frac{4\hat{L}^2}{M}  \sum_{m = 1}^M  \mathbb{E}\big\| y^{(m)}_t - y^{(m)}_{x^{(m)}_{t}} \big\|^2 + \frac{12\bar{L}^2I\eta^2}{M} \sum_{m = 1}^M   \sum_{\ell = \bar{t}_{s-1}}^{t-1} \mathbb{E}\big\|   \nu_\ell^{(m)} - \bar{\nu}_\ell\big\|^2 \nonumber\\
&\qquad +   8\zeta^2 + 4G_1^2 + \frac{4G_2^2}{b_x}
\end{align*}
for $ t = \bar{t}_s, s\in [S]$, we have:
\begin{align*}
\frac{1}{M} \sum_{m = 1}^M  \mathbb{E}\big\|\big(  \nu_t^{(m)} -\bar{\nu}_t  \big) \big\|^2 \leq \frac{4\hat{L}^2}{M}  \sum_{m = 1}^M  \mathbb{E}\big\| y^{(m)}_t - y^{(m)}_{x^{(m)}_{t}} \big\|^2 +   8\zeta^2 + 4G_1^2 + \frac{4G_2^2}{b_x}
\end{align*}
\end{lemma}
\begin{proof}
For $t \in [T]$, we have:
\begin{align}
\label{Eq: ConsensusError_FedAvg}
\mathbb{E}\big\|\big(  \nu_t^{(m)} -\bar{\nu}_t  \big) \big\|^2
& \overset{(a)}{\leq}  2\mathbb{E}\big\| \big(\nu_t^{(m)} - \nabla h^{(m)} (x_t^{(m)})  \big) - \big(  \bar{\nu}_t - \frac{1}{M} \sum_{m=1}^M \nabla h^{(j)} (x_t^{(j)})\big)\big\|^2 \nonumber\\
&\qquad \qquad \qquad \qquad  +  2\mathbb{E}\big\|   \big( \nabla h^{(m)} (x_t^{(m)})  - \frac{1}{M} \sum_{j = 1}^M \nabla h^{(j)} (x_t^{(j)}) \big) \big\|^2  \nonumber\\
& \overset{(b)}{\leq}  2\mathbb{E}\big\|\nu_t^{(m)} - \nabla h^{(m)} (x_t^{(m)}) \big\|^2 +  2\mathbb{E}\big\|\nabla h^{(m)} (x_t^{(m)})  - \frac{1}{M} \sum_{j = 1}^M \nabla h^{(j)} (x_t^{(j)}) \big\|^2 \nonumber\\
&\leq 4\hat{L}^2 \mathbb{E}\big\| y^{(m)}_t - y^{(m)}_{x^{(m)}_{t}} \big\|^2 + 4G_1^2 + \frac{4G_2^2}{b_x} + 2\mathbb{E}\big\|\nabla h^{(m)} (x_t^{(m)})  - \frac{1}{M} \sum_{j = 1}^M \nabla h^{(j)} (x_t^{(j)}) \big\|^2
\end{align}
where the equality $(a)$ uses triangle inequality and $(b)$ follows from the application of Proposition~\ref{prop: Sum_Mean_Kron}. Next, for the second term of \ref{Eq: ConsensusError_FedAvg} we have:
\begin{align}
& \sum_{m = 1}^M  \big\| \nabla h^{(m)} (x_t^{(m)})  - \frac{1}{M} \sum_{j = 1}^M \nabla h^{(j)} (x_t^{(j)})\big\|^2  \nonumber\\
& \overset{(a)}{\leq} 2\sum_{m = 1}^M       \big\|    \nabla h^{(m)} (x_t^{(m)})  -  \nabla h^{(m)}(\bar{x}_t)   \big\|^2 + 4M  \big\|    \nabla h (\bar{x}_t)  - \frac{1}{M} \sum_{j = 1}^M \nabla h^{(j)} (x_t^{(j)})   \big\|^2 \nonumber\\
&\qquad +  4\sum_{m = 1}^M  \big\|    \nabla h^{(m)} (\bar{x}_t)  -   \nabla h (\bar{x}_t)   \big\|^2
\overset{(b)}{\leq}  6\bar{L}^2 \sum_{m = 1}^M       \big\|x_t^{(m)}  -  \bar{x}_t \big\|^2  +  4M\zeta^2
\label{Eq: InterNodeVar_FedAvg}
\end{align}
where $(a)$ follows the generalized triangle inequality; $(b)$ utilizes the heterogeneity Assumption~\ref{assumption:hetero}. Next for the first term, it is 0 when $t = \bar{t}_s$ and when $t \neq \bar{t}_s$, we use Lemma~\ref{lem: x_drift_FedAvg}. Substituting~\ref{Eq: InterNodeVar_FedAvg} back to \ref{Eq: ConsensusError_FedAvg}, we get the results in the lemma.
\end{proof}

\begin{lemma}
\label{lemma:d_bound_fedavg}
For $s \in [S]$, we have:
\begin{align*}
    (1 - 12\bar{L}^2I^2\eta^2)\sum_{t = \bar{t}_{s-1}}^{\bar{t}_s-1} D_{t} &\leq 4\hat{L}^2\sum_{t=\bar{t}_{s-1}}^{\bar{t}_s-1} B_t  +   8I\zeta^2 + 4IG_1^2 + \frac{4IG_2^2}{b_x}
\end{align*}
\end{lemma}
\begin{proof}
For ease of notation, we denote $D_t = \frac{1}{M} \sum_{m = 1}^M \mathbb{E}\|(  \nu_t^{(m)} -\bar{\nu}_t)\|^2$ and $B_t = \frac{1}{M} \sum_{m=1}^M \mathbb{E}\|y^{(m)}_t - y^{(m)}_{x^{(m)}_{t}}\|^2 $. Based on Lemma~\ref{lem: ErrorAccumulation_Iterates_FedAvg},  we have:
\begin{align*}
D_t \leq 4\hat{L}^2B_t + 12\bar{L}^2I\eta^2 \sum_{\ell = \bar{t}_{s-1}}^{t-1} D_\ell +   8\zeta^2 + 4G_1^2 + \frac{4G_2^2}{b_x}
\end{align*}
Next, we sum from $\bar{t}_{s-1} + 1$ to $\bar{t}_{s}-1$, we have:
\begin{align*}
    \sum_{t=\bar{t}_{s-1}+1}^{\bar{t}_s-1} D_t  &\leq 4\hat{L}^2\sum_{t=\bar{t}_{s-1}+1}^{\bar{t}_s-1} B_t + 12\bar{L}^2I\eta^2 \sum_{t=\bar{t}_{s-1}+1}^{\bar{t}_s-1}\sum_{\ell = \bar{t}_{s-1}}^{t-1} D_\ell +   8(I-1)\zeta^2 + 4(I-1)G^2 \nonumber\\
    &\leq 4\hat{L}^2\sum_{t=\bar{t}_{s-1}+1}^{\bar{t}_s-1} B_t + 12\bar{L}^2I^2\eta^2 \sum_{\ell = \bar{t}_{s-1}}^{\bar{t}_s-1} D_\ell +   8(I-1)\zeta^2 + 4(I-1)G_1^2 + \frac{4(I-1)G_2^2}{b_x}
\end{align*}
In the second inequality, we use $t - 1 \leq \bar{t}_s-1$, combine with the case when $t = \bar{t}_{s-1}$ in lemma~\ref{lem: x_drift_FedAvg}, we have:
\begin{align*}
    (1 - 12\bar{L}^2I^2\eta^2)\sum_{t = \bar{t}_{s-1}}^{\bar{t}_s-1} D_{t} &\leq 4\hat{L}^2\sum_{t=\bar{t}_{s-1}}^{\bar{t}_s-1} B_t  +   8I\zeta^2 + 4IG_1^2 + \frac{4IG_2^2}{b_x}
\end{align*}
This completes the proof.
\end{proof}

\subsubsection{Descent Lemma}
\begin{lemma}
\label{lemma:hg_bound}
For all $t \in [\bar{t}_{s-1}, \bar{t}_s - 1]$, the iterates generated satisfy:
\begin{align*}
 \mathbb{E}\big\|  \nabla h(\bar{x}_{t})  - \mathbb{E}_{\xi}[\bar{\nu}_t]   \big\|^2 \leq \frac{2\hat{L}^2}{M}\sum_{m=1}^M  \big(  4\kappa^2\mathbb{E}\big\|x_t^{(m)} - \bar{x}_t \big\|^2 + 2\mathbb{E}\big\| y^{(m)}_t - y^{(m)}_{x^{(m)}_{t}} \big\|^2\big) + 2G_1^2
\end{align*}
\end{lemma}

\begin{proof}
By definition of $\bar{\nu}_t$ and $\nabla h(\bar{x}_{t})$, we have:
\begin{align*}
      \mathbb{E}\big\|  \nabla h(\bar{x}_{t})  - \mathbb{E}_{\xi}[\bar{\nu}_t]  \big\|^2  &\overset{(a)}{\leq}  \frac{1}{M}\sum_{m=1}^M   \mathbb{E}\big\|\mathbb{E}_{\xi}[\nu^{(m)}_t] - \nabla h^{(m)}(\bar{x}_t) \big\|^2  \nonumber \\ &\leq  \frac{2}{M}\sum_{m=1}^M   \mathbb{E}\big[\big\|\mathbb{E}_{\xi}[\nu^{(m)}_t] - \mu^{(m)}_t\big\|^2 +  \big\|\mu^{(m)}_t - \nabla h^{(m)}(\bar{x}_t) \big\|^2\big] \nonumber \\
      & \overset{(b)}{\leq} \frac{\hat{L}^2}{M}\sum_{m=1}^M  \big( \mathbb{E}\big\|x_t^{(m)} - \bar{x}_t \big\|^2 +  \mathbb{E}\big\| y_t^{(m)} - y^{(m)}_{\bar{x}_t} \big\|^2\big)  + 2G_1^2  \\
      & \leq \frac{2\hat{L}^2}{M}\sum_{m=1}^M  \big( \mathbb{E}\big\|x_t^{(m)} - \bar{x}_t \big\|^2 + \mathbb{E}\big\| y^{(m)}_t - y^{(m)}_{x^{(m)}_{t}} + y^{(m)}_{x^{(m)}_{t}} - y^{(m)}_{\bar{x}_{t}} \big\|^2\big) + 2G_1^2  \\
      & \leq \frac{2\hat{L}^2}{M}\sum_{m=1}^M  \big( \big(1 + 2\kappa^2\big)\mathbb{E}\big\|x_t^{(m)} - \bar{x}_t \big\|^2 + 2\mathbb{E}\big\| y^{(m)}_t - y^{(m)}_{x^{(m)}_{t}} \big\|^2\big) + + 2G_1^2
\end{align*}
where inequality (a) follows the generalized triangle inequality; inequality (b) follows the Proposition~\ref{some smoothness local} and Proposition~\ref{prop:hg_var_app}.

\end{proof}

\begin{lemma}
\label{lemma:desent}
For $t \neq \bar{t}_s$, the iterates generated satisfy:
\begin{align*}
     \mathbb{E}[h(\bar{x}_{t + 1})] & \leq \mathbb{E}[h(\bar{x}_{t })] - \frac{\eta}{2} \mathbb{E}\|\nabla h(\bar{x}_{t})\|^2 - \frac{\eta}{4}\mathbb{E}\big\| \mathbb{E}_{\xi}[\bar{\nu}_t]  \big\|^2 + \frac{\eta^2 \bar{L}G_2^2}{2b_xM} + \eta G_1^2 \nonumber\\
     &\qquad + \frac{\eta\hat{L}^2}{M}\sum_{m=1}^M  \big(  4\kappa^2 I\eta^2 \sum_{\ell = \bar{t}_{s-1}}^{t-1} \mathbb{E}\big\|   \nu_\ell^{(m)} - \bar{\nu}_\ell\big\|^2 + 2\mathbb{E}\big\| y^{(m)}_t - y^{(m)}_{x^{(m)}_{t}} \big\|^2\big)
\end{align*}
for $ t = \bar{t}_s$, we have:
\begin{align*}
     \mathbb{E}[h(\bar{x}_{t + 1})] & \leq \mathbb{E}[h(\bar{x}_{t })] - \frac{\eta}{2} \mathbb{E}\|\nabla h(\bar{x}_{t})\|^2 - \frac{\eta}{4}\mathbb{E}\big\| \mathbb{E}_{\xi}[\bar{\nu}_t]  \big\|^2 + \frac{\eta^2 \bar{L}G_2^2}{2b_xM} + \eta G_1^2 + \frac{2\eta\hat{L}^2}{M}\sum_{m=1}^M \mathbb{E}\big\| y^{(m)}_t - y^{(m)}_{x^{(m)}_{t}} \big\|^2
\end{align*}
where the expectation is w.r.t the stochasticity of the algorithm.
\end{lemma}
\begin{proof}
Using the smoothness of $f$ we have:
\begin{align*}
    \mathbb{E} [h(\bar{x}_{t + 1})]
    & \leq   \mathbb{E} [h(\bar{x}_{t })] + \mathbb{E}\langle \nabla h(\bar{x}_{t}),  \bar{x}_{t + 1} - \bar{x}_{t}\rangle + \frac{\bar{L}}{2} \mathbb{E}\| \bar{x}_{t + 1} - \bar{x}_{t } \|^2  \nonumber\\
    &  \overset{(a)}{=}  \mathbb{E} [h(\bar{x}_{t})] - \eta \mathbb{E}\langle \nabla h(\bar{x}_{t}),  \mathbb{E}_{\xi}[\bar{\nu}_t] \rangle + \frac{\eta^2 \bar{L}}{2} \mathbb{E}\| \mathbb{E}_{\xi}[\bar{\nu}_{t}]  \|^2 +  \frac{\eta^2 \bar{L}G_2^2}{2b_x M} \nonumber\\
    &  \overset{(b)}{=}  \mathbb{E}[h(\bar{x}_{t})] - \frac{\eta}{2} \mathbb{E}\|\nabla h(\bar{x}_{t})\|^2  + \frac{\eta}{2}\mathbb{E}\|\nabla h(\bar{x}_{t}) - \mathbb{E}_{\xi}[\bar{\nu}_t]\|^2   - \left(\frac{\eta}{2} - \frac{\eta^2 \bar{L}}{2}\right) \mathbb{E}\| \mathbb{E}_{\xi}[\bar{\nu}_t]  \|^2 +  \frac{\eta^2 \bar{L}G_2^2}{2b_x M}\nonumber\\
    & \overset{(c)}{\leq} \mathbb{E}[h(\bar{x}_{t })] - \frac{\eta}{2} \mathbb{E}\|\nabla h(\bar{x}_{t})\|^2 - \frac{\eta}{4}\mathbb{E}\big\| \mathbb{E}_{\xi}[\nu^{(m)}_t]  \big\|^2 + \frac{\eta^2 \bar{L}G_2^2}{2b_xM} + \eta G_1^2 \nonumber\\
    &\qquad + \frac{\eta\hat{L}^2}{M}\sum_{m=1}^M  \big(  \underbrace{4\kappa^2 \mathbb{E}\big\|x_t^{(m)} - \bar{x}_t \big\|^2}_{T_1} + 2\mathbb{E}\big\| y^{(m)}_t - y^{(m)}_{x^{(m)}_{t}} \big\|^2 \big)
\end{align*}
where equality $(a)$ follows from the iterate update given in Step 6 of Algorithm~\ref{alg:FedBiO}; $(b)$ uses $\langle a , b \rangle = \frac{1}{2} [\|a\|^2 + \|b\|^2 - \|a - b \|^2]$;  (c) follows the assumption that  $\eta < 1/2\bar{L}$. Finally, use lemma~\ref{lemma:hg_bound} to bound $T_1$ when $t \neq \bar{t}_s$ finishes the proof.
\end{proof}

\subsubsection{Proof of Convergence Theorem}\label{appendix:fedbio}
We first denote the following potential function $\mathcal{G}(t)$:
\begin{align*}
    \mathcal{G}_t &= \mathbb{E}[h(\bar{x}_{t})] + \frac{9\eta\hat{L}^2}{\mu\gamma}\times \frac{1}{M} \sum_{m=1}^M \mathbb{E}\big\|y^{(m)}_t - y^{(m)}_{x^{(m)}_{t}} \big\|^2
\end{align*}

\begin{theorem}
\label{theorem:FedBiO}
Suppose we have constant $\bar{\eta} = \min\big(\frac{1}{2C_1^{1/2}}, \frac{\mu\gamma}{12\kappa\hat{L}}, \frac{1}{2\bar{L}}, \frac{1}{6I\hat{L}}\big)$, if we choose $\eta = \min\big(\bar{\eta}, \left(\frac{2\Delta}{C_\eta T}\right)^{1/3}\big)$ and $\gamma = \frac{1}{2L}$, we have:
\begin{align*}
   \frac{1}{T} \sum_{t = 1}^{T} \|\nabla h(\bar{x}_{t})\|^2  =  O\left(\frac{\kappa^5}{T} + \left(\frac{\kappa^{16}}{T^{2}}\right)^{1/3} + \frac{\kappa^5\sigma^2}{b_y} + \frac{G_2^2}{b_xM} + G_1^2\right)
\end{align*}
To reach an $\epsilon$ stationary point, we choose the inner batch size $b_y = O(\kappa^5\epsilon^{-1})$, upper batch size $b_x = O(M^{-1}\epsilon^{-1})$ and $Q = O(\kappa\log(\frac{\kappa}{\epsilon}))$ in Eq.~\ref{eq:outer_grad_est_local}, and $T = O(\kappa^8\epsilon^{-1.5})$ number of iterations.
\end{theorem}

\begin{proof}
Similar to Lemma~\ref{lemma:d_bound_fedavg}, we denote $D_t = \frac{1}{M} \sum_{m = 1}^M \|\big(  \nu_t^{(m)} -\bar{\nu}_t  \big)\|^2$, $B_t = \frac{1}{M} \sum_{m=1}^M \|y^{(m)}_t - y^{(m)}_{x^{(m)}_{t}}\|^2 $, additionally, we denote $E_t = \|\mathbb{E}_{\xi}[\bar{\nu}_{t}]\|^2$. First, by Lemma~\ref{lemma: inner_drift}, when $t \neq \bar{t}_s$, by the triangle inequality, we have:
\begin{align*}
B_t - B_{t-1} & \leq -\frac{\mu\gamma}{2}B_{t-1} +\frac{10\kappa^2\eta^2}{\mu\gamma} D_{t-1} + \frac{10\kappa^2\eta^2}{\mu\gamma} E_{t-1} + \frac{10\kappa^2\eta^2G_2^2}{\mu\gamma b_x M} + \frac{3\gamma^2\sigma^2}{b_y}
\end{align*}
When $t = \bar{t}_s$, we have:
\begin{align*}
    B_t - B_{t-1}  &\leq  - \frac{\mu\gamma}{2}B_{t-1} + \frac{20\kappa^2\eta^2}{\mu\gamma} D_{t-1} + \frac{20\kappa^2\eta^2}{\mu\gamma} E_{t-1} + \frac{10\kappa^2I\eta^2}{\mu\gamma}\sum_{\ell=\bar{t}_{s-1}}^{t-1}D_{\ell}+ \frac{10\kappa^2\eta^2G_2^2}{\mu\gamma b_xM} + \frac{3\gamma^2\sigma^2}{b_y}
\end{align*}
We telescope from $\bar{t}_{s-1} + 1$ to $\bar{t}_{s}$ and have:
\begin{align}\label{eq:B_tele_fedavg}
    B_{\bar{t}_s} - B_{\bar{t}_{s-1}}  & \leq  - \frac{\mu\gamma}{2} \sum_{t=\bar{t}_{s-1}}^{\bar{t}_s-1}B_{t} + \frac{40\kappa^2I\eta^2}{\mu\gamma} \sum_{t=\bar{t}_{s-1}}^{\bar{t}_s-1}D_{t} + \frac{20\kappa^2\eta^2}{\mu\gamma} \sum_{t=\bar{t}_{s-1}}^{\bar{t}_s-1}E_{t} + \frac{10I\kappa^2\eta^2G_2^2}{\mu\gamma b_xM} + \frac{3I\gamma^2\sigma^2}{b_y}
\end{align}
Next, by Lemma~\ref{lemma:desent}, when $t \neq \bar{t}_s$, we have:
\begin{align*}
     \mathbb{E} [h(\bar{x}_{t + 1})] - \mathbb{E}[h(\bar{x}_{t })] \leq - \frac{\eta}{2} \mathbb{E}\|\nabla h(\bar{x}_{t})\|^2  -\frac{\eta}{4}E_t + 4\kappa^2\hat{L}^2I\eta^3 \sum_{\ell = \bar{t}_{s-1}}^{t-1}D_l + 2\eta\hat{L}^2B_t + \frac{\eta^2 \bar{L}G_2^2}{2b_xM} + \eta G_1^2
\end{align*}
and when $ t = \bar{t}_s$, we have:
\begin{align*}
     \mathbb{E} [h(\bar{x}_{t + 1})] - \mathbb{E}[h(\bar{x}_{t })] \leq - \frac{\eta}{2} \mathbb{E}\|\nabla h(\bar{x}_{t})\|^2  -\frac{\eta}{4}E_t + 2\eta\hat{L}^2B_t + \frac{\eta^2 \bar{L}G_2^2}{2b_xM} + \eta G_1^2
\end{align*}
We telescope from $\bar{t}_{s-1}$ to $\bar{t}_{s}$ to have:
\begin{align}
    \mathbb{E}[h(\bar{x}_{\bar{t}_{s}})] - \mathbb{E}[h(\bar{x}_{\bar{t}_{s - 1} })] & \leq - \sum_{t = \bar{t}_{s-1}}^{\bar{t}_s-1} \frac{\eta}{2} \mathbb{E}\|\nabla h(\bar{x}_{t})\|^2 - \sum_{t = \bar{t}_{s-1}}^{\bar{t}_s-1}\frac{\eta}{4} E_t + 4\kappa^2\hat{L}^2I\eta^3\sum_{t = \bar{t}_{s-1}+1}^{\bar{t}_s-1} \sum_{\ell = \bar{t}_{s-1}}^{t-1} D_l \nonumber\\
    &\qquad + \sum_{t = \bar{t}_{s-1}}^{\bar{t}_s-1}2\hat{L}^2\eta B_t + \frac{I\eta^2 \bar{L}G_2^2}{2b_x M} + I\eta G_1^2\nonumber\\
    & \leq - \sum_{t = \bar{t}_{s-1}}^{\bar{t}_s-1} \frac{\eta}{2} \mathbb{E}\|\nabla h(\bar{x}_{t})\|^2 -\sum_{t = \bar{t}_{s-1}}^{\bar{t}_s-1}\frac{\eta}{4} E_t  + 4\kappa^2\hat{L}^2I^2\eta^3\sum_{t = \bar{t}_{s-1}}^{\bar{t}_s-1}  D_l \nonumber\\
    &\qquad + 2\hat{L}^2\eta\sum_{t = \bar{t}_{s-1}}^{\bar{t}_s-1} B_t  + \frac{I\eta^2 \bar{L}G_2^2}{2b_x M} + I\eta G_1^2
\label{eq:h_tele_fedavg}
\end{align}
In the last inequality, we use the fact that $\bar{t}_s - \bar{t}_{s-1}  \leq I$. 

Next, by the definition of the potential function and combine with Eq.~\ref{eq:B_tele_fedavg} and Eq.~\ref{eq:h_tele_fedavg}, we have:
\begin{align*}
    \mathcal{G}_{\bar{t}_s} - \mathcal{G}_{\bar{t}_{s-1}} &\leq -\frac{\eta}{2} \sum_{t = \bar{t}_{s-1}}^{\bar{t}_s-1} \mathbb{E}\|\nabla h(\bar{x}_{t})\|^2 -\frac{5\eta\hat{L}^2}{2} \sum_{t=\bar{t}_{s-1}}^{\bar{t}_s-1}B_{t} - \frac{\eta}{4}\left(1 - \frac{720\kappa^2\eta^2\hat{L}^2}{\mu^2\gamma^2}\right)\sum_{t=\bar{t}_{s-1}}^{\bar{t}_s-1}E_{t} \nonumber\\
    & \qquad + \left(\frac{360}{\mu^2\gamma^2} + 4I\right) \eta^3\kappa^2\hat{L}^2I\sum_{t=\bar{t}_{s-1}}^{\bar{t}_s-1}D_{t} + \frac{27I\hat{L}^2\gamma\eta\sigma^2}{b_y\mu} \nonumber\\
    &\qquad + \frac{90I\kappa^2\hat{L}^2\eta^3G_2^2}{\mu^2\gamma^2 b_xM} + \frac{I\eta^2 \bar{L}G_2^2}{2b_xM} + I\eta G_1^2
\end{align*}
to bound the coefficients above, we choose $\eta \leq \frac{\mu\gamma}{48\kappa\hat{L}}$. Then we have:
\begin{align*}
    \mathcal{G}_{\bar{t}_s} - \mathcal{G}_{\bar{t}_{s-1}} &\leq -\frac{\eta}{2} \sum_{t = \bar{t}_{s-1}}^{\bar{t}_s-1} \mathbb{E}\|\nabla h(\bar{x}_{t})\|^2 -\frac{5\eta\hat{L}^2}{2} \sum_{t=\bar{t}_{s-1}}^{\bar{t}_s-1}B_{t} - \frac{\eta}{8}\sum_{t=\bar{t}_{s-1}}^{\bar{t}_s-1}E_{t} \nonumber\\
    &\qquad + \left(\frac{360}{\mu^2\gamma^2} + 4I\right) \eta^3\kappa^2\hat{L}^2I\sum_{t=\bar{t}_{s-1}}^{\bar{t}_s-1}D_{t} + \frac{27I\hat{L}^2\gamma\eta\sigma^2}{b_y\mu} \nonumber\\
    &\qquad + \frac{90I\kappa^2\hat{L}^2\eta^3G_2^2}{\mu^2\gamma^2 b_xM} + \frac{I\eta^2 \bar{L}G_2^2}{2b_xM} + I\eta G_1^2
\end{align*}
By lemma~\ref{lemma:d_bound_fedavg}, and choosing $\eta < \frac{1}{6I\bar{L}}$, we have:
\begin{align*}
    \sum_{t = \bar{t}_{s-1}}^{\bar{t}_s-1} D_{t} &\leq 6\hat{L}^2\sum_{t=\bar{t}_{s-1}}^{\bar{t}_s-1} B_t  +   18I\zeta^2 + 6IG_1^2 + \frac{6IG_2^2}{b_x}
\end{align*}
Next, we denote $C_1 = \left(\frac{360}{\mu^2\gamma^2} + 4I\right)\kappa^2\hat{L}^2$, and choose $\eta < \min(\frac{1}{2C_1^{1/2}}, \frac{\mu\gamma}{12\kappa\hat{L}}, \frac{1}{2\bar{L}})$ then we have:
\begin{align*}
    \mathcal{G}_{\bar{t}_s} - \mathcal{G}_{\bar{t}_{s-1}} &\leq -\frac{\eta}{2} \sum_{t = \bar{t}_{s-1}}^{\bar{t}_s-1} \mathbb{E}\|\nabla h(\bar{x}_{t})\|^2 -\frac{5\eta\hat{L}^2}{2} \sum_{t=\bar{t}_{s-1}}^{\bar{t}_s-1}B_{t} - \frac{\eta}{8}\sum_{t=\bar{t}_{s-1}}^{\bar{t}_s-1}E_{t} \nonumber\\
    &\qquad +  C_1I\eta^3\left(6\hat{L}^2\sum_{t=\bar{t}_{s-1}+1}^{\bar{t}_s} B_t  +   18I\zeta^2 + 6IG_1^2 + \frac{6IG_2^2}{b_x}\right) \nonumber\\
    &\qquad + \frac{27I\hat{L}^2\gamma\eta\sigma^2}{b_y\mu}
    + \frac{5I\eta G_2^2}{8b_xM} + \frac{I\eta^2 \bar{L}G_2^2}{2b_x} + I\eta G_1^2\nonumber\\
    &\leq -\frac{\eta}{2} \sum_{t = \bar{t}_{s-1}}^{\bar{t}_s-1} \mathbb{E}\|\nabla h(\bar{x}_{t})\|^2 + 18C_1I\hat{L}^2\eta^3\zeta^2 + 6C_1I\hat{L}^2\eta^3G_1^2 + \frac{6C_1I\hat{L}^2\eta^3G_2^2}{b_x} + \frac{5I\eta G_2^2}{8b_xM}\nonumber\\
    &\qquad + \frac{27I\hat{L}^2\gamma\eta\sigma^2}{b_y\mu}  + \frac{I\eta^2 \bar{L}G_2^2}{2b_xM} + I\eta G_1^2
\end{align*}
Sum over all $s \in [S]$ (assume $T = SI + 1$ without loss of generality) to obtain:
\begin{align*}
    \frac{\eta}{2} \sum_{t = 1}^{T} \mathbb{E}\|\nabla h(\bar{x}_{t})\|^2 &\leq \mathcal{G}_{1} - \mathcal{G}_{T} +18C_1T\hat{L}^2\eta^3\zeta^2 + 6C_1T\hat{L}^2\eta^3G_1^2 + \frac{6C_1T\hat{L}^2\eta^3G_2^2}{b_x} \nonumber\\
    &\qquad +\frac{5T\eta G_2^2}{8b_xM} + \frac{27T\hat{L}^2\gamma\eta\sigma^2}{b_y\mu}  + \frac{T\eta^2 \bar{L}G_2^2}{2b_xM} + T\eta G_1^2 \nonumber\\
    &\leq \Delta + \frac{9\eta\hat{L}^2\Delta_y}{\mu\gamma} + 18C_1T\hat{L}^2\eta^3\zeta^2 + 6C_1T\hat{L}^2\eta^3G_1^2 + \frac{6C_1T\hat{L}^2\eta^3G_2^2}{b_x} \nonumber\\
    &\qquad + \frac{5T\eta G_2^2}{8b_xM} + \frac{27T\hat{L}^2\gamma\eta\sigma^2}{b_y\mu}  + \frac{T\eta^2 \bar{L}G_2^2}{2b_xM} + T\eta G_1^2
\end{align*}
we define $\Delta = h(x_{1}) - h^*$ as the initial sub-optimality of the function and $\Delta_y = \frac{1}{M} \sum_{m=1}^M \big\|y^{(m)}_1 - y^{(m)}_{x_{1}} \big\|^2$ as the initial sub-optimality of the inner variable estimation, then we divide by $\eta T/2$ on both sides and have:
\begin{align*}
    \frac{1}{T} \sum_{t = 1}^{T} \mathbb{E}\|\nabla h(\bar{x}_{t})\|^2 
    &\leq \underbrace{\frac{2\Delta}{\eta T}  + \frac{\eta \bar{L}G_2^2}{2b_xM}  + \left(36C_1\hat{L}^2\zeta^2 + 12C_1\hat{L}^2G_1^2 +  \frac{12C_1\hat{L}^2G_2^2}{b_x}\right)\eta^2}_{T_1} \nonumber\\
    &\qquad + \underbrace{\frac{18\hat{L}^2\Delta_y}{\mu\gamma T} + \frac{54\hat{L}^2\gamma\sigma^2}{b_y\mu}}_{T_2}  + \underbrace{\frac{5 G_2^2}{4b_xM} +  G_1^2}_{T_3}
\end{align*}
As shown in the inequality, we break the bound into three parts. The $T_1$ part has a structure similar to that for the single level federated learning problems. Then the $T_2$ part includes the optimization error of the lower problem, and the statistical error of sampling. Finally, the $T_3$ part includes the bias and variance of the hyper-gradient estimate.

Next, by $\eta < \frac{1}{2\bar{L}}$, we have
\begin{align}\label{eq:fedavg}
    \frac{1}{T} \sum_{t = 1}^{T} \mathbb{E}\|\nabla h(\bar{x}_{t})\|^2 
    &\leq \frac{2\Delta}{\eta T}  + \left(36C_1\hat{L}^2\zeta^2 + 12C_1\hat{L}^2G_1^2 + \frac{12C_1\hat{L}^2G_2^2}{b_x} \right)\eta^2 \nonumber\\
    &\qquad + \frac{18\hat{L}^2\Delta_y}{\mu\gamma T} + \frac{54\hat{L}^2\gamma\sigma^2}{b_y\mu}  + \frac{G_2^2}{4b_xM} + \frac{5 G_2^2}{4b_xM} +  G_1^2
\end{align}
Next, we denote constant
$\bar{\eta} = \min\big(\frac{1}{2C_1^{1/2}}, \frac{\mu\gamma}{12\kappa\hat{L}}, \frac{1}{2\bar{L}}, \frac{1}{6I\hat{L}}\big)$
and $C_\eta = \left(36C_1\hat{L}^2\zeta^2 + 12C_1\hat{L}^2G_1^2 + \frac{12C_1\hat{L}^2G_2^2}{b_x} \right)$
we choose
\[
\eta = \min\big(\bar{\eta}, \left(\frac{2\Delta}{C_\eta T}\right)^{1/3}\big)
\]
and $\gamma = \frac{1}{2L}$, and obtain:
\begin{align*}
    \frac{1}{T} \sum_{t = 1}^{T} \mathbb{E}\|\nabla h(\bar{x}_{t})\|^2 
    &\leq \frac{2\Delta}{\bar{\eta}T} + \frac{36\kappa\hat{L}^2\Delta_y}{T} + \left(\frac{4C_\eta\Delta^2}{T^2}\right)^{1/3} + \frac{27\hat{L}^2\sigma^2}{\mu b_y L} + \frac{3G_2^2}{2b_x M} + 2G_1^2
\end{align*}
Finally, since $\hat{L} = O(\kappa^2)$ , $\bar{L} = O(\kappa^3)$ and and $\mu\gamma = O(\kappa^{-1})$. Suppose we choose $I = O(1)$, then $\bar{\eta} = O(\kappa^{-4})$ and $C_1 = O(\kappa^{8})$, $\zeta = O(\kappa^2)$,  $C_\eta = O(\kappa^{16})$, thus, we have
\begin{align*}
   \frac{1}{T} \sum_{t = 1}^{T} \|\nabla h(\bar{x}_{t})\|^2  =  O\left(\frac{\kappa^5}{T} + \left(\frac{\kappa^{16}}{T^{2}}\right)^{1/3} + \frac{\kappa^5\sigma^2}{b_y} + \frac{G_2^2}{b_xM} + G_1^2\right)
\end{align*}
and to reach an $\epsilon$ stationary point, we choose the inner batch size $b_y = O(\kappa^5\epsilon^{-1})$, upper batch size $b_x = O(M^{-1}\epsilon^{-1})$ and $Q = O(\kappa\log(\frac{\kappa}{\epsilon}))$ in Eq.~\ref{eq:outer_grad_est_local}, and $T = O(\kappa^8\epsilon^{-1.5})$ number of iterations.
\end{proof}

\section{Useful Propositions}
In this section, we state some propositions useful in the proof:
\begin{proposition}[Lemma 3 of~\cite{karimireddy2020scaffold}] (generalized triangle inequality)
\label{prop:generali_tri}
Let $\{x_k\}, k\in{K}$ be $K$ vectors. Then the following are true:
\begin{enumerate}
	\item $||x_i + x_j||^2 \le (1 + a)||x_i||^2 + (1 + \frac{1}{a})||x_j||^2$ for any $a > 0$, and
	\item  $||\sum_{k=1}^K x_k||^2 \le K\sum_{k=1}^{K} ||x_k||^2$
\end{enumerate}
\end{proposition}

\begin{proposition}[Lemma C.1 of~\cite{khanduri2021near}]
\label{prop: Sum_Mean_Kron}
For a finite sequence $x^{(k)} \in \mathbb{R}^d$ for $k \in [K]$ define $\bar{x} \coloneqq \frac{1}{K} \sum_{k = 1}^K x^{(k)}$, we then have $\sum_{k=1}^K    \| x^{(k)} - \bar{x} \|^2 \leq \sum_{k=1}^K    \| x^{(k)} \|^2.$
\end{proposition}

\begin{proposition}[Lemma C.2 of~\cite{khanduri2021near}]
\label{Lem: AD_Sum_1overT}
Let $a_0 > 0$ and $a_1,a_2, \ldots, a_T \geq 0$. We have
$$\sum_{t=1}^T \frac{a_t}{a_0 + \sum_{i=t}^t a_i} \leq \ln \big(1 + \frac{\sum_{i=1}^t a_i}{a_0} \big).$$
\end{proposition}

\begin{proposition}\label{prop:strong-prog1}
Suppose we have function $g(y)$, which is L-smooth and $\mu$-strongly-convex, then suppose $\gamma < \frac{1}{L}$, the progress made by one step of gradient descent is:
\begin{align*}
\mathbb{E}\|y_{t+1}-y^*\|^2 & \leq ( 1-\mu\gamma)\|y^*-y_t\|^2 + 2\gamma^2\sigma^2
\end{align*}
where $y^*$ is the minimum of $g(y)$ and we have update rule $g(y_{t+1}) = g(y_{t}) - \gamma\nabla g(y_t, \xi)$, where the error of stochastic gradient estimate is bounded by $\sigma^2$.
\end{proposition}

\begin{proof}
First, by the strong convexity of of function $g(y)$, we have:
\begin{align*}
g(y^*) & \geq g(y_t) + \langle\nabla_y g(y_t), y^*-y_t\rangle + \frac{\mu}{2}\|y^*-y_t\|^2 \nonumber \\
& =  g(y_t) + \langle \nabla_y g(y_t), y^*-y_{t+1}\rangle + \langle\nabla_y g(y_t), y_{t+1} - y_{t}\rangle+ \frac{\mu}{2}\|y^*-y_t\|^2
\end{align*}
Then by $L$-smoothness, we have:
$
 \frac{L}{2}\|y_{t+1}-y_t\|^2 \geq g(y_{t+1}) - g(y_{t}) -\langle\nabla_y g(y_t), y_{t+1}-y_t\rangle,
$
Combining above two inequalities and take expectation on both sides, we have
\begin{align*}
 g(y^*) & \geq \mathbb{E}g(y_{t+1}) +  \mathbb{E}\langle \nabla_y g(y_t), y^*-y_{t+1}\rangle  + \frac{\mu}{2}\|y^*-y_t\|^2 - \frac{L}{2}\mathbb{E}\|y_{t+1}-y_t\|^2\nonumber\\
 &\geq \mathbb{E}g(y_{t+1}) +  \gamma\| \nabla_y g(y_t)\|^2 + \langle \nabla_y g(y_t), y^*-y_{t}\rangle  + \frac{\mu}{2}\|y^*-y_t\|^2 - \frac{L\gamma^2}{2}\mathbb{E}\|\nabla_y g(y_t, \xi)\|^2\nonumber\\
 &\geq \mathbb{E}g(y_{t+1}) +  \gamma\| \nabla_y g(y_t)\|^2 + \langle \nabla_y g(y_t), y^*-y_{t}\rangle  + \frac{\mu}{2}\|y^*-y_t\|^2 - \frac{L\gamma^2}{2}\mathbb{E}\|\nabla_y g(y_t)\|^2 - \frac{L\gamma^2\sigma^2}{2}\nonumber\\
 &\geq \mathbb{E}g(y_{t+1}) + \langle \nabla_y g(y_t), y^*-y_{t}\rangle + \frac{\mu}{2}\|y^*-y_t\|^2 + \left(\gamma - \frac{L\gamma^2}{2}\right)\|\nabla_y g(y_t)\|^2 - \frac{L\gamma^2\sigma^2}{2}
\end{align*}
By definition of $y^{*}$, we have $g(y^{*}) \geq g(y_{t+1})$. Thus, we obtain
\begin{align*}
 0 & \geq \langle \nabla_y g(y_t), y^*-y_{t}\rangle + \frac{\mu}{2}\|y^*-y_t\|^2 + \left(\gamma - \frac{L\gamma^2}{2}\right)\|\nabla_y g(y_t)\|^2 - \frac{L\gamma^2\sigma^2}{2}
\end{align*}
By $y_{t+1} = y_t - \gamma\nabla_y g(y_t, \xi)$, we have:
\begin{align*}
 \mathbb{E}\|y_{t+1}-y^{*}\|^2 &= \mathbb{E}\|y_t - \gamma\nabla_y g(y_t, \xi) -y^{*}\|^2 = \|y_t-y^{*}\|^2 - 2\gamma\langle \nabla_y g(y_t), y_t-y^{*}\rangle + \gamma^2\mathbb{E}\|\nabla_y g(y_t, \xi)\|^2\nonumber\\
 &\leq \big(1 - \mu\gamma\big)\|y_t - y^*\|^2 - 2\gamma\big(\gamma - \frac{L\gamma^2}{2} - \frac{\gamma}{2}\big) \|\nabla_y g(y_t)\|^2 + \left(L\gamma^3 + \gamma^2\right)\sigma^2
\end{align*}
Then since we choose $\gamma < \frac{1}{L}$, we obtain:
\begin{align*}
\mathbb{E}\|y_{t+1}-y^*\|^2 & \leq ( 1-\mu\gamma)\|y^*-y_t\|^2 + 2\gamma^2\sigma^2
\end{align*}
This completes the proof.
\end{proof}

\begin{proposition}\label{prop:strong-prog}
Suppose we have function $g(y)$, which is L-smooth and $\mu$-strongly-convex, then suppose $\gamma < \frac{1}{2L}$ and $\alpha_t < 1$, the progress made by one step of gradient descent is:
\begin{align*}
\|y_{t+1}-y^*\|^2 & \leq ( 1-\frac{\mu\gamma\alpha_t}{2})\|y_t - y^*\|^2 - \frac{\gamma^2\alpha_t}{4} \|\omega_t\|^2 \nonumber\\
&\qquad + \frac{4\gamma\alpha_t}{\mu}\|\nabla_y g(x_t,y_t)-\mathbb{E}[w_t]\|^2 + \frac{3\gamma^2\alpha_t}{2}Var[\omega_t].
\end{align*}
where $y^*$ is the minimum of $g(y)$ and we have update rule $g(y_{t+1}) = g(y_{t}) - \gamma\alpha_t\omega_t$.
\end{proposition}

\begin{proof}
First, Suppose we denote $\tilde{y}_{t+1} = y_t - \gamma\omega_t$, then we have $y_{t+1} = y_t +\alpha_t(\tilde{y}_{t+1} - y_t)$. By the strong convexity of of function $g(y)$, we have:
\begin{align} \label{eq:E1}
g(y^*) & \geq g(y_t) + \langle\nabla_y g(y_t), y^*-y_t\rangle + \frac{\mu}{2}\|y^*-y_t\|^2 \nonumber \\
& =  g(y_t) + \mathbb{E}\langle \mathbb{E} [w_t], y^*-\tilde{y}_{t+1}\rangle + \mathbb{E}\langle\nabla_y g(y_t)- \mathbb{E} [w_t], y^*- \tilde{y}_{t+1}\rangle \nonumber\\
&\qquad - \gamma\langle\nabla_y g(y_t), \mathbb{E} [w_t]\rangle+ \frac{\mu}{2}\|y^*-y_t\|^2
\end{align}
where the expectation is \emph{w.r.t} the stochasity of $\omega_t$. Then by $L$-smoothness, we have:
\begin{align} \label{eq:E2}
 \frac{L}{2}\mathbb{E}\|\tilde{y}_{t+1}-y_t\|^2 &\geq \mathbb{E}g(\tilde{y}_{t+1}) - g(y_{t}) + \gamma \langle\nabla_y g(y_t), \mathbb{E}[\omega_t]\rangle
\end{align}
Combining the \ref{eq:E1} with \ref{eq:E2}, we have
\begin{align*}
 g(y^*) & \geq \mathbb{E} g(\tilde{y}_{t+1}) +  \mathbb{E}\langle \mathbb{E} [w_t], y^*-\tilde{y}_{t+1}\rangle + \mathbb{E}\langle\nabla_y g(y_t)- \mathbb{E} [w_t], y^*- \tilde{y}_{t+1}\rangle \nonumber\\
 &\qquad + \frac{\mu}{2}\|y^*-y_t\|^2 - \frac{L}{2}\mathbb{E}\|\tilde{y}_{t+1}-y_t\|^2\nonumber\\
 &\geq \mathbb{E} g(\tilde{y}_{t+1}) +  \gamma\| \mathbb{E} [w_t]\|^2 + \langle \mathbb{E} [w_t], y^*-y_{t}\rangle + \mathbb{E} \langle\nabla_y g(y_t)- \mathbb{E} [w_t], y^*- \tilde{y}_{t+1}\rangle \nonumber\\
 &\qquad + \frac{\mu}{2}\|y^*-y_t\|^2 - \frac{L\gamma^2}{2}\mathbb{E}\|\omega_t\|^2\nonumber\\
 &\geq \mathbb{E} g(\tilde{y}_{t+1}) + \langle \mathbb{E} [w_t], y^*-y_{t}\rangle + \mathbb{E} \langle\nabla_y g(y_t)-\mathbb{E} [w_t], y^*- \tilde{y}_{t+1}\rangle \nonumber\\
 &\qquad + \frac{\mu}{2}\|y^*-y_t\|^2 + \left(\gamma - \frac{L\gamma^2}{2}\right)\|\mathbb{E}[\omega_t]\|^2 - \frac{L\gamma^2}{2}Var[\omega_t] 
\end{align*}
where $Var$ denotes the variance. By definition of $y^{*}$, we have $g(y^{*}) \geq g(\tilde{y}_{t+1})$. Thus, we obtain
\begin{align} \label{eq:E5}
 0 & \geq \langle \mathbb{E} [w_t], y^*-y_{t}\rangle + \mathbb{E} \langle\nabla_y g(y_t)-\mathbb{E} [w_t], y^*- \tilde{y}_{t+1}\rangle \nonumber\\
 &\qquad + \frac{\mu}{2}\|y^*-y_t\|^2 + \left(\gamma - \frac{L\gamma^2}{2}\right)\|\mathbb{E}[\omega_t]\|^2 - \frac{L\gamma^2}{2}Var[\omega_t] 
\end{align}
Considering the upper bound of the second term $\langle\nabla_y g(y_t)-w_t, y^{*}- y_{t+1}\rangle$, we have
\begin{align*}
 &-\mathbb{E}\langle\nabla_y g(y_t)-\mathbb{E} [w_t], y^{*}- \tilde{y}_{t+1}\rangle \nonumber \\
 & = -\langle\nabla_y g(y_t)-\mathbb{E} [w_t], y^{*}- y_{t}\rangle + \langle\nabla_y g(y_t)-\mathbb{E}[w_t], \mathbb{E} [\omega_t]\rangle \nonumber \\
 & \leq \frac{1}{\mu} \|\nabla_y g(y_t)-\mathbb{E} [w_t]\|^2 + \frac{\mu}{4}\|y^{*}- y_{t}\|^2 ] + \frac{1}{\mu} \|\nabla_y g(y_t)- \mathbb{E} [w_t]\|^2 + \frac{\mu\gamma^2}{4}\|\mathbb{E}[\omega_t]\|^2 \nonumber\\
 & = \frac{2}{\mu} \|\nabla_y g(y_t)-\mathbb{E}[w_t]\|^2 + \frac{\mu}{4}\|y^{*}- y_{t}\|^2 + \frac{\mu\gamma^2}{4}\|\mathbb{E}[\omega_t]\|^2.
\end{align*}
Combining with Eq.~\ref{eq:E5}:
\begin{align*}
 0 & \geq \langle \mathbb{E} [w_t], y^{*}-y_{t}\rangle -\frac{2}{\mu} \|\nabla_y g(y_t)-\mathbb{E} [w_t]\|^2  + \big(\gamma - \frac{3L\gamma^2}{4}\big)\|\mathbb{E} [w_t]\|^2 +  \frac{\mu}{4}\|y^{*}-y_t\|^2 - \frac{L\gamma^2}{2}Var[\omega_t] 
\end{align*}
By $y_{t+1} = y_t - \gamma\alpha_t\omega_t$, we have:
\begin{align*}
 \mathbb{E}\|y_{t+1}-y^{*}\|^2 &= \mathbb{E}\|y_t - \gamma\alpha_t\omega_t -y^{*}\|^2 = \|y_t-y^{*}\|^2 - 2\gamma\alpha_t\langle \mathbb{E}[\omega_t], y_t-y^{*}\rangle + \gamma^2\alpha_t^2\mathbb{E}[\|\omega_t\|^2]\nonumber\\
 &\leq \big(1 - \frac{\mu\gamma\alpha_t}{2}\big)\|y_t - y^*\|^2 - 2\gamma\alpha_t\big(\gamma - \frac{\gamma\alpha_t}{2} - \frac{3L\gamma^2}{4}\big) \|\mathbb{E}[\omega_t]\|^2 \nonumber\\
 &\qquad + \frac{4\gamma\alpha_t}{\mu} \|\nabla_y g(y_t)- \mathbb{E} [w_t]\|^2 + (L\gamma^3\alpha_t + \gamma^2\alpha_t^2)Var[\omega_t] 
\end{align*}
Then since we choose $\gamma < \frac{1}{2L}$, $\alpha_t < 1$, we obtain:
\begin{align*}
\mathbb{E}\|y_{t+1}-y^*\|^2 & \leq ( 1-\frac{\mu\gamma\alpha_t}{2})\|y^*-y_t\|^2 - \frac{\gamma^2\alpha_t}{4} \|\mathbb{E}[\omega_t]\|^2 \nonumber\\
&\qquad + \frac{4\gamma\alpha_t}{\mu}\|\nabla_y g(x_t,y_t)-\mathbb{E}[w_t]\|^2 + \frac{3\gamma^2\alpha_t}{2}Var[\omega_t] .
\end{align*}
This completes the proof.
\end{proof}

\end{document}